\documentclass{article}

\usepackage[in]{fullpage}
\usepackage{natbib}
\usepackage{mathpazo}

\usepackage[utf8]{inputenc} %
\usepackage[T1]{fontenc}    %
\usepackage{hyperref}       %
\usepackage{url}            %
\usepackage{booktabs}       %
\usepackage{amsfonts}       %
\usepackage{nicefrac}       %
\usepackage{microtype}      %

\newcommand{\Loss}{\text{Loss}}

\newcommand{\reals}{{\mathbb{R}}}

\newcommand{\argmin}{\mathop{\rm argmin}}

\newcommand{\dquote}[1]{``#1''}

\usepackage{url}            %
\usepackage{booktabs}       %
\usepackage{multirow}    
\usepackage{amsfonts}       %
\usepackage{nicefrac}       %
\usepackage{microtype}      %
\usepackage{enumerate}
\usepackage{hhline}
\usepackage{makecell}
\usepackage{pifont}

\usepackage{graphicx} %
\usepackage{caption}
\usepackage{subcaption}
\usepackage{amsmath}
\usepackage{amsthm}
\usepackage{amssymb}
\usepackage{tikz}
\usepackage{xcolor}
\usetikzlibrary{arrows}

\allowdisplaybreaks

\usepackage{mathrsfs}

\usepackage{algorithm}
\usepackage{algorithmic}
\usepackage{hyperref}
\usepackage{bm}

\allowdisplaybreaks

\newcommand{\poly}{\mathrm{poly}}

\newcommand{\labs}{\left\vert}
\newcommand{\rabs}{\right\vert}
\newcommand{\lnorm}{\left\Vert}
\newcommand{\rnorm}{\right\Vert}

\newcommand{\tr}{\operatorname{tr}}
\newcommand{\opt}{\mathrm{opt}}

\def\TV{\mathrm{TV}}

\newcommand{\expect}{\mathbb{E}}

\newcommand{\indict}{\mathbb{I}}

\newtheorem{thm}{Theorem}
\newtheorem{lem}{Lemma}

\newtheorem{prop}{Proposition}
\newtheorem{asmp}{Assumption}
\newtheorem{defn}{Definition}

\newtheorem{fact}{Fact}

\newtheorem{rem}{Remark}
\newtheorem{example}{Example}

\newtheorem{ec}{Empirical Observation}

\usepackage[capitalize,noabbrev]{cleveref}
\crefname{thm}{Theorem}{Theorems}
\crefname{lem}{Lemma}{Lemmas}
\crefname{cor}{Corollary}{Corollaries}
\crefname{prop}{Proposition}{Propositions}
\crefname{asmp}{Assumption}{Assumptions}
\crefname{defn}{Definition}{Definitions}
\crefname{oracle}{Oracle}{Oracles}
\crefname{fact}{Fact}{Facts}
\crefname{conj}{Conjecture}{Conjectures}
\crefname{rem}{Remark}{Remarks}
\crefname{claim}{Claim}{Claims}
\crefname{ec}{Empirical Observation}{Empirical Observations}

\renewcommand{\algorithmicrequire}{\textbf{Input:}}
\renewcommand{\algorithmicensure}{\textbf{Output:}}

\definecolor{red}{rgb}{1, 0, 0}

\definecolor{green}{rgb}{0, 1, 0}
\definecolor{darkgreen}{rgb}{0.0, 0.2, 0.13}
\definecolor{darkseagreen}{rgb}{0.56, 0.74, 0.56}
\definecolor{officegreen}{rgb}{0.0, 0.5, 0.0}

\newcommand{\GREEN}[1]{{\color{green} #1}}

\definecolor{blue}{rgb}{0, 0, 1}
\newcommand{\BLUE}[1]{{\color{blue} #1}}

\definecolor{orange}{rgb}{1, 0.4, 0.0}

\input{packages/math_commands}

\renewcommand{\opt}{\operatorname{opt}}
\newcommand{\constant}{\operatorname{constant}}
\newcommand{\rl}{\operatorname{RL}}

\newcommand{\piE}{\pi^{\operatorname{E}}}
\newcommand{\bc}{\operatorname{BC}}

\newcommand{\ail}{\operatorname{AIL}}
\newcommand{\piail}{\pi^{\ail}}
\newcommand{\goodS}{\gS^{\operatorname{G}}}
\newcommand{\badS}{\gS^{\operatorname{B}}}

\newcommand{\gooda}{{\GREEN{a}}}
\newcommand{\bada}{{\BLUE{a}}}
\newcommand{\cdp}{\operatorname{DP}}
\newcommand{\eval}{\operatorname{EVAL}}
\renewcommand{\opt}{\operatorname{OPT}}

\usepackage{soul}
\renewcommand{\textsf}[1]{#1}
\usepackage{enumitem}
\usepackage{mathabx}
\setlist{topsep=0pt,parsep=0pt,partopsep=0pt}
\usepackage{tabulary}

\newcommand{\meanstd}[2]{$#1 {\scriptscriptstyle \pm #2}$}

\definecolor{dmorange500}{HTML}{FF5F19}
\definecolor{dmblue300}{HTML}{2267EB}
\definecolor{dmred300}{HTML}{FF617B}
\hypersetup{
    colorlinks=true,
    citecolor=dmorange500,
    linkcolor=dmorange500,
    urlcolor=dmorange500}

\theoremstyle{plain}

\theoremstyle{definition}

\theoremstyle{remark}

\usepackage{authblk}

\title{Understanding Adversarial Imitation Learning in Small Sample Regime: A Stage-coupled Analysis}

\author[1]{Tian Xu\thanks{Equal contribution. Author ordering is determined by coin flip. Email: \texttt{xut@lamda.nju.edu.cn} and \texttt{ziniuli@link.cuhk.edu.cn}}}
\author[2,3]{Ziniu Li$^*$}
\author[1]{Yang Yu\thanks{Corresponding authors. Email: \texttt{yuy@nju.edu.cn} and \texttt{luozq@cuhk.edu.cn}}}
\author[2,3]{Zhi-Quan Luo$^\dagger$}
\affil[1]{National Key Laboratory for Novel Software Technology and School of Artificial Intelligence, Nanjing University}
\affil[2]{The Chinese University of Hong Kong, Shenzhen}
\affil[3]{Shenzhen Research Institute of Big Data}

\date{\today}

\begin{document}

\maketitle
\begingroup
\renewcommand{\thefootnote}{}
\footnotetext{This paper is accepted in IEEE Transactions on Pattern Analysis and Machine Intelligence (TPAMI).}
\addtocounter{footnote}{0}
\endgroup

\begin{abstract}

Imitation learning (IL) learns a policy from expert trajectories, serving as a fundamental paradigm in both large language model training and embodied AI. This process is challenging due to the nature of sequential decision-making where errors can accumulate and distributions may shift over horizons. However, it has been found that a kind of IL approach, adversarial imitation learning (AIL), can have exceptional empirical performance. With just one expert trajectory, AIL often matches the expert performance even in a long horizon, on tasks such as robotic locomotion control. There are two fundamental yet unsolved questions: why does AIL perform well with so \emph{few} trajectories, and why does it maintain good performance over \emph{long} horizons? Previous theoretical results fail to answer these questions as they are meaningful only in large sample regime (i.e., lots of expert trajectories) {and have dependence on the decision horizon.} In this paper, we analyze a total-variation-distance-based AIL (called TV-AIL), showing a horizon-free imitation gap $\gO(\min\{1, \sqrt{|\gS|/N} \})$ on a class of instances abstracted from robotic locomotion control tasks. Here $|\gS|$ is the state space size for a Markov Decision Process (MDP), and $N$ is the number of expert trajectories. We emphasize two important features of our bound. First, this bound is meaningful in both small and large sample regimes. Second, this bound suggests that the imitation gap of TV-AIL does not increase with the decision horizon. Together, our bound can therefore explain the empirical observations and provide insights into how AIL addresses the distribution shift issue. Our analysis leverages the multi-stage policy optimization structure in TV-AIL and presents a new stage-coupled analysis. This tool also helps analyze the worst-case imitation gap of TV-AIL, disclosing its limitations in general MDPs.

\end{abstract}

\section{Introduction}\label{sec:introduction}
Imitation learning (IL) aims to train a policy by learning from expert demonstrations \citep{argall2009survey, hussein2017survey, osa2018survey}. It serves as a fundamental paradigm across diverse domains, from large language model training (through pre-training and supervised fine-tuning) to embodied AI (where agents learn skills by mimicking human demonstrations). One popular approach is behavioral cloning (BC), which directly matches the policy to the expert's actions using supervised learning (i.e., maximum likelihood estimation) \citep{Pomerleau91bc}. Although simple and widely used in various applications \citep{ross11dagger, silver2016mastering, levin16_end_to_end}, BC suffers from compounding errors \citep{ross11dagger}. Specifically, IL operates under the sequential decision-making framework \citep{puterman2014markov}, where future observations depend on previous decisions. Therefore, decision errors may accumulate and the distribution may shift over horizons, leading to poor practical performance. This issue becomes critical when the number of expert demonstrations is limited and the decision horizon is long \citep{ghasemipour2019divergence,xu2020error}, which is often the case in practical applications. See Table~\ref{tab:mujoco_sample_size} for the empirical evidence in the (robotic) locomotion control tasks. Thus, effectively addressing the compounding errors is believed to be fundamental in imitation learning \citep{xu2021error,rajaraman2020fundamental}.

Generative adversarial imitation learning (GAIL) \citep{ho2016gail} is introduced as an alternative to behavioral cloning (BC). While BC implements direct policy matching via supervised learning, GAIL performs state-action distribution matching\footnote{The difference between these two matching principles can be understood as follows: state-action distribution is a marginal distribution, while policy is a conditional distribution. Both matching principles lead to the expert policy optimally when infinite expert trajectories are available, but differ with finite trajectories.}. Practically, AIL first recovers a reward function from expert demonstrations and subsequently uses reinforcement learning (RL) \citep{sutton2018reinforcement} methods to maximize the reward.  This principle has led to the development of many adversarial imitation learning (AIL) methods, which minimize a certain divergence function between an agent's \emph{state-action} distribution and an expert's \citep{fu2018airl, Kostrikov19dac, Kostrikov20value_dice, Brantley20disagreement}. These methods have shown exceptional performance in various applications. For example, {as shown in Table \ref{tab:mujoco_sample_size}}, in locomotion control tasks with only one expert trajectory and a decision horizon of 1000, AIL methods can nearly match expert performance, whereas BC cannot \citep{ho2016gail}.

It is widely accepted that AIL methods are effective in addressing the issue of compounding errors in empirical settings. However, the theoretical underpinnings and understanding remain lacking. Specifically, two critical questions remain unresolved: why do AIL methods perform well with \emph{few} expert trajectories, and how do they maintain good performance over \emph{long} horizons? While existing theoretical results \citep{sun2019provably, wang2020computation, zhang2020generative, rajaraman2020fundamental, nived2021provably, xu2021error, swamy2021moments, liu2021provably} offer solid convergence and sample complexity guarantees for AIL, none adequately address these two areas of practical concern. In particular, previous theoretical models primarily focus on the large sample regime, where a large number of data ensures good performance. However, practitioners are often more concerned with the small sample regime, for which, to the best of our knowledge, there is no existing theoretical framework. We note that existing analytical tools have certain limitations and advocate for the development of a new theoretical framework. Further details are discussed below.

\begin{table*}[t]
\centering
\caption{Scaled imitation gap on Hopper, HalfCheetah and Walker2d with $H = 1000$. The reward scales over three tasks are different. We divide the original imitation gap by a scalar, which is proportional to the expert's policy value (see \cref{table:expert_value_mujoco} in the Appendix). We report the mean of the scaled imitation gap with the standard deviation over 5 independent experiments (same with \cref{tab:mujoco_horizon}). {For the original imitation gap and policy return, please refer to Appendix \ref{appendix:experiment_details}.}}
\label{tab:mujoco_sample_size}
\begin{tabular}{@{}cccccc@{}}
\toprule
     &    & $N=1$ & $N=4$ & $N=7$ &$N=10$ \\ \midrule
\multirow{2}{*}{Hopper (scale = 3.2)} &BC      & \meanstd{784.97}{28.09} &  \meanstd{887.04}{31.26} &  \meanstd{666.44}{106.58} & \meanstd{460.72}{74.95}       \\
 &TV-AIL      & \meanstd{10.38}{11.25}  & \meanstd{1.81}{2.88}       & \meanstd{-4.96}{10.88}  & \meanstd{2.33}{10.59} \\
\midrule
\multirow{2}{*}{HalfCheetah (scale = 7.7)} &BC      & \meanstd{1058.48}{8.27}& \meanstd{1066.21}{22.76} & \meanstd{988.07}{35.52} & \meanstd{579.53}{171.28}       \\
 &TV-AIL      & \meanstd{-22.45}{101.65}  & \meanstd{-84.96}{16.84}       & \meanstd{-78.69}{6.98} & \meanstd{-79.29}{7.95} \\
 \midrule
\multirow{2}{*}{Walker2d (scale = 5.0)} &BC      & \meanstd{1002.13}{9.68} & \meanstd{939.27}{19.10} & \meanstd{528.91}{181.46} & \meanstd{222.98}{52.97}       \\
 &TV-AIL      & \meanstd{12.89}{19.61}  & \meanstd{9.04}{17.99}       & \meanstd{-4.91}{8.14}  & \meanstd{14.36}{12.48} \\ \bottomrule
\end{tabular}
\end{table*}

Previous works primarily employed the reduction-and-estimation framework for analyzing AIL methods, a method dating back to \citep{pieter04apprentice, syed07game}. For tabular and episodic MDPs with finite states and actions, it has been proven that two AIL methods, FEM \citep{pieter04apprentice} and GTAL \citep{syed07game}, have worst-case imitation gap bounds $V(\piE) - V(\pi)$ of $\gO(\min \{ H, H\sqrt{|\gS||\gA|/N}) \})$ and $\gO(\min \{H, H|\gS||\gA|/\sqrt{N} \})$, respectively. In the statistical learning setting, the expert policy $\piE$ collects $N$ trajectory for the learner $\pi$. Furthermore, $|\gS|$ and $|\gA|$ denote the sizes of the state and action spaces, $H$ denotes the decision horizon, and $V(\pi)$ denotes the sum of rewards obtained by a policy $\pi$. These results have limitations and cannot explain the empirical observation well. On the one hand, these bounds are only meaningful in the large sample regime (i.e., $N \gtrsim |\gS||\gA|$); otherwise, the first term dominates in the small sample regime, and the bounds become trivial\footnote{Note that one-step reward is assumed to be between 0 and 1,  and hence the maximum imitation gap is at most $H$.} (in this paper, $\gtrsim$ and $\lesssim$ denote greater or less than up to constants, respectively). On the other hand, the bounds suggest that the performance of AIL methods may degenerate a lot for long-horizon tasks, which is rarely observed in practice.

Now we discuss the technical limitations of the reduction-and-estimation framework, which we will further elaborate on in \cref{subsec:key_analysis_ideas}. This framework consists of two main steps: (1) proving that the imitation gap is upper-bounded by the statistical estimation error of the expert state-action distribution and (2) controlling the estimation error through proper concentration inequalities. This framework suggests that if the statistical estimation error is small, the imitation gap is small. However, we believe that this worst-case analysis could be too loose for practical tasks, as it considers the multi-step decision errors in a \emph{decoupled} and \emph{independent} manner in step (1). To support our claim, we conducted experiments and observed that even if the statistical estimation error is large, the imitation gap of AIL methods could still be very small. See the evidence in \cref{tab:reset_cliff_horizon_estimation_error}. In this paper, we propose a new theoretical analysis that explains why AIL methods perform well even when the estimation error is large.

\subsection{Our Contribution}

To investigate the properties of AIL methods, we introduce a class of instances\footnote{An imitation learning instance refers to the underlying MDP and associated expert policy. When the context is clear, we refer to an instance as an MDP.} called RBAS MDPs, which are defined in detail in \cref{asmp:reset_cliff}. These instances have reachable bad absorbing states (RBAS), which is a characteristic found in practical tasks such as locomotion control and Atari games. When an agent makes a wrong action in these tasks, it typically goes to a terminal state with a zero reward. RBAS MDPs capture this feature, as shown in \cref{fig:toy_reset_cliff}. We confirm through numerical experiments that the well-known empirical observations about AIL methods hold for RBAS MDPs, suggesting that these instances serve as a suitable mathematical model for studying the algorithmic properties of AIL methods.

Now, we provide an informal preview of our main result. 

\begin{thm}[Informal Statement of \cref{theorem:ail_reset_cliff}] \label{thm:informal_main_result}
There is an AIL method with the total variation distance (called TV-AIL) that achieves the horizon-free imitation gap $\gO(\min\{1, \sqrt{|\gS|/N}\})$ for any instance in the class of RBAS MDPs.
\end{thm}

\begin{figure}[tbp]
     \centering
     \begin{subfigure}[b]{0.49\linewidth}
         \centering
         \includegraphics[width=\linewidth]{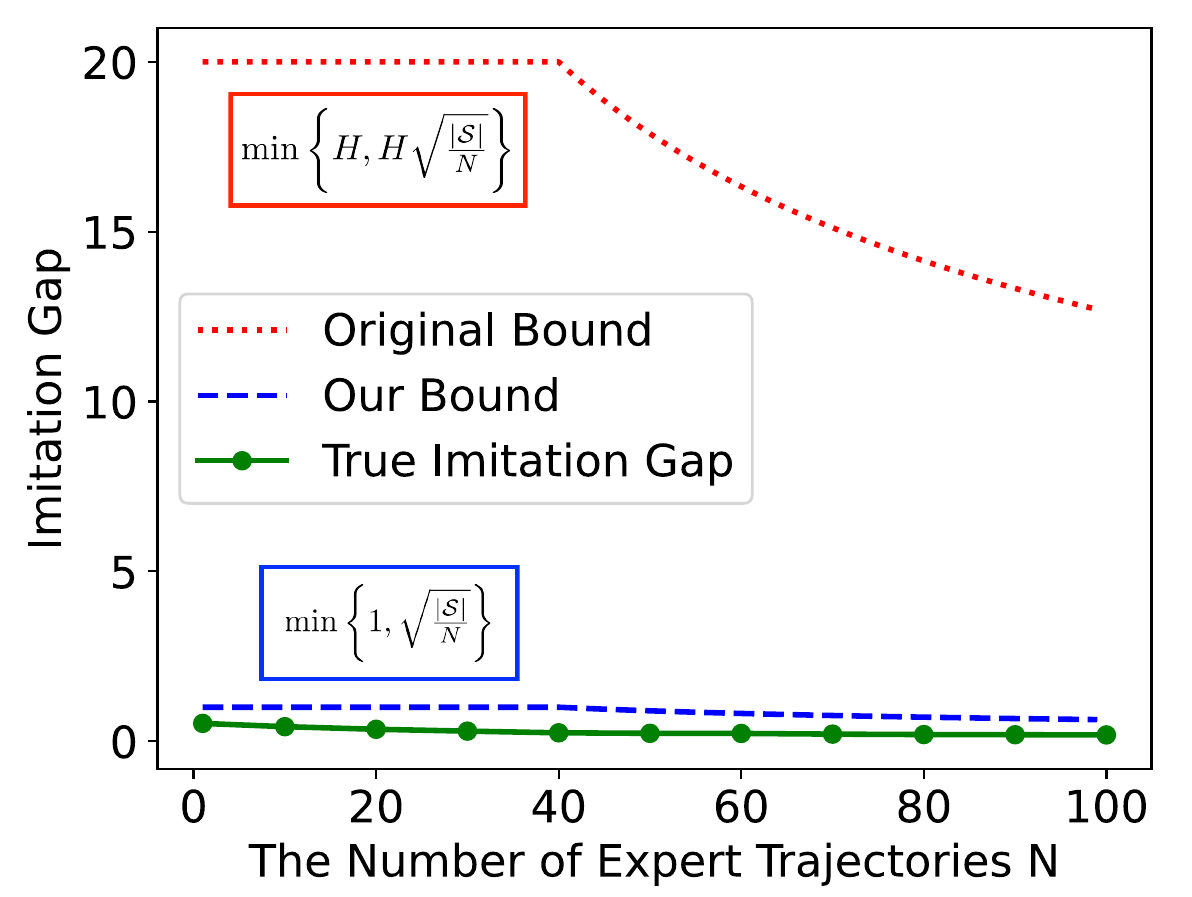}
         \caption{Imitation gap with different number of expert trajectories $N$ and horizon $H=20$}
         \label{fig:y equals x}
     \end{subfigure}
     \hfill
     \begin{subfigure}[b]{0.49\linewidth}
         \centering
         \includegraphics[width=\linewidth]{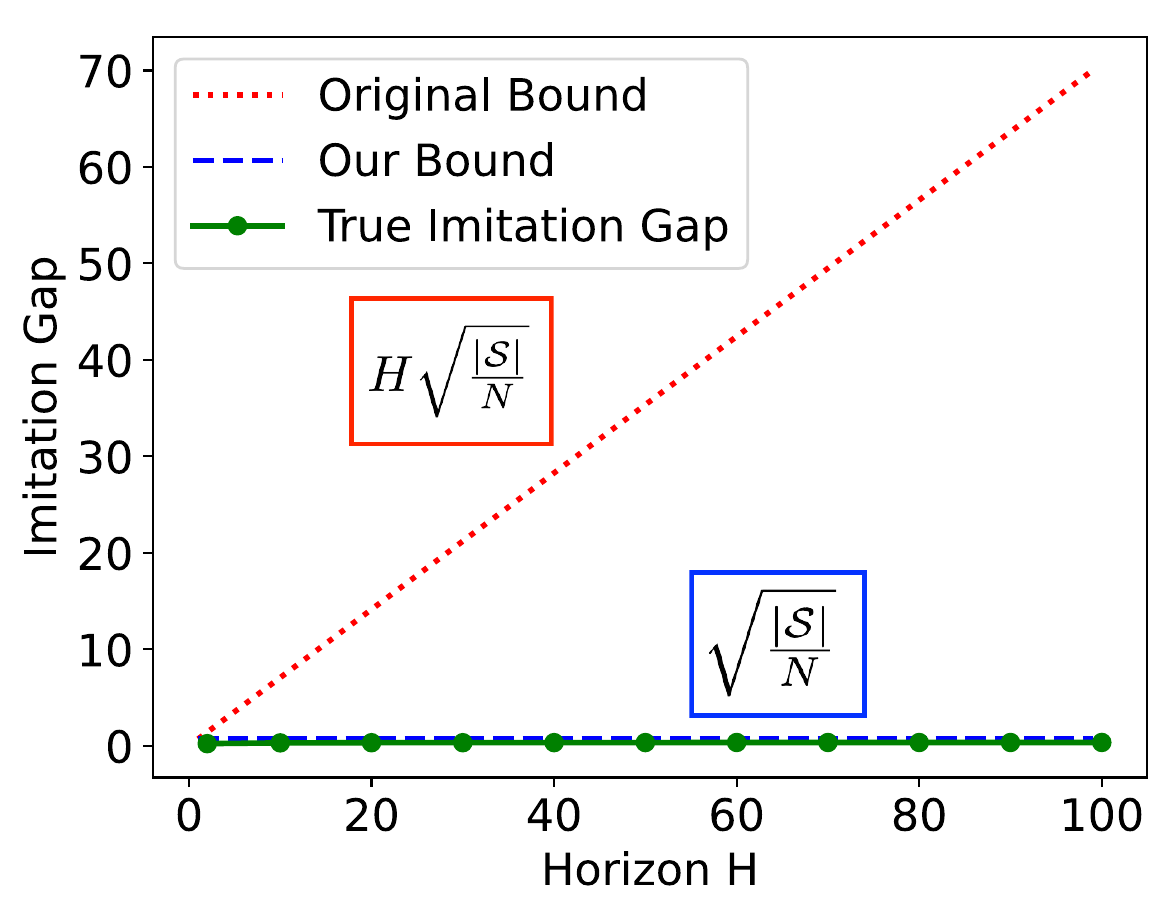}
         \caption{Imitation gap with different horizon $H$ and the number of expert trajectories $N=80$}
         \label{fig:three sin x}
     \end{subfigure}
        \caption{A comparison between the original bound, our bound and the true imitation gap on RBAS MDPs. Our theoretical bound predicts the empirical observation well.}
        \label{fig:bound_comparison}
\end{figure}

We remark that \cref{thm:informal_main_result} is meaningful in both small and large sample regimes, as the bound on the imitation gap is at most 1, which is smaller than the maximum value of $H$. Moreover, this theorem indicates that the performance of TV-AIL remains stable for long-horizon tasks, a finding that aligns well with empirical observations. To our best knowledge, \cref{thm:informal_main_result} provides the first horizon-free bound that is also meaningful in the small sample regime.

It is essential to note that this positive outcome is not solely due to the total variation distance but is instead derived from a new theoretical analysis. Interestingly, the reduction-and-estimation framework would suggest that the imitation gap of TV-AIL should be bounded by $\gO(\min \{H, H \sqrt{|\gS|/N} \})$, which is inadequate to explain the superior performance of TV-AIL. {As an illustration, we make a comparison between the original bound, our new bound and the true imitation gap on RBAS MDPs, which is depicted in Figure \ref{fig:bound_comparison}. In particular, our new bound predicts the true imitation gap in both situations with varying numbers of expert trajectories and horizons. In contrast, the original bound provides a quite loose prediction.} {Furthermore, the performance prediction given by our new bound also aligns well with the AIL's performance in practical locomotion control tasks shown in Tables \ref{tab:mujoco_sample_size} and \ref{tab:mujoco_horizon}, suggesting that the developed theory can explain the exceptional empirical performance of AIL methods.}

To derive \cref{thm:informal_main_result}, we develop a new stage-coupled analysis that characterizes the fundamental dynamic programming structure inherent to AIL. This structural insight distinguishes our analysis from the conventional reduction-and-estimation analysis, which fails to capture this essential property. Technically, the optimization problem in TV-AIL involves a multi-stage optimization and is generally non-convex. In particular, environment transition, policy, and the induced state-action distribution are coupled across stages. Previous analyses deal with this issue independently, whereas we do not. We overcome the analysis difficulty through a careful backward induction analysis. Our findings suggest that, from a forward perspective, an action decision can determine the future state-action distribution. Conversely, from a backward perspective, feedback from future time steps, via dynamic programming or RL, can enable TV-AIL to replicate the expert's actions in states not encountered in the expert dataset. This discloses that the stage-coupling structure in multi-stage optimization helps TV-AIL identify the expert action on states out of the demonstration distribution, thereby addressing the distribution shift issue and mitigating the compounding errors over horizons.

After obtaining the horizon-free imitation gap bound for TV-AIL, one might question whether this good algorithmic behavior holds for all instances or only for a specific set of instances. To address this concern, we provide a lower bound and a matching upper bound for TV-AIL.

\begin{thm}[Informal Statement of \cref{prop:lower_bound_vail} and \cref{thm:worst_case_vail}]   \label{thm:informal_lower_bound}
There exists a class of instances (refer to \cref{asmp:standard_imitation}) such that the worst-case imitation gap of TV-AIL $\gO(\min\{H, H\sqrt{|\gS|/N} \})$ is tight.
\end{thm}

We note that \cref{thm:informal_lower_bound} does not contradict \cref{thm:informal_main_result} since the instance classes in the two theorems do not overlap. Furthermore, \cref{thm:informal_lower_bound} reveals that for the hard instances satisfying \cref{asmp:standard_imitation}, each state is absorbing, and there is no connection between states. Consequently, the multi-stage policy optimization reduces to $H$ independent imitation problems, which is quite different from RBAS MDPs. Therefore, we conclude that without reasonable assumptions about MDPs, obtaining the nice algorithmic behavior of AIL methods is impossible.

The paper is structured as follows. Section \ref{sec:related_work} provides a review of related works, while Section \ref{sec:preliminary} provides a background on imitation learning. The main result of the paper, which concerns the horizon-free imitation gap of TV-AIL for RBAS MDPs, is presented in Section \ref{sec:horizon_free_sample_complexity}. Section \ref{sec:a_matching_lower_bound} discusses the worst-case performance of TV-AIL, and Section \ref{sec:conclusion} offers concluding remarks.

\section{Related Work}  \label{sec:related_work}

Over the years, there have been significant efforts to comprehend the behavior of the classical algorithm behavioral cloning (BC) \citep{ross2010efficient, syed2010reduction, ross11dagger, xu2020error, rajaraman2020fundamental, rajaraman2021value, nived2021provably, swamy2021moments}. It was shown in the seminal work \citep{ross2010efficient} that BC could result in compounding errors at the population level (when there are \emph{infinite} expert trajectories). Recent research \citep{rajaraman2020fundamental} established an imitation gap of ${\gO}(\min \{H, |\gS| H^2/N \})$ when expert trajectories are \emph{finite}. Furthermore, from an information-theoretic viewpoint, \citep{rajaraman2020fundamental} provided an impressive lower bound of $\Omega(H^2 |\gS|/N)$ for offline imitation learning\footnote{\dquote{Offline} here means that the agent has access to only a dataset collected by the expert policy. The agent does not know the transition function and cannot interact with the environment.}. This result suggests that for all offline imitation learning algorithms, including BC, the $\gO(H^2)$ term in the imitation gap is inevitable, i.e., a fundamental limit in the offline setting. Similar conclusions also hold for discounted and infinite-horizon MDPs \citep{xu2021error}.

The idea of state-action distribution matching in adversarial imitation learning (AIL) can be traced back to the apprenticeship learning algorithms \citep{pieter04apprentice, syed07game, syed08lp, ziebart2008maximum}. Previous AIL methods, such as FEM \citep{pieter04apprentice} and GTAL \citep{syed07game}, employed $\ell_2$-norm-based and $\ell_{\infty}$-norm-based metrics to measure the discrepancy between the state-action distributions, respectively. A reduction-and-estimation-based analysis was developed in \citep{pieter04apprentice, syed07game} to upper bound the imitation gap as $\gO(\min \{H, H\poly(|\gS||\gA|)/\sqrt{N} \})$ for FEM and GTAL, where $H$ is the horizon, $N$ is the number of demonstrations, $|\gS|$ is the state space size, and $|\gA|$ is the action space size. This analysis has been widely used in recent works \citep{wang2020computation, rajaraman2020fundamental, xu2021error, liu2021provably}.

Generative adversarial imitation learning (GAIL) \citep{ho2016gail} has become a significant milestone in AIL research. Unlike its predecessors, such as FEM and GTAL, GAIL employs the Jensen-Shannon divergence to measure the discrepancy between state-action distributions. Moreover, GAIL leverages neural networks to learn feature representations adaptively. GAIL, along with other AIL methods like FEM and GTAL, has been found to match expert performance for long-horizon tasks ($H=1000$) in the small sample regime ($N=1$) in locomotion control tasks \citep{ho2016gail}, and similar results are observed for Atari games such as Pong and Breakout \citep{yu2020intrinsic, cai2021imitation}. These empirical findings have garnered significant attention. Since the initial work of \citep{ho2016gail}, many empirical advancements have been made in the field of AIL \citep{fu2018airl, Kostrikov19dac, Kostrikov20value_dice, Brantley20disagreement, dadashi2021primal, liu2021energy}.

The theoretical understanding of why AIL methods can perform well is currently limited. CoRL 2019's best paper \citep{ghasemipour2019divergence} conjectured that \dquote{In common MDPs of interest, the reward function depends more on the state than action. Hence encouraging policies to explicitly match expert state marginals is an important learning criterion}. However, our lower bound argument shows that the conditions in this conjecture are \emph{not} enough to establish the superior performance of AIL.

Recent works \citep{xu2020error, swamy2021moments} have provided an answer by showing that the imitation gap of a class of AIL methods has a linear dependence on the decision horizon $H$, which is better than the quadratic dependence in BC's imitation gap. However, this result is true only at the population level, which differs from the practical scenario where expert trajectories are finite. To extend the results in \citep{xu2020error, swamy2021moments} to the finite sample case, one can use the reduction-and-estimation analysis, leading to an imitation gap bound similar to that of FEM and GTAL. Alternatively, \citep{rajaraman2020fundamental} proposed an AIL algorithm called MIMIC-MD, which achieves a better imitation gap $\gO(\min \{H, H^{3/2} |\gS|/N, H\sqrt{|\gS|/N} \})$ compared with BC and classical AIL methods (e.g., FEM and GTAL). However, this improvement is based on a more accurate estimation of the expert state-action distribution and only holds in the large sample regime, whereas we are interested in algorithmic behaviors in the small sample regime. Furthermore, \citep{rajaraman2020fundamental} established an information-theoretic lower bound of $\Omega(H |\gS|/N)$ that applies to AIL methods, which holds only in the large sample regime. However, the superior performance of AIL methods can be observed in the small sample regime, creating a gap between theory and practice.

In this paper, we focus on tabular MDPs, where no function approximation is used. We notice that \citep{cai2019lqr} and \citep{liu2021provably} studied AIL methods with linear function approximation, while the neural network approximation case is studied in \citep{wang2020computation, zhang2020generative, xu2021error}. Besides, \citep{jung2024sample} proposed improving AIL by enhancing feature representations, motivated by the function approximation theory that learning invariant features could reduce the VC dimension. The main message in the function approximation is that under structural assumptions, the dependence on $|\gS|$ in the imitation gap can be improved to the inherent dimension $d$ with function approximation. Nevertheless, we mainly study algorithmic behaviors in terms of the dependence on $H$, which is usually unrelated to function approximation. One may still worry that the function approximation is crucial to the superior performance of AIL methods. To address this concern, we show that AIL methods can recover the expert policy with one expert trajectory on constructed tabular MDPs, suggesting that the function approximation is not a key factor for our research problem. Finally, this paper primarily investigates a class of MDPs termed RBAS MDPs. This approach mirrors prior theoretical works that develop sharp analysis by focusing on specific instance classes: \citep{rajaraman2021value} analyzed MDPs satisfying recoverability conditions, while \citep{foster2024behavior} investigated deterministic expert policies.

\section{Preliminary}
\label{sec:preliminary}

This section introduces the Markov decision process, imitation learning setup, and representative algorithms, including behavioral cloning and adversarial imitation learning.

\subsection{Episodic Markov Decision Process}

In this paper, we study episodic Markov decision processes (MDPs) defined by the tuple $\gM = (\gS, \gA, \gP, r, H, \rho)$. Here, $\gS$ and $\gA$ denote the state and action spaces, respectively. The decision horizon $H$ represents the total number of actions taken by the agent in a trajectory, and $\rho$ is the initial state distribution. The transition function of the MDP is defined by $\gP = \{P_1, \cdots, P_{H}\}$, where $P_h(s_{h+1}|s_h, a_h)$ denotes the probability of transitioning to state $s_{h+1}$ at time step $h+1$ given the current state $s_h$ and action $a_h$ taken at time step $h$, for $h \in [H]$\footnote{$[x]$ denotes the set of integers from $1$ to $x$.}. Similarly, the reward function is denoted by $r = \{r_1, \cdots, r_{H}\}$, where $r_h: \gS \times \gA \rar [0, 1]$ assigns a reward to each state-action pair at time step $h$, for $h \in [H]$. A policy $\pi = \lb \pi_1, \cdots, \pi_H \rb$ is a sequence of non-stationary policies, where $\pi_h: \gS \rar \Delta(\gA)$ and $\Delta(\gA)$ is the probability simplex over $\gA$. The function $\pi_h (a|s)$ gives the probability of selecting action $a$ at time step $h$ given the current state $s$, for $h \in [H]$.

The sequential decision process proceeds as follows: at the start of each episode, the environment resets to an initial state sampled from $\rho$. Then, the agent observes the state $s_h$ and selects an action $a_h$ according to the policy $\pi_h(a_h|s_h)$. After that, the environment transitions to the next state $s_{h+1}$ following $P_h(s_{h+1}|s_h, a_h)$ and sends a reward signal $r_h(s_h, a_h)$ to the agent. The process repeats for a total of $H$ steps until the end of the episode.

The effectiveness of a policy is evaluated based on its expected long-term return, which is also known as its \emph{policy value}. It is defined as follows:
\begin{align*}
    V(\pi) := \expect \bigg[ \sum_{h=1}^{H} r_h(s_h, a_h) \mid s_1\sim \rho; a_h \sim \pi_h (\cdot|s_h), s_{h+1} \sim P_h(\cdot|s_h, a_h), \forall h \in [H] \bigg].
\end{align*}
To facilitate later analysis, we introduce the state-action distribution induced by a policy $\pi$: 
\begin{align*}
    d_h^{\pi}(s, a) := \sP ( s_h = s, a_h = a \mid s_1 \sim \rho, a_\ell \sim \pi_h(\cdot|s_\ell), s_{\ell+1} \sim P_\ell(\cdot|s_\ell, a_\ell), \forall \ell \in [h]).
\end{align*}
In other words, $d_h^{\pi}(s, a)$ quantifies the visitation probability of state-action pair $(s, a)$ in time step $h$. Then according to the definition, we obtain the dual form of policy value \citep{puterman2014markov}:
\begin{align}
\label{eq:value_dual_representation}
    V(\pi) = \sum_{h=1}^H \sum_{(s, a)} d_h^{\pi}(s, a) r_h (s, a). 
\end{align}
This formula will be used in the analysis of AIL methods. Likewise, we can define the state distribution $d^{\pi}_h(s)$. According to the definition, we have $d^{\pi}_h(s) = \sum_{a} d^{\pi}_h(s, a)$. Unless mentioned, when we use the symbol $d^{\pi}_h$, we mean the state-action distribution.

\subsection{Imitation Learning} 

The goal of imitation learning is to learn a high quality policy directly from expert demonstrations. Typically, it is assumed that a (nearly optimal) expert policy $\piE$ is available to interact with the environment and generate a dataset of $N$ trajectories, each with length $H$:
\begin{align*}
    \gD = \bigg\{ \tr = \lp s_1, a_1, s_2, a_2, \cdots, s_H, a_H \rp; s_1 \sim \rho, a_h \sim \piE_h(\cdot|s_h), s_{h+1} \sim P_h(\cdot|s_h, a_h), \forall h \in [H] \bigg\}.
\end{align*}
Each trajectory in the dataset $\gD$ is obtained independently by executing the expert policy $\piE$. The learner can then use $\gD$ to mimic the expert and learn a good policy. The quality of imitation is evaluated using the (expected) \emph{imitation gap}:
\begin{align*}
V(\piE) - \expect \ls V(\pi) \rs,
\end{align*}
where the expectation is taken over the randomness in generating $N$ trajectories. It is important to note that $\pi$ is a random variable that depends on $\gD$. Furthermore, during the training phase, IL algorithms do \emph{not} have access to reward information. Ideally, a good learner should be able to perfectly imitate the expert, resulting in a small imitation gap. For theoretical analysis purposes, it is common to assume that the expert policy is deterministic, which is a widely-used assumption in the literature \citep{xu2020error, rajaraman2020fundamental, nived2021provably}.

\subsection{Behavioral Cloning}

Behavioral cloning \citep{Pomerleau91bc} is a classic offline algorithm for solving the imitation learning task. Its main idea is to perform maximum likelihood estimation for the expert trajectory. Specifically, the likelihood for a trajectory $\tr$ is given by:
\begin{align*}
    \sP\lp \tr \rp = \rho(s_1) \prod_{h=1}^{H} \pi_h (a_h|s_h) P_{h}(s_{h+1} | s_h, a_h).
\end{align*}
Then, we obtain that 
\begin{align*}
    \log \lp \sP\lp \tr \rp \rp = \sum_{h=1}^{H} \log \lp \pi_h(a_h |s_h) \rp + \text{constant},  
\end{align*}
where the constant term is unrelated to policy $\pi$. By extending this idea to $N$ expert trajectories in the given dataset $\gD$, we obtain the following optimization problem for BC:
\begin{align}   \label{eq:objective_bc}
   \max_{\pi \in \Pi} \sum_{h=1}^{H} \sum_{(s_h, a_h) \in \gD} \log \lp \pi_h(a_h |s_h) \rp,
\end{align}
where $\Pi$ is the set of stochastic policies. In the tabular setting, the optimal solution could be 
\begin{align}   \label{eq:pi_bc}
 \pi^{\bc}_h(a|s) = \left\{ \begin{array}{ll}
      \frac{\# \tr_h(\cdot, \cdot) = (s, a)}{\sum_{a^\prime} \# \tr_h(\cdot, \cdot) = (s, a^\prime)}  & \text{if } \# \tr_h(\cdot) = s > 0  \\
        \frac{1}{|\gA|} & \text{otherwise}  
    \end{array} \right.
\end{align}
Here $\# \tr_h(\cdot, \cdot) = (s, a)$ ($\# \tr_h(\cdot) = s$) refers to the number of trajectories such that their state-action pairs (states) are equal to $(s, a)$ ($s$) in time step $h$. That is, $\pi^{\bc}_h(a|s)$ computes the empirical conditional distribution for visited states and the uniform distribution for non-visited states. Since BC does not know the expert action on non-visited states, it may suffer from the compounding errors issue \citep{ross2010efficient}. On the theoretical side, the imitation gap of BC has been analyzed in \citep{rajaraman2020fundamental}.

\begin{thm}[\citep{rajaraman2020fundamental}]  \label{thm:imitation_gap_bc}
For any tabular and episodic MDP, given $N$ expert trajectories, the imitation gap of BC is $\gO(\min\{H,  |\gS|H^2/N \})$.
\end{thm}

\subsection{Adversarial Imitation Learning}

Adversarial imitation learning (AIL) is a type of algorithm that aims to match the state-action distribution between the expert and the learner's policy. One example of an AIL algorithm is GAIL \citep{ho2016gail}, which employs the Jensen-Shannon (JS) divergence to quantify the difference in state-action distributions:
\begin{align*}
    \min_{\pi \in \Pi} \sum_{h=1}^{H} \JS \lp d^{\pi}_h (\cdot, \cdot),  \widehat{d^{\piE}_h}(\cdot, \cdot) \rp,
\end{align*}
where for two distributions $p$ and $q$ on the set $\gX$, $\JS(p, q) = \KL(p, (p + q)/2) + \KL(q, (p+q)/2)$, where $\KL(p, q) = \sum_{x} p(x) \log (p(x)/ q(x))$. Moreover, $\widehat{d^{\piE}_h}$ is the estimation of the marginal distribution $d^{\piE}_h$:
\begin{align}   \label{eq:estimate_by_count}
   \widehat{d^{\piE}_h} (s, a) := \frac{  \sum_{\tr \in \gD}  \indict\lb \tr(\cdot, \cdot) = (s, a) \rb }{N},
\end{align} 
where we use $\indict\{ \cdot \}$ to denote the indicator function, and $ \sum_{\tr \in \gD}  \indict\lb \tr(\cdot, \cdot) = (s, a) \rb$ counts the number of expert trajectories that have the state-action pair $(s, a)$ in time step $h$. If we consider the dual form of the JS divergence, we recover the common min-max formulation for GAIL:
\begin{align*}
    \min_{\pi \in \Pi} \max_{c \in \gC_{\operatorname{JS}}} \,\, \sum_{h=1}^H \expect_{(s, a) \sim d^{\pi}_h}  \left[ \log c_h(s, a) \right] + \expect_{(s, a)\sim \widehat{d^{\piE}_h} } \ls  \log (1 - c_h(s, a) )\rs ,
\end{align*}
where $\gC_{\operatorname{JS}}$ is the set of functions $c_h: \gS \times \gA \rar (0, 1)$, which is often called discriminator. This min-max formulation can be implemented with neural networks in practice. In particular, the discriminator recovers a cost (negative reward) function from expert data, and the goal of the policy is to minimize this cost using reinforcement learning (RL) approaches \citep{sutton2018reinforcement}.

In this paper, we mainly consider the total variation (TV) distance to measure the state-action distribution discrepancy, which leads to the following formulation:
\begin{align}  \label{eq:ail}
  \min_{\pi \in \Pi} \sum_{h=1}^{H} \sum_{(s, a) \in \gS \times \gA} | d^{\pi}_h(s, a) - \widehat{d^{\piE}_h}(s, a) |,
\end{align}
We call such an approach TV-AIL (total-variation-distance-based AIL) and use $\piail$ to denote any optimal solution to \eqref{eq:ail}. Similarly, there exists a min-max formulation for TV-AIL:
\begin{align}  
    \min_{\pi \in \Pi} \max_{c \in \gC_{\TV}}  \,\,  & \sum_{h=1}^H \expect_{(s, a) \sim d^{\pi}_h} \left[c_h(s, a) \right] - \expect_{(s, a)\sim \widehat{d^{\piE}_h} } \left[c_h(s, a) \right], \label{eq:ail_min_max} 
\end{align}
where $\gC_{\TV}$ is the set of functions $c_h: \gS \times \gA \rar [-1, 1]$. In practice, we employ multi-layer neural networks to represent $c$ and apply gradient-descent-ascent to optimize \eqref{eq:ail_min_max}. It is expected that the set of neural networks could approximate $\gC_{\TV}$ well due to its expressive power \citep{hornik1989multilayer}. We note that TV-AIL has a comparative performance with GAIL for practical tasks.

In this paper, we primarily focus on analyzing TV-AIL in the known-transition setting, where computing the state-action distribution $d^{\pi}$ in the optimization step is straightforward. This setting has been considered in the existing literature \citep{rajaraman2020fundamental, rajaraman2021value, nived2021provably, xu2021error}. It is worth noting that in the known-transition setting, we still need expert trajectories to recover the expert policy since the reward information is inaccessible in the training phase of imitation learning. In addition, our analysis of TV-AIL can be extended to other AIL methods with additional efforts, but we omit this for the sake of clarity. Specifically, the main difference between AIL methods is the divergence function (in terms of measuring the discrepancy of state-action distributions). As mentioned, FEM and GTAL use the $\ell_2$-norm-based and $\ell_{\infty}$-norm-based divergences, respectively, while the $\ell_1$-norm-based divergence is considered in TV-AIL. Nevertheless, all norms are \dquote{equivalent} in the finite dimension space in the sense that they can be upper bounded mutually \citep{yosida2012functional}. JS divergence and TV distance are connected via Pinsker's inequality; see \citep{xu2020error} for more discussion. Therefore, AIL methods are likely to share similar algorithmic properties, and we omit a general analysis to maintain clarity in our presentation.

\section{A Horizon-free Imitation Gap of TV-AIL}
\label{sec:horizon_free_sample_complexity}

In this section, we present the main conclusion of the paper, which is that TV-AIL is capable of achieving the horizon-free imitation gap for MDPs that are abstracted from tasks like locomotion control. To provide evidence in support of this claim, we first examine two empirical observations from the literature \citep{ho2016gail, ghasemipour2019divergence, yu2020intrinsic, cai2021imitation}. The first observation pertains to the sample size.

\begin{ec}  \label{ec:sample_size}
For practical tasks (e.g., locomotion control), with a few expert trajectories (e.g., 1 expert trajectory), AIL methods (e.g., TV-AIL) can achieve a small imitation gap, whereas BC cannot.
\end{ec}

\begin{figure}[t]
    \centering
    \includegraphics[width=0.5\linewidth]{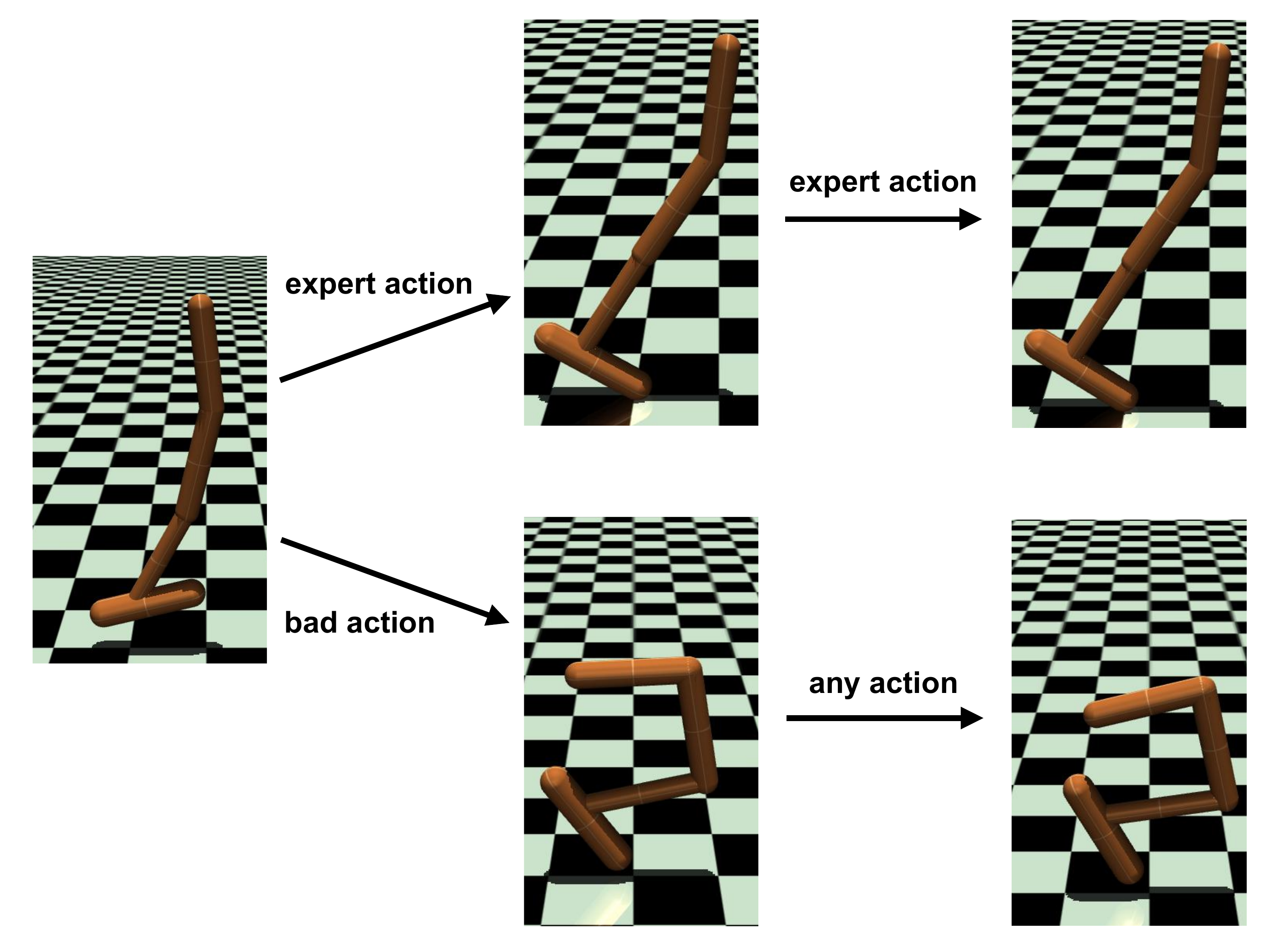}
    \caption{An illustration of the Hopper task in the MuJoCo locomotion benchmark. The top row shows the robot executing the expert action and successfully jumping forward, resulting in a positive reward. On the other hand, the bottom row illustrates the robot taking a non-expert action, resulting in a fall and a zero reward on the terminal (absorbing) state.}
    \label{fig:hopper_illustration}
\end{figure}

To fully support the conclusion in \cref{ec:sample_size}, we provide the experimental results in \cref{tab:mujoco_sample_size}, which show that TV-AIL can perform well even with limited expert trajectories on MuJoCo locomotion control tasks. Moreover, an additional empirical observation to consider is about the decision horizon.

\begin{ec}  \label{ec:horizon}
For practical tasks (e.g., locomotion control), even if the decision horizon is large (e.g., $H=1000$), with limited expert trajectories (e.g., 1 expert trajectory), AIL methods (e.g., TV-AIL) can achieve a small imitation gap.
\end{ec}

We present evidence supporting the second empirical observation in \cref{tab:mujoco_horizon}, which has received less attention in the literature except for a few studies such as \citep{xu2020error, xu2021error}. Figure 2 in \citep{xu2020error} shows that AIL methods' performance is less affected by the decision horizon, but there is no rigorous theoretical explanation for this observation. Nonetheless, we believe that this observation is crucial in gaining insights into the algorithmic behavior of AIL methods.

\begin{table*}[t]
\centering
\caption{Scaled imitation gap on Hopper, HalfCheetah and Walker2d with $N=1$.}
\label{tab:mujoco_horizon}
\begin{tabular}{@{}cccccc@{}}
\toprule
     &    & $H=100$ & $H=500$ & $H=1000$ &$H=2000$ \\ \midrule
\multirow{2}{*}{Hopper (scale=3.2)} &BC      &\meanstd{0.80}{1.72} & \meanstd{178.56}{79.27} & \meanstd{784.97}{28.09} & \meanstd{1950.42}{36.20}       \\
 &TV-AIL       & \meanstd{4.96}{4.56}   & \meanstd{-1.73}{6.77}       & \meanstd{10.38}{11.25}  & \meanstd{-9.92}{37.15} \\
\midrule
\multirow{2}{*}{HalfCheetah (scale=7.7)} &BC      &  \meanstd{56.44}{9.23} & \meanstd{491.91}{21.01} & \meanstd{1058.48}{8.27} & \meanstd{2198.61}{12.93}       \\
 &TV-AIL      & \meanstd{8.49}{8.26}  & \meanstd{-24.73}{11.46}       & \meanstd{-22.45}{101.65}  & \meanstd{-169.83}{16.06} \\
 \midrule
\multirow{2}{*}{Walker2d (scale=5.0)} &BC      & \meanstd{10.25}{7.89} & \meanstd{413.02}{13.65} & \meanstd{1002.13}{9.68} & \meanstd{2158.84}{2.05}       \\
 &TV-AIL      & \meanstd{-0.18}{1.04}  & \meanstd{16.26}{22.64}       & \meanstd{12.93}{19.61}  & \meanstd{71.69}{66.30} \\ \bottomrule
\end{tabular}
\end{table*}

While exploring the theoretical aspects of the above observations is intriguing, it is a challenging task due to the many essential factors that contribute to superior performance, such as environment preprocessing, neural network architectures and optimizers \citep{orsini2021what}. Capturing all these factors in a simple and intuitive theory is difficult. However, in this paper, we provide an answer in the tabular setting. Here, the expert policy can be well approximated, and efficient computation procedures are available, rendering approximation and optimization errors irrelevant. To begin, we investigate a class of MDPs with structural transitions.

\subsection{RBAS MDPs}
\label{subsec:reset_cliff}

In this part, we introduce a class of tabular and episodic MDPs called RBAS MDPs, which will be used to study algorithmic behaviors of TV-AIL.

\begin{asmp}[RBAS MDPs]   \label{asmp:reset_cliff}
For a tabular and episodic MDP and an expert policy, we assume that 
\begin{itemize}   
    \item State space is divided into the sets of \dquote{good} states and \dquote{bad} states, i.e., $\gS = \goodS \cup \badS, \goodS \cap \badS = \emptyset$. All bad states only have transitions to themselves, i.e., for any $b \in \badS$, $a \in \gA$ and $h \in [H]$, $\sum_{b^\prime \in \badS} P_h(b^\prime|b, a) = 1$. For any good state, $a^1$ is the expert action and the others are non-expert actions.\footnote{For the simplicity of notations, we assume that expert actions are the same on all good states. Nevertheless, our results can be seamlessly extended to the case where expert actions are different on good states.}

    \item The initial state distribution supports on the set of good states. That is, $\rho (s) > 0, \forall s \in \goodS$ and $\sum_{s \in \goodS} \rho (s) = 1$.
    \item For action $a^1$, we have for any state $s, s^\prime \in \goodS$ and $h \in [H]$, $P_h(s^\prime|s, a^{1}) > 0$. 
    \item For action $a \ne a^{1}$, we have for any state $s, s^\prime \in \goodS$ and $h \in [H]$, $P_h(s^\prime|s, a) = 0$.
\end{itemize}
\end{asmp}

The first assumption is sound in practice. For example, in locomotion control tasks, "good" states mean that the robot can walk well, while "bad" states mean that the robot falls down and cannot recover back. Furthermore, the expert action is crucial to maintaining such good status. The second assumption is posed to avoid the trivial case where the agents start from a bad absorbing state. The third assumption is widely applicable and means that the expert action does not lead to a bad state. The last assumption means that non-expert actions lead to transitions into bad states. Together, we see that these assumptions hold in practical tasks. For instance, in locomotion control tasks, once taking the non-expert (wrong) action, the robot falls and goes into a bad terminal (absorbing) state; refer to \cref{fig:hopper_illustration}.

While slightly idealized, RBAS MDPs provide the controlled analysis setup to isolate and rigorously investigate the core phenomena of AIL. Besides, the properties of RBAS MDPs can be relaxed in the sense that 1). with a small probability, the agent can also transit into good states by taking non-expert actions on good states; 2). with a small probability, the agent can return to good states starting from bad states. For interested readers, please refer to \cref{subsubsec:rbas_mdps_extensions} for details.

A simple example of RBAS MDPs with three states and two actions is shown in \cref{fig:toy_reset_cliff}. For this MDP, $s^{1}$ and $s^{2}$ are good states, while $b$ is a bad state. Here we use the superscript to indicate the state index, which can avoid confusion with the time step in the subscript. Since there are only two actions, we use colors, rather than the superscript, to distinguish them: $\GREEN{a}$ is the expert action and $\BLUE{a}$ is the non-expert action. Besides, $H = 2$ is considered in this MDP and transition probabilities are indicated by digits on the arrows in \cref{fig:toy_reset_cliff}. The initial state distribution $\rho$ is 
\begin{align*}
    \rho = (\rho (s^{1}), \rho (s^{2}), \rho (b)) = \lp \frac{1}{2}, \frac{1}{2}, 0 \rp.
\end{align*}
As one can see, if the agent takes the non-expert action $\BLUE{a}$ on a good state, it goes to the bad absorbing state $b$. In this case, the imitation gap can be at most $2$.

\begin{figure}[htbp]
    \centering
    \includegraphics[width=0.5\linewidth]{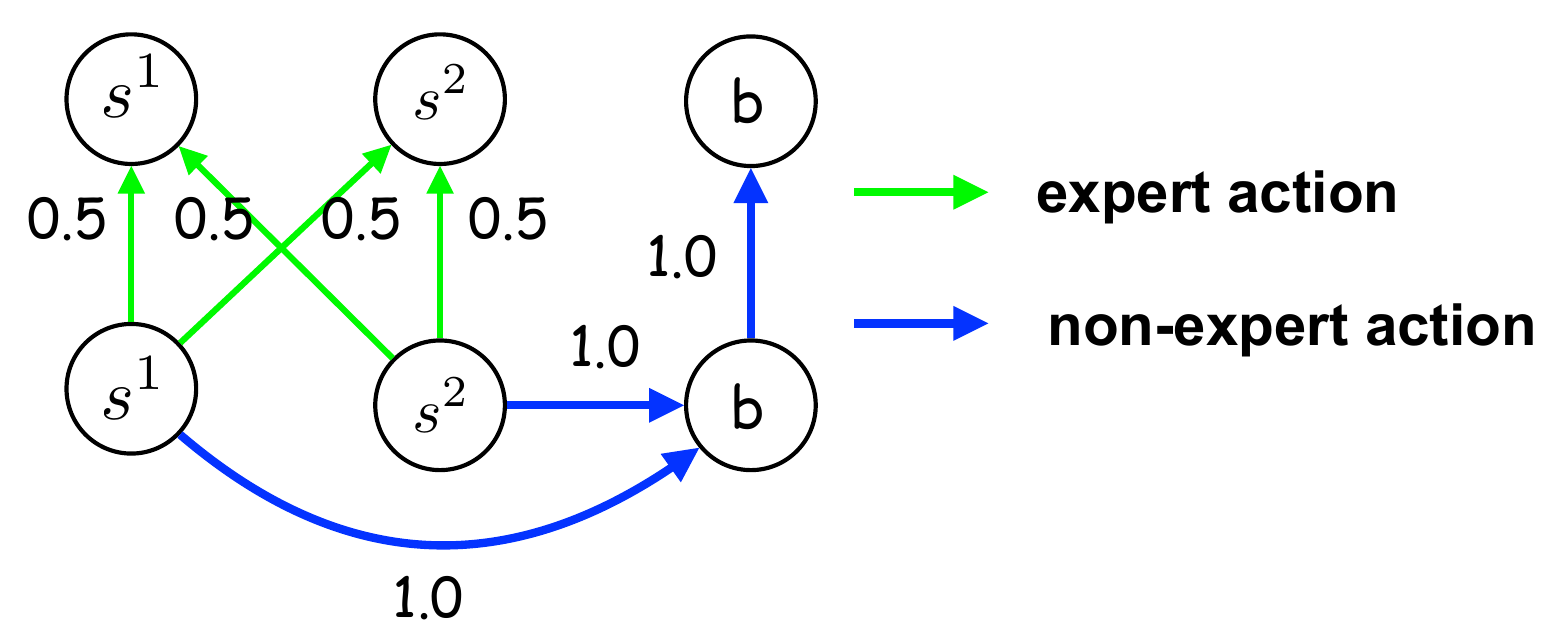}
    \caption{A simple MDP corresponding to \cref{asmp:reset_cliff}. Digits indicate the transition probabilities.}
    \label{fig:toy_reset_cliff}
\end{figure}

It is worth mentioning that the MDP instance used as a lower bound for offline imitation learning algorithms, as proposed by \citep{rajaraman2020fundamental}, also satisfies the assumptions of \cref{asmp:reset_cliff}. Specifically, the instance has $|\gS|-1$ good states and $1$ bad state, and the initial state distribution $\rho$ is formulated as follows:
\begin{align*}
\rho &= \lp \rho (s^{1}), \cdots, \rho (s^{|\gS|-2}), \rho (s^{|\gS|-1}), \rho (b) \rp \\
&= \lp \frac{1}{N+1}, \cdots, \frac{1}{N+1}, 1 - \frac{\vert \gS \vert-2}{N+1}, 0 \rp.
\end{align*}
The transition function for this instance is carefully designed to satisfy the RBAS property as well. Specifically, for each good state, executing the expert action $\GREEN{a}$ leads to a state transition according to $\rho$, i.e., $P_{h}(\cdot|s^{i}, \GREEN{a}) = \rho(\cdot)$ for $i \in [|\gS| - 1]$. For the bad state, it is absorbing, i.e., $P(b|b, \BLUE{a}) = 1$. It has been proved that any offline imitation learning approach, including BC, suffers an imitation gap of at least $\Omega( \vert \gS \vert H^2 / N )$ on this instance \citep{rajaraman2020fundamental}. Note that the imitation gap bound of BC in \cref{thm:imitation_gap_bc} matches this lower bound, and hence, we conclude that BC has a tight imitation gap of $\Theta ( \min \{ H, \vert \gS \vert H^2 / N \} )$ on RBAS MDPs.

We confirm that the conclusions drawn in \cref{ec:sample_size} and \cref{ec:horizon} also apply to RBAS MDPs. To illustrate, we consider a specific RBAS MDP with $|\gS|=20$ and $|\gA| = 2$, similar to the one shown in \cref{fig:toy_reset_cliff}, consisting of $19$ good states and $1$ bad state. On good states, the reward equals 1 on the expert action and 0 on the other actions. On the bad state, all actions have zero rewards. The initial state distribution is uniform over good states.  Please refer to Appendix \ref{appendix:experiment_details} for details. First, we examine the imitation gaps of BC and TV-AIL with respect to the number of expert demonstrations. The results, summarized in \cref{tab:reset_cliff_sample_size}, support our claim that TV-AIL can closely approximate the expert's performance with only one expert trajectory, whereas BC struggles in this scenario, in line with \cref{ec:sample_size}. Second, we evaluate TV-AIL and BC on the same MDP with varying horizons, as shown in \cref{tab:reset_cliff_horizon}, and we observe that the imitation gap of TV-AIL remains largely unchanged, indicating the validity of \cref{ec:horizon}. Our empirical findings support the idea that the algorithmic behaviors of offline imitation learning approaches are consistent with RBAS MDPs, which are therefore a suitable model for studying these behaviors.

\begin{table}[htbp]
\centering
\caption{Imitation gap on RBAS MDPs with $H=1000$. We report the mean of imitation gap with the standard deviation over 20 independent experiments (same with the remaining tables).}
\label{tab:reset_cliff_sample_size}

\begin{tabular}{@{}ccccc@{}}
\toprule
         & $N=1$ & $N=4$ & $N=7$ &$N=10$ \\ \midrule
BC       & $998.87 \pm 0.15$& $998.57 \pm 0.14$ & $998.12 \pm 0.16$  & $997.60 \pm 0.30$        \\
TV-AIL      & $0.71 \pm 0.00$ & $0.64 \pm 0.01$       & $0.61 \pm 0.02$  & $0.55 \pm 0.02$ \\ \bottomrule
\end{tabular}

\end{table}

\begin{table}[htbp]
\centering
\caption{Imitation gap on the RBAS MDP with $N=1$.}
\label{tab:reset_cliff_horizon}

\begin{tabular}{@{}ccccc@{}}
\toprule
         &$H=100$  & $H=500$  & $H=1000$ &$H=2000$ \\ \midrule
BC      & $98.89 \pm 0.14$ & $498.91 \pm 0.10$ & $998.87 \pm 0.15$& $1998.88 \pm 0.10$       \\
TV-AIL     & $0.69 \pm 0.00$ & $0.70 \pm 0.00$  & $0.71 \pm 0.00$   & $0.71 \pm 0.00$ \\ \bottomrule
\end{tabular}

\end{table}

\subsection{Main Results}

In this section, we introduce TV-AIL's horizon-free imitation gap guarantee based on \cref{prop:ail_general_reset_cliff}. This proposition establishes the optimality condition for the problem in \eqref{eq:ail} for tabular and episodic MDPs that satisfy \cref{asmp:reset_cliff}.

\begin{prop}
\label{prop:ail_general_reset_cliff}
For any tabular and episodic MDP satisfying \cref{asmp:reset_cliff},  suppose that $\piail$ is a minimizer of \eqref{eq:ail}. When $N \geq 1$, we have the following optimality condition almost surely:
\begin{align*}
    \piail_h (a^1|s) = 1, \forall s \in \goodS, h \in [H-1].
\end{align*}
\end{prop}

\cref{prop:ail_general_reset_cliff} implies that TV-AIL can recover the expert actions on both visited and non-visited states within the first $H-1$ time steps. In contrast, BC cannot achieve this and may choose non-expert actions on non-visited states. The reason for this difference in behavior is as follows.

\begin{rem}

It is important to note that the state-action distribution matching problem in \eqref{eq:ail} is a multi-stage policy optimization problem, where decision variables may be interdependent. To illustrate this,  recall the definition of $d^{\pi}_h(s, a)$: \begin{align}
    &\quad d^{\pi}_h(s, a) \nonumber
    \\
    &= d^{\pi}_h(s) \pi_h(a|s) \nonumber
    \\
    &= \ls \sum_{(s^\prime, a^\prime)} d^{\pi}_{h-1}(s^\prime, a^\prime) P_{h-1}(s|s^\prime, a^\prime) \rs  \pi_h(a|s) \nonumber  \\
    &= \ls \sum_{(s^\prime, a^\prime)} d^{\pi}_{h-1}(s^\prime) \pi_{h-1}(a^\prime | s^\prime) P_{h-1}(s|s^\prime, a^\prime) \rs  \pi_h(a|s) . \label{eq:flow_link}
\end{align}
As can be seen, the variables $\pi_{h-1}$ and $\pi_h$ are linked by the intermediate variable $d^{\pi}_{h}(s, a)$, which means that the optimization of $\pi_{h-1}$ and $\pi_{h}$ is carried out jointly. This coupling structure is a significant difference between TV-AIL and BC. Although BC also optimizes a non-stationary policy, its objective does not have any coupling structure, as shown in \eqref{eq:objective_bc}. Furthermore, while BC solves a convex optimization problem in \eqref{eq:objective_bc}, TV-AIL may solve a non-convex policy optimization problem due to the coupling structure in \eqref{eq:flow_link}.

\begin{prop}    \label{prop:non_convex}
There exist tabular and episodic MDPs such that the objective of TV-AIL in \eqref{eq:ail} is non-convex. 
\end{prop}

Proof of \cref{prop:non_convex} is deferred to Appendix \ref{appendix:proof_prop:non_convex}. The intuition is that function $f(x, y) = xy$ is non-convex. In our context, $x$ and $y$ may refer to $\pi_{h-1}$ and $\pi_{h}$, respectively. Before the follow-up discussion about \cref{prop:ail_general_reset_cliff}, we comment that even though the policy optimization in TV-AIL could be non-convex, there exists a linear-programming-based procedure for \eqref{eq:ail}, which runs in a polynomial time. Furthermore, gradient-based methods can also return an approximately optimal solution for \eqref{eq:ail_min_max} in a polynomial time. Thus, we do not need to worry much about the computation efficiency; please see Appendix \ref{appendix:optimization_procedures} for details.

Despite the non-convexity, we establish the global optimality condition in \cref{prop:ail_general_reset_cliff}, based on a staged-coupled analysis. We defer details to \cref{subsec:key_analysis_ideas}. We remark that the recursive structure in \eqref{eq:flow_link} is essential to recover the expert action on non-visited states, as the state-action distribution matching loss in a large time step (say $\Vert d^{\pi}_h - \widehat{d^{\piE}_h} \Vert_1$) can affect the decision variables in a small time step (say $\pi_{1}$). For an explanation of this point, please refer to \cref{example:ail_success} in Appendix~\ref{appendix:proof_claim_example_ail_success}. Since there is no future guidance in the last time step, we cannot guarantee that the obtained policy in the last time step follows the expert policy, as stated in \cref{prop:ail_general_reset_cliff}.
\end{rem}

With \cref{prop:ail_general_reset_cliff}, it is immediate to obtain the imitation gap of TV-AIL for RBAS MDPs.

\begin{thm}[Horizon-free Imitation Gap of TV-AIL on RBAS MDPs]     \label{theorem:ail_reset_cliff}
For any tabular and episodic MDP satisfying \cref{asmp:reset_cliff}, suppose that $\piail$ is any minimizer of \eqref{eq:ail}. Then we have that
\begin{align}   \label{eq:ail_reset_cliff}
    V({\piE}) - \expect \ls V(\piail) \rs \leq \gO \lp \min \lb 1, \sqrt{\frac{\vert \gS \vert}{N}} \rb \rp,
\end{align}
where the expectation is taken over the randomness in collecting $N$ expert trajectories.
\end{thm}

\begin{rem}
\cref{theorem:ail_reset_cliff} says that TV-AIL has two types of imitation gaps for RBAS MDPs, depending on the sample size. In the small sample regime (i.e., $N \lesssim |\gS|$), the first term dominates in \eqref{eq:ail_reset_cliff}, indicating the imitation gap of TV-AIL is at most 1. In particular, this guarantee holds for any $H$ and $|\gS|$. On the other hand, in the large sample regime (i.e., $N \gtrsim |\gS|$), the second term in \eqref{eq:ail_reset_cliff} dominates and the imitation gap diminishes to 0 as $N$ goes to infinity. This result can explain the empirical results in {\cref{tab:reset_cliff_sample_size} and \cref{tab:reset_cliff_horizon}}. By the similarity between RBAS MDPs and locomotion tasks from the MuJoCo benchmark, \cref{theorem:ail_reset_cliff} can also help understand the superior performance of TV-AIL in practice. To our best knowledge, this represents the first horizon-free imitation gap bound that is also meaningful in the small sample regime.

We clarify that the good result in \cref{theorem:ail_reset_cliff} is not due to the total variation distance. Instead, we observe that other AIL methods (e.g., FEM, GTAL and GAIL) also have comparative performance with TV-AIL on RBAS MDPs; see the numerical result in Appendix~\ref{appendix:algorithmic_behaviors_of_other_ail_methods}.
\end{rem}

\begin{rem}

Recall that for RBAS MDPs, the imitation gap of BC is ${\gO}(\min \{H, |\gS| H^2/N \})$. In the small sample regime where $N \lesssim |\gS|$, the imitation gap bound of TV-AIL is much smaller than that of BC. This result suggests that through the coupled multi-stage optimization, TV-AIL can effectively overcome the issue of compounding errors in offline imitation.

However, in the large sample regime where $N \gtrsim |\gS|H^4$, careful readers may notice that BC has a better imitation gap bound than TV-AIL. It is important to note, however, that in this regime, the imitation gap is less than $\gO(1/H^{2})$, which is an extremely small value. Thus, this observation may not often be significant in practice. Nonetheless, we would like to comment that the sample barrier issue is not a fundamental problem for TV-AIL. By making a slight modification to TV-AIL, we can achieve an improved imitation gap bound of $\gO(\min \{1, |\gS|/N \})$ for RBAS MDPs, which is better than the imitation gap bound of BC for the entire sample regime. Please refer to Appendix \ref{appendix:vail_sample_barrier} for further discussion on this topic.

\end{rem}

\begin{rem}
It is important to note that the horizon-free imitation gap presented in \cref{theorem:ail_reset_cliff} does not contradict the lower bound $\Omega(H |\gS|/N )$ found in \citep{rajaraman2020fundamental}. There are two reasons for this. First, the lower bound is applicable only in the large sample regime (i.e., $N \gtrsim |\gS|$). Second, the instance that provides the lower bound in \cref{theorem:ail_reset_cliff} does not satisfy \cref{asmp:reset_cliff}.
\end{rem}

{Finally, to gain an intuitive understanding of the horizon-free imitation gap and the coupling structure in the state-action distribution matching, we provide an example in Appendix~\ref{appendix:proof_claim_example_ail_success}.}

\subsection{Toward A Stage-coupled Analysis}
\label{subsec:key_analysis_ideas}

In this part, we outline the primary analysis technique employed to establish the horizon-free imitation gap of TV-AIL. Initially, we give a concise overview of the classical reduction-and-estimation analysis that has been utilized in previous studies \citep{pieter04apprentice, syed07game, xu2020error, rajaraman2020fundamental}. We then discuss why this technique fails to offer a tight bound on RBAS MDPs. Finally, we introduce a novel stage-coupled analysis that uncovers the algorithmic characteristics of TV-AIL for RBAS MDPs.

\textbf{Reduction-and-Estimation Analysis.} To analyze AIL methods, the reduction-and-estimation analysis reduces the imitation gap to the statistical estimation error of the expert's state-action distribution. Concretely, we have that
\begin{align}
& \quad \labs  V({\piE}) - V({\piail}) \rabs \nonumber
\\
&\overset{(a)}{=} \bigg\vert \sum_{h=1}^{H}  \sum_{(s, a) \in \gS \times \gA}  d^{\piE}_h(s, a) r_h(s, a) - d^{\piail}_h(s, a) r_h(s, a) \bigg\vert  
\nonumber \\
&\overset{(b)}{\leq} \sum_{h=1}^{H}  \sum_{(s, a) \in \gS \times \gA} \labs d^{\piE}_h(s, a) - d^{\piail}_h(s, a) \rabs \nonumber \\
&\overset{(c)}{\leq} \sum_{h=1}^{H}  \sum_{(s, a) \in \gS \times \gA} \labs d^{\piE}_h(s, a) - \widehat{d^{\piE}_h}(s, a) \rabs  + \sum_{h=1}^{H}  \sum_{(s, a) \in \gS \times \gA} \labs  \widehat{d^{\piE}_h}(s, a) - d^{\piail}_h(s, a) \rabs \nonumber 
\\
&\overset{(d)}{\leq} 2 \sum_{h=1}^{H}  \sum_{(s, a) \in \gS \times \gA} \labs  \widehat{d^{\piE}_h}(s, a) - d^{\piE}_h(s, a) \rabs, \label{eq:estimation_reduction_final_step}
\end{align}
where equation $(a)$ follows the dual representation of policy value in \eqref{eq:value_dual_representation}, inequality $(b)$ is based on the assumption that $r_h(s, a) \in [0, 1]$, and inequality $(d)$ holds because $\piail$ is the optimal solution to \eqref{eq:ail}, i.e., $\sum_{h=1}^{H}\Vert d^{\piail}_h - \widehat{d^{\piE}_h} \Vert_1 \leq \sum_{h=1}^{H}\Vert d^{\piE}_h - \widehat{d^{\piE}_h} \Vert_1$. Then, the estimation error $\sum_{h=1}^{H}\Vert d^{\piE}_h - \widehat{d^{\piE}_h} \Vert_1$ can be further upper bounded via proper concentration inequalities. For instance, the $\ell_1$-risk estimation error typically concentrates in a rate $\gO(\sqrt{|\gX|/N})$ \citep{weissman2003inequalities, han2015minimax, kamath2015learning}, where $|\gX|$ is the cardinality of the symbol set $\gX$ (i.e., $|\gX|$ is the estimation dimension) and $N$ is the sample size. In the context of imitation learning, we have that $\Vert d^{\piE}_h - \widehat{d^{\piE}_h} \Vert_1 \lesssim \sqrt{|\gS|/N}$ for $h \in [H]$, where we consider the assumption that the expert policy is deterministic so that the error bound does not depend on $|\gA|$. Combing the above two steps, one can obtain the imitation gap bound.

\begin{thm} \label{thm:worst_case_vail}
For any tabular and episodic MDP, including RBAS MDPs, the imitation gap of TV-AIL is $\gO(\min\{H, H \sqrt{|\gS|/N} \})$.
\end{thm}
This bound, though comparable to the results of classical algorithms such as FEM and GTAL, does not provide a satisfactory explanation for empirical observations made on RBAS MDPs and MuJoCo locomotion tasks. This is because this bound is only meaningful in the large sample regime $N \gtrsim |\gS|$; otherwise, the first term $\gO(H)$ in the imitation gap dominates and this bound becomes trivial.

Furthermore, empirical evidence, as demonstrated in \cref{tab:reset_cliff_horizon_estimation_error}, indicates that the estimation error can be substantial even when the imitation gap of TV-AIL is small for RBAS MDPs. As such, the reduction-and-estimation analysis falls short of closing the gap between theory and practice.

\begin{table}[htbp]
\centering
\caption{Imitation gap and estimation error of TV-AIL on RBAS MDPs with $N=1$.}
\label{tab:reset_cliff_horizon_estimation_error}

\begin{tabular}{@{}ccccc@{}}
\toprule
         &$H=100$  & $H=500$  & $H=1000$ &$H=2000$ \\ \midrule

Imitation Gap     &  \meanstd{0.69}{0.00} &  \meanstd{0.70}{0.00}  & \meanstd{0.71}{0.00}   & \meanstd{0.71}{0.00} \\ 
Estimation Error      & \meanstd{189.47}{0.00} & \meanstd{947.37}{0.00} & \meanstd{1894.74}{0.00} & \meanstd{3789.47}{0.00}       \\
\bottomrule
\end{tabular}
\end{table}

\textbf{Stage-coupled Analysis.} To overcome the limitations of the reduction-and-estimation analysis, we develop a stage-coupled analysis approach to study the optimal solution of TV-AIL. Specifically, we seek to determine the optimal solution $\pi^{\ail} = \{\pi_1^{\ail}, \pi_2^{\ail}, \ldots, \pi_H^{\ail} \}$ in a backward-inductive manner.

Our analysis technique exploits the transition properties of RBAS MDPs to establish that, for any optimal policy $\pi^{\ail}$ to the state-action distribution matching problem defined in (\ref{eq:ail}), we must have
\begin{align}   \label{eq:visit_good_states}
    \forall h \in [H], \forall s \in \goodS: \quad d^{\piail}_{h}(s) > 0,
\end{align}
where $d^{\piail}_{h}(s)$ is calculated by $(\pi^{\ail}_1, \ldots, \pi^{\ail}_{h-1})$. We provide a detailed proof of this result in \cref{lem:condition_for_ail_optimal_solution} in the Appendix. It is important to note that while the condition in \eqref{eq:visit_good_states} does not establish that $\piail$ must take the expert action on good states, it does demonstrate that optimal policies can visit good states in each step. Further, based on the transition properties of RBAS MDPs, we can infer that for all $h \in [H]$, there is at least one good state $s \in \goodS$ such that $\piail_h(a^{1} | s) > 0$; we provide details on this in \cref{lem:condition_for_ail_optimal_solution} in the Appendix.

Next, we aim to eliminate the possibility of the optimal policy $\piail$ taking bad actions on good states. To do so, we apply the backward induction technique based on the idea of dynamic programming \citep{bertsekas2012dynamic}. We argue that for each time-dependent policy $\piail_h$, it is necessary to take the expert action on good states. Otherwise, it would lead to a higher cumulative state-action distribution matching loss and hence cannot be optimal. Our previous result in \eqref{eq:visit_good_states} serves as a crucial building block for this argument. Below, we provide a brief outline of the proof and leave the detailed proofs to Appendix \ref{appendix:proof_of_prop:ail_general_reset_cliff}.

In the proof, we use the optimality conditions of multi-stage optimization, expressed as
\begin{align}   \label{eq:single_optimality}
    \pi^{\ail}_{h} \in \argmin_{\pi_h} f_h(\pi_h; \pi^{\ail}_1, \ldots, \pi^{\ail}_{h-1}, \pi^{\ail}_{h+1}, \ldots, \pi^{\ail}_{H})
\end{align}
for all $h \in [H]$, where 
\begin{align*}
    f_h(\pi_h; \piail_1, \ldots, \piail_{h-1}, \piail_{h+1}, \ldots, \piail_H) = \sum_{h=1}^{H} \sum_{(s, a)} \labs d^{\pi}_h(s, a) - \widehat{d^{\piE}_h}(s, a) \rabs 
\end{align*}
is a single-variable loss function that takes $\pi_h$ as the variable and other time-dependent policies \\$(\piail_{1}, \ldots, \piail_{h-1}, \piail_{h+1}, \ldots, \piail_{H})$ as fixed parameters.  This loss function measures the discrepancy between the state-action distribution under $\pi_h$ and the corresponding expert distribution. Mathematically speaking, the condition in \eqref{eq:single_optimality} means that the global optimality implies the directional optimality in each coordinate.  The details of \eqref{eq:single_optimality} are presented in \cref{lem:n_vars_opt_greedy_structure} in the Appendix.

Using the optimality conditions in \eqref{eq:single_optimality}, we proceed with the backward induction proof. The base step is to prove $\pi^{\ail}_{H-1} (\cdot|s) = \piE_{H-1} (\cdot|s)$ for all $s \in \goodS$ (note that we do not guarantee the quality of $\piail_H$). Our strategy is to prove that $\pi^{\ail}_{H-1} (a^{1}|s) = \piE_{H-1} (a^{1}|s) = 1$ for all $s \in \goodS$ is the \emph{unique} optimal solution of
\begin{align}   \label{eq:base_step_proof}
    \min_{\pi_{H-1}} f_{H-1}(\pi_{H-1}; \pi_1^{\ail}, \ldots, \pi_{H-2}^{\ail}, \pi_{H}^{\ail}).
\end{align}
We can express $f_{H-1}$ as 
\begin{align*}
f_{H-1} = \sum_{(s, a)} \labs d^{\pi}_{H-1}(s, a) - \widehat{d^{\piE}_{H-1}}(s, a) \rabs + \sum_{(s, a)} \labs d^{\pi}_{H}(s, a) - \widehat{d^{\piE}_{H}}(s, a)  \rabs + \constant,
\end{align*}
where $d^{\pi}_{H-1}(s, a)$ and $d^{\pi}_{H}(s, a)$ are induced by $(\piail_1, \ldots, \piail_{H-2}, \pi_{H-1}, \piail_H)$. Besides, $\constant$ is the sum of state-action distribution matching losses from $h=1$ to $h=H-2$ incurred by $(\piail_1, \piail_2, \ldots, \piail_{H-2})$, and is independent of $\pi_{H-1}$. For notation simplicity, let $\Loss_{h}$ be the state-action distribution matching loss in time step $h$. That is, let $\Loss_{H-1} = \Vert d^{\pi}_{H-1} - \widehat{d^{\piE}_{H-1}} \Vert_1$, and $\Loss_{H} = \Vert d^{\pi}_{H} -  \widehat{d^{\piE}_{H}} \Vert_1$. To prove the optimality of $\piE_{H-1}$ in \eqref{eq:base_step_proof}, we first show that (I) $\pi_{H-1}(a^{1}|s) = 1$ for all $s \in \goodS$ is the optimal solution with respect to $\min_{\pi_{H-1}} \Loss_{H-1}$, and then show that (II) it is also the \emph{unique} optimal solution with respect to $\min_{\pi_{H-1}} \Loss_{H}$. We elaborate on these two steps as follows.

\textbf{Step (I):} we will demonstrate that $\pi_{h} (a^1|s) = 1$ for all $s \in \goodS$ is an optimal solution to minimizing $\Loss_h$. Since $\piE$ always takes the expert action and does not visit bad states, we have that $\widehat{d^{\piE}_{H-1}}(s) = 0$ for $s \in \badS$. Moreover, we have $\widehat{d^{\piE}_{H-1}}(s, a) = 0$ for $s \in \goodS$ and $a \ne a^{1}$, as the expert policy always executes $a^{1}$. With these facts, we obtain that 
\begin{align*}
     \Loss_{H-1} &= \sum_{(s, a)} \labs d^{\pi}_{H-1}(s, a) - \widehat{d^{\piE}_{H-1}}(s, a) \rabs \nonumber \\
    &= \sum_{s \in \badS} d^{\pi^{\text{AIL}}}_{H-1} (s) + \sum_{s \in \goodS} \bigg( d^{\pi^{\text{AIL}}}_{H-1} (s) \lp 1- \pi_{H-1} (a^1|s) \rp + \labs \widehat{d^{\piE}_{H-1}} (s, a^1) - d^{\pi^{\text{AIL}}}_{H-1} (s) \pi_{H-1} (a^1|s) \rabs  \bigg) . 
\end{align*}
Since the first term $d^{\pi^{\text{AIL}}}_{H-1}(s)$ is irrelevant to $\pi_{H-1}$, we only need to consider the second term. For a specific $s \in \goodS$, where $d^{\pi^{\text{AIL}}}_{H-1}(s) > 0$ (as stated in \eqref{eq:visit_good_states}), we can show that $\pi_h(a^{1}|s) = 1$ is an optimal solution for such kind of piece-wise linear function, regardless of the estimation $\widehat{d^{\piE}_{H-1}}$. This proof is provided in \cref{lem:single_variable_opt} in the Appendix. Therefore, we conclude that $\pi_{H-1} (a^1|s) = 1$ for all $s \in \goodS$ is an optimal solution regarding $\min_{\pi_{H-1}} \Loss_{H-1}$.

\textbf{Step (II):} We prove that $\pi_{H-1} (a^1|s) = 1$ for all $s \in \goodS$ is the \emph{unique} optimal solution with respect to $\min_{\pi_{H-1}} \Loss_{H}$. This property reflects the interdependence of policy optimization across different stages, where the objective at a later stage $H$ influences the optimization process at an earlier stage $H-1$. To illustrate this, consider that
\begin{align}
     d^{\pi}_{H}(s, a)
    =& d^{\pi}_{H}(s) \pi^{\text{AIL}}_{H}(a|s) \nonumber \\
    =& \bigg[ \sum_{(s^\prime, a^\prime)} d^{\pi^{\text{AIL}}}_{H-1}(s^\prime) \pi_{H-1}(a^\prime | s^\prime) P_{H-1}(s|s^\prime, a^\prime) \bigg]  \pi^{\text{AIL}}_{H}(a|s) \label{eq:main_text_eq_1} 
\end{align}

For our purpose, we break down $\Loss_{H}$ into two parts according to visited states and non-visited states. Specifically, we define the set of visited states in time step $H$ as $\gV_{H} := \{s \in \gS: \widehat{d^{\piE}_{H}}(s) > 0 \}$.  Then, using \eqref{eq:main_text_eq_1}, we have 
\begin{align*}
    \Loss_{H}
    =& \sum_{(s, a)} \labs d^{\pi}_H(s, a) - \widehat{d^{\piE}_{H}}(s, a) \rabs \\
    =& \sum_{s \in \gV_{H}} \bigg\vert \sum_{s^\prime \in \goodS}  A(s, s^\prime) -  \widehat{d^{\piE}_{H}} (s) \bigg\vert  - \sum_{s^\prime \in \goodS} d(s^\prime) \pi_{H-1} (a^{1}|s^\prime) + \constant,
\end{align*}
where $\constant$ is independent of $\pi_{H-1}$. We introduce two new notations: $A(s, s^\prime)$ and $d(s^\prime)$. Here are their definitions: 
\begin{align*}
   A(s, s^\prime) &=  d^{\pi^{\text{AIL}}}_{H-1} (s^\prime)  P_{H-1} (s | s^\prime, a^{1}) \pi^{\ail}_{H} (a^{1}|s) \pi_{H-1} (a^{1}|s^\prime) \\
   d(s^\prime) &= \sum_{s \in \gV_{H}} d^{\pi^{\text{AIL}}}_{H-1} (s^\prime) P_{H-1} (s | s^\prime, a^{1}) \pi_{H}^{\ail} (a^{1}|s).
\end{align*}
For the detailed derivation, please see \eqref{eq:proof_vail_reset_cliff_1} in the Appendix. To analyze the optimization problem $\min_{\pi_{H-1}} \Loss_{H}$, we rely on a specialized result from the Appendix, \cref{lem:mn_variables_opt_unique}, which proves that $\pi_{H-1}(a^{1}|s) = 1$ for all $s \in \goodS$ is the unique optimal solution. The proof of \cref{lem:mn_variables_opt_unique} is technical and beyond the scope of this main text. However, interested readers can refer to the Appendix for further details.

After performing the previous two steps, we have shown that $\pi_{H-1} (a^{1}|s) = 1$ for all $s \in \goodS$ is optimal for minimizing both $\Loss_{H-1}$ and $\Loss_{H}$. Moreover, since $\pi_{H-1} (a^{1}|s) = 1$ for all $s \in \goodS$ is the unique minimizer for the latter objective, we can conclude that it is also the unique minimizer for the sum objective $f_{H-1} = \Loss_{H-1} + \Loss_{H} + \constant$, by applying \cref{lem:unique_opt_solution_condition} from the Appendix. This completes the proof of the base step.

The proof for the induction step follows a similar approach. Assuming that for stage $h$ we have $\piail_{h^\prime}(\cdot|s) = \piE_{h^\prime}(\cdot|s)$ for all $s \in \goodS$ for all $h+1 \leq h^\prime \leq H-1$, we aim to prove that $\piail_h(\cdot|s) = \piE_h(\cdot|s)$ for all $s \in \goodS$. We follow the same proof strategy as the base step. Specifically, we have
\begin{align}
    \piail_{h} &\in \argmin_{\pi_h} f_{h}(\pi_h; \pi^{\ail}_1, \ldots, \piail_{h-1}, \piail_{h+1}, \ldots, \piail_{H}) \label{eq:main_text_3} \\
    &\in \argmin_{\pi_h}  f_{h}(\pi_h; \piail_{1}, \ldots, \piail_{h-1}, \piE_{h+1}, \ldots, \piail_{H}). \nonumber
\end{align}
As before, we can decompose $f_h$ into three parts: the state-action distribution matching loss in the current stage, the cumulative state-action distribution losses in the future stages, and the constant term in the early stages.
\begin{align*}
    f_h &= \sum_{(s, a)} \labs d^{\pi}_h(s, a) - \widehat{d^{\piE}_h}(s, a) \rabs + \sum_{h^\prime = h + 1}^{H} \sum_{(s, a)} \labs d^{\pi}_{h^\prime} (s, a) - \widehat{d^{\piE}_{h^\prime}}(s, a) \rabs + \constant,
\end{align*}
where $\constant$ is the sum of the losses incurred in matching state-action distributions by $(\piail_1, \ldots, \piail_{h-1})$ from stage $1$ to $h-1$. Using the same notation as before, we have $f_h = \Loss_h + \sum_{h^\prime = h+1}^{H} \Loss_{h^\prime} + \constant$. Following the proof strategy for the base step, we can show that the expert policy is the unique optimal solution by separately minimizing $\Loss_h$ and $\sum_{h^\prime = h+1}^{H} \Loss_{h^\prime}$. We omit the details here and refer readers to Appendix \ref{appendix:proof_of_prop:ail_general_reset_cliff} for more information.

\begin{rem}[Stage-coupled Analysis V.S. Reduction-and-estimation Analysis]
    In RBAS MDPs, the stage-coupled analysis derives an imitation gap $\gO(\min\{1, \sqrt{|\gS|/N} \})$ for TV-AIL, while the classical reduction-and-estimation analysis leads to a looser bound $\gO(\min\{H, H \sqrt{|\gS|/N} \})$. Both analysis connects the imitation gap with the cumulative state-action distribution discrepancy $\sum_{h=1}^{H} \| d^{\piE}_h - d^{\piail}_h  \|_1$, but they differ fundamentally in how to analyze this distribution discrepancy. On one hand, as shown in inequality (c) in Eq.(\ref{eq:estimation_reduction_final_step}), the reduction-and-estimation analysis relates the distribution discrepancy with the cumulative statistical estimation error $\sum_{h=1}^{H} \| d^{\piE}_h - \widehat{d^{\piE}_h}  \|_1$. Such a cumulative estimation error does not diminish under RBAS MDP assumptions and results in the horizon dependence in the imitation gap. On the other hand, the stage-coupled analysis employs the mentioned backward induction-based method to provide a sharp characterization of TV-AIL's policy. Concretely, Proposition \ref{prop:ail_general_reset_cliff} establishes that TV-AIL can exactly recover the expert policy in the first $H-1$ stages, suggesting that the state-action distribution discrepancy equals zero for these stages. As such, only the distribution discrepancy in the final stage contributes to the imitation gap, leading to a horizon-free bound.   
\end{rem}

\begin{rem}[Difference with the Dynamic-programming-based Proof]
Our analysis differs from the dynamic programming (DP) proof. In particular, our proof of \cref{prop:ail_general_reset_cliff} utilizes RBAS MDPs properties directly to characterize the optimal policies and does not require forward substitution as in the direct dynamic programming technique. In contrast, the DP proof computes a functional by backward induction and then uses forward substitution to find the optimal policy. For interested readers, please refer to Appendix \ref{appendix:difference_with_the_dp_analysis} for a detailed discussion.

\end{rem}

We note that the assumption of reachable bad absorbing states is crucial to our stage-coupled analysis, especially in establishing the optimality of $\piE_h$ at each stage. Our analysis demonstrates that selecting an action that leads to bad absorbing states can result in significant matching loss in future stages. As such, guidance from distribution matching in the future can provably assist TV-AIL in recovering the expert action on non-visited states in the first $H-1$ stages. This uncovers that the stage-coupling structure in the distribution matching loss helps TV-AIL identify the expert action on states out of the demonstration distribution, thereby addressing the distribution shift issue and mitigating the compounding errors over horizons.

\subsection{Extensions}
\label{subsec:extension}
{In the previous part, we focus on the setting with exact solutions to the distribution matching problem and known transition functions. The subsequent sections will provide extensions by considering approximate solutions and unknown transitions.}

\subsubsection{When Exact Solutions Are Not Available}
\label{sec:approximate_solutions}

In the previous section, we have explored the imitation gap of solutions that are \emph{exactly} optimal. However, in practice, gradient-based methods are typically employed to solve \eqref{eq:ail_min_max}, leading to solutions that are only \emph{approximately} optimal. In this part, we demonstrate that our previous findings remain valid when accounting for optimization errors.

\begin{defn}[$\varepsilon$-optimal solution]   \label{defn:approximate_solution}
A policy $\widebar{\pi}$ is an $\varepsilon$-optimal solution, if 
\begin{align*}
     \sum_{h=1}^{H} \lnorm d^{\widebar{\pi}}_h - \widehat{d^{\piE}_h} \rnorm_1 \leq \min_{\pi \in \Pi} \sum_{h=1}^{H} \lnorm d^{\pi}_h - \widehat{d^{\piE}_h} \rnorm_1 + \varepsilon. 
\end{align*}
\end{defn}

We begin by highlighting that incorporating optimization error in the reduction-and-estimation framework is a straightforward process. By following the steps in \eqref{eq:estimation_reduction_final_step} and applying an additional triangle inequality, we obtain the bound:
\begin{align*}
   \labs V(\widebar{\pi}) - V(\piE) \rabs \leq 2 \sum_{h=1}^{H} \lnorm  \widehat{d^{\piE}_h} - d^{\piE}_h \rnorm_1 + \varepsilon.
\end{align*}
In comparison with the error bound presented in \eqref{eq:estimation_reduction_final_step}, this inequality has an extra term, namely, the optimization error $\varepsilon$. However, incorporating the optimization error within our stage-coupled analysis is not trivial. One may guess that due to optimization error, the approximately optimal policy $\widebar{\pi}$ may select a non-expert action with a small probability. Consequently, the agent may experience compounding errors, and the horizon-free imitation gap may not hold. This conjecture is reasonable when the optimization error is large. However, we will show that when the optimization error is well-controlled, the horizon-free guarantee remains unchanged. We present our formal claim below.

\begin{thm}[Horizon-free Imitation Gap of Approximate \textsf{TV-AIL} on RBAS MDPs]
\label{theorem:ail_approximate_reset_cliff}
For each tabular and episodic MDP satisfying \cref{asmp:reset_cliff}, the candidate policy set is defined as $\Pi^{\opt} = \{ \pi \in \Pi: \forall h \in [H], \exists s \in \goodS, \pi_h (a^{1}|s) > 0 \}$. Suppose that $\widebar{\pi} \in \Pi^{\opt} $ is an $\varepsilon$-optimal solution of \eqref{eq:ail}. Let us define $\varepsilon^\prime = 8 \varepsilon / c (\widebar{\pi})$, then we have 
\begin{align*}
    & V({\piE}) - \expect \ls V(\widebar{\pi}) \rs \lesssim \min \lb 1 + \varepsilon^\prime, \sqrt{\frac{\vert \gS \vert}{N}} + \varepsilon^\prime  \rb,
\end{align*}
where $c(\widebar{\pi}) > 0$ is defined as 
\begin{align*}
     c (\widebar{\pi}) := \min_{1 \leq \ell < h \leq H, s, s^\prime \in \goodS} \{ \sP^{\widebar{\pi}} \lp s_{h} = s |s_{\ell} = s^\prime, a_{\ell} = a^{1} \rp \}.
\end{align*}
Here $\sP^{\widebar{\pi}} \lp s_{h} = s |s_{\ell} = s^\prime, a_{\ell} = a^{1} \rp $ is the visitation probability of $s$ in time step $h$ by starting from $s^\prime, a^{1}$ in time step $\ell$, which is jointly determined by the transition function and policy $\widebar{\pi}$. 
\end{thm}

The proof of \cref{theorem:ail_approximate_reset_cliff} relies on a highly technical sensitivity analysis and is presented in Appendix \ref{appendix:proof_theorem:ail_approximate_reset_cliff}. The technical challenge arises because an $\varepsilon$-optimal solution is not necessarily much closer to the optimal solution (in fact, the opposite is often true). The intuition behind \cref{theorem:ail_approximate_reset_cliff} is that for an approximate solution, it is crucial to control decision errors on non-visited states, which is captured by $c(\widebar{\pi})$. If $\widebar{\pi}$ is the expert policy (i.e., the exactly optimal policy), $c(\widebar{\pi})$ is large. Hence, we expect that the horizon-free imitation gap still holds if the optimization error $\varepsilon$ is small. However, if $c(\widebar{\pi})$ is small, the requirement in \cref{theorem:ail_approximate_reset_cliff} may not hold, so we cannot make any strong claims about the $\varepsilon$-optimal solution. It is worth noting that the experiments presented in \cref{subsec:reset_cliff} used gradient-based methods to obtain approximately optimal solutions. Interestingly, we observed that these methods could find solutions with a relatively large $c(\widebar{\pi})$ and a small $\varepsilon$ when the iteration number was large (see the empirical evaluation in \cref{tab:reset_cliff_rel_opt_error}). Understanding why gradient-based methods can find such good solutions remains a topic for future research.

\begin{table}[htbp]
\centering
\caption{Evaluation of Approximate TV-AIL on the RBAS MDP with $N=1$.}
\label{tab:reset_cliff_rel_opt_error}

\begin{tabular}{@{}ccccc@{}}
\toprule
         &$H=100$  & $H=500$  & $H=1000$ &$H=2000$ \\ \midrule
         $\varepsilon / c (\widebar{\pi})$ & $0.42 \pm 0.00$ & $0.43 \pm 0.00$  & $0.44 \pm 0.00$   & $0.44 \pm 0.00$ \\ \bottomrule
\end{tabular}
\end{table}

\subsubsection{When The Transition Function Is Unknown}
\label{subsec:unknown_transition}
This paper primarily addresses scenarios where the transition function is known, enabling precise calculation of the state-action distribution $d^{\pi}$  during the optimization process. In practice, however, the transition function often remains unknown. Nevertheless, learners can estimate it by engaging with the environment through policy roll-outs. This section explores how our findings are applicable in such cases. Notably, if the interactions are substantial, the approximated transition model may be sufficiently accurate for effective imitation learning. Consequently, theoretical guarantees for known transitions may be adapted to scenarios with unknown transitions. This concept has been explored in previous research, notably in \citep{xu2021nearly}. We offer a concise discussion here and direct readers seeking more comprehensive details to the Appendix.

\begin{defn}[Uniform Policy Evaluation]
\label{def:uniform_policy_evaluation}
Given an MDP $\gM$, an algorithm is said to be $(\varepsilon, \delta)$-PAC for uniform policy evaluation in terms of state-action distribution if
\begin{align*}
    \sP \lp \forall \pi \in \Pi, \sum_{h=1}^H \lnorm d^{\pi, \gP}_h - d^{\pi, \widehat{\gP}}_h \rnorm_1 \leq \varepsilon \rp \geq 1-\delta,
\end{align*}
where $d^{\pi, \gP}_h$ and $d^{\pi, \widehat{\gP}}_h$ are the state-action distributions of policy $\pi$ under the real transition function $\gP$ and the transition function $\widehat{\gP}$ learned by the algorithm, respectively.
\end{defn}

With a transition function learned by the algorithm for uniform policy evaluation, we can perform the state-action distribution matching with this empirical transition function.
\begin{align}
\label{eq:AIL_with_empirical_transition_model}
    \min_{\pi \in \Pi} \sum_{h=1}^H \lnorm d^{\pi, \widehat{\gP}}_h - \widehat{d^{\piE}_h} \rnorm_1.
\end{align}

\begin{algorithm}[htbp]
\caption{Model-based TV-AIL}
\label{algo:mbtvail}
\begin{algorithmic}[1]
\REQUIRE{Expert demonstrations $\gD$.}
\STATE{$\widehat{\gP} \leftarrow$ Invoke an algorithm that is $(\varepsilon_{\eval}, \delta)$-PAC for uniform policy evaluation to interact with the environment and learn a transition model.}
\STATE{$\widebar{\pi} \lar$ Apply an algorithm to solve the optimization problem in \eqref{eq:AIL_with_empirical_transition_model} up to an error $\varepsilon_{\opt}$.}
\ENSURE{Policy $\widebar{\pi}$.}
\end{algorithmic}
\end{algorithm}

\begin{prop}    \label{prop:transfer_error}
Under the unknown transition setting, consider Model-based TV-AIL displayed in Algorithm \ref{algo:mbtvail} and $\widebar{\pi}$ is output policy, with probability at least $1-\delta$, $\widebar{\pi}$ is an $(2\varepsilon_{\eval} + \varepsilon_{\opt})$-optimal solution:
\begin{align*}
    \sum_{h=1}^{H} \lnorm d^{\widebar{\pi}}_h - \widehat{d^{\piE}_h} \rnorm_1 \leq \min_{\pi \in \Pi} \sum_{h=1}^{H} \lnorm d^{\pi}_h - \widehat{d^{\piE}_h} \rnorm_1 + 2\varepsilon_{\eval} +  \varepsilon_{\opt}.
\end{align*}
\end{prop}

The additional term $2 \varepsilon_{\eval} + \varepsilon_{\opt}$ in \cref{prop:transfer_error} can be interpreted as the approximation error $\varepsilon$ defined in \cref{defn:approximate_solution}. Therefore, the analysis presented in the previous section is applicable. Moreover, the requirement of uniform policy evaluation can be met using reward-free exploration methods \citep{chi20reward-free, menard20fast-active-learning}, as pointed out by \citep{xu2021nearly}. Hence, our theoretical results hold even in the case of unknown transitions. For further discussion, please refer to Appendix~\ref{appendix:subsec:mb-tvail}.

\subsubsection{When the RBAS MDP Assumption is Violated}
\label{subsubsec:rbas_mdps_extensions}

\begin{figure}[htbp]
\centering
\includegraphics[width=0.6\linewidth]{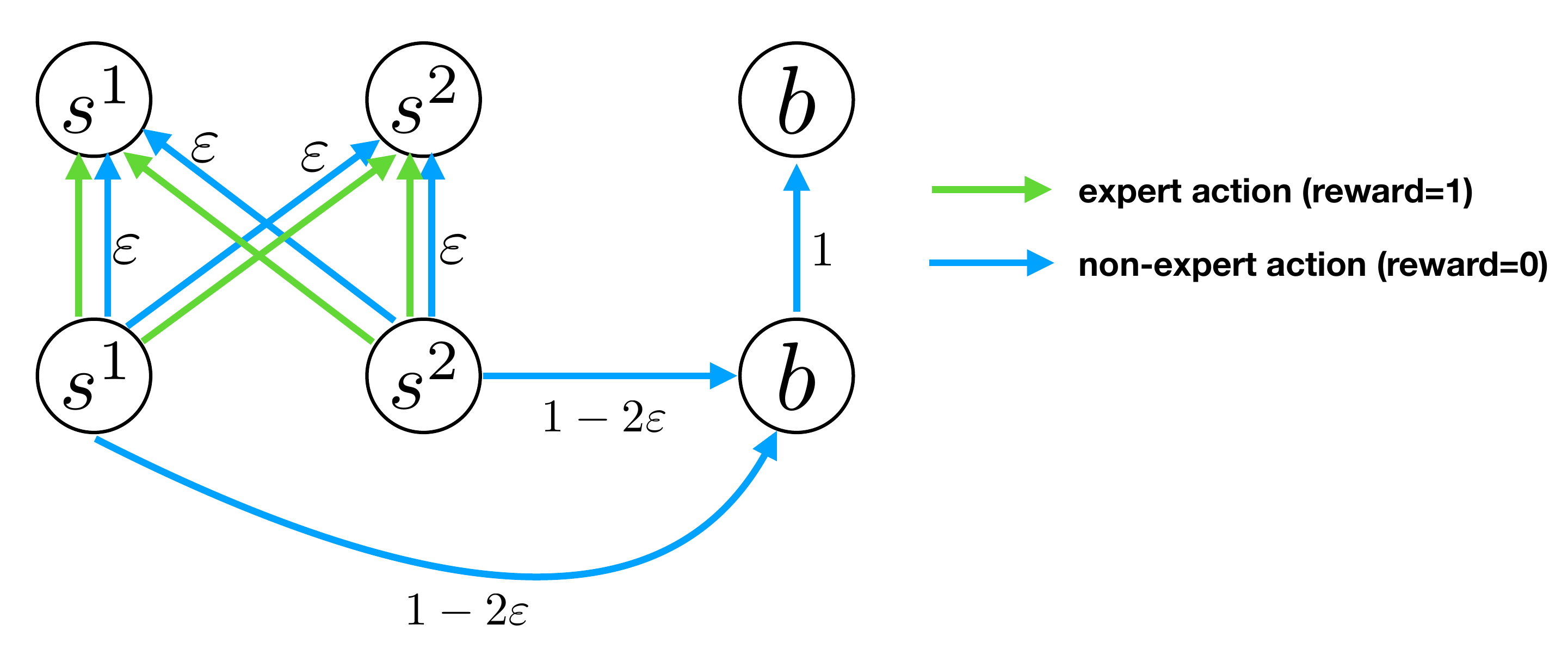}
\caption{Example I of extended RBAS MDPs. Digits indicate the transition probabilities.}
\label{fig:rabs_mdp_extension_one}
\end{figure}

Our main analysis focuses on RBAS MDPs satisfying \cref{asmp:reset_cliff}. In this part, we consider relaxing \cref{asmp:reset_cliff} along two dimensions: (i) taking a non-expert action may lead to a transition into good states with a small probability, and (ii) agents can also return to good states from bad states with a small probability.

For the first aspect, we consider an example of extended RBAS MDPs illustrated in Figure~\ref{fig:rabs_mdp_extension_one} for clarity. The extension to general $|\goodS|$ and $H$ follows straightforwardly via the same induction-based analysis. In this extended RBAS MDP, any non-expert action $a \neq a^1$ satisfies $P_1(s' | s,a)=\varepsilon, \forall s, s^\prime \in \goodS$. Despite this relaxation, we show that TV-AIL can still exactly recover the expert action on preceding unvisited states, analogous to the guarantee in \cref{prop:ail_general_reset_cliff}.

\begin{prop}
\label{prop:rabs_mdps_extension_one}
Consider example I of extended RBAS MDPs shown in Figure~\ref{fig:rabs_mdp_extension_one}, and suppose $\piail$ is a minimizer of \eqref{eq:ail}. If $\varepsilon \le P_1(s' \mid s,a^1) / 2$, $\forall s,s' \in \goodS$, then for any $N \ge 1$, the following optimality condition holds almost surely:
\begin{align*}
    \piail_1(a^1 \mid s) = 1, \qquad \forall s \in \goodS.
\end{align*}
\end{prop}

The proof is provided in Appendix~\ref{appendix:proof_of_extension_one}. Proposition~\ref{prop:rabs_mdps_extension_one} indicates that as long as $\varepsilon$ remains smaller than the expert-action transition probability (up to a constant factor), TV-AIL can still correctly identify expert actions on preceding states. 

\begin{figure}[htbp]
    \centering
    \includegraphics[width=0.6\linewidth]{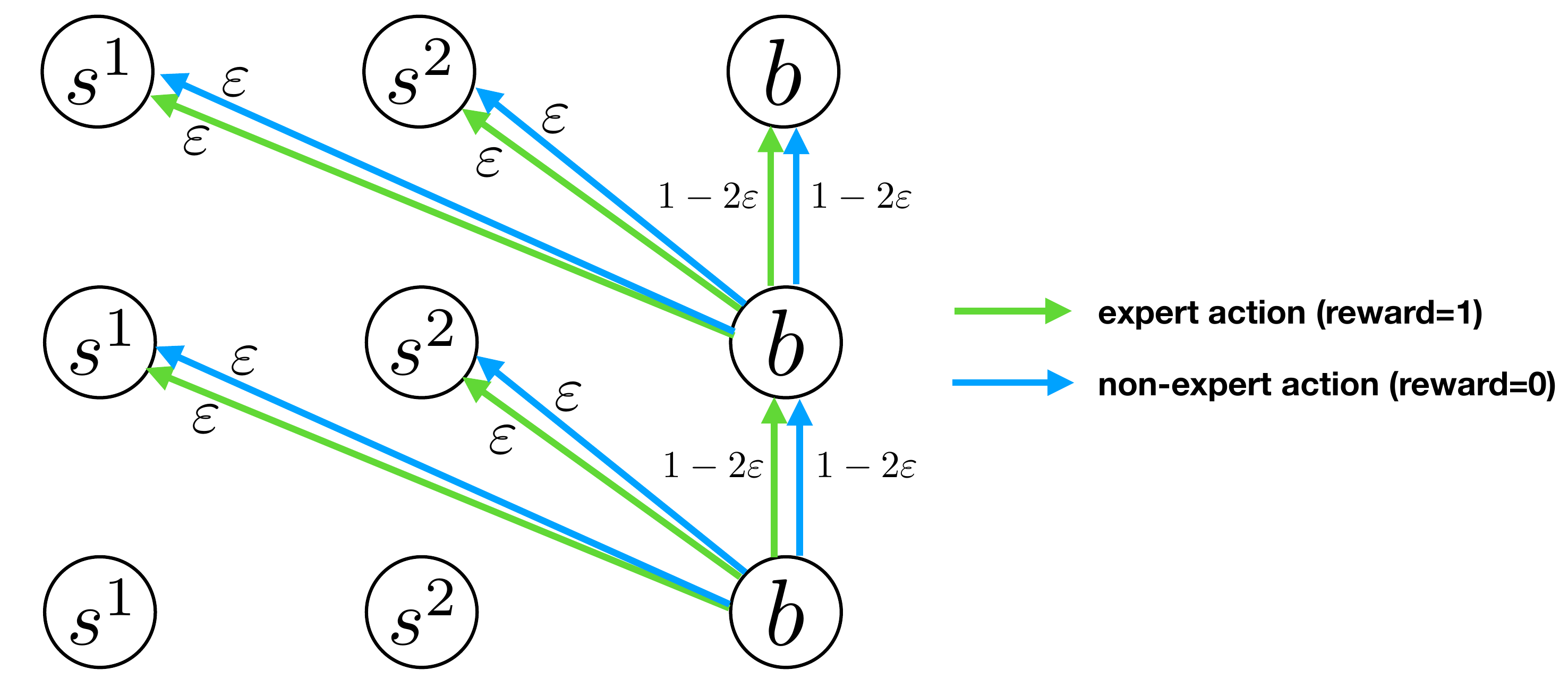}
    \caption{Example II of extended RBAS MDPs. Digits indicate the transition probabilities.}
    \label{fig:rabs_mdp_extension_two}
\end{figure}
For the second aspect, we consider an example of extended RBAS MDPs shown in Figure~\ref{fig:rabs_mdp_extension_two}.\footnote{This relaxation is meaningful when $H\geq 3$ since agents start to visit bad states from $h=2$.} Compared with the original RBAS MDP, the difference lies in the transitions from the bad state: $\forall h \in [2]$ and $\forall a \in \gA$, $P_h (s^\prime|b, a) = \varepsilon, \; \forall s^\prime \in \goodS$ and $P_h (b|b, a) = 1-2\varepsilon$.
\begin{prop}
\label{prop:rabs_mdps_extension_two}
    Consider example II of extended RBAS MDPs shown in Figure \ref{fig:rabs_mdp_extension_two}, suppose that $\piail$ is a minimizer of \eqref{eq:ail}. If $\varepsilon \leq d^{\piE}_{3} (s) / 2, \; \forall s \in \goodS$, $\varepsilon \leq \sP^{\piE} (s_3=s^\prime |s_1 = s, a_1=a^1), \; \forall s, s^\prime \in \goodS$, then for any $N \geq 1$, the following optimality condition almost surely:
\begin{align*}
    \piail_h (a^1|s) = 1, \forall  h \in [2], \forall s \in \goodS.
\end{align*}
\end{prop}
The proof is provided in Appendix \ref{appendix:proof_of_extension_two}. Proposition \ref{prop:rabs_mdps_extension_two} indicates that when the probability of returning to good states from bad states is small, TV-AIL can still identify expert actions on proceeding states. Intuitively, taking non-expert actions induces a one-point distribution on the bad state, and then the majority of probability mass on the bad state will be maintained in future time steps when $\varepsilon$ is small. This incurs a large distribution matching loss as the expert policy never visits the bad state.

\section{Beyond RBAS MDPs And Horizon-free Imitation Gap}
\label{sec:a_matching_lower_bound}

We have previously identified a horizon-free imitation gap for TV-AIL in RBAS MDPs. This discovery leads us to question whether TV-AIL consistently demonstrates a horizon-free imitation gap across all instances. In this section, we address this question by presenting a horizon-dependent lower bound and a corresponding upper bound for the imitation gap of TV-AIL in specific challenging instances. These challenging instances will be formally introduced in the subsequent sections.

\begin{asmp}[{MDPs with Isolated Absorbing States}]    \label{asmp:standard_imitation}
For a tabular and episodic MDP and an expert policy, we assume that 
\begin{itemize}
    \item Each state is absorbing and each action has the same transitions. i.e., $\forall (s, a) \in \gS \times \gA, h \in [H]$, we have $ P_h(s|s, a) = 1$.
    \item For any state, $a^{1}$ is the expert action with a reward 1 and the others are non-expert actions with a reward 0.
\end{itemize}
\end{asmp}

\begin{figure}[t]
    \centering
    \includegraphics[width=0.6\linewidth]{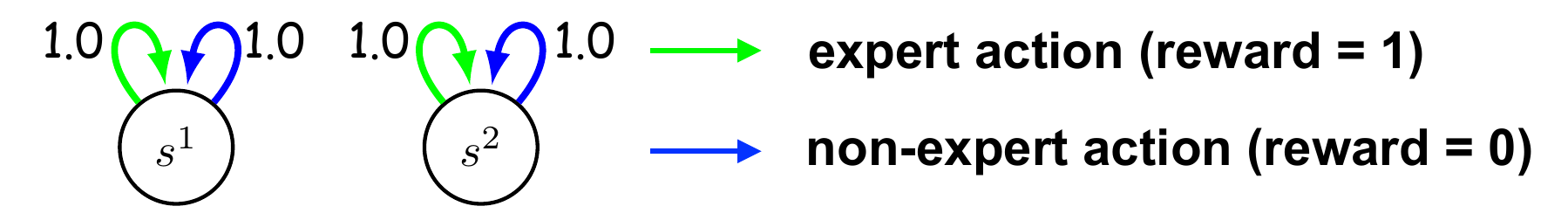}
    \caption{A simple MDP corresponding to \cref{asmp:standard_imitation}. Digits indicate the transition probabilities.}
    \label{fig:toy_standard_imitation}
\end{figure}

We would like to highlight two main characteristics of the hard instances that satisfy \cref{asmp:standard_imitation}. First, each state is isolated and absorbing, implying that the states are disconnected. Second, while all actions lead to the same transition, only the expert action provides a positive reward. These features do not apply to the instances that satisfy \cref{asmp:reset_cliff}. We provide a basic example that satisfies \cref{asmp:standard_imitation} with two states and two actions in \cref{fig:toy_standard_imitation}.

Now, we explain why the above two features make the imitation problem difficult. First, the self-absorbing characteristic implies that decision variables over stages become disconnected. For any instance satisfying \cref{asmp:standard_imitation}, we can deduce that for any $h \in [H]$,
\begin{align*}
    d^{\pi}_h(s, a) &= d^{\pi}_h(s) \pi_h(a|s) = \sum_{a^\prime }d^{\pi}_{h-1}(s, a^\prime)\pi_h(a|s) \\
    &= d^{\pi}_{h-1}(s) \pi_h(a|s) \\
    &= \ldots \\
    &= \rho(s) \pi_h(a|s).
\end{align*}
Then, we can obtain that 
\begin{align*}
    d^{\pi}_{h}(s) = \sum_{a} d^{\pi}_h(s, a) = \sum_{a} \rho(s) \pi_h(a|s) = \rho(s).
\end{align*}
This means that the state visitation distribution is equal to the initial state distribution, indicating that the policy does not affect the state visitation distribution. Thus, when we use the backward-induction-based approach to analyze the optimal policy $\piail_h$, we find that it is independent of $(\piail_1, \ldots, \piail_{h-1})$. Moreover, it can be shown that $\piail_h$ is also unrelated to $(\piail_{h+1}, \ldots, \piail_{H})$. Therefore, the obtained time-dependent policies are decoupled in this case. Thus, the multi-stage policy optimization reduces to $H$ independent one-step state-action distribution matching problems. Mathematically speaking, the policy optimization problem in \eqref{eq:single_optimality} becomes: for all $h \in [H]$,
\begin{align}
     \piail_h &\in \argmin_{\pi_h} f_h(\pi_h; \pi^{\ail}_1, \ldots, \pi^{\ail}_{h-1}, \pi^{\ail}_{h+1}, \ldots, \pi^{\ail}_{H}) \nonumber \\
     &\in \argmin_{\pi_h} \sum_{(s, a)}  \labs \rho(s) \pi_h(a|s)  - \widehat{d^{\piE}_h}(s, a) \rabs, \label{eq:ail_piece_wise_linear}
\end{align}
i.e., a piece-wise linear optimization problem. 

Second, we argue that the one-step state-action distribution matching cannot guarantee optimality even on visited states. That is, TV-AIL may select a wrong action even on visited states. This is mainly because the matching is performed in the marginal distribution space. {We illustrate this point by providing an example in Appendix \ref{appendix:example_standard_imitation}.}

Based on the above discussion, we formally state the imitation gap of TV-AIL on instances satisfying \cref{asmp:standard_imitation}.
\begin{prop}    
\label{proposition:ail_policy_value_gap_standard_imitation}
For any tabular and episodic MDP satisfying \cref{asmp:standard_imitation},  for each time step $h$, we define a set of states $\gW_h := \{s \in \gS: \widehat{d^{\piE}_h} (s) < \rho (s)   \}$. Then,
\begin{itemize}
    \item For each time step $h$, for an optimal solution $\piail$, it satisfies
    \begin{align*}
        \piail_h (a^1|s) \in [\widehat{d^{\piE}_h}(s) / \rho (s), 1], \forall s \in \gW_h, \quad \piail_h (a^1|s) = 1, \forall s \in \gW_h^c, 
    \end{align*}
    where $\gW_h^{c}$ is the complement set of $\gW_{h}$. 
    \item Among all possible optimal solutions, in the worst-case, we have 
    \begin{align*}
        \max_{\pi \in \Pi^{\ail}} V({\piE}) - \expect \ls V({\pi})\rs = \frac{1}{2} \expect \ls \sum_{h=1}^H \lnorm \widehat{d^{\piE}_h} - d^{\piE}_h   \rnorm_1 \rs.  
    \end{align*}
    \item The largest imitation gap is achieved by the policy $\pi_h (a^1|s) = \widehat{d^{\piE}_h}(s) / \rho (s), \forall s \in \gW_h$ and $\pi_h (a^1|s) = 1, \forall s \in \gW_h^c$.
\end{itemize}
\end{prop}

Proof of \cref{proposition:ail_policy_value_gap_standard_imitation} can be found in Appendix \ref{appendix:proof_proposition:ail_policy_value_gap_standard_imitation}. This proposition implies that for instances satisfying \cref{asmp:standard_imitation}, the imitation gap of TV-AIL is equal to the statistical estimation error of the expert state-action distribution (up to constants). This matches the upper bound in \eqref{eq:estimation_reduction_final_step} obtained from the estimation-and-reduction-based analysis. However, we still need to determine whether this upper bound is tight or not. To answer this question, we provide a lower bound on the imitation gap in the following section.

Using \cref{proposition:ail_policy_value_gap_standard_imitation}, we can establish a lower bound on the imitation gap by proving a lower bound for the $\ell_1$-risk estimation.

\begin{thm}
\label{thm:tv_error_lower_bound_small_data}
Consider a {categorical} distribution $Q$ over a finite set $\gX$. Given $N$ i.i.d. samples ($X_1, \cdots, X_N$) from $Q$, consider the  estimator $\widehat{Q}$:
\begin{align*}
    \widehat{Q}(i) = \frac{\sum_{j} \mathbb{I}(X_j = i)}{N}.
\end{align*}
If  $N \lesssim |\gX|$, we have that 
\begin{align*}
    \max_{Q \in \gQ} \expect \ls \lnorm Q - \widehat{Q} \rnorm_1 \rs \gtrsim  1.
\end{align*}
If $N \gtrsim |\gX|$, we have that 
\begin{align*}
    \max_{Q \in \gQ} \expect \ls \lnorm Q - \widehat{Q} \rnorm_1 \rs \gtrsim  \sqrt{\frac{|\gX|}{N}}.
\end{align*}
Here $\gQ$ is the set of all {categorical} distributions on the set $\gX$.
\end{thm}

The proof of \cref{thm:tv_error_lower_bound_small_data} can be found in Appendix \ref{appendix:proof_thm:tv_error_lower_bound_small_data}. We should note that while a lower bound has been previously established in the large sample regime \citep{kamath2015learning, han2015minimax}, our contribution is providing a lower bound in the small sample regime. By utilizing \cref{proposition:ail_policy_value_gap_standard_imitation} and \cref{thm:tv_error_lower_bound_small_data}, we can derive the following lower bound for the imitation gap.

\begin{prop}  \label{prop:lower_bound_vail}
To break the tie, suppose that {TV-AIL} outputs an optimal policy $\piail$ by uniformly sampling from all possible optimal solutions. Then, there exists a tabular and episodic MDP satisfying \cref{asmp:standard_imitation} such that

\begin{align*}
    V({\piE}) - \expect \ls V(\piail) \rs \geq \Omega \lp \min \lb H, H\sqrt{ \frac{|\gS|}{N}}  \rb \rp.
\end{align*}
\end{prop}

\begin{rem}
The established lower bound $\Omega (H, H \sqrt{|\gS| / N})$ matches the upper bound $\gO (H, H \sqrt{|\gS| / N})$ in \cref{thm:worst_case_vail}, demonstrating our result is tight. Moreover, using \cref{thm:worst_case_vail}, we can confirm that in the worst-case scenario, the imitation gap of TV-AIL must have a linear relationship with $H$. As a result, the question that was posed at the beginning of this section cannot be answered positively: there cannot be a pleasant imitation gap that is free from any horizon limitations for any instance.
\end{rem}

We further empirically validate this lower bound. Specifically, we test BC, TV-AIL and other representative AIL methods such as FEM, GTAL and GAIL on instances satisfying \cref{asmp:standard_imitation}. The imitation gaps on both small and large sample regimes are reported in \cref{tab:standard_imitation_other_ail_small_data} and \cref{tab:standard_imitation_other_ail_large_data}, respectively. Across both regimes, we observe that the imitation gaps of both BC and TV-AIL increase when the horizon grows, closely matching our theoretical prediction in \cref{prop:lower_bound_vail}. Besides, other AIL methods such as FEM, GTAL and GAIL exhibit qualitatively similar behavior, suggesting that this phenomenon is not specific to TV-AIL.

Finally, we emphasize that the worst-case scenario for TV-AIL, where the imitation gap suffers a linear dependence on the horizon $H$, may rarely occur in practice due to the unique natures of isolation and self-absorption. Therefore, the lower bound we have established does not contradict the observed excellent performance of AIL methods in practice. Instead, this result can provide insights into when TV-AIL may fail and what factors (e.g., expert action transitions are discriminative with respect to non-expert actions and there exist reachable bad self-absorbing states) are crucial to its success.

\begin{table}[htbp]
\centering
\caption{Imitation gap on the lower bound instance with $N=1$.}
\label{tab:standard_imitation_other_ail_small_data}
\begin{tabular}{@{}ccccc@{}}
\toprule
         &$H=100$  & $H=500$  & $H=1000$ &$H=2000$ \\ \midrule
BC & \meanstd{49.49}{0.24} & \meanstd{247.54}{0.48} & \meanstd{494.7}{0.6} & \meanstd{990.25}{1.19} \\
         TV-AIL & \meanstd{49.50}{0.01} & \meanstd{247.50}{0.00} & \meanstd{495.00}{0.01} & \meanstd{990.00}{0.00} \\
         FEM &  \meanstd{49.50}{0.00} & \meanstd{247.50}{0.00} & \meanstd{495.00}{0.00} & \meanstd{990.00}{0.00}       \\
GTAL     &   \meanstd{50.05}{4.93} & \meanstd{250.52}{3.96}  & \meanstd{495.05}{4.95}   & \meanstd{989.06}{4.85} \\
GAIL &   \meanstd{49.50}{0.00} & \meanstd{247.50}{0.00} & \meanstd{495.00}{0.00} &  \meanstd{990.00}{0.00} \\
\bottomrule
\end{tabular}

\end{table}

\begin{table}[htbp]
\centering
\caption{Imitation gap on the lower bound instance with $N=100$.}
\label{tab:standard_imitation_other_ail_large_data}
\begin{tabular}{@{}ccccc@{}}
\toprule
         &$H=100$  & $H=500$  & $H=1000$ &$H=2000$ \\ \midrule
                  BC & \meanstd{18.17}{1.56} & \meanstd{94.12}{6.93} & \meanstd{179.74}{16.36} & \meanstd{359.87}{38.14}
         \\
         TV-AIL & \meanstd{18.23}{1.63} & \meanstd{94.30}{6.85} & \meanstd{179.80}{16.31} & \meanstd{360.05}{38.21} \\
         FEM & \meanstd{18.18}{1.63} & \meanstd{92.85}{7.06} & \meanstd{192.21}{15.89} & \meanstd{378.15}{30.52}       \\
GTAL     & \meanstd{21.52}{2.22} & \meanstd{94.63}{7.83}  & \meanstd{188.02}{15.06}  & \meanstd{373.98}{30.82} \\
GAIL &    \meanstd{18.27}{1.62}  & \meanstd{91.71}{7.55} & \meanstd{184.34}{14.34} & \meanstd{371.09}{30.64}
\\
\bottomrule
\end{tabular}
\end{table}

\section{Conclusion}
\label{sec:conclusion}

This paper introduces a new theoretical framework to explain the success of adversarial imitation learning (AIL) methods in matching expert performance with limited demonstrations. We begin by identifying a class of MDPs abstracted from locomotion control tasks where AIL excels empirically. In these MDPs, we prove that TV-AIL, a representative AIL method, can achieve a horizon-free imitation gap bound that is meaningful in the small sample regime. This sharp theory is proved through a newly developed stage-coupled analysis. This technique reveals a key mechanism: the stage-coupling structure inherent in distribution matching enables TV-AIL to identify the expert action on states out of the demonstration distribution. This provides fundamental insights into how AIL mitigates the distribution shift issue. Finally, our theory provides guidance for practitioners: AIL is most effective on tasks where policy decisions substantially shape future state distributions, helping determine when AIL is the appropriate method.

There are several promising avenues for future research in this area. One direction is to explore the use of function approximation in AIL. This paper focused on tabular AIL, where there is no extrapolation, but it demonstrated that TV-AIL can generalize well on non-visited states under certain assumptions. However, extending AIL to parameterized functions may present new challenges, and more assumptions are needed to establish a horizon-free imitation gap bound. Furthermore, in the function approximation setting, it has been empirically validated that AIL methods offer the distinct advantage of learning good features \citep{yunzhu2017infogail}. Therefore, it would be interesting to investigate the theoretical understanding of feature learning in AIL.

Another direction is to investigate AIL methods for other problems related to imitating policies. For example, imitation learning approaches can be used to recover environment transitions \citep{venkatraman2015, xu2020error}, which are critical for model-based reinforcement learning methods \citep{sutton2018reinforcement}. It would be interesting to explore whether the nice horizon-free guarantee also applies in the context of environment learning.

\section*{Acknowledgments}

The work of Tian Xu is supported by the Fundamental Research Program for Young Scholars (PhD Candidates) of the National Science Foundation of China (623B2049). The work of Yang Yu is supported by the National Key Research and Development Program of China (2024CSJZN00300), NSFC (62495093), and Jiangsu Science Foundation (BK20243039). The work of Zhi-Quan Luo is supported by the National Natural Science Foundation of China (No. 61731018) and the Guangdong Provincial Key Laboratory of Big Data Computation Theories and Methods.

\bibliographystyle{abbrvnat}
\bibliography{reference.bib}

@inproceedings{Brantley20disagreement,
 author = {Kiant{\'{e}} Brantley and
Wen Sun and
Mikael Henaff},
 booktitle = {Proceedings of the 8th International Conference on Learning Representations},
 title = {Disagreement-Regularized Imitation Learning},
 year = {2020}
}

@article{cai2019lqr,
 author = {Qi Cai and
Mingyi Hong and
Yongxin Chen and
Zhaoran Wang},
 journal = {ar{X}iv},
 title = {On the Global Convergence of Imitation Learning: {A} Case for Linear
Quadratic Regulator},
 volume = {1901.03674},
 year = {2019}
}

@inproceedings{wang2020computation,
 author = {               Yizhou Wang and
Tianyi Liu and
Zhuoran Yang and
Xingguo Li and
Zhaoran Wang and
Tuo Zhao},
 booktitle = {Proceedings of the 8th International Conference on Learning Representations},
 title = {On Computation and Generalization of Generative Adversarial Imitation
Learning},
 year = {2020}
}

@inproceedings{chi20reward-free,
 author = {Chi Jin and
Akshay Krishnamurthy and
Max Simchowitz and
Tiancheng Yu},
 booktitle = {Proceedings of the 37th International Conference on Machine Learning},
 pages = {4870--4879},
 title = {Reward-Free Exploration for Reinforcement Learning},
 year = {2020}
}

@inproceedings{zhang2020generative,
  title = 	 {Generative Adversarial Imitation Learning with Neural Network Parameterization: Global Optimality and Convergence Rate},
  author =       {Zhang, Yufeng and Cai, Qi and Yang, Zhuoran and Wang, Zhaoran},
  booktitle = 	 {Proceedings of the 37th International Conference on Machine Learning},
  pages = 	 {11044--11054},
  year = 	 {2020},
}

@inproceedings{fu2018airl,
 author = {Justin Fu and Katie Luo and Sergey Levine},
 booktitle = {Proceedings of the 6th International Conference on Learning Representations},
 title = {Learning Robust Rewards with Adverserial Inverse Reinforcement Learning},
 year = {2018}
}

@inproceedings{ghasemipour2019divergence,
 author = {Seyed Kamyar Seyed Ghasemipour and
Richard S. Zemel and
Shixiang Gu},
 booktitle = {Proceedings of the 3rd Annual Conference on Robot Learning},
 pages = {1259--1277},
 title = {A Divergence Minimization Perspective on Imitation Learning Methods},
 year = {2019}
}

@inproceedings{haarnoja2018sac,
 author = {Tuomas Haarnoja and
Aurick Zhou and
Pieter Abbeel and
Sergey Levine},
 booktitle = {Proceedings of the 35th International Conference on Machine Learning},
 pages = {1856--1865},
 title = {Soft Actor-Critic: Off-Policy Maximum Entropy Deep Reinforcement Learning
with a Stochastic Actor},
 year = {2018}
}

@inproceedings{ho2016gail,
 author = {Jonathan Ho and
Stefano Ermon},
 booktitle = {Advances in Neural Information Processing Systems 29},
 pages = {4565--4573},
 title = {Generative Adversarial Imitation Learning},
 year = {2016}
}

@inproceedings{ke2019imitation,
  title={Imitation learning as f-divergence minimization},
  author={Ke, Liyiming and 
  Choudhury, Sanjiban and Barnes, Matt and Sun, Wen and Lee, Gilwoo and Srinivasa, Siddhartha},
  booktitle={International Workshop on the Algorithmic Foundations of Robotics},
  pages={313--329},
  year={2020}
}

@inproceedings{Kostrikov19dac,
 author = {Ilya Kostrikov and
Kumar Krishna Agrawal and
Debidatta Dwibedi and
Sergey Levine and
Jonathan Tompson},
 booktitle = {Proceedings of the 7th International Conference on Learning Representations},
 title = {Discriminator-Actor-Critic: Addressing Sample Inefficiency and Reward
Bias in Adversarial Imitation Learning},
 year = {2019}
}

@inproceedings{Kostrikov20value_dice,
 author = {Ilya Kostrikov and
Ofir Nachum and
Jonathan Tompson},
 booktitle = {Proceedings of the 8th International Conference on Learning Representations},
 title = {Imitation Learning via Off-Policy Distribution Matching},
 year = {2020}
}

@article{levin16_end_to_end,
 author = {Sergey Levine and
Chelsea Finn and
Trevor Darrell and
Pieter Abbeel},
 journal = {Journal of Machine Learning Research},
 number = {39},
 pages = {1--40},
 title = {End-to-End Training of Deep Visuomotor Policies},
 volume = {17},
 year = {2016}
}

@article{liu2021provably,
 author = {Liu, Zhihan and Zhang, Yufeng and Fu, Zuyue and Yang, Zhuoran and Wang, Zhaoran},
 journal = {arXiv},
 title = {Provably Efficient Generative Adversarial Imitation Learning for Online and Offline Setting with Linear Function Approximation},
 volume = {2108.08765},
 year = {2021}
}

@article{mcallester03concentration,
 author = {David A. McAllester and
Luis E. Ortiz},
 journal = {Journal of  Machine Learning Research},
 pages = {895--911},
 title = {Concentration Inequalities for the Missing Mass and for Histogram
Rule Error},
 volume = {4},
 year = {2003}
}

@article{good1953population,
  title={The population frequencies of species and the estimation of population parameters},
  author={Good, Irving J},
  journal={Biometrika},
  volume={40},
  number={3-4},
  pages={237--264},
  year={1953},
}

@inproceedings{menard20fast-active-learning,
 author = {Pierre M{\'{e}}nard and
Omar Darwiche Domingues and
Anders Jonsson and
Emilie Kaufmann and
Edouard Leurent and
Michal Valko},
 booktitle = {Proceedings of the 38th International Conference on Machine Learning},
 pages = {7599--7608},
 title = {Fast active learning for pure exploration in reinforcement learning},
 year = {2021}
}

@article{nived2021provably,
 author = {Nived Rajaraman and
Yanjun Han and
Lin F. Yang and
Kannan Ramchandran and
Jiantao Jiao},
 journal = {ar{X}iv},
 title = {Provably Breaking the Quadratic Error Compounding Barrier in Imitation
Learning, Optimally},
 volume = {2102.12948},
 year = {2021}
}

@article{Orabona19a_modern_introduction_to_ol,
 author = {Francesco Orabona},
 journal = {ar{X}iv},
 title = {A Modern Introduction to Online Learning},
 volume = {1912.13213},
 year = {2019}
}

@inproceedings{yunzhu2017infogail,
  author       = {Yunzhu Li and
                  Jiaming Song and
                  Stefano Ermon},
  title        = {InfoGAIL: Interpretable Imitation Learning from Visual Demonstrations},
  booktitle    = {Advances in Neural Information Processing Systems 30},
  pages        = {3812--3822},
  year         = {2017}
}

@inproceedings{pieter04apprentice,
 author = {Pieter Abbeel and
Andrew Y. Ng},
 booktitle = {Proceedings of the 21st International Conference on Machine Learning},
 title = {Apprenticeship learning via inverse reinforcement learning},
 pages = {1--8},
 year = {2004}
}

@article{Pomerleau91bc,
 author = {Dean Pomerleau},
 journal = {Neural Computation},
 number = {1},
 pages = {88--97},
 title = {Efficient Training of Artificial Neural Networks for Autonomous Navigation},
 volume = {3},
 year = {1991}
}

@book{puterman2014markov,
 author = {Martin L. Puterman},
 publisher = {John Wiley \& Sons},
 title = {Markov Decision Processes: Discrete Stochastic Dynamic Programming},
 year = {2014}
}

@inproceedings{rajaraman2020fundamental,
 author = {Nived Rajaraman and
Lin F. Yang and
Jiantao Jiao and
Kannan Ramchandran},
 booktitle = {Advances in Neural Information Processing Systems 33},
 pages = {2914--2924},
 title = {Toward the Fundamental Limits of Imitation Learning},
 year = {2020}
}

@inproceedings{ross11dagger,
 author = {St{\'{e}}phane Ross and
Geoffrey J. Gordon and
Drew Bagnell},
 booktitle = {Proceedings of the 14th International Conference on Artificial
Intelligence and Statistics},
 pages = {627--635},
 title = {A Reduction of Imitation Learning and Structured Prediction to No-Regret
Online Learning},
 year = {2011}
}

@inproceedings{ross2010efficient,
 author = {Ross, St{\'e}phane and Bagnell, Drew},
 booktitle = {Proceedings of the 13rd International Conference on Artificial Intelligence and Statistics},
 pages = {661--668},
 title = {Efficient reductions for imitation learning},
 year = {2010}
}

@article{shalev12online-learning,
 author = {Shai Shalev{-}Shwartz},
 journal = {Foundations and Trends in Machine Learning},
 number = {2},
 pages = {107--194},
 title = {Online Learning and Online Convex Optimization},
 volume = {4},
 year = {2012}
}

@article{silver2016mastering,
 author = {Silver, David and Huang, Aja and Maddison, Chris J and Guez, Arthur and Sifre, Laurent and Van Den Driessche, George and Schrittwieser, Julian and Antonoglou, Ioannis and Panneershelvam, Veda and Lanctot, Marc and others},
 journal = {Nature},
 number = {7587},
 pages = {484--489},
 title = {Mastering the game of Go with deep neural networks and tree search},
 volume = {529},
 year = {2016}
}

@book{sutton2018reinforcement,
 author = {Sutton, Richard S and Barto, Andrew G},
 publisher = {MIT press},
 title = {Reinforcement {L}earning: {A}n {I}ntroduction},
 year = {2018}
}

@inproceedings{syed07game,
 author = {Umar Syed and
Robert E. Schapire},
 booktitle = {Advances in Neural Information Processing Systems 20},
 pages = {1449--1456},
 title = {A Game-Theoretic Approach to Apprenticeship Learning},
 year = {2007}
}

@inproceedings{syed08lp,
 author = {Umar Syed and
Michael H. Bowling and
Robert E. Schapire},
 booktitle = {Proceedings of the 25th International Conference on Machine Learning},
 pages = {1032--1039},
 title = {Apprenticeship learning using linear programming},
 year = {2008}
}

@article{weissman2003inequalities,
 author = {Weissman, Tsachy and Ordentlich, Erik and Seroussi, Gadiel and Verdu, Sergio and Weinberger, Marcelo J},
 journal = {Hewlett-Packard Labs, Techical Report},
 title = {Inequalities for the L1 deviation of the empirical distribution},
 year = {2003}
}

@inproceedings{xu2020error,
 author = {Tian Xu and
Ziniu Li and
Yang Yu},
 booktitle = {Advances in Neural Information Processing Systems 33},
 pages = {15737--15749},
 title = {Error Bounds of Imitating Policies and Environments},
 year = {2020}
}

@article{xu2021error,
 author = {Xu, Tian and Li, Ziniu and Yu, Yang},
 publisher = {IEEE},
 title = {Error Bounds of Imitating Policies and Environments for Reinforcement Learning},
  journal={IEEE Transactions on Pattern Analysis and Machine Intelligence},
  volume={44},
  number={10},
  pages={6968--6980},
  year={2021},
}

@article{xu2021nearly,
  title={More Efficient Adversarial Imitation Learning Algorithms With Known and Unknown Transitions},
  author={Xu, Tian and Li, Ziniu and Yu, Yang},
  volume={2106.10424, v2},
  journal={arXiv},
  year={2021}
}

@article{rajaraman2021value,
  title={On the Value of Interaction and Function Approximation in Imitation Learning},
  author={Rajaraman, Nived and Han, Yanjun and Yang, Lin and Liu, Jingbo and Jiao, Jiantao and Ramchandran, Kannan},
  journal={Advances in Neural Information Processing Systems 34},
  year={2021}
}

@article{argall2009survey,
  title={A survey of robot learning from demonstration},
  author={Argall, Brenna D and Chernova, Sonia and Veloso, Manuela and Browning, Brett},
  journal={Robotics and autonomous systems},
  volume={57},
  number={5},
  pages={469--483},
  year={2009},
}

@article{hussein2017survey,
  author    = {Ahmed Hussein and
               Mohamed Medhat Gaber and
               Eyad Elyan and
               Chrisina Jayne},
  title     = {Imitation Learning: {A} Survey of Learning Methods},
  journal   = {{ACM} Computing Surveys},
  volume    = {50},
  number    = {2},
  pages     = {1--35},
  year      = {2017},
}

@article{osa2018survey,
  author    = {Takayuki Osa and
               Joni Pajarinen and
               Gerhard Neumann and
               J. Andrew Bagnell and
               Pieter Abbeel and
               Jan Peters},
  title     = {An Algorithmic Perspective on Imitation Learning},
  journal   = {Foundations and Trends in Robotic},
  volume    = {7},
  number    = {1-2},
  pages     = {1--179},
  year      = {2018},
}

@article{orsini2021what,
  title     = {What Matters for Adversarial Imitation Learning?},
  author    = {Manu Orsini and
               Anton Raichuk and
               L{\'{e}}onard Hussenot and
               Damien Vincent and
               Robert Dadashi and
               Sertan Girgin and
               Matthieu Geist and
               Olivier Bachem and
               Olivier Pietquin and
               Marcin Andrychowicz},
  journal={Advances in Neural Information Processing Systems 34},
  year={2021}
}

@book{bertsekas2012dynamic,
  title={Dynamic Programming and Optimal Control: Volume I},
  author={Bertsekas, Dimitri},
  year={2012},
  publisher={Athena scientific}
}

@inproceedings{li2022rethinking,
  title={Rethinking ValueDice: Does It Really Improve Performance?},
  author={Ziniu Li and
    Tian Xu and
    Yang Yu and
    Zhi-Quan Luo},
  booktitle={Proceedings of the 11st International Conference on Learning Representations},
  year={2022},
}

@inproceedings{sun2019provably,
  title={Provably efficient imitation learning from observation alone},
  author={Sun, Wen and Vemula, Anirudh and Boots, Byron and Bagnell, Drew},
  booktitle={Proceeding of the 36th International Conference on Machine Learning},
  pages={6036--6045},
  year={2019},
}

@inproceedings{swamy2021moments,
  title={Of moments and matching: A game-theoretic framework for closing the imitation gap},
  author={Swamy, Gokul and Choudhury, Sanjiban and Bagnell, J Andrew and Wu, Steven},
  booktitle={Proceeding of the 38th International Conference on Machine Learning},
  pages={10022--10032},
  year={2021},
}

@inproceedings{dadashi2021primal,
  author    = {Robert Dadashi and
               L{\'{e}}onard Hussenot and
               Matthieu Geist and
               Olivier Pietquin},
  title     = {Primal Wasserstein Imitation Learning},
  booktitle = {Proceeedings of the 9th International Conference on Learning Representations},
  year      = {2021},
}

@inproceedings{ziebart2008maximum,
  author    = {Brian D. Ziebart and
               Andrew L. Maas and
               J. Andrew Bagnell and
               Anind K. Dey},
  title     = {Maximum Entropy Inverse Reinforcement Learning},
  booktitle = {Proceedings of the 23rd {AAAI} Conference on Artificial Intelligence},
  pages     = {1433--1438},
  year      = {2008},
}

@inproceedings{kamath2015learning,
  title={On learning distributions from their samples},
  author={Kamath, Sudeep and Orlitsky, Alon and Pichapati, Dheeraj and Suresh, Ananda Theertha},
  booktitle={Proceedings of the 28th Conference on Learning Theory},
  pages={1066--1100},
  year={2015},
}

@article{han2015minimax,
  author    = {Yanjun Han and
               Jiantao Jiao and
               Tsachy Weissman},
  title     = {Minimax Estimation of Discrete Distributions Under $\ell_1$  Loss},
  journal   = {{IEEE} Transactions on Information Theory},
  volume    = {61},
  number    = {11},
  pages     = {6343--6354},
  year      = {2015},
}

@inproceedings{syed2010reduction,
  author    = {Umar Syed and
               Robert E. Schapire},
  title     = {A Reduction from Apprenticeship Learning to Classification},
  booktitle = {Advances in Neural Information Processing Systems 23},
  pages     = {2253--2261},
  year      = {2010},
}

@book{yosida2012functional,
  title={Functional analysis},
  author={Yosida, K{\"o}saku},
  year={2012},
  publisher={Springer Science \& Business Media}
}

@inproceedings{cai2021imitation,
  author    = {Xin{-}Qiang Cai and
               Yao{-}Xiang Ding and
               Yuan Jiang and
               Zhi{-}Hua Zhou},
  title     = {Imitation Learning from Pixel-Level Demonstrations by HashReward},
  booktitle = {Proceedings of the 20th International Conference on Autonomous Agents and
               Multiagent Systems},
  pages     = {279--287},
  year      = {2021},
}

@inproceedings{yu2020intrinsic,
  author    = {Xingrui Yu and
               Yueming Lyu and
               Ivor W. Tsang},
  title     = {Intrinsic Reward Driven Imitation Learning via Generative Model},
  booktitle = {Proceedings of the 37th International Conference on Machine Learning},
  volume    = {119},
  pages     = {10925--10935},
  year      = {2020},
}

@inproceedings{liu2021energy,
  author    = {Minghuan Liu and
               Tairan He and
               Minkai Xu and
               Weinan Zhang},
  title     = {Energy-Based Imitation Learning},
  booktitle = {Proceedings of the 20th International Conference on Autonomous Agents and
               Multiagent Systems},
  pages     = {809--817},
  year      = {2021},
}

@inproceedings{venkatraman2015,
  author    = {Arun Venkatraman and
               Martial Hebert and
               J. Andrew Bagnell},
  title     = {Improving Multi-Step Prediction of Learned Time Series Models},
  booktitle = {Proceedings of the 29th {AAAI} Conference on Artificial Intelligence},
  pages     = {3024--3030},
  year      = {2015},
}

@article{hornik1989multilayer,
  title={Multilayer feedforward networks are universal approximators},
  author={Hornik, Kurt and Stinchcombe, Maxwell and White, Halbert},
  journal={Neural networks},
  volume={2},
  number={5},
  pages={359--366},
  year={1989}
}

@article{jung2024sample,
  title={Sample-efficient adversarial imitation learning},
  author={Jung, Dahuin and Lee, Hyungyu and Yoon, Sungroh},
  journal={Journal of Machine Learning Research},
  volume={25},
  number={31},
  pages={1--32},
  year={2024}
}

@article{foster2024behavior,
  title={Is behavior cloning all you need? understanding horizon in imitation learning},
  author={Foster, Dylan J and Block, Adam and Misra, Dipendra},
  journal={Advances in Neural Information Processing Systems 37},
  pages={120602--120666},
  year={2024}
}

\clearpage
\appendix

\section{Proof of Results in Section \ref{sec:horizon_free_sample_complexity}}

\subsection{RBAS MDPs and Useful Properties}
\label{appendix:reset_cliff_and_ail_properties}

In this part, we present some useful properties of \textsf{TV-AIL} on RBAS MDPs. For RBAS MDPs, we know the expert policy never visits bad states. Thus, we have the following fact. 
\begin{fact}    \label{fact:ail_estimation}
For any tabular and episodic MDP satisfying \cref{asmp:reset_cliff}, and the estimation $\widehat{d^{\piE}_h} (s, a)$, we have that 
\begin{align*}
   &\forall h \in [H], \forall s \in \badS, \forall a \in \gA: \,  \widehat{d^{\piE}_h}(s) = 0 \text{ and } \widehat{d^{\piE}_h}(s, a) = 0, \\
   &\forall h \in [H]: \, \sum_{s \in \goodS} \widehat{d^{\piE}_h}(s, a^{1}) = 1, \\
   &\forall h \in [H], \forall s \in \goodS, \forall a \ne a^{1}: \,  \widehat{d^{\piE}_h}(s, a) = 0. 
\end{align*}
\end{fact}

The following lemma states that on RBAS MDPs, in each time step, the optimal solution $\piail$ must take the expert action on certain good state with a positive probability.
\begin{lem}
\label{lem:condition_for_ail_optimal_solution}
For any tabular and episodic MDP satisfying \cref{asmp:reset_cliff}, suppose that $\piail$ is an optimal solution to the state-action distribution matching problem \eqref{eq:ail}. Then for all $h \in [H]$, there exists a state $s \in \goodS$ such that $\piail_h (a^{1} |s) > 0$. Consequently, we have $d^{\piail}_h(s) > 0$ for all $ h \in [H]$ and $s \in \goodS$. 
\end{lem}

\begin{proof}[Proof of \cref{lem:condition_for_ail_optimal_solution}]
The proof is based on contradiction. Assume that the original first statement is false: there exists a policy $\piail$, which is an optimal solution of \eqref{eq:ail}, such that $\exists h \in [H]$, $\forall s \in \goodS$, $\piail_h (a^{1}|s) = 0$. Let $h$ denote the smallest time step index such that $\forall s \in \goodS, \piail_{h} (a^{1}|s) = 0$. It also implies that $\forall s \in \goodS, \sum_{a \in \gA \setminus \{a^1 \}} \piail_{h} (a|s) = 1$.

We construct another policy $\widetilde{\pi}^{\operatorname{AIL}}$, which is only different from $\piail$ in time step $h$. In particular, in time step $h$, we assume that $\widetilde{\pi}^{\operatorname{AIL}} (a^{1}|s) = 1, \forall s \in \goodS$. Here we compare objective values of $\piail$ and $\widetilde{\pi}^{\operatorname{AIL}}$. Since $\piail$ is the same as $\widetilde{\pi}^{\operatorname{AIL}}$ in the first $h-1$ steps, their objective values are the same in the first $h-1$ steps. We only need to compare state-action distribution matching losses from time step $h$. Notice that $d^{\widetilde{\pi}^{\operatorname{AIL}}}_h (s) =d^{\piail}_h (s)$, we obtain
\begin{align*}
    &\quad \Loss_h(\piail) \\
    &= \sum_{(s, a)} \labs \widehat{d^{\piE}_h} (s, a) - d^{\piail}_h (s, a)  \rabs \\
    &= \sum_{s \in \goodS} \bigg[ \labs \widehat{d^{\piE}_h} (s, a^{1}) - d^{\piail}_h (s, a^{1})  \rabs + \sum_{a \ne a^{1}} \labs \widehat{d^{\piE}_h} (s, a) - d^{\piail}_h (s, a)  \rabs \bigg] + \sum_{s \in \badS} \sum_{a} \labs \widehat{d^{\piE}_h} (s, a) - d^{\piail}_h (s, a)  \rabs  \\
    &= \sum_{s \in \goodS} \ls \labs \widehat{d^{\piE}_h} (s, a^{1}) - 0  \rabs + \sum_{a \ne a^{1}} \labs 0 - d^{\piail}_h (s, a)  \rabs \rs + \sum_{s \in \badS} \sum_{a} \labs 0 - d^{\piail}_h (s, a)  \rabs \\
    &=  \sum_{s \in \goodS} \lp \widehat{d^{\piE}_h} (s) + d^{\piail}_h (s) \rp + \sum_{s \in \badS}  d^{\piail}_h (s),
\end{align*}
and 
\begin{align*}
    &\quad \Loss_h(\widetilde{\pi}^{\operatorname{AIL}}) \\
    &= \sum_{(s, a)} \labs \widehat{d^{\piE}_h} (s, a) - d^{\widetilde{\pi}^{\operatorname{AIL}}}_h (s, a)  \rabs \\
    &= \sum_{s \in \goodS} \bigg[ \labs \widehat{d^{\piE}_h} (s, a^{1}) - d^{\widetilde{\pi}^{\operatorname{AIL}}}_h (s, a^{1})  \rabs + \sum_{a \ne a^{1}} \labs \widehat{d^{\piE}_h} (s, a) - d^{\widetilde{\pi}^{\operatorname{AIL}}}_h (s, a)  \rabs \bigg] + \sum_{s \in \badS} \sum_{a} \labs \widehat{d^{\piE}_h} (s, a) - d^{\widetilde{\pi}^{\operatorname{AIL}}}_h (s, a)  \rabs  \\
    &=  \sum_{s \in \goodS} \ls \labs \widehat{d^{\piE}_h} (s, a^{1}) - d^{\widetilde{\pi}^{\operatorname{AIL}}}_h (s, a^{1})  \rabs + \sum_{a \ne a^{1}} \labs 0 - 0 \rabs \rs + \sum_{s \in \badS} \sum_{a} \labs 0 - d^{\piail}_h (s, a)  \rabs \\
    &= \sum_{s \in \goodS} \labs \widehat{d^{\piE}_h} (s) - d^{\piail}_h (s) \rabs + \sum_{s \in \badS}  d^{\piail}_h (s).
\end{align*}
Then we have
\begin{align*}
    &\quad \Loss_h(\widetilde{\pi}^{\operatorname{AIL}}) - \Loss_h(\piail) \\
    &= \sum_{s \in \goodS} \labs \widehat{d^{\piE}_h} (s) - d^{\piail}_h (s) \rabs - \widehat{d^{\piE}_h} (s) - d^{\piail}_h (s) \\
    &< 0,  
\end{align*}
where the last strict inequality follows that there always exists $s \in \goodS$ such that $\widehat{d^{\piE}_h} (s)>0$ and $d^{\piail}_h (s) >0$, so $\vert \widehat{d^{\piE}_h} (s) - d^{\piail}_h (s) \vert < \widehat{d^{\piE}_h} (s) + d^{\piail}_h (s)$. To argue $ d^{\piail}_h (s) >0$ for $s \in \goodS$,  we note that $\forall h^\prime \in [h-1]$, there exists $s \in \goodS$ such that $\piail_{h^\prime}(a^{1} | s) > 0$. With the reachable assumption (refer to \cref{asmp:reset_cliff}) that $\forall s, s^\prime \in \goodS,  P_h (s^\prime |s, a^1) > 0$, we therefore know  $ d^{\piail}_h (s) >0$ for all $s \in \goodS$.
\begin{align*}
    &\quad \text{Loss}_{h^\prime}(\piail) \\
    &=  \sum_{(s, a)} \labs \widehat{d^{\piE}_{h^\prime}} (s, a) - d^{\piail}_{h^\prime} (s, a)  \rabs  \\
    &= \sum_{s \in \goodS} \sum_{a}  \labs \widehat{d^{\piE}_{h^\prime}} (s, a) - d^{\piail}_{h^\prime} (s, a)  \rabs + \sum_{s \in \badS} \sum_{a} \labs \widehat{d^{\piE}_{h^\prime}} (s, a) - d^{\piail}_{h^\prime} (s, a)  \rabs  \\
    &=  \sum_{s \in \goodS} \sum_{a}  \labs \widehat{d^{\piE}_{h^\prime}} (s, a) - 0 \rabs   + \sum_{s \in \badS} \sum_{a} \labs 0 - d^{\piail}_{h^\prime} (s, a)  \rabs \\
    &= \sum_{s \in \goodS} \widehat{d^{\piE}_{h^\prime}} (s) + \sum_{s \in \badS}  d^{\piail}_{h^\prime} (s) \\
    &= 1 + 1 =2,
\end{align*}
which is the maximal value of TV-AIL's objective in each time step. Thus, we have that $\text{Loss}_{h^\prime}(\widetilde{\pi}^{\operatorname{AIL}}) \leq \text{Loss}_{h^\prime}(\piail)$.

Combing the above two arguments, we have that $\sum_{h = 1}^{H} \Loss_h(\widetilde{\pi}^{\operatorname{AIL}})< \sum_{h = 1}^{H} \Loss_h(\piail)$. This contradicts the fact that $\piail$ is the optimal solution to TV-AIL's objective. Hence the original statement is true and we finish the proof of the first statement.

Now we proceed to prove the second statement. The second statement follows the first statement and the properties of RBAS MDPs. The proof is based on the forward induction. In the base step where $h=1$, we directly have that $\forall s \in \goodS, d^{\piail}_1 (s) = \rho (s) > 0$ as the initial state distribution only supports on the set of good states on RBAS MDPs defined in \cref{asmp:reset_cliff}.

In the induction step, we assume that in time step $h$, $\forall s \in \goodS, d^{\piail}_h (s) > 0$. We aim to prove that $\forall s \in \goodS, d^{\piail}_{h+1} (s) > 0$. For each $s \in \goodS$, according to the Bellman flow equation, we have that
\begin{align*}
    d^{\piail}_{h+1} (s) &= \sum_{s^\prime, a^\prime} d^{\piail}_{h} (s^\prime) \piail_h (a^\prime|s^\prime) P_h (s|s^\prime, a^\prime)
    \\
    &= \sum_{s^\prime \in \goodS} d^{\piail}_{h} (s^\prime) \piail_h (a^1|s^\prime) P_h (s|s^\prime, a^1). 
\end{align*}
The last equation follows that only by taking the expert action on good states, the agent can transit into good states. According to the first statement, there exists $\widetilde{s} \in \goodS$ such that $\piail_h (a^1|\widetilde{s}) > 0$. Therefore, we have that
\begin{align*}
    d^{\piail}_{h+1} (s) \geq d^{\piail}_{h} (\widetilde{s}) \piail_h (a^1|\widetilde{s}) P_h (s|\widetilde{s}, a^1). 
\end{align*}
Due to the assumption in the induction step, we have that $d^{\piail}_{h} (\widetilde{s}) > 0$. Due to the reachable property of RBAS MDPs, we have that $P_h (s|\widetilde{s}, a^1) > 0$. In summary, we derive that $d^{\piail}_{h+1} (s) > 0$, which completes the proof in the induction step. Thus, we finish the proof of the second statement.

\end{proof}

\cref{lem:condition_for_ail_optimal_solution} claims that there exists certain good state such that the optimal policy must take the expert action with a positive probability. The following lemma characterizes such states in the last time step.

\begin{lem}
\label{lem:general_condition_ail_policy_at_last_step}
Consider any tabular and episodic MDP satisfying \cref{asmp:reset_cliff}. For the estimation $\widehat{d^{\piE}_h} (s)$, we define the set of visited states as $\gV_{h} := \{ s \in \gS: \widehat{d^{\piE}_h} (s) > 0  \}$. Suppose that $\piail = (\piail_1, \cdots, \piail_H)$ is an optimal solution of \eqref{eq:ail}, then $\forall s \in \gV_{H}$, we have $\piail_{H} \lp a^{1} |s \rp > 0$. 
\end{lem}

\begin{proof}[Proof of \cref{lem:general_condition_ail_policy_at_last_step}]
With \cref{lem:n_vars_opt_greedy_structure}, if $\piail = (\piail_1, \ldots, \piail_{H-1}, \piail_{H})$ is an optimal solution to \eqref{eq:ail}, then we have 
\begin{align*}
    \piail_{H} & \in \argmin_{\pi_{H}}  \sum_{(s, a)} \labs d^{\pi}_H(s) \pi_{H}(a|s) - \widehat{d^{\piE}_H}(s, a) \rabs,
\end{align*}
where $d^{\pi}_{H}(s)$ is computed by $\piail_{1}, \ldots, \piail_{H-1}$. Then, we obtain  
\begin{align*}
   \piail_{H} &\in \argmin_{\pi_{H}}  \bigg\{ \sum_{s \in \gV_H} \sum_{a \in \gA} \labs d^{\pi}_H(s) \pi_{H}(a|s) - \widehat{d^{\piE}_H}(s, a) \rabs + \sum_{s \notin \gV_H} \sum_{a \in \gA} \labs d^{\pi}_H(s) \pi_H(a|s) - \widehat{d^{\piE}_H}(s, a) \rabs \bigg\}
   \\
   &= \argmin_{\pi_{H}} \bigg\{  \sum_{s \in \gV_H} \bigg( \labs d^{\pi}_{H}(s) \pi_H ( a^{1} |s) - \widehat{d^{\piE}_H}(s, a^{1} ) \rabs + \sum_{a \ne a^{1} } \labs d^{\pi}_{H}(s) \pi_H (a|s) - \widehat{d^{\piE}_H}(s, a) \rabs \bigg) 
   \\
   &\; + \sum_{s \notin \gV_H} \sum_{a \in \gA } \labs d^{\pi}_h(s) \pi_H(a|s) - \widehat{d^{\piE}_H}(s, a) \rabs \bigg\}.
\end{align*}
Since $\widehat{d^{\piE}_H}(s, a) = 0$ for $a \not= a^{1}$, we obtain
\begin{align*}
\piail_{H} &\in \argmin_{\pi_{H}} \bigg\{  \sum_{s \in \gV_H} \bigg[ \labs d^{\pi}_H(s) \pi_H ( a^{1} |s) - \widehat{d^{\piE}_H}(s, a^{1}) \rabs + d^{\pi}_H(s) \lp 1 - \pi_H (a^{1} |s) \rp \bigg] 
\\
&\;+ \sum_{s \notin \gV_H} \sum_{a \in \gA } \labs d^{\pi}_H(s) \pi_H(a|s) - \widehat{d^{\piE}_H}(s, a) \rabs \bigg\}
   \\
   &= \argmin_{\pi_{H}}  \bigg\{ \sum_{s \in \gV_H} \bigg[ \labs d^{\pi}_H(s) \pi_H ( a^{1} |s)  - \widehat{d^{\piE}_H}(s, a^{1} ) \rabs - d^{\pi}_H(s) \pi_H ( a^{1} |s)   \bigg] + \sum_{s \notin \gV_H} \sum_{a \in \gA } \labs d^{\pi}_H(s) \pi_H(a|s) - \widehat{d^{\piE}_H}(s, a) \rabs \bigg\}.
\end{align*}
The last equation follows that $d^{\pi}_H(s)$ is independent of $\pi_H$. Note that for different $s \in \gV_H$, $\pi_H (a^{1} |s)$ are independent by the tabular formulation. Thus, we can consider the optimization problem for each $s \in \gV_H$ separately. Specifically, for each $s \in \gV_H$, we have
\begin{align*}
    \piail_{H} (a^{1} |s) &= \argmin_{\pi_H (a^{1} |s)} \bigg\{ \labs d^{\pi}_H(s) \pi_H ( a^{1} |s) - \widehat{d^{\piE}_H}(s, a^{1}) \rabs - d^{\pi}_H(s) \pi_H ( a^{1} |s) \bigg\}.
\end{align*}
For the above one-dimension optimization problem, \cref{lem:single_variable_opt_condition} claims that the optimal solution must be positive, i.e., $\piail_{H} (a^{1} |s)  > 0$. Thus, we finish the proof if we can verify the conditions in \cref{lem:single_variable_opt_condition}. 

In the following part, we verify the conditions required by \cref{lem:single_variable_opt_condition} by setting $a = d^{\pi}_H(s), c = \widehat{d^{\piE}_H}(s, a^{1})$. Since $\piail$ is an optimal solution of TV-AIL's objective, with \cref{lem:condition_for_ail_optimal_solution}, we have that $\forall h \in [H]$, $\exists s \in \goodS$, $\piail_h (a^1 |s) > 0$. With the assumption that $\forall h \in [H], s, s^\prime \in \goodS, P_h (s^\prime |s, a^1) > 0$, we have that $d^{\pi}_h(s) > 0, \forall s \in \goodS$. Based on the definition, for each $s \in \gV_{H}$, $\widehat{d^{\piE}_H}(s, a^{1}) > 0$. Now conditions required by \cref{lem:single_variable_opt_condition} are verified and we obtain that $\piail_H (a^{1} |s) > 0, \forall s \in \gV_{H}$. 
 
\end{proof}

\subsection{Proof of Proposition \ref{prop:ail_general_reset_cliff}}
\label{appendix:proof_of_prop:ail_general_reset_cliff}

\begin{proof}[Proof of \cref{prop:ail_general_reset_cliff}]
The proof is based on backward induction. First, we establish the optimality conditions of multi-stage optimization in the backward induction proof. From \cref{lem:n_vars_opt_greedy_structure}, we have 
\begin{align*}
    \pi^{\ail}_{h} \in \argmin_{\pi_h} f_h(\pi_h; \pi^{\ail}_1, \ldots, \pi^{\ail}_{h-1}, \pi^{\ail}_{h+1}, \ldots, \pi^{\ail}_{H})
\end{align*}
for all $h \in [H]$, where 
\begin{align*}
    f_h(\pi_h; \piail_1, \ldots, \piail_{h-1}, \piail_{h+1}, \ldots, \piail_H) &= \sum_{h=1}^{H} \sum_{(s, a)} \labs d^{\pi}_h(s, a) - \widehat{d^{\piE}_h}(s, a) \rabs 
\end{align*}
is a single-variable loss function that takes $\pi_h$ as the variable and other time-dependent policies\\ $(\piail_1, \ldots, \piail_{h-1}, \piail_{h+1}, \ldots, \piail_H)$ as fixed parameters. Here we emphasize that the state-action distribution $d^{\pi}_h(s, a)$ is calculated by $(\piail_1, \ldots, \piail_{h-1}, \pi_h)$.

Now we proceed to the induction-based proof. Specifically, our induction assumption is: for each $h+1 \leq h^\prime \leq H-1$, we assume that $\piail_{h^\prime} (a^1|s) = \piE_{h^\prime} (a^1|s) = 1, \forall s \in \goodS$. We first consider the base case, i.e., we need to prove that $\piail_{H-1} (a^1|s) = \piE_{H-1} (a^1|s) = 1, \forall s \in \goodS$.

\textbf{Base Case.} Recall the optimality condition in time step $H-1$.
\begin{align*}
    \pi^{\ail}_{H-1} &\in \argmin_{\pi_{H-1}} f_h(\pi_{H-1}; \pi^{\ail}_1, \ldots, \pi^{\ail}_{H-2}, \pi^{\ail}_{H}).
\end{align*}
If we can prove that $\pi_{H-1}(a^1|s) = \piE_{H-1} (a^1|s) = 1, \forall s \in \goodS$ is the unique optimal solution with respect to 
\begin{align*}
    \min_{\pi_{H-1}} f_{H-1}(\pi_{H-1}; \pi_1^{\ail}, \ldots, \pi_{H-2}^{\ail}, \pi_{H}^{\ail}),
\end{align*}
then we can derive that $\piail_{H-1} (a^1|s) = \piE_{H-1} (a^1|s) = 1, \forall s \in \goodS$. To achieve this target, we decompose $f_{H-1}(\pi_{H-1}; \pi_1^{\ail}, \ldots, \pi_{H-2}^{\ail}, \pi_{H}^{\ail})$ into three parts.
\begin{align*}
    f_{H-1}(\pi_{H-1}) &= \underbrace{\sum_{(s, a)} \labs d^{\pi}_{H-1}(s, a) - \widehat{d^{\piE}_{H-1}}(s, a) \rabs}_{\Loss_{H-1}} + \underbrace{\sum_{(s, a)} \labs d^{\pi}_{H}(s, a) - \widehat{d^{\piE}_{H}}(s, a)  \rabs}_{\Loss_{H}} + \underbrace{\sum_{h=1}^{H-2} \sum_{(s, a)} \labs d^{\pi}_{h}(s, a) - \widehat{d^{\piE}_{h}}(s, a)  \rabs}_{\constant}. 
\end{align*}
Here the state-action distributions $d^{\pi}_1 (s, a), \ldots, d^{\pi}_H (s, a)$ are calculated by $( \pi_1^{\ail}, \ldots, \pi_{H-2}^{\ail}, \pi_{H-1}, \pi_{H}^{\ail})$. Notice that the terms $\Loss_{H-1}$ and $\Loss_{H}$ depend on $\pi_{h-1}$ while the term $\constant$ does not depend on $\pi_{h-1}$. Therefore, we have that
\begin{align*}
    \argmin_{\pi_{H-1}} f_{H-1}(\pi_{H-1}; \pi_1^{\ail}, \ldots, \pi_{H-2}^{\ail}, \pi_{H}^{\ail}) &= \argmin_{\pi_{H-1}} \Loss_{H-1} + \Loss_{H}.   
\end{align*}
We will argue that $\pi_{H-1}(a^1|s) = \piE_{H-1} (a^1|s) = 1, \forall s \in \goodS$ is the optimal solution with respect to $\min_{\pi_{H-1}} \Loss_{H-1}$ and the unique optimal solution with respect to $\min_{\pi_{H-1}} \Loss_{H}$. Therefore, we can claim that $\pi_{H-1}(a^1|s) = \piE_{H-1} (a^1|s) = 1, \forall s \in \goodS$ is the unique optimal solution with respect to $\min_{\pi_{H-1}} f_{H-1}(\pi_{H-1}; \pi_1^{\ail}, \ldots, \pi_{H-2}^{\ail}, \pi_{H}^{\ail})$ by \cref{lem:unique_opt_solution_condition}.

For $\Loss_{H-1}$, we have that 
    \begin{align*}
\Loss_{H-1} &= \sum_{s \in \gS} \sum_{a \in \gA} \labs \widehat{d^{\piE}_{H-1}} (s, a) - d^{\piail}_{H-1} (s) \pi_{H-1} (a|s) \rabs
        \\
        &= \sum_{s \in \goodS}  \bigg[ \labs \widehat{d^{\piE}_{H-1}} (s, a^{1}) - d^{\piail}_{H-1} (s) \pi_{H-1} (a^1|s) \rabs + \sum_{a \ne a^{1}} \labs \widehat{d^{\piE}_{H-1}} (s, a) - d^{\piail}_{H-1} (s) \pi_{H-1} (a|s)  \rabs \bigg] 
        \\
        &\; + \sum_{s \in \badS} \sum_{a \in \gA} \labs \widehat{d^{\piE}_{H-1}} (s, a) - d^{\piail}_{H-1} (s) \pi_{H-1} (a|s) \rabs \\
        &= \sum_{s \in \goodS} \bigg( \labs \widehat{d^{\piE}_{H-1}} (s) - d^{\piail}_{H-1} (s) \pi_{H-1} (a^{1}|s)  \rabs + d^{\piail}_{H-1} (s) \lp 1 - \pi_{H-1} (a^{1}|s) \rp \bigg) + \sum_{s \in \badS} d^{\piail}_{H-1} (s). 
    \end{align*}
    The last equation follows \cref{fact:ail_estimation}. Notice that $d^{\piail}_{H-1} (s)$ is fixed and independent of $\pi_{H-1}$, so we can obtain the following optimization problem: 
	    \begin{align*}
	        \argmin_{\pi_{H-1}} \mathrm{Loss}_{H-1} &= \argmin_{\pi_{H-1}} \bigg\{ \sum_{s \in \goodS} \bigg\vert \widehat{d^{\piE}_{H-1}} (s) - d^{\piail}_{H-1} (s) \pi_{H-1} (a^{1}|s)  \bigg\vert - d^{\piail}_{H-1} (s) \pi_{H-1} (a^{1}|s) \bigg\}.
	    \end{align*}
    Since elements in $\{ \pi_{H-1} (\cdot|s): s \in \goodS \}$ are independent, we can consider the above optimization problem for each $s \in \goodS$ individually:
    \begin{align*}
        \argmin_{\pi_{H-1} (a^{1}|s) \in [0, 1]} \bigg\{ \labs \widehat{d^{\piE}_{H-1}} (s) - d^{\piail}_{H-1} (s) \pi_{H-1} (a^{1}|s)  \rabs - d^{\piail}_{H-1} (s) \pi_{H-1} (a^{1}|s) \bigg\}.
    \end{align*}
    For this one-dimension optimization problem, we can use \cref{lem:single_variable_opt} to show that $\pi_{H-1} (a^{1}|s) = \piE_{H-1} (a^{1}|s)  = 1$ is an optimal solution. 
    
    For $\Loss_{H}$, let us introduce the notation $\gV_{H} :=\{ s \in \gS: \widehat{d^{\piE}_H} (s) > 0  \}$, i.e., the set of visited states in time step $H$. For any $s \notin \gV_{H}$,  we have that $\widehat{d^{\piE}_{H}} (s) = 0$. Then, we obtain 
    \begin{align}
      \Loss_{H} &= \sum_{s \in \gS} \sum_{a \in \gA} \labs \widehat{d^{\piE}_{H}} (s, a) - d^{\pi}_{H} (s, a) \rabs \nonumber
        \\
        &= \sum_{s \in \goodS} \sum_{a \in \gA} \labs \widehat{d^{\piE}_{H}} (s, a) - d^{\pi}_{H} (s, a) \rabs + \sum_{s \in \badS} d^{\pi}_H (s)  \nonumber 
        \\
        &= \underbrace{\sum_{s \in \gV_{H}} \labs \widehat{d^{\piE}_{H}} (s) - d^{\pi}_{H} (s, a^{1}) \rabs}_{\text{Term I}} + \underbrace{ \sum_{s \in \gV_{H}} \sum_{a\not= a^1} d^{\pi}_H (s, a)}_{\text{Term II}} + \underbrace{ \sum_{s \in \goodS \text{ and } s \notin \gV_{H}} d^{\pi}_H (s) }_{\text{Term III}} + \underbrace{\sum_{s \in \badS} d^{\pi}_H (s)}_{\text{Term IV}}. \label{eq:main_prop_proof_1}
    \end{align}
    Note that the state-action distributions $d^{\pi}_{H} (s, a)$ and $d^{\pi}_H (s)$ are computed by $(\piail_1, \ldots, \piail_{H-2}, \pi_{H-1}, \piail_{H})$, where $\pi_{H-1}$ is the decision variable and the others are given. Now let us consider the first three terms in \eqref{eq:main_prop_proof_1}. For $d^{\pi}_H(s)$  with $s \in \goodS$,  with the Bellman-flow equation in \eqref{eq:flow_link}, we have 
    \begin{align*}
        d^{\pi}_{H} (s) &= \sum_{s^\prime \in \gS} \sum_{a \in \gA} d^{\piail}_{H-1} (s^\prime) \pi_{H-1} (a|s^\prime) P_{H-1} (s | s^\prime, a)
        \\
        &= \sum_{s^\prime \in \goodS} d^{\piail}_{H-1} (s^\prime) \pi_{H-1} (a^{1}|s^\prime) P_{H-1} (s | s^\prime, a^{1}).
    \end{align*}
   Accordingly, we have 
    \begin{align*}
        &\quad \text{Term I} 
        \\
    &= \sum_{s \in \gV_{H}} \bigg\vert \widehat{d^{\piE}_{H}} (s) - \bigg( \sum_{s^\prime \in \goodS} d^{\piail}_{H-1} (s^\prime) \pi_{H-1} (a^{1}|s^\prime) P_{H-1} (s | s^\prime, a^{1})  \bigg) \piail_{H} (a^{1}|s) \bigg\vert \\
         &= \sum_{s \in \gV_{H}} \bigg\vert \widehat{d^{\piE}_{H}} (s) - \sum_{s^\prime \in \goodS} d^{\piail}_{H-1} (s^\prime)  P_{H-1} (s | s^\prime, a^{1}) \piail_{H} (a^{1}|s) \pi_{H-1} (a^{1}|s^\prime) \bigg\vert,
    \end{align*}
    and
    \begin{align*}
        &\quad \text{Term II}
        \\
    &= \sum_{s \in \gV_{H} } \lp 1- \piail_{H} (a^{1}|s) \rp \cdot  \lp \sum_{s^\prime \in \goodS} d^{\piail}_{H-1} (s^\prime) \pi_{H-1} (a^{1}|s^\prime) P_{H-1} (s | s^\prime, a^{1})  \rp  \\
    &= \sum_{s^\prime \in \goodS} \pi_{H-1} (a^{1}|s^\prime) \cdot  \lp \sum_{s \in \gV_{H}} d^{\piail}_{H-1} (s^\prime) P_{H-1} (s | s^\prime, a^{1}) \lp 1- \piail_{H} (a^{1}|s) \rp \rp,
    \end{align*}
   and 
\begin{align*}
    &\quad \text{Term III}
    \\
    &=  \sum_{s \in \goodS \text{ and } s \notin \gV_{H}} \bigg( \sum_{s^\prime \in \goodS} d^{\piail}_{H-1} (s^\prime) \pi_{H-1} (a^{1}|s^\prime)  P_{H-1} (s | s^\prime, a^{1}) \bigg) \\
    &= \sum_{s^\prime \in \goodS} \pi_{H-1} (a^{1}|s^\prime) \cdot \lp \sum_{s \in \goodS \text{ and } s \notin \gV_{H}} d^{\piail}_{H-1} (s^\prime) P_{H-1} (s | s^\prime, a^{1}) \rp.
\end{align*}
Next, we consider the last term in \eqref{eq:main_prop_proof_1}.  For $d^{\pi}_{H} (s)$ with $s \in \badS$, recall that when the agent takes a non-expert action,  it transits into bad states. Therefore, the probability of visiting bad states in time step $H$ arises from two parts. One is the probability of visiting bad states in time step $H-1$ and the other is the probability of visiting good states and taking non-expert actions in time step $H-1$. Accordingly, we obtain 
    \begin{align*}
        &\quad \sum_{s \in \badS} d^{\pi}_H (s)
        \\
        &= \sum_{s^\prime \in \badS} d^{\piail}_{H-1} (s^\prime) + \sum_{s^\prime \in \goodS} d^{\piail}_{H-1} (s^\prime) \bigg( \sum_{a \not= a^{1}} \pi_{H-1} (a|s^\prime) \bigg)
        \\
        &= \sum_{s^\prime \in \badS} d^{\piail}_{H-1} (s^\prime) + \sum_{s^\prime \in \goodS} d^{\piail}_{H-1} (s^\prime) \lp 1 - \pi_{H-1} (a^{1}|s^\prime) \rp.
    \end{align*}
   Then, it is ready to get 
   \begin{align*}
       &\quad \text{Term IV}
       \\
       &= \sum_{s^\prime \in \badS} d^{\piail}_{H-1} (s) + \sum_{s^\prime \in \goodS} d^{\piail}_{H-1} (s^\prime) \lp 1 - \pi_{H-1} (a^{1}|s^\prime) \rp \\
       &= \sum_{s^\prime \in \badS} d^{\piail}_{H-1} (s^\prime) + \sum_{s^\prime \in \goodS} d^{\piail}_{H-1} (s^\prime) - \sum_{s^\prime \in \goodS} d^{\piail}_{H-1}(s^\prime) \pi_{H-1} (a^{1} | s^\prime).
   \end{align*}
    Subsequently, we merge the optimization variable $\pi_{H-1}$ in the second, third, and fourth terms to obtain 
    \begin{align}
       &\quad \text{Term II} + \text{Term III} + \text{Term IV} \nonumber  \\
       &=    \sum_{s^\prime \in \goodS} \pi_{H-1} (a^{1}|s^\prime)  \nonumber
       \\
       &\quad \cdot \lp \sum_{s \in \gV_{H}} d^{\piail}_{H-1} (s^\prime) P_{H-1} (s | s^\prime, a^{1}) \lp 1- \piail_{H} (a^{1}|s) \rp    \rp  \nonumber + \sum_{s^\prime \in \goodS} \pi_{H-1} (a^{1}|s^\prime)  \nonumber
        \\
        &\quad \cdot \bigg( \sum_{s \in \goodS \text{ and } s \notin \gV_{H}} d^{\piail}_{H-1} (s^\prime) P_{H-1} (s | s^\prime, a^{1}) \bigg) \nonumber - \sum_{s^\prime \in \goodS} \pi_{H-1} (a^{1}|s^\prime) d^{\piail}_{H-1} (s^\prime) + \constant  \nonumber 
        \\
        &= \sum_{s^\prime \in \goodS} \pi_{H-1} (a^{1}|s^\prime) \Bigg( \sum_{s \in \gV_{H}} d^{\piail}_{H-1} (s^\prime) P_{H-1} (s | s^\prime, a^{1}) \nonumber - \sum_{s \in \gV_{H}} d^{\piail}_{H-1} (s^\prime) P_{H-1} (s | s^\prime, a^{1}) \piail_{H} (a^{1}|s)  \nonumber 
        \\
        &\;+ \sum_{s \in \goodS \text{ and } s \notin \gV_{H}}  d^{\piail}_{H-1} (s^\prime) P_{H-1} (s | s^\prime, a^{1})  - d^{\piail}_{H-1} (s^\prime)   \Bigg) \nonumber + \constant   \nonumber 
        \\
        &= \sum_{s^\prime \in \goodS} \pi_{H-1} (a^{1}|s^\prime) \Bigg( \sum_{s \in \goodS} d^{\piail}_{H-1} (s^\prime) P_{H-1} (s | s^\prime, a^{1}) - \sum_{s \in \gV_{H}}   d^{\piail}_{H-1} (s^\prime) P_{H-1} (s | s^\prime, a^{1}) \piail_{H} (a^{1}|s)  - d^{\piail}_{H-1} (s^\prime)   \Bigg) \label{eq:main_prop_proof_2}
        \\
        &= - \sum_{s^\prime \in \goodS} \pi_{H-1} (a^{1}|s^\prime)   \cdot \lp \sum_{s \in \gV_{H}} d^{\piail}_{H-1} (s^\prime) P_{H-1} (s | s^\prime, a^{1}) \piail_{H} (a^{1}|s)  \rp + \constant , \nonumber 
    \end{align}
    where in the last equation we use the fact that for $s^\prime \in \goodS$, we have $\sum_{s \in \goodS} P_{H-1}(s|s^\prime, a^1) = 1$, so the first term and the third term in \eqref{eq:main_prop_proof_2} are canceled. In the above equations, $\constant = \sum_{s^\prime \in \badS} d^{\piail}_{H-1} (s^\prime) + \sum_{s^\prime \in \goodS} d^{\piail}_{H-1} (s^\prime)$, which is independent of $\pi_{H-1}$. Back to \eqref{eq:main_prop_proof_1}, we get that 
    \begin{equation}    \label{eq:proof_vail_reset_cliff_1}
    \begin{split}
          \Loss_{H} &= \sum_{s \in \gV_{H}} \bigg\vert \widehat{d^{\piE}_{H}} (s) - \sum_{s^\prime \in \goodS} d^{\piail}_{H-1} (s^\prime)  P_{H-1} (s | s^\prime, a^{1}) \piail_{H} (a^{1}|s) \pi_{H-1} (a^{1}|s^\prime) \bigg\vert 
         \\
         &\;- \sum_{s^\prime \in \goodS} \pi_{H-1} (a^{1}|s^\prime) \cdot \lp \sum_{s \in \gV_{H}} d^{\piail}_{H-1} (s^\prime) P_{H-1} (s | s^\prime, a^{1}) \piail_{H} (a^{1}|s)  \rp + \constant.     
    \end{split}
    \end{equation}
    Then we have that
    \begin{align*}
        &\quad \argmin_{\pi_{H-1}} \Loss_{H}
        \\
        &= \argmin_{\pi_{H-1}} \bigg\{ \sum_{s \in \gV_{H}} \bigg\vert \widehat{d^{\piE}_{H}} (s) - \sum_{s^\prime \in \goodS} d^{\piail}_{H-1} (s^\prime)  P_{H-1} (s | s^\prime, a^{1}) \piail_{H} (a^{1}|s) \pi_{H-1} (a^{1}|s^\prime) \bigg\vert 
        \\
        &\; - \sum_{s^\prime \in \goodS} \pi_{H-1} (a^{1}|s^\prime) \cdot \bigg( \sum_{s \in \gV_{H}} d^{\piail}_{H-1} (s^\prime) P_{H-1} (s | s^\prime, a^{1}) \piail_{H} (a^{1}|s)  \bigg)  \bigg\}.
    \end{align*}
    For this optimization problem, we will apply \cref{lem:mn_variables_opt_unique} to show that $\forall s \in \goodS, \pi_{H-1}(a^{1} | s) = \piE_{H-1}(a^{1} | s) = 1$ is the unique optimal solution. In particular, we can verify the conditions required by \cref{lem:mn_variables_opt_unique} by defining the following terms:
    \begin{align*}
        &m = \labs \gV_{H}  \rabs, n = \labs \goodS \rabs, \forall s \in \gV_{H}, c(s) = \widehat{d^{\piE}_{H}} (s),
        \\
        & \forall s \in \gV_{H}, s^\prime \in \goodS, A (s, s^\prime) = d^{\piail}_{H-1} (s^\prime)  P_{H-1} (s | s^\prime, a^{1}) \piail_{H} (a^{1}|s),
        \\
        & \forall s^\prime \in \goodS, d(s^\prime) = \sum_{s \in \gV_{H}} d^{\piail}_{H-1} (s^\prime) P_{H-1} (s | s^\prime, a^{1}) \piail_{H} (a^{1}|s). 
    \end{align*}
    Now we verify the conditions in \cref{lem:mn_variables_opt_unique}. To start with, we note that 
    \cref{lem:condition_for_ail_optimal_solution} implies that if $\piail$ is an optimal solution to \eqref{eq:ail_reset_cliff} on RBAS MDPs, then $\forall s^\prime \in \goodS, d^{\piail}_{H-1}(s^\prime) > 0$.  With \cref{lem:general_condition_ail_policy_at_last_step}, we have that $\forall s \in \gV_{H}, \piail_{H} (a^{1}|s) > 0$. Hence we have that $A > 0$, where $>$ means element-wise comparison. Besides, on the one hand, we have that
    \begin{align*}
         \sum_{s \in  \gV_{H} } c(s) = \sum_{s \in  \gV_{H} } \widehat{d^{\piE}_{H}} (s) = 1. 
    \end{align*}
    On the other hand, we have that
    \begin{align*}
         \sum_{s \in  \gV_{H} } \sum_{s^\prime \in \goodS} A (s, s^\prime) &\leq \sum_{s \in  \gV_{H} } \sum_{s^\prime \in \goodS} d^{\piail}_{H-1} (s^\prime)  P_{H-1} (s | s^\prime, a^{1}) \leq  1.
    \end{align*}
    Therefore, we obtain
    \begin{align*}
        \sum_{s \in  \gV_{H} } \sum_{s^\prime \in \goodS} A (s, s^\prime) \leq \sum_{s \in  \gV_{H} } c(s). 
    \end{align*}
    For each $s^\prime \in \goodS$, it holds that 
 \begin{align*}
     \sum_{s \in \gV_{H}} A (s, s^\prime)  &=  \sum_{s \in \gV_{H}} d^{\piail}_{H-1} (s^\prime)  P_{H-1} (s | s^\prime, a^{1}) \piail_{H} (a^{1}|s) = d (s^\prime). 
    \end{align*}
    Thus, we have verified the conditions in \cref{lem:mn_variables_opt_unique}. With \cref{lem:mn_variables_opt_unique}, we obtain that $\pi_{H-1} (a^{1}|s) =\piE_{H-1} (a^{1}|s) = 1, \forall s \in \goodS$ is the unique optimal solution of $\mathrm{Loss}_{H}$.
    By \cref{lem:unique_opt_solution_condition}, $\pi_{H-1} (a^{1}|s) =\piE_{H-1} (a^{1}|s) = 1, \forall s \in \goodS$ is the unique optimal solution of $\min_{\pi_{H-1}} \Loss_{H-1} + \Loss_{H}$, which completes the proof of the base case.

\textbf{Induction Step.} The main proof strategy is similar to what we have used in the proof of the base case but is more tricky. We assume that for step $h^\prime = h+1, h+2, \cdots, H-1$, $\piail_{h^\prime} (a^{1}|s) = \piE_{h^\prime} (a^{1}|s) = 1, \forall s \in \goodS$. We aim to prove that for step $h$, $\piail_{h} (a^{1}|s) = \piE_{h} (a^{1}|s) = 1, \forall s \in \goodS$.

Recall the optimality condition in time step $h$.
\begin{align*}
    \pi^{\ail}_{h} &\in \argmin_{\pi_h} f_h(\pi_h; \pi^{\ail}_1, \ldots, \pi^{\ail}_{h-1}, \pi^{\ail}_{h+1}, \ldots, \pi^{\ail}_{H}).  
\end{align*}
Our target becomes to prove that $\pi_{h} (a^{1}|s) = \piE_{h} (a^{1}|s) = 1, \forall s \in \goodS$ is the unique optimal solution regarding
\begin{align*}
    \min_{\pi_h} f_h(\pi_h; \pi^{\ail}_1, \ldots, \pi^{\ail}_{h-1}, \pi^{\ail}_{h+1}, \ldots, \pi^{\ail}_{H}).
\end{align*}
Similar to the analysis in the base step, we decompose $f_h(\pi_h; \pi^{\ail}_1, \ldots, \pi^{\ail}_{h-1}, \pi^{\ail}_{h+1}, \ldots, \pi^{\ail}_{H})$ into three parts.
\begin{align*}
    f_{h}(\pi_{h}) &= \underbrace{\sum_{(s, a)} \labs d^{\pi}_{h}(s, a) - \widehat{d^{\piE}_{h}}(s, a) \rabs}_{\Loss_{h}} + \sum_{h^\prime=h+1}^{H} \underbrace{ \sum_{(s, a)}   \labs d^{\pi}_{h^\prime} (s, a) - \widehat{d^{\piE}_{h^\prime}} (s, a) \rabs}_{\Loss_{h^\prime}} + \underbrace{\sum_{h^\prime=1}^{h-1} \sum_{(s, a)} \labs d^{\pi}_{h^\prime}(s, a) - \widehat{d^{\piE}_{h^\prime}}(s, a)  \rabs}_{\constant}. 
\end{align*}
Notice that the state-action distributions appeared in $f_{h}(\pi_{h})$ are computed by $(\piail_1, \ldots, \piail_{h-1},\pi_h, \piail_{h+1}, \ldots, \piail_{H})$. In particular, for each $1 \leq h^\prime \leq h-1$, $d^{\pi}_{h^\prime}(s, a)$ is independent of the optimization variable $\pi_h$. Therefore, we obtain that
\begin{align*}
    \argmin_{\pi_h} f_h(\pi_h; \pi^{\ail}_1, \ldots, \pi^{\ail}_{h-1}, \pi^{\ail}_{h+1}, \ldots, \pi^{\ail}_{H}) &= \argmin_{\pi_h} \Loss_{h} + \sum_{h^\prime=h+1}^{H} \Loss_{h^\prime}.   
\end{align*}
Similarly, we will first prove that $\pi_{h} (a^1|s) = \piE_{h} (a^1|s) = 1, \forall s \in \goodS$ is an optimal solution with respect to $\min_{\pi_{h}} \Loss_{h}$. Then we will prove that for each $h+1 \leq h^\prime \leq H$, $\pi_{h} (a^1|s) = \piE_{h} (a^1|s) = 1, \forall s \in \goodS$ is the {unique} optimal solution with respect to $\min_{\pi_{h}} \Loss_{h^\prime}$. In this way, we can argue that $\pi_{h} (a^1|s) = \piE_{h} (a^1|s) = 1, \forall s \in \goodS$ is the unique optimal solution with respect to
\begin{align*}
    \min_{\pi_h} f_h(\pi_h; \pi^{\ail}_1, \ldots, \pi^{\ail}_{h-1}, \pi^{\ail}_{h+1}, \ldots, \pi^{\ail}_{H}).
\end{align*}
For $\Loss_h$, we have that
\begin{align*}
    &\quad \Loss_{h}
    \\
    &= \sum_{s \in \gS} \sum_{a \in \gA} \labs \widehat{d^{\piE}_{h}} (s, a) - d^{\piail}_{h} (s) \pi_h(a|s) \rabs
    \\
    &= \sum_{s \in \goodS}  \bigg( \labs \widehat{d^{\piE}_{h}} (s, a^1) - d^{\piail}_{h} (s) \pi_h(a^1|s) \rabs + \sum_{a \not=  a^1} d^{\piail}_{h} (s) \pi_h (a|s)  \bigg) + \sum_{s \in \badS} d^{\piail}_{h} (s)
    \\
    &= \sum_{s \in \goodS} \bigg( \labs \widehat{d^{\piE}_{h}} (s) - d^{\piail}_{h} (s) \pi_{h} (a^{1}|s)  \rabs + d^{\piail}_{h} (s) \lp 1 - \pi_{h} (a^{1}|s) \rp \bigg) + \sum_{s \in \badS} d^{\piail}_{h} (s). 
\end{align*}
Notice that $d^{\piail}_{h} (s)$ is independent of $\pi_{h}$, then we have that
	\begin{align*}
	    \argmin_{\pi_{h}} \mathrm{Loss}_{h} &= \argmin_{\pi_{h}} \sum_{s \in \goodS} \bigg( \labs \widehat{d^{\piE}_{h}} (s) - d^{\piail}_{h} (s) \pi_{h} (a^{1}|s)  \rabs - d^{\piail}_{h} (s) \pi_{h} (a^{1}|s) \bigg).
	\end{align*}
By the tabular formulation, we can consider the above optimization problem for each $s \in \goodS$ individually:
\begin{align*}
    &\argmin_{\pi_{h} (a^{1}|s) \in [0, 1]} \labs \widehat{d^{\piE}_{h}} (s) - d^{\piail}_{h} (s) \pi_{h} (a^{1}|s)  \rabs - d^{\piail}_{h} (s) \pi_{h} (a^{1}|s).
\end{align*}
For this one-dimension optimization problem, we can show that $\pi_h(a^{1}|s) = \piE_h(a^{1}|s) = 1$ is an optimal solution by \cref{lem:single_variable_opt}. Thus, we obtain that $\pi_{h} (a^{1}|s) =  \piE_h(a^{1}|s) = 1, \forall s \in \goodS$ is an optimal solution of $\min_{\pi_h} \mathrm{Loss}_{h}$.

    Next, for each $h+1 \leq h^\prime \leq H-1$, we consider the optimization problem of $\min_{\pi_h} \Loss_{h^\prime}$. Notice that we assume that for each $h+1 \leq h^\prime \leq H-1$, $\piail_{h^\prime} (a^{1}|s) = 1, \forall s \in \goodS$. Then we have
    \begin{align}
        &\quad \Loss_{h^\prime} \nonumber
        \\
        &= \sum_{s \in \gS} \sum_{a \in \gA} \labs \widehat{d^{\piE}_{h^\prime}} (s, a) - d^{\pi}_{h^\prime} (s, a) \rabs
        \nonumber \\
        &= \sum_{s \in \goodS} \sum_{a \in \gA} \labs \widehat{d^{\piE}_{h^\prime}} (s, a) - d^{\pi}_{h^\prime} (s) \piail_{h^\prime} (a|s)  \rabs \nonumber + \sum_{s \in \badS} \sum_{a \in \gA} d^{\pi}_{h^\prime} (s, a)
        \nonumber \\
        &= \sum_{s \in \goodS} \labs \widehat{d^{\piE}_{h^\prime}} (s, a^1) - d^{\pi}_{h^\prime} (s) \piail_{h^\prime}( a^1|s) \rabs + \sum_{s \in \badS} d^{\pi}_{h^\prime} (s) 
       \nonumber  \\
        &= \underbrace{\sum_{s \in \goodS} \labs \widehat{d^{\piE}_{h^\prime}} (s) - d^{\pi}_{h^\prime} (s) \rabs}_{\text{Term I}} + \underbrace{\sum_{s \in \badS} d^{\pi}_{h^\prime} (s)}_{\text{Term II}}.  \label{eq:eq:main_prop_proof_3}
    \end{align}
    Here $d^{\pi}_{h^\prime} (s)$ and $d^{\pi}_{h^\prime} (s, a)$ are induced by $(\piail_1, \ldots, \piail_{h-1}, \pi_h, \piail_{h+1}, \ldots, \piail_{h^\prime})$. Note that only through taking the expert action on good states, the agent could visit good states. With the Bellman-flow equation in \eqref{eq:flow_link}, we have that $ \forall s \in \goodS$, 
    \begin{align*}
        &\quad d^{\pi}_{h^\prime} (s)
        \\
        &= \sum_{s^\prime \in \gS} \sum_{a \in \gA} d^{\piail}_h (s^\prime) \pi_h (a|s^\prime) \sP^{\piail} \lp s_{h^\prime} = s |s_h = s^\prime, a_h = a \rp
        \\
        &= \sum_{s^\prime \in \goodS} d^{\piail}_h (s^\prime) \pi_h (a^{1}|s^\prime) \sP^{\piail} \lp s_{h^\prime} = s |s_h = s^\prime, a_h = a^{1} \rp,
    \end{align*}
    where $\sP^{\piail}(s_{h^\prime} = s |s_h = s^\prime, a_h = a^{1})$ refers to the transition probability of $s$ in time step $h^\prime$ by starting from $(s^\prime, a^{1})$ in time step $h$ via policy $\piail_{h+1}, \ldots, \piail_{h^\prime-1}$. Besides, $d^{\piail}_h (s^\prime)$ is induced by $(\piail_1, \ldots, \piail_{h-1})$. 
    Then we obtain that
    \begin{align*}
        \text{Term I} = \sum_{s \in \goodS} \bigg\vert \widehat{d^{\piE}_{h^\prime}} (s) - \sum_{s^\prime \in \goodS} d^{\piail}_h (s^\prime)  \sP^{\piail} \lp s_{h^\prime} = s |s_h = s^\prime, a_h = a^{1} \rp \pi_h (a^{1}|s^\prime) \bigg\vert.
    \end{align*}
    First of all, we notice that the conditional probability of $\sP^{\piail} \lp s_{h^\prime} = s |s_h = s^\prime, a_h = a^{1} \rp$ is independent of $\pi_h$. Besides, as for each $h^\prime = h+1, h+2, \ldots, H-1$, $\piail_{h^\prime} (a^{1}|s) = 1, \forall s \in \goodS$, the visitation probability of bad states in step $h^\prime$ comes from two parts in step $h$. One is the visitation probability of bad states in step $h$. The other is the probability of visiting good states and taking non-expert actions in step $h$. We obtain
    \begin{align*}
        \text{Term II} &= \sum_{s \in \badS} d^{\piail}_{h} (s) + \sum_{s^\prime \in \goodS} \sum_{a \not= a^{1}} d^{\piail}_h (s^\prime)  \pi_h (a|s^\prime) = \sum_{s \in \badS} d^{\piail}_{h} (s) + \sum_{s^\prime \in \goodS} d^{\piail}_h (s^\prime) \lp 1 - \pi_h (a^{1}|s^\prime) \rp .
    \end{align*}
    Plugging $\text{Term I}$ and $\text{Term II}$ into $\Loss_{h^\prime}$ in \eqref{eq:eq:main_prop_proof_3} yields
    \begin{align*}
        &\quad \Loss_{h^\prime}
        \\
        &= \sum_{s \in \goodS} \bigg\vert \widehat{d^{\piE}_{h^\prime}} (s) - \sum_{s^\prime \in \goodS} d^{\piail}_h (s^\prime)  \sP^{\piail} \lp s_{h^\prime} = s |s_h = s^\prime, a_h = a^{1} \rp \pi_h (a^{1}|s^\prime) \bigg\vert + \sum_{s \in \badS} d^{\piail}_{h} (s) 
        \\
        &\;+ \sum_{s^\prime \in \goodS} d^{\piail}_h (s^\prime) \lp 1 - \pi_h (a^{1}|s^\prime) \rp
        \\
        &= \sum_{s \in \goodS} \bigg\vert \widehat{d^{\piE}_{h^\prime}} (s) - \sum_{s^\prime \in \goodS} d^{\piail}_h (s^\prime)  \sP^{\piail} \lp s_{h^\prime} = s |s_h = s^\prime, a_h = a^{1} \rp \pi_h (a^{1}|s^\prime) \bigg\vert - \sum_{s^\prime \in \goodS} d^{\piail}_h (s^\prime) \pi_h (a^{1}|s^\prime) + \constant. 
    \end{align*}
    Here $\constant =  \sum_{s \in \badS} d^{\piail}_{h} (s) + \sum_{s^\prime \in \goodS} d^{\piail}_h (s^\prime)$ is independent of $\pi_h$. This equation is similar to \eqref{eq:proof_vail_reset_cliff_1} in the proof of the base case. Then we have that
    \begin{align*}
        & \quad \argmin_{\pi_h} \Loss_{h^\prime}
        \\
        &= \argmin_{\pi_h} \sum_{s \in \goodS} \bigg\vert \widehat{d^{\piE}_{h^\prime}} (s) - \sum_{s^\prime \in \goodS} d^{\piail}_h (s^\prime)  \sP^{\piail} \lp s_{h^\prime} = s |s_h = s^\prime, a_h = a^{1} \rp \pi_h (a^{1}|s^\prime) \bigg\vert - \sum_{s^\prime \in \goodS} d^{\piail}_h (s^\prime) \pi_h (a^{1}|s^\prime)  .
    \end{align*}
    For this optimization problem, we can again use \cref{lem:mn_variables_opt_unique} to prove that $\forall s \in \goodS, \pi_h(a^{1} | s) = \piE_h(a^{1} | s)= 1$ is the unique global optimal solution with respect to $\min_{\pi_h} \Loss_{h^\prime}$. To check the conditions required by \cref{lem:mn_variables_opt_unique}, we define 
    \begin{align*}
        & m = n = \labs \goodS \rabs, \forall s \in \goodS, c(s) = \widehat{d^{\piE}_{h^\prime}} (s), \\
        & \forall s, s^\prime \in \goodS, A (s, s^\prime) = d^{\piail}_h (s^\prime)  \sP^{\piail} \lp s_{h^\prime} = s |s_h = s^\prime, a_h = a^{1} \rp,
        \\
        & \forall s^\prime \in \goodS, d(s^\prime) = d^{\piail}_h (s^\prime). 
    \end{align*}
    Now we verify the conditions in \cref{lem:mn_variables_opt_unique}. Following the same argument in the proof of the base case, we have that $\forall s, s^\prime \in \goodS$, 
    \begin{align*}
   d^{\piail}_{h} (s^\prime)  > 0, \text{ and } \sP^{\piail} \lp s_{h^\prime} = s |s_h = s^\prime, a_h = a^{1} \rp > 0.
    \end{align*}
    Then we can obtain that $A > 0$ where $>$ means element-wise comparison. Besides, on the one hand, we have that
    \begin{align*}
        \sum_{s \in \goodS} c (s) = \sum_{s \in \goodS} \widehat{d^{\piE}_{h^\prime}} (s)  = 1.
    \end{align*}
    On the other hand, we have that
    \begin{align*}
    &\quad \sum_{s \in \goodS} \sum_{s^\prime \in \goodS} A (s, s^\prime)
    \\
    &= \sum_{s \in \goodS} \sum_{s^\prime \in \goodS} d^{\piail}_h (s^\prime)  \sP^{\piail} \lp s_{h^\prime} = s |s_h = s^\prime, a_h = a^{1} \rp
    \\
    &\leq \sum_{s^\prime \in \goodS} d^{\piail}_h (s^\prime)
    \\
    &\leq 1.
    \end{align*}
    To summarize, we obtain 
    \begin{align*}
        \sum_{s \in \goodS} \sum_{s^\prime \in \goodS} A (s, s^\prime) \leq \sum_{s \in \goodS} c (s). 
    \end{align*}
	    For each $s^\prime \in \goodS$, we further have that 
	    \begin{align*}
	        \sum_{s \in \goodS} A (s, s^\prime) &= \sum_{s \in \goodS} d^{\piail}_h (s^\prime)  \sP^{\piail} \lp s_{h^\prime} = s |s_h = s^\prime, a_h = a^{1} \rp = d^{\piail}_h (s^\prime) = d(s^\prime). 
	    \end{align*}
    In the penultimate equation, we utilize the argument that $\sum_{s \in \goodS} \sP^{\piail} \lp s_{h^\prime} = s |s_h = s^\prime, a_h = a^{1} \rp = 1$. This is because that conditioned on $s_h=s^\prime, a_h=a^{1}$, taking the policy $\piail$ only visits good states due to the assumption that for each $h+1 \leq \ell \leq H-1$, $\piail_{\ell} (a^1|s) = 1, \forall s \in \goodS$ in the induction step. Thus, we have verified the conditions in \cref{lem:mn_variables_opt_unique}. By Lemma \ref{lem:mn_variables_opt_unique}, we obtain that $\pi_{h} (a^{1}|s) = \piE_{h} (a^{1}|s) = 1, \forall s \in \goodS$ is the {unique} optimal solution of $\mathrm{Loss}_{h^\prime}$ for each time step $h^\prime$, where $h+1 \leq h^\prime \leq H-1$.

    Finally, we consider $\Loss_{H}$. Recall the definition that $\gV_{H} = \{s \in \gS, \widehat{d^{\piE}_H} (s) > 0  \}$. Then we derive
    \begin{align*}
      \Loss_{H} &= \sum_{s \in \gS} \sum_{a \in \gA} \labs \widehat{d^{\piE}_{H}} (s, a) - d^{\pi}_H (s, a) \rabs
        \\
        &= \sum_{s \in \goodS} \sum_{a \in \gA} \labs \widehat{d^{\piE}_{H}} (s, a) - d^{\pi}_H (s, a) \rabs + \sum_{s \in \badS} d^{\pi}_H (s) 
        \\
        &= \underbrace{\sum_{s \in \gV_{H}} \labs \widehat{d^{\piE}_{H}} (s) - d^{\pi}_H (s, a^{1}) \rabs}_{\text{Term I}} + \underbrace{ \sum_{s \in \gV_{H}} \sum_{a \not= a^1} d^{\pi}_H (s, a)}_{\text{Term II}}  + \underbrace{ \sum_{s \in \goodS \text{ and } s \notin \gV_{H}} d^{\pi}_H (s) }_{\text{Term III}} + \underbrace{\sum_{s \in \badS} d^{\pi}_H (s)}_{\text{Term IV}}. 
    \end{align*}
    Here $d^{\pi}_H (s, a)$ and $d^{\pi}_H (s)$ are induced by $(\piail_1, \ldots, \piail_{h-1}, \pi_h, \piail_{h+1}, \ldots, \piail_H)$. With the Bellman-flow equation in \eqref{eq:flow_link}, we have $\forall s \in \goodS$, 
    \begin{align*}
          &\quad d^{\pi}_H (s)
          \\
          &= \sum_{s^\prime \in \gS} \sum_{a \in \gA} d^{\piail}_{h} (s^\prime) \pi_{h} (a|s^\prime) \sP^{\piail} (s_H = s | s_h = s^\prime, a_h = a)
        \\
        &= \sum_{s^\prime \in \goodS} d^{\piail}_{h} (s^\prime) \pi_{h} (a^{1}|s^\prime) \sP^{\piail} (s_H = s | s_h = s^\prime, a_h = a^{1}).  
    \end{align*}
    Here $d^{\piail}_{h} (s^\prime)$ is induced by $(\piail_1, \ldots, \piail_{h-1})$. Besides, $\sP^{\piail} (s_H = s | s_h = s^\prime, a_h = a^{1})$ refers to the transition probability of $s$ in time step $H$ by starting from $(s^\prime, a^{1})$ in time step $h$ via policy $\piail_{h+1}, \ldots, \piail_{H-1}$.
    Accordingly, we have 
    \begin{align*}
         \text{Term I}
     = \sum_{s \in \gV_{H}} \bigg\vert \widehat{d^{\piE}_{H}} (s) - \sum_{s^\prime \in \goodS} d^{\piail}_{h} (s^\prime) \cdot \sP^{\piail} (s_H = s | s_h = s^\prime, a_h = a^{1}) \piail_H (a^{1}|s) \pi_h (a^{1}|s^\prime)    \bigg\vert, 
    \end{align*}
    and 
    \begin{align*}
         \text{Term II}= \sum_{s^\prime \in \goodS} \pi_h (a^1|s^\prime) \bigg( \sum_{s \in \gV_H} d^{\piail}_h (s^\prime) \cdot \sP^{\piail} (s_H = s | s_h = s^\prime, a_h = a^{1}) \lp 1-\piail_{H} (a^1|s) \rp  \bigg),
    \end{align*}
    and
    \begin{align*}
         \text{Term III}
      = \sum_{s^\prime \in \goodS} \pi_h (a^1|s^\prime) \cdot \lp \sum_{s \in \goodS \text{ and } s \notin \gV_{H}} d^{\piail}_{h} (s^\prime) \sP^{\piail} (s_H = s | s_h = s^\prime, a_h = a^{1}) \rp.
    \end{align*}
    With the assumption that for time step $h^\prime = h+1, h+2, \cdots, H-1$, $\piail_{h^\prime} (a^{1}|s) = 1, \forall s \in \goodS$, the visitation probability of bad states in step $H$ comes from two parts in step $h$. By a similar argument with the previous analysis of $\Loss_{h^\prime}$, we get 
    \begin{align*}
        \text{Term IV} &= \sum_{s \in \badS} d^{\piail}_h (s) + \sum_{s^\prime \in \goodS} \sum_{a \not= a^1} d^{\piail}_h (s^\prime)  \pi_{h} (a|s^\prime) = \sum_{s \in \badS} d^{\piail}_h (s) + \sum_{s^\prime \in \goodS} d^{\piail}_h (s^\prime) \lp 1 - \pi_{h} (a^{1}|s^\prime) \rp.
    \end{align*}
   Plugging $\text{Term I, II, III, and IV}$ into $\Loss_{H}$ yields that
    \begin{align*}
        &\quad \Loss_{H}
        \\
        &= \sum_{s \in \gV_{H}} \bigg\vert \widehat{d^{\piE}_{H}} (s) - \sum_{s^\prime \in \goodS} d^{\piail}_h (s^\prime) \sP^{\piail} (s_H = s | s_h = s^\prime, a_h = a^{1}) \piail_H (a^{1}|s) \pi_h (a^{1}|s^\prime)    \bigg\vert 
        \\
        &\; - \sum_{s^\prime \in \goodS} \pi_h (a^{1}|s^\prime) \cdot \bigg( \sum_{s \in \gV_{H}} d^{\piail}_h (s^\prime) \sP^{\piail} (s_H = s | s_h = s^\prime, a_h = a^{1}) \piail_H (a^{1}|s) \bigg) + \constant,
    \end{align*}
    where $\constant = \sum_{s \in \badS} d^{\piail}_h (s) + \sum_{s^\prime \in \goodS} d^{\piail}_h (s^\prime)$, which is independent of $\pi_h$. This equation is similar to \eqref{eq:proof_vail_reset_cliff_1} in the proof of the base case. Then we have that
    \begin{align*}
        & \quad \argmin_{\pi_h} \Loss_{H}
        \\
        &= \argmin_{\pi_h} \bigg\{ \sum_{s \in \gV_{H}} \bigg\vert \widehat{d^{\piE}_{H}} (s) - \sum_{s^\prime \in \goodS} d^{\piail}_h (s^\prime) \sP^{\piail} (s_H = s | s_h = s^\prime, a_h = a^{1}) \piail_H (a^{1}|s) \pi_h (a^{1}|s^\prime)    \bigg\vert 
        \\
        & \; - \sum_{s^\prime \in \goodS} \pi_h (a^{1}|s^\prime) \cdot \bigg( \sum_{s \in \gV_{H}} d^{\piail}_h (s^\prime) \sP^{\piail} (s_H = s | s_h = s^\prime, a_h = a^{1}) \piail_H (a^{1}|s) \bigg) \bigg\}.
    \end{align*}
    For this optimization problem, we can again use \cref{lem:mn_variables_opt_unique} to prove that $\forall s \in \goodS, \pi_h(a^{1} | s) = \piE_h (a^{1}|s) = 1$ is the unique optimal solution. To check the conditions in \cref{lem:mn_variables_opt_unique}, we define 
	    \begin{align*}
	        & m = \labs \gV_{H} \rabs, n = \labs \goodS \rabs, \forall s \in \gV_{H}, c(s) = \widehat{d^{\piE}_{H}} (s), 
	        \\
	        & \forall s \in \gV_{H}, s^\prime \in \goodS, A (s, s^\prime) = d^{\piail}_h (s^\prime) \sP^{\piail} (s_H = s | s_h = s^\prime, a_h = a^{1}) \piail_H (a^{1}|s),
	        \\
	        & \forall s^\prime \in \goodS, d(s^\prime) =  \sum_{s \in \gV_{H}} d^{\piail}_h (s^\prime) \sP^{\piail} (s_H = s | s_h = s^\prime, a_h = a^{1}) \piail_H (a^{1}|s). 
	    \end{align*}
    Following the same argument in the proof of the base case, we have that $A >0$. On the one hand, we have 
    \begin{align*}
        \sum_{s \in \gV_{H}} c(s) = \sum_{s \in \gV_{H}} \widehat{d^{\piE}_{H}} (s) = 1. 
    \end{align*}
    On the other hand,
	    \begin{align*}
	        \sum_{s \in \gV_{H}} \sum_{s^\prime \in \goodS} A (s, s^\prime) &\leq \sum_{s \in \gV_{H}} \sum_{s^\prime \in \goodS} d^{\piail}_h (s^\prime) \sP^{\piail} (s_H = s | s_h = s^\prime, a_h = a^{1}) \leq \sum_{s^\prime \in \goodS} d^{\piail}_h (s^\prime) \leq 1.
	    \end{align*}
    Therefore, we obtain
    \begin{align*}
        \sum_{s \in \gV_{H}} \sum_{s^\prime \in \goodS} A (s, s^\prime) \leq \sum_{s \in \gV_{H}} c(s). 
    \end{align*}
    We further obtain $\forall s^\prime \in \goodS$,
    \begin{align*}
        \sum_{s \in \gV_{H}} A (s, s^\prime) &= \sum_{s \in \gV_{H}} d^{\piail}_h (s^\prime) \sP^{\piail} (s_H = s | s_h = s^\prime, a_h = a^{1}) \piail_H (a^{1}|s) = d(s^\prime).
    \end{align*}
    Thus we have verified the conditions in \cref{lem:mn_variables_opt_unique}. By \cref{lem:mn_variables_opt_unique}, we obtain $\pi_h(a^{1} | s) = \piE_h(a^{1} | s)  = 1, \forall s \in \goodS$ is the unique optimal solution of $\min_{\pi_h} \Loss_{H}$. Therefore, we prove that $\pi_h(a^{1} | s) = \piE_h(a^{1} | s)  = 1, \forall s \in \goodS$ is the unique optimal solution of $\argmin_{\pi_h} \Loss_h + \sum_{h^\prime=h+1}^H \Loss_{h^\prime}$. Thus, we finish the induction proof and the whole proof is done. 

\end{proof}

\subsection{Proof of Proposition \ref{prop:non_convex}}
\label{appendix:proof_prop:non_convex}

\begin{figure}[t]
    \centering
    \includegraphics[width=\linewidth]{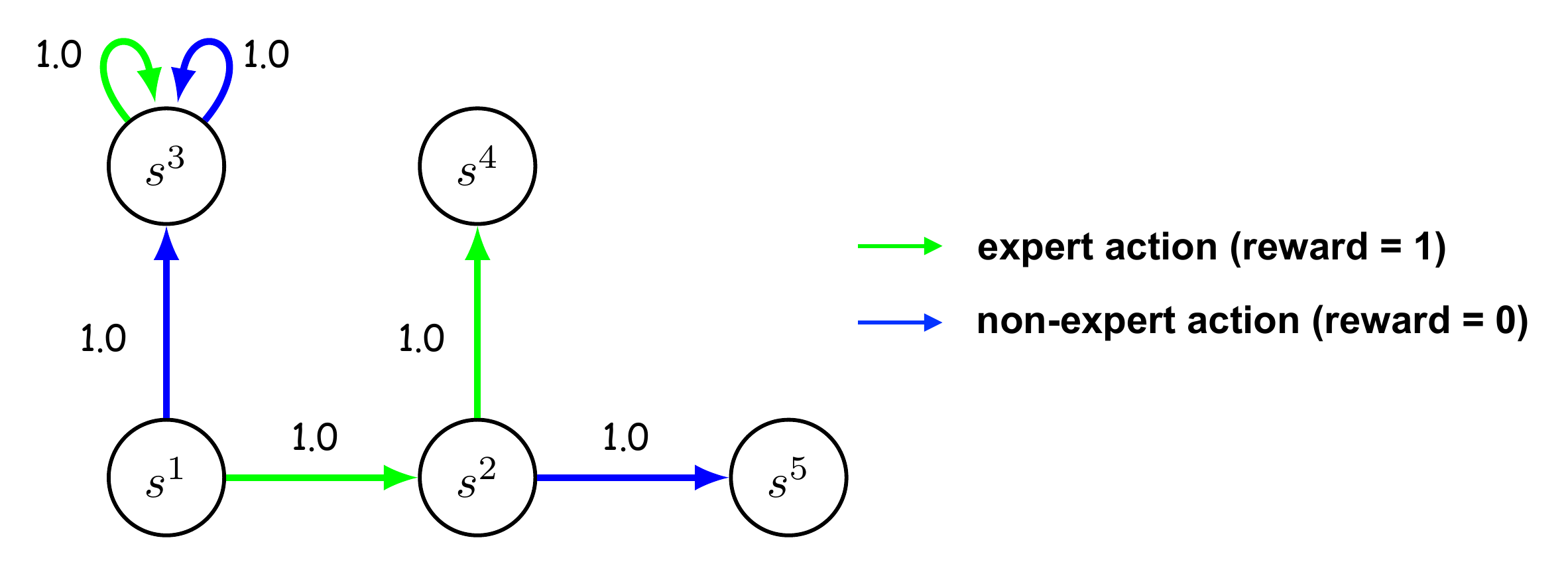}
    \caption{A simple example to show that {TV-AIL}'s objective in \eqref{eq:ail} is non-convex.}
    \label{fig:non_convex}
\end{figure}

\begin{proof}[Proof of \cref{prop:non_convex}]
To prove that the objective in \eqref{eq:ail} could be non-convex, we provide an instance shown in \cref{fig:non_convex}. In particular, there are 5 states $(s^{1}, s^{2}, s^{3}, s^{4}, s^{5})$ and two actions $(\GREEN{a}, \BLUE{a})$. Each arrow shows a deterministic transition. The initial state is $s^1$ and decision horizon is $2$. Assume  $\GREEN{a}$ is the expert action and there is only one expert trajectory: $(s^{1}, \GREEN{a}) \rar (s^{2}, \BLUE{a})$. We can calculate the empirical distribution:
\begin{align*}
    \widehat{d^{\piE}_{1}}(s^{1}, \GREEN{a}) = 1, \quad  \widehat{d^{\piE}_{2}}(s^{2}, \BLUE{a}) = 1.
\end{align*}
Let us use the following notations: $x := \pi_1(a^{1} | s^{1})$ and $y := \pi_2(a^{1} | s^{2})$. In time step $h=1$, we have 
\begin{align*}
    &\quad \Loss_1 \\
    &= \sum_{(s, a)} \labs d^{\pi}_1(s, a) - \widehat{d^{\piE}_1}(s, a) \rabs \\
    &= \labs d^{\pi}_{1}(s^{1}, \GREEN{a}) - \widehat{d^{\piE}_{1}} (s^{1}, \GREEN{a}) \rabs + \labs d^{\pi}_{1} (s^{1}, \BLUE{a}) - \widehat{d^{\piE}_1}(s^{1}, \BLUE{a}) \rabs \\
    &= \labs x - 1 \rabs + \labs 1-x - 0 \rabs = 2(1 - x).
\end{align*}
In time step $h=2$, we have 
\begin{align*}
    &\quad \Loss_2 \\
    &= \sum_{(s, a)} \labs d^{\pi}_2(s, a) - \widehat{d^{\piE}_2}(s, a) \rabs \\
    &= \labs d^{\pi}_{2}(s^{2}, \GREEN{a}) - \widehat{d^{\piE}_2}(s^{2}, \GREEN{a}) \rabs +  \labs d^{\pi}_{2}(s^{2}, \BLUE{a}) - \widehat{d^{\piE}_2}(s^{2}, \BLUE{a}) \rabs + \labs d^{\pi}_{2}(s^{3}, \GREEN{a}) - \widehat{d^{\piE}_2}(s^{3}, \GREEN{a}) \rabs + \labs d^{\pi}_{2}(s^{3}, \BLUE{a}) - \widehat{d^{\piE}_2}(s^{3}, \BLUE{a}) \rabs \\
    &= \labs xy - 1 \rabs + \labs x(1-y) - 0 \rabs + (1-x)  \\
    &= 1 - xy + x - xy + 1 - x = 2(1 - xy).
\end{align*}
Thus, we have that 
\begin{align*}
    f(x, y) = \Loss_1 + \Loss_2 = 2 \lp 2 - x - xy \rp.
\end{align*}
Furthermore, we can compute that 
    \begin{align*}
    \nabla f(x, y) = \bvec{-2 - 2y \\ -2x}, \quad \nabla^2 f(x, y) = \bvec{0 & -2 \\ -2 & 0}.
    \end{align*}
Since $\nabla^2 f(x, y)$ is not a PSD, we claim that $f(x, y)$ is non-convex w.r.t. $(x, y)$.

\end{proof}

\subsection{An Example for TV-AIL in RABS MDPs}
\label{appendix:proof_claim_example_ail_success}

To gain a deeper understanding of the horizon-free imitation gap and the coupling structure in the state-action distribution matching, let us examine the following example.

\begin{example} \label{example:ail_success}

Consider the mentioned MDP shown in \cref{fig:toy_reset_cliff}. Furthermore, assume that 
the agent is provided with 2 expert trajectories: $\tr^1 = (s^{1}, \GREEN{a}) \rar (s^{1}, \GREEN{a}) $ and $\tr^2 = (s^{1}, \GREEN{a}) \rar (s^{2}, \GREEN{a})$, where $\GREEN{a}$ is the expert action.

Let us first study the performance of BC. According to \eqref{eq:pi_bc}, we find that BC exactly recovers the expert action except that it poses a uniform policy on the {non-visited} $s^{2}$ in time step $h = 1$. As a result, BC makes a mistake with probability $\rho(s^{2}) \cdot \pi_1^{\bc}(\BLUE{a} | s^{2}) = 0.25$. Accordingly, its imitation gap is $0.25 \cdot 2 = 0.5$.

For {TV-AIL}, it makes sense to guess that the expert action is recovered on visited states (otherwise, it incurs a state-action distribution matching loss). We argue that {TV-AIL} exactly recovers the expert action even on the {non-visited} state $s^{2}$ in time step $h=1$. Consequently, the imitation gap of {TV-AIL} is 0, which is smaller than BC. Here we mainly explain intuition. For the formal proof, please refer to Appendix \ref{appendix:proof_of_prop:ail_general_reset_cliff}.

Assume that TV-AIL takes the expert action on visited states and let $\pi_{1}(\BLUE{a} | s^{2}) = 1 - \beta$, where $\beta \in [0, 1]$. We note that a positive $\beta$ makes no difference for the loss function in time step $h = 1$, since 
\begin{align*}
&\quad |d^{\pi}_1(s^{2}, \GREEN{a}) - \widehat{d^{\piE}_1}(s^{2}, \GREEN{a}) | + |d^{\pi}_1(s^{2}, \BLUE{a}) - \widehat{d^{\piE}_1}(s^{2}, \BLUE{a}) |
\\
&= |d^{\pi}_1(s^{2}, \GREEN{a}) - 0 | + |d^{\pi}_1(s^{2}, \BLUE{a}) - 0 | = \rho(s^{2}).
\end{align*}
However, it matters for the loss function in time step $h=2$. By \eqref{eq:flow_link}, we can compute the state-action distribution in time step $h=2$: 
\begin{align*}
   d^{\pi}_2(s^{1}, \GREEN{a}) = 0.25(1+\beta), d^{\pi}_2(s^{1}, \BLUE{a}) = 0, \quad d^{\pi}_2(s^{2}, \GREEN{a}) = 0.25(1+\beta) , d^{\pi}_2(s^{2}, \BLUE{a}) = 0, \quad d^{\pi}_2(b, \BLUE{a}) = 0.5(1-\beta).
\end{align*}
Then the state-action distribution matching loss becomes
\begin{align*}
    \text{Loss} (\beta) &= \sum_{h=1}^{2} \sum_{(s, a)} | d^{\pi}_h(s, a) - \widehat{d^{\piE}_h}(s, a) | \\
    &= 1.0 +  \sum_{(s, a)}  | d^{\pi}_2(s, a) - \widehat{d^{\piE}_2}(s, a) | \\
    &= 1.0 + 2 |0.25(1+\beta)  - 0.5| + |0.5(1-\beta) - 0| \\
    &= 2 - \beta,
\end{align*}
which has a unique globally optimal solution at $\beta = 1$. In plain language, if the agent selects a wrong action in the first time step, it may go to the bad absorbing state in the second time step. This results in a large loss because the expert policy never visits the bad absorbing state. Therefore, to minimize the cumulative state-action distribution matching losses, TV-AIL has to select the action that can avoid the bad status.

\end{example}

\subsection{Discussion on the Difference with the Dynamic-programming-based Proof} 
\label{appendix:difference_with_the_dp_analysis}

In this part, we elaborate on the difference between our proof and the dynamic-programming-based proof. Our proof of \cref{prop:ail_general_reset_cliff} utilizes RBAS MDPs properties directly to characterize the optimal policies and does not require forward substitution as in the direct dynamic programming (DP) technique. In contrast, the DP proof computes a functional by backward induction and then uses forward substitution to find the optimal policy.

Specifically, the backward induction of DP's proof computes a functional by solving the following \dquote{cost-to-go} minimization problem \citep{bertsekas2012dynamic}:
\begin{align*}
    \pi^{\cdp}_{h} &\in \argmin_{\pi_h}  f^{\cdp}_h (\pi_h; \pi_1, \ldots, \pi_{h-1}, \pi_{h+1}, \ldots, \pi_{H+1})
\end{align*}
with 
\begin{align*}
      f^{\cdp}_h &=  \sum_{(s, a)} \labs d^{\pi}_h(s, a) - \widehat{d^{\piE}_h}(s, a) \rabs + \sum_{h^\prime = h + 1}^{H} \sum_{(s, a)} \labs d^{\pi}_{h^\prime} (s, a) - \widehat{d^{\piE}_{h^\prime}}(s, a) \rabs + \constant. 
\end{align*}
Here $f^{\cdp}_h$ is conditioned on an arbitrary choice of $(\pi_1, \ldots, \pi_{h-1})$. Thus, the resultant $\pi^{\cdp}_h$ is a functional. Instead, the objective in \eqref{eq:main_text_3} relies on the optimal solution of $(\piail_1, \ldots, \piail_{h-1})$. This difference matters in two aspects. One the one hand, when characterizing the optimal solution to $\min_{\pi_h} f_h$, we use the condition in \eqref{eq:visit_good_states} to argue that $d^{\piail}_h(s) > 0$ for all $s \in \goodS$. This auxiliary property facilitates later arguments in Steps (I) and (II). In contrast, the DP proof cannot use this property as $\pi_1, \ldots, \pi_{h-1}$ are chosen arbitrarily. On the other hand, the resultant optimal policy to minimizing $f_h$ in our proof is an optimal solution set, while the DP proof computes a functional.

\subsection{Proof of Theorem \ref{theorem:ail_reset_cliff}}

\begin{proof}[Proof of \cref{theorem:ail_reset_cliff}]

According to \cref{eq:value_dual_representation}, we have that  
\begin{align*}
\labs  V({\piE}) - V({\piail}) \rabs &= \bigg\vert \sum_{h=1}^{H}  \sum_{(s, a) }  d^{\piE}_h(s, a) r_h(s, a) - d^{\piail}_h(s, a) r_h(s, a) \bigg\vert.
\end{align*}
From \cref{prop:ail_general_reset_cliff}, we have that for any $h \in [H-1]$, $\piail_{h} (a^{1}|s) = \piE_h (a^{1}|s) = 1, \forall s \in \goodS$. Therefore, $\piE$ and $\piail$ never visit bad states and for any $h \in [H-1]$, $d^{\piE}_h(s, a) = d^{\piail}_h(s, a)$. As a result, the policy value gap is upper bounded by the state-action distribution discrepancy in the last time step.
\begin{align*}
 \labs  V({\piE}) - V({\piail}) \rabs &=  \bigg\vert \sum_{(s, a) } d^{\piE}_H (s, a) r_{H}(s, a) - d^{\piail}_H (s, a) r_{H}(s, a) \bigg\vert.
 \end{align*}
On the one hand, since $r_{H} (s, a) 
\in [0, 1]$, we have that
\begin{align*}
    \bigg\vert \sum_{(s, a) } d^{\piE}_H (s, a) r_{H}(s, a) - d^{\piail}_H (s, a) r_{H}(s, a) \bigg\vert \leq 1.
\end{align*}
On the other hand, by triangle inequality, we have
 \begin{align*}
  \bigg\vert \sum_{(s, a) } d^{\piE}_H (s, a) r_{H}(s, a) - d^{\piail}_H (s, a) r_{H}(s, a) \bigg\vert
\leq \sum_{(s, a) } \labs d^{\piE}_H (s, a) - \widehat{d^{\piE}_H} (s, a) \rabs + \sum_{(s, a)} \labs \widehat{d^{\piE}_H} (s, a) - d^{\piail}_H (s, a)  \rabs.
\end{align*}
Therefore, we derive that
\begin{align*}
    \labs  V({\piE}) - V({\piail}) \rabs  \leq \min \bigg\{ 1, \sum_{(s, a) } \labs d^{\piE}_H (s, a) - \widehat{d^{\piE}_H} (s, a) \rabs + \sum_{(s, a)} \labs \widehat{d^{\piE}_H} (s, a) - d^{\piail}_H (s, a)  \rabs. \bigg\}
\end{align*}
In the following part, we want to prove that the optimality of $\piail$ implies that $\sum_{(s, a)} \vert \widehat{d^{\piE}_H} (s, a) - d^{\piail}_H (s, a)  \vert \leq \sum_{(s, a)} \vert d^{\piE}_H (s, a) - \widehat{d^{\piE}_H} (s, a) \vert$.

By \cref{lem:n_vars_opt_greedy_structure}, it holds that
\begin{align*}
    \piail_{H} &\in \argmin_{\pi_H} \sum_{(s, a) } \labs \widehat{d^{\piE}_H} (s, a) - d^{\pi}_H (s, a) \rabs,
\end{align*}
where $d^{\pi}_{H}$ is computed by $(\piail_1, \ldots, \piail_{H-1}, \pi_H)$. From \cref{prop:ail_general_reset_cliff}, we know that $(\piail_1, \ldots, \piail_{H-1}) = \\ (\piE_1, \ldots, \piE_{H-1})$ and thus $d^{\piail}_H (s) = d^{\piE}_H (s), \forall s \in \gS$.
Then, we arrive at 
\begin{align*}
    \piail_{H}  &\in \argmin_{\pi_H} \sum_{(s, a) } \labs \widehat{d^{\piE}_H} (s, a) - d^{\piE}_H (s, a) \rabs, 
\end{align*}
This implies that
\begin{align*}
    \sum_{(s, a) } \labs \widehat{d^{\piE}_H} (s, a) - d^{\piail}_H (s, a)  \rabs &\leq \sum_{(s, a) } \labs \widehat{d^{\piE}_H} (s, a) - d^{\piE}_H(s, a)  \rabs.
\end{align*}
Then we obtain
\begin{align}
&\quad \labs  V({\piE}) - V({\piail}) \rabs  \nonumber   \\
&\leq \min\bigg\{ 1,  2 \sum_{(s, a) } \labs d^{\piE}_H (s, a) - \widehat{d^{\piE}_H} (s, a) \rabs \bigg\} \nonumber \\
&= \min \bigg\{ 1, 2\sum_{s \in \gS} \labs d^{\piE}_H (s) - \widehat{d^{\piE}_H} (s)  \rabs \bigg\}. \label{eq:vail_reset_cliff_value_gap_last_step_estimation_error} 
\end{align}
Finally, we apply \citep[Theorem 1]{han2015minimax} to upper bound the estimation error in \eqref{eq:vail_reset_cliff_value_gap_last_step_estimation_error}:
\begin{align*}
    &\quad \expect \ls \labs  V({\piE}) - V({\piail}) \rabs \rs  \\
    &\leq \min \bigg\{ 1, 2 \expect \ls \lnorm \widehat{d^{\piE}_H} (\cdot) - d^{\piE}_H (\cdot) \rnorm_{1} \rs \bigg\} \\
    &\leq \min \bigg\{1, 2 \sqrt{ \frac{|\gS| - 1}{N}} \bigg\}.
\end{align*}
\end{proof}

\subsection{Proof of Theorem \ref{theorem:ail_approximate_reset_cliff} }
\label{appendix:proof_theorem:ail_approximate_reset_cliff}

To prove \cref{theorem:ail_approximate_reset_cliff}, we need to re-build the optimality condition by a stage-coupled analysis. This part is stated in \cref{prop:ail_general_reset_cliff_approximate_solution}. Before we present \cref{prop:ail_general_reset_cliff_approximate_solution}, we mention a useful property of \textsf{TV-AIL}'s objective on RBAS MDPs.

\begin{lem}
\label{lemma:ail_policy_ail_objective_equals_expert_policy_ail_objective}
Consider RBAS MDPs satisfying \cref{asmp:reset_cliff}. Suppose that $\piail$ is an optimal solution of \eqref{eq:ail}, then  $\piail$ and $\piE$ achieve the same loss, i.e., $\sum_{h=1}^{H} \Vert d^{\piail}_h - \widehat{d^{\piE}_h} \Vert_1 = \sum_{h=1}^{H} \Vert d^{\piE}_h - \widehat{d^{\piE}_h} \Vert_1$.
\end{lem}
Refer to Appendix \ref{appendix:proof_lemma:ail_policy_ail_objective_equals_expert_policy_ail_objective} for the proof. Note that $\piail$ may be different with $\piE$ in the last time step on RBAS MDPs.  The following proposition demonstrates that the distance between the approximately optimal solution and the exactly optimal solution can be properly controlled. 

\begin{prop}
\label{prop:ail_general_reset_cliff_approximate_solution}
For each tabular and episodic MDP satisfying \cref{asmp:reset_cliff}, define the candidate policy set $\Pi^{\opt} = \{ \pi \in \Pi: \forall h \in [H], \exists s \in \goodS, \pi_h (a^{1}|s) > 0 \}$. Given the expert state-action distribution estimation $\widehat{d^{\piE}_{H}}$, suppose that $\widebar{\pi} \in \Pi^{\opt}$ is an $\varepsilon$-optimal solution (refer to \cref{defn:approximate_solution}). For any $N\geq 1$, we have the following approximate optimality condition almost surely:
\begin{align}   \label{eq:approximate_optimality_condition}
&\quad c (\widebar{\pi}) \bigg[ \sum_{h=1}^{H} \sum_{\ell=1}^{h-1} \sum_{s \in \goodS} d^{\widebar{\pi}}_{\ell} (s) \lp 1 - \widebar{\pi}_{\ell} (a^{1}|s)  \rp + \sum_{s \in  \gS^{\widebar{\pi}}_H}  d^{\widebar{\pi}}_{H} (s) \big( \min \big\{ 1,  \frac{\widehat{d^{\piE}_{H}} (s)}{d^{\piE}_H (s)} \big\} - \widebar{\pi}_H (a^1|s)  \big)  \bigg] \leq \varepsilon,
\end{align}
where $c(\widebar{\pi}) > 0$ is defined as 
\begin{align*}
     c (\widebar{\pi}) := \min_{1 \leq \ell < h \leq H, s, s^\prime \in \goodS} \lb \sP^{\widebar{\pi}} \lp s_{h} = s |s_{\ell} = s^\prime, a_{\ell} = a^{1} \rp \rb.
\end{align*}
Here $\sP^{\widebar{\pi}} \lp s_{h} = s |s_{\ell} = s^\prime, a_{\ell} = a^{1} \rp $ is the visitation probability of $s$ in time step by starting from $s^\prime, a^\prime$ in time step $\ell$, which is jointly determined by the transition function and policy $\widebar{\pi}$. In addition,  
\begin{align*}
    \gS^{\widebar{\pi}}_H  = \lb s \in \goodS: \widebar{\pi}_H (a^1|s) \leq \min \{1,  \widehat{d^{\piE}_{H}} (s) / d^{\piE}_H (s) \}  \rb.
\end{align*}
\end{prop}

Proof of \cref{prop:ail_general_reset_cliff_approximate_solution} is rather technical and is deferred to Appendix \ref{appendix:proof_prop:ail_general_reset_cliff_approximate_solution}. We explain \cref{prop:ail_general_reset_cliff_approximate_solution} by connecting it with \cref{prop:ail_general_reset_cliff}. In particular, if $\varepsilon = 0$, we can show that the optimality condition in  \cref{prop:ail_general_reset_cliff_approximate_solution} reduces to that in \cref{prop:ail_general_reset_cliff}. To see this, for each $h \in [H-1]$ and $s \in \goodS$, since $c(\widebar{\pi}) > 0$ and $d^{\widebar{\pi}}_h (s) > 0$, we must have $\widebar{\pi}_{h}(a^{1} | s) = 1$ for all $h \in [H-1]$, while there may exist many optimal solutions in the last step policy optimization (corresponding to the second term in \eqref{eq:approximate_optimality_condition}).

Equipped with \cref{prop:ail_general_reset_cliff_approximate_solution}, we can obtain the horizon-free sample complexity for the approximately optimal solution of {TV-AIL} in \cref{theorem:ail_approximate_reset_cliff}.

\begin{proof}[Proof of \cref{theorem:ail_approximate_reset_cliff}]
Given the estimation $\widehat{d^{\piE}}$, we consider {TV-AIL}'s objective.
\begin{align*}
    \min_{\pi} \sum_{h=1}^{H} \sum_{(s, a) } \labs d^{\pi}_h(s, a) - \widehat{d^{\piE}_h}(s, a) \rabs.
\end{align*}
Suppose $\widebar{\pi}$ is an $\varepsilon$-optimal solution (refer to \cref{defn:approximate_solution}). First,  we construct an optimal solution $\piail$ based on \cref{prop:ail_general_reset_cliff}:
\begin{itemize}
    \item By \cref{prop:ail_general_reset_cliff}, we have $\forall h \in [H-1], s \in \goodS, \piail_{h} (a^{1}|s) = \piE_{h} (a^{1}|s)$.
    
    \item For the last time step $H$, we defined a set of states $\gS_H^1 := \{s \in \goodS: \widehat{d^{\piE}_H} (s) < d^{\piE}_H (s)   \}$. The policy in the last time step is defined as $\forall s \in \gS_H^1$, $\piail_{H} (a^{1}|s) = \widehat{d^{\piE}_H} (s) / d^{\piE}_H (s)$ and $\forall s \in \goodS \setminus \gS_H^1$, $\piail_{H} (a^{1}|s) = 1$. In a word, $\forall s \in \goodS$, $\piail_{H} (a^{1}|s) = \min \{\widehat{d^{\piE}_H} (s) / d^{\piE}_H (s),1  \}$.
    
    \item For simplicity of analysis, we also define the policy on bad states,  although $\piail$ never visit bad states. $\forall h \in [H], s \in \badS, \piail_{h} (\cdot|s) = \widebar{\pi}_{h} (\cdot|s)$.   
\end{itemize}

Next, we verify that such defined $\piail$ is an optimal solution of \eqref{eq:ail}. According to \cref{prop:ail_general_reset_cliff}, it suffices to show that $\pi_H$ achieves the optimality. According to \cref{lem:n_vars_opt_greedy_structure}, we need to prove that 
\begin{align*}
    \piail_{H} \in \min_{\pi_H} \sum_{(s, a)} \labs d^{\pi}_{H}(s, a) - \widehat{d^{\piE}_H}(s, a) \rabs,
\end{align*}
where $d^{\pi}_{H}$ is computed by $(\piail_1, \ldots, \piail_{H-1})$. In fact, the argument here is the same with that in \cref{proposition:ail_policy_value_gap_standard_imitation}. Hence, we omit details.

Now we consider the imitation gap of $\widebar{\pi}$.
\begin{align}
\label{eq:value_gap_decomposition_approximate_solution}
    V({\piE}) - V({\widebar{\pi}}) = V({\piE}) - V({\piail}) + V({\piail}) -  V({\widebar{\pi}}). 
\end{align}
By \eqref{eq:vail_reset_cliff_value_gap_last_step_estimation_error} in the proof of \cref{theorem:ail_reset_cliff}, we have that
\begin{align}
   V({\piE}) - V({\piail}) \leq \min \bigg\{ 1, 2\sum_{s \in \gS} \labs d^{\piE}_H (s) - \widehat{d^{\piE}_H} (s)  \rabs \bigg\}.\label{eq:ail_policy_piE_value_gap}
\end{align}
Then we consider the policy value gap between $\piail$ and $\widebar{\pi}$. With the dual form of policy value in \eqref{eq:value_dual_representation}, we get that
\begin{align}   
&\quad V({\piail}) -  V({\widebar{\pi}}) \nonumber \\
&= \sum_{h=1}^{H-1} \sum_{(s, a)} \lp d^{\piail}_h (s, a) - d^{\widebar{\pi}}_h (s, a)  \rp r_h (s, a) \nonumber + \sum_{(s, a)} \lp d^{\piail}_H (s, a) - d^{\widebar{\pi}}_H (s, a)  \rp r_H (s, a)
    \nonumber \\
    &\leq \sum_{h=1}^{H-1} \lnorm d^{\piail}_h(\cdot, \cdot) - d^{\widebar{\pi}}_h(\cdot, \cdot)  \rnorm_{1} + \sum_{(s, a)} \lp d^{\piail}_H (s, a) - d^{\widebar{\pi}}_H (s, a)  \rp r_H (s, a), \label{eq:approximate_solution_step_1}
\end{align}
where we use $d^{\pi}_h (\cdot, \cdot)$ to explicitly denote the state-action distribution induced by $\pi$. Recall the definition of $\gS^{\widebar{\pi}}_H  = \{s \in \goodS: \widebar{\pi}_H (a^1|s) \leq \min \{1,  \widehat{d^{\piE}_{H}} (s) / d^{\piE}_H (s) \}  \}$ introduced in \cref{prop:ail_general_reset_cliff_approximate_solution}. For the second term in \eqref{eq:approximate_solution_step_1}, we have
\begin{align*}
 &\quad \sum_{(s, a)} \lp d^{\piail}_H  (s, a) - d^{\widebar{\pi}}_H (s, a)  \rp r_H (s, a) \\
 &= \sum_{s \in \goodS}  \lp d^{\piail}_H  (s, a^{1}) - d^{\widebar{\pi}}_H (s, a^{1})  \rp
    \\
    &= \sum_{s \in \goodS}  \lp d^{\piail}_H  (s) \piail_{H} ( a^{1} | s) - d^{\widebar{\pi}}_H (s) \widebar{\pi}_H (a^{1}|s)  \rp
    \\
    &=  \sum_{s \in \goodS}   \lp d^{\piail}_H  (s) - d^{\widebar{\pi}}_H (s) \rp  \piail_{H} ( a^{1}|s ) + \sum_{s \in \goodS} d^{\widebar{\pi}}_H (s) \lp \piail_{H} ( a^{1} | s) -   \widebar{\pi}_H (a^{1}|s) \rp
    \\
    &\leq \lnorm d^{\piail}_H  (\cdot) - d^{\widebar{\pi}}_H (\cdot) \rnorm_1 + \sum_{\substack{ s: s \in \goodS, \\ \piail_{H} ( a^{1}|s ) \geq   \widebar{\pi}_H (a^{1}|s)}} d^{\widebar{\pi}}_H (s) \lp \piail_{H} (a^{1}|s ) -   \widebar{\pi}_H (a^{1}|s) \rp
    \\
    &= \lnorm d^{\piail}_H  (\cdot) - d^{\widebar{\pi}}_H (\cdot) \rnorm_1 + \sum_{s \in \gS^{\widebar{\pi}}_H} d^{\widebar{\pi}}_H (s) \lp \piail_{H} ( a^{1}|s ) -   \widebar{\pi}_H (a^{1}|s) \rp,
\end{align*}
where we use $d^{\pi}_h (\cdot)$ to explicitly denote the state distribution induced by $\pi$. Plugging the above inequality into the policy value gap yields
\begin{align*}
&\quad V({\piail}) -  V({\widebar{\pi}}) \\
&\leq \sum_{h=1}^{H-1} \lnorm d^{\piail}_H  (\cdot, \cdot) - d^{\widebar{\pi}}_H (\cdot, \cdot) \rnorm_{1} + \lnorm d^{\piail}_H  (\cdot) - d^{\widebar{\pi}}_H (\cdot) \rnorm_1 + \sum_{s \in \gS^{\widebar{\pi}}_H} d^{\widebar{\pi}}_H (s) \lp \piail_{H} ( a^{1}|s ) -   \widebar{\pi}_H (a^{1}|s) \rp.
\end{align*}
With \cref{lemma:state_dist_discrepancy}, we have that
\begin{align*}
    &\quad \sum_{h=1}^{H-1} \lnorm d^{\piail}_H  (\cdot, \cdot) - d^{\widebar{\pi}}_H (\cdot, \cdot) \rnorm_{1} + \lnorm d^{\piail}_H  (\cdot) - d^{\widebar{\pi}}_H (\cdot) \rnorm_1 
    \\
    &\leq \sum_{h=1}^{H-1} \lnorm d^{\piail}_H  (\cdot) - d^{\widebar{\pi}}_H (\cdot) \rnorm_{1} + \sum_{h=1}^{H-1} \expect_{s \sim d^{\widebar{\pi}}_H (\cdot)} \ls \lnorm \piail_h (\cdot|s) - \widebar{\pi}_h (\cdot|s) \rnorm_1 \rs + \lnorm d^{\piail}_H  (\cdot) - d^{\widebar{\pi}}_H (\cdot) \rnorm_1
    \\
    &= \sum_{h=1}^{H} \lnorm d^{\piail}_H  (\cdot) - d^{\widebar{\pi}}_H (\cdot) \rnorm_{1} + \sum_{h=1}^{H-1} \expect_{s \sim d^{\widebar{\pi}}_H (\cdot)} \ls \lnorm \piail_h (\cdot|s) - \widebar{\pi}_h (\cdot|s) \rnorm_1 \rs
    \\
    &= \sum_{h=1}^{H} \sum_{\ell=1}^{h-1} \expect_{s \sim d^{\widebar{\pi}}_\ell (\cdot)} \ls \lnorm \piail_\ell (\cdot|s) - \widebar{\pi}_\ell (\cdot|s) \rnorm_1 \rs + \sum_{h=1}^{H-1} \expect_{s \sim d^{\widebar{\pi}}_H (\cdot)} \ls \lnorm \piail_h (\cdot|s) - \widebar{\pi}_h (\cdot|s) \rnorm_1 \rs
    \\
    &\leq 2  \sum_{h=1}^{H} \sum_{\ell=1}^{h-1} \expect_{s \sim d^{\widebar{\pi}}_\ell (\cdot)} \ls \lnorm \piail_\ell (\cdot|s) - \widebar{\pi}_\ell (\cdot|s) \rnorm_1 \rs.
\end{align*}
Then we have that
\begin{align}
    &\quad V({\piail}) -  V({\widebar{\pi}}) \nonumber \\
    &\leq 2  \sum_{h=1}^{H} \sum_{\ell=1}^{h-1} \expect_{s \sim d^{\widebar{\pi}}_\ell (\cdot)} \ls \lnorm \piail_\ell (\cdot|s) - \widebar{\pi}_\ell (\cdot|s) \rnorm_1 \rs + \sum_{s \in \gS^{\widebar{\pi}}_H} d^{\widebar{\pi}}_H (s) \lp \piail_{H} ( a^{1}|s ) -   \widebar{\pi}_H (a^{1}|s) \rp.  \label{eq:approximate_solution_step_3}
\end{align}
Notice that $\piail$ agrees with $\widebar{\pi}$ on bad states and we therefore obtain
\begin{align*}
     &\quad \expect_{s \sim d^{\widebar{\pi}}_{\ell} (\cdot)} \ls \lnorm \piail_{\ell} (\cdot|s) - \widebar{\pi}_{\ell} (\cdot|s) \rnorm_1 \rs \\
     &= \sum_{s \in \goodS } d^{\widebar{\pi}}_{\ell} (s)    \lnorm \piail_{\ell} (\cdot|s) - \widebar{\pi}_{\ell} (\cdot|s) \rnorm_1
    \\
    &=  \sum_{s \in \goodS } d^{\widebar{\pi}}_{\ell} (s) \bigg( \labs \piail_{\ell} (a^{1}|s) - \widebar{\pi}_{\ell} (a^{1}|s)  \rabs + \sum_{a \ne a^{1}} \widebar{\pi}_{\ell} (a|s)  \bigg)
    \\
    &= 2 \sum_{s \in \goodS } d^{\widebar{\pi}}_{\ell} (s)    \lp 1 - \widebar{\pi}_{\ell} (a^{1}|s) \rp.
\end{align*}
In the penultimate inequality, we use the fact that $\forall \ell \in [H-1], s \in \goodS, \piail_{\ell} (a^{1}|s) = 1$. Then we have that
\begin{align}
    &\quad V({\piail}) -  V({\widebar{\pi}}) \nonumber \\
    &\leq 4  \sum_{h=1}^{H} \sum_{\ell=1}^{h-1} \sum_{s \in \goodS } d^{\widebar{\pi}}_{\ell} (s)    \lp 1 - \widebar{\pi}_{\ell} (a^{1}|s) \rp + \sum_{s \in \gS^{\widebar{\pi}}_H} d^{\widebar{\pi}}_H (s) \lp \piail_{H} (a^{1}|s) -   \widebar{\pi}_H (a^{1}|s) \rp. \label{eq:approximate_solution_step_2}
\end{align}
Then we consider the second term in \eqref{eq:approximate_solution_step_2}. For the last time step $H$, notice that by construction, we have $\piail_{H} (a^{1}|s)  = \min\{ \widehat{d^{\piE}_H} (s) / d^{\piE}_H (s), 1  \}$. Then we get 
\begin{align}
    &\quad \sum_{s \in \gS^{\widebar{\pi}}_H} d^{\widebar{\pi}}_H (s) \lp \piail_{H} ( a^{1} | s) -   \widebar{\pi}_H (a^{1}|s) \rp \nonumber \\
    &\leq \sum_{s \in \gS^{\widebar{\pi}}_H} d^{\widebar{\pi}}_H (s) \lp \min\lb \frac{\widehat{d^{\piE}_H} (s)}{d^{\piE}_H (s)},  1  \rb -   \widebar{\pi}_H (a^{1}|s) \rp. \label{eq:approximate_solution_step_4}
\end{align}
Plugging \eqref{eq:approximate_solution_step_4} to \eqref{eq:approximate_solution_step_3} yields
\begin{align*}
    &\quad V({\piail}) -  V({\widebar{\pi}}) \\
    &\leq 4  \sum_{h=1}^{H} \sum_{\ell=1}^{h-1} \sum_{s \in \goodS } d^{\widebar{\pi}}_{\ell} (s)    \lp 1 - \widebar{\pi}_{\ell} (a^{1}|s) \rp + \sum_{s \in \gS^{\widebar{\pi}}_H} d^{\widebar{\pi}}_H (s) \lp \min\lb \frac{\widehat{d^{\piE}_H} (s)}{d^{\piE}_H (s)},  1  \rb -   \widebar{\pi}_H (a^{1}|s) \rp.
\end{align*}
Subsequently, we can apply \cref{prop:ail_general_reset_cliff_approximate_solution} and get that
\begin{align*}
    &\quad V({\piail}) -  V({\widebar{\pi}}) \\
    &\leq 4 \bigg[ \sum_{h=1}^{H} \sum_{\ell=1}^{h-1} \sum_{s \in \goodS } d^{\widebar{\pi}}_{\ell} (s)    \lp 1 - \widebar{\pi}_{\ell} (a^{1}|s) \rp + \sum_{s \in \gS^{\widebar{\pi}}_H} d^{\widebar{\pi}}_H (s) \big( \min\big\{ \frac{\widehat{d^{\piE}_H} (s)}{d^{\piE}_H (s)},  1  \big\} -   \widebar{\pi}_H (a^{1}|s) \big) \bigg]
    \\
    &\leq \frac{4}{c (\widebar{\pi})} \varepsilon_{}.
\end{align*}

We proceed to upper bound the imitation gap. With \eqref{eq:value_gap_decomposition_approximate_solution} and \eqref{eq:ail_policy_piE_value_gap}, we have 
\begin{align*}
   &\quad V({\piE})  - \expect \ls V({\widebar{\pi}}) \rs \\
   &=  V( {\piE}) - \expect \ls  V({\piail}) \rs + \expect \ls V({\piail}) -  V({\widebar{\pi}}) \rs
    \\
    &\leq \min \lb 1, 2 \expect \ls \sum_{s \in \gS} \labs d^{\piE}_H (s) - \widehat{d^{\piE}_H} (s)  \rabs \rs \rb + \frac{4}{c (\widebar{\pi})} \varepsilon
    \\
    &= \min \lb 1, 2 \expect \ls  \lnorm \widehat{d^{\piE}_H} (\cdot) - d^{\piE}_H (\cdot) \rnorm_{1}  \rs \rb + \frac{4}{c (\widebar{\pi})} \varepsilon, 
\end{align*}
where the expectation is taken over the randomness in collecting $N$ expert trajectories. Finally, we apply \citep[Theorem 1]{han2015minimax} to upper bound the estimation error in the first term:
\begin{align*}
    V({\piE}) - \expect \ls   V({\widebar{\pi}}) \rs &\leq \min \lb 1 +\frac{4}{c (\widebar{\pi})} \varepsilon,  2 \sqrt{ \frac{|\gS| - 1}{N}} + \frac{4}{c (\widebar{\pi})} \varepsilon \rb.
\end{align*}
\end{proof}

\subsection{Proof of Proposition \ref{prop:transfer_error}}

\begin{proof}[Proof of Proposition \ref{prop:transfer_error}]
Notice that $\widehat{\gP}$ is the transition model learned by an algorithm that is $(\varepsilon_{\eval}, \delta)$-PAC for uniform policy evaluation (see \cref{def:uniform_policy_evaluation}). Then with probability at least $1-\delta$, for any policy $\pi$, we have that
\begin{align*}
    &\quad \sum_{h=1}^{H} \lnorm d^{\pi}_h - \widehat{d^{\piE}_h} \rnorm_1
    \\
    &= \sum_{h=1}^{H} \lnorm d^{\pi, \gP}_h - \widehat{d^{\piE}_h} \rnorm_1
    \\
    &\leq \sum_{h=1}^{H} \lnorm d^{\pi, \widehat{\gP}}_h - \widehat{d^{\piE}_h} \rnorm_1 + \sum_{h=1}^{H} \lnorm d^{\pi, \widehat{\gP}}_h - d^{\pi, \gP}_h \rnorm_1
    \\
    &\leq \sum_{h=1}^{H} \lnorm d^{\pi, \widehat{\gP}}_h - \widehat{d^{\piE}_h} \rnorm_1 + \varepsilon_{\eval}. 
\end{align*}
This implies that 
\begin{align*}
    \sum_{h=1}^{H} \lnorm d^{\widebar{\pi}}_h - \widehat{d^{\piE}_h} \rnorm_1 \leq \sum_{h=1}^{H} \lnorm d^{\widebar{\pi}, \widehat{\gP}}_h - \widehat{d^{\piE}_h} \rnorm_1 + \varepsilon_{\eval}, 
\end{align*}
where $\widebar{\pi}$ is an $\varepsilon_{\opt}$-approximately optimal solution with respect to the optimization problem in \eqref{eq:AIL_with_empirical_transition_model}. By definition, we also have that
\begin{align*}
    \sum_{h=1}^{H} \lnorm d^{\widebar{\pi}, \widehat{\gP}}_h - \widehat{d^{\piE}_h} \rnorm_1 \leq \min_{\pi \in \Pi} \sum_{h=1}^{H} \lnorm d^{\pi, \widehat{\gP}}_h - \widehat{d^{\piE}_h} \rnorm_1 + \varepsilon_{\opt}. 
\end{align*}
Combining the above two inequalities yields that
\begin{align*}
    &\quad \sum_{h=1}^{H} \lnorm d^{\widebar{\pi}}_h - \widehat{d^{\piE}_h} \rnorm_1
    \\
    &\leq \min_{\pi \in \Pi} \sum_{h=1}^{H} \lnorm d^{\pi, \widehat{\gP}}_h - \widehat{d^{\piE}_h} \rnorm_1 + \varepsilon_{\opt} + \varepsilon_{\eval}
    \\
    &\leq \min_{\pi \in \Pi} \sum_{h=1}^{H} \lnorm d^{\pi, \gP}_h - \widehat{d^{\piE}_h} \rnorm_1 + \varepsilon_{\opt} + 2\varepsilon_{\eval},
\end{align*}
where the last inequality again follows \cref{def:uniform_policy_evaluation}. According to \cref{defn:approximate_solution}, we have that $\widebar{\pi}$ is an $(\varepsilon_{\opt} + 2\varepsilon_{\eval})$-optimal solution with respect to the state-action distribution matching problem and thus complete the proof.
\end{proof}

\subsection{Proof of Proposition \ref{prop:rabs_mdps_extension_one}}
\label{appendix:proof_of_extension_one}
Here we present the proof of Proposition \ref{prop:rabs_mdps_extension_one}. The overall proof strategy is similar to that of Proposition \ref{prop:ail_general_reset_cliff}. However, to analyze the extended RABS MDPs, we need to carefully characterize the visitation probability of good states by taking the non-expert action.

To prove Proposition \ref{prop:rabs_mdps_extension_one}, we need the following auxiliary lemma. We use $\gooda$ and $\bada$ to denote the expert action and non-expert action, respectively.

\begin{lem}
\label{lem:auxiliary_lemma_extension_one}
    Consider the extended RABS MDPs shown in Figure \ref{fig:rabs_mdp_extension_one}, assume that $\forall s, s^{\prime} \in \goodS, \; \varepsilon \leq P_1 (s^\prime |s, \gooda) $. We have that
    \begin{align*}
        \exists s \in \goodS, \piail_1 (\gooda|s) > 0.
    \end{align*}
    Furthermore, we define $\gV_{2} := \{s \in  \goodS: \widehat{d^{\piE}_2} (s) > 0 \}$ as the set of states visited in the expert demonstrations. It holds that
    \begin{align*}
        \forall s \in \gV_2, \; \piail_2 (\gooda|s) > 0.
    \end{align*}
\end{lem}

\begin{proof}
We first prove that for time step $h=1$, $\exists s \in \goodS$, $\piail_1 (\gooda|s) > 0$. We prove this result by contradiction. In particular, we assume that $\forall s \in \goodS, \piail_1 (\gooda|s) = 0$. Then it is easy to calculate that
    \begin{align*}
        & d^{\piail}_1 (s^1) = \rho (s^1), \; d^{\piail}_1 (s^2) = \rho (s^2), \; d^{\piail}_1 (b) = 0,
        \\
        & d^{\piail}_2 (s^1) = \sum_{\widetilde{s} \in \goodS} \rho (\widetilde{s}) \varepsilon, \; d^{\piail}_2 (s^2) = \sum_{\widetilde{s} \in \goodS} \rho (\widetilde{s}) \varepsilon,
        \\
        & d^{\piail}_2 (b) = 1 - \sum_{s \in \goodS} \sum_{\widetilde{s} \in \goodS} \rho (\widetilde{s}) \varepsilon.  
    \end{align*}
    \begin{align*}
        & \quad \Loss_1 (\piail) = 2,
        \\
        &\quad \Loss_2 (\piail)
        \\
        &= \sum_{s \in \goodS} \bigg(  \labs d^{\piail}_2 (s) \piail_2 (\gooda|s) - \widehat{d^{\piE}_h} (s) \rabs + d^{\piail}_2 (s) \lp  1- \piail_2 (\gooda|s) \rp  \bigg) + d^{\piail}_2 (b)
        \\
        &\geq \sum_{s \in \goodS} \bigg( \widehat{d^{\piE}_h} (s) -    d^{\piail}_2 (s) \piail_2 (\gooda|s) + d^{\piail}_2 (s) \lp  1- \piail_2 (\gooda|s) \rp  \bigg) + d^{\piail}_2 (b)
        \\
        &= \sum_{s \in \goodS} \lp \widehat{d^{\piE}_h} (s)    + d^{\piail}_2 (s) -    2 d^{\piail}_2 (s) \piail_2 (\gooda|s)  \rp + d^{\piail}_2 (b)
        \\
        &= 2 - 2 \sum_{s \in \goodS} d^{\piail}_2 (s) \piail_2 (\gooda|s)
        \\
        &\geq 2 - 2 \sum_{s \in \goodS} d^{\piail}_2 (s)
        \\
        &= 2 - 2 \sum_{s \in \goodS} \sum_{\widetilde{s} \in \goodS} \rho (\widetilde{s}) \varepsilon. 
    \end{align*}
     We construct another policy $\widetilde{\pi}^{\ail}$: $\widetilde{\pi}^{\ail}_h (\gooda|s) = 1, \; \forall s \in \goodS, \; \forall h \in [2]$. Then we can obtain that
     \begin{align*}
         &\Loss_1 (\widetilde{\pi}^{\ail}) = \sum_{s \in \goodS} \labs d^{\piE}_1 (s) - \widehat{d^{\piE}_1} (s)  \rabs < 2 = \Loss_1 (\piail),
         \\
         & \Loss_2 (\widetilde{\pi}^{\ail}) = \sum_{s \in \goodS} \labs d^{\piE}_2 (s) - \widehat{d^{\piE}_2} (s)  \rabs. 
     \end{align*}
     To analyze $\Loss_2 (\widetilde{\pi}^{\ail})$, we define the set of states $\gS^{\operatorname{G}, 1} := \{ s \in \goodS: \widehat{d^{\piE}_2} (s) \geq d^{\piE}_2 (s)  \}$ and $\gS^{\operatorname{G}, 2} = \gS \setminus \gS^{\operatorname{G}, 1}$. It is easy to derive that $| \gS^{\operatorname{G}, 1} | \geq 1$  as both $\widehat{d^{\piE}_2} (\cdot)$ and $d^{\piE}_2 (\cdot)$ are valid state distributions. Then we have that
     \begin{align*}
         &\quad \Loss_2 (\widetilde{\pi}^{\ail})
         \\
         &= \sum_{s \in \goodS} \labs d^{\piE}_2 (s) - \widehat{d^{\piE}_2} (s)  \rabs
         \\
         &= \sum_{s \in \gS^{\operatorname{G}, 1}} \widehat{d^{\piE}_2} (s) -  d^{\piE}_2 (s) + \sum_{s \in \gS^{\operatorname{G}, 2}} d^{\piE}_2 (s) -  \widehat{d^{\piE}_2} (s)
         \\
         &= 2 - 2 \sum_{s \in \gS^{\operatorname{G}, 1}} d^{\piE}_2 (s) - 2 \sum_{s \in \gS^{\operatorname{G}, 2}} \widehat{d^{\piE}_2} (s)
         \\
         &\leq 2 - 2 \sum_{s \in \gS^{\operatorname{G}, 1}} d^{\piE}_2 (s)
         \\
         &= 2 - 2 \sum_{s \in \gS^{\operatorname{G}, 1}} \sum_{\widetilde{s} \in \goodS} \rho (\widetilde{s}) P_1 (s|\widetilde{s}, \gooda)
         \\
         &\leq 2 - 2 \min_{s \in \goodS} \sum_{\widetilde{s} \in \goodS} \rho (\widetilde{s}) P_1 (s|\widetilde{s}, \gooda)
         \\
         &\leq 2 - 2  \sum_{\widetilde{s} \in \goodS} \rho (\widetilde{s}) \min_{s \in \goodS} P_1 (s|\widetilde{s}, \gooda) .
     \end{align*}
    According to the assumption that $\forall s, s^\prime \in \goodS, \; \varepsilon \leq   P_1 (s^\prime |s, \gooda) / 2$, then $\Loss_2 (\widetilde{\pi}^{\ail}) \leq \Loss_2 (\piail)$. In summary, we have obtained 
    \begin{align*}
         \Loss_1 (\widetilde{\pi}^{\ail}) <  \Loss_1 (\piail), \Loss_2 (\widetilde{\pi}^{\ail}) \leq \Loss_2 (\piail).  
    \end{align*}

    This implies that $\widetilde{\pi}^{\ail}$ achieves a strictly smaller distribution matching loss than $\piail$, which contradicts the fact that $\piail$ is the optimal solution. Then we can derive that for time step $h=1$, $\exists s \in \goodS$, $\piail_1 (\gooda|s) > 0$.

    Then we prove the second statement. According to the optimality condition, we have that
    \begin{align*}
        \piail_2 \in \argmin_{\pi_2} f_2 (\pi_2; \piail_1).
    \end{align*}
    Notice that $\Loss_1$ is independent of $\pi_2$. Then it holds that
    \begin{align*}
        \piail_2 \in \argmin_{\pi_2} \Loss_2.
    \end{align*}
    For $\Loss_2$, we have that
    \begin{align*}
        \Loss_2 &= \sum_{s \in \goodS} \bigg( \labs \widehat{d^{\piE}_2 } (s) - d^{\piail}_2 (s) \pi_2 (\gooda|s) \rabs + d^{\piail}_2 (s) \lp 1-  \pi_2 (\gooda|s) \rp \bigg)
        \\
        &= \sum_{s \in \gV_2} \bigg( \labs \widehat{d^{\piE}_2 } (s) - d^{\piail}_2 (s) \pi_2 (\gooda|s) \rabs + d^{\piail}_2 (s) \lp 1-  \pi_2 (\gooda|s) \rp \bigg) + \sum_{s \notin \gV_2, s \in \goodS} d^{\piail}_2 (s) + d^{\piail}_2 (b)
        \\
        &= \sum_{s \in \gV_2} \bigg( \labs \widehat{d^{\piE}_2 } (s) - d^{\piail}_2 (s) \pi_2 (\gooda|s) \rabs - d^{\piail}_2 (s)  \pi_2 (\gooda|s) \bigg) + \sum_{s \in \gV_2} d^{\piail}_2  (s)   + \sum_{s \notin \gV_2, s \in \goodS} d^{\piail}_2  (s) + d^{\piail}_2 (b)
        \\
        &= \sum_{s \in \gV_2} \bigg( \labs \widehat{d^{\piE}_2 } (s) - d^{\piail}_2 (s) \pi_2 (\gooda|s) \rabs - d^{\piail}_2 (s)  \pi_2 (\gooda|s) \bigg) + 1.
    \end{align*}
    Then we have that
    \begin{align*}
        \piail_2 \in &\argmin_{\pi_2} \sum_{s \in \gV_2} \bigg( \labs \widehat{d^{\piE}_2 } (s) - d^{\piail}_2 (s) \pi_2 (\gooda|s) \rabs - d^{\piail}_2 (s)  \pi_2 (\gooda|s) \bigg) 
    \end{align*}
    For each $s \in \gV_2$, it holds that
    \begin{align*}
        \piail_2 (\gooda|s) \in &\argmin_{\pi_2 (\gooda|s) \in [0, 1]} \bigg( \labs \widehat{d^{\piE}_2 } (s) - d^{\piail}_2 (s) \pi_2 (\gooda|s) \rabs - d^{\piail}_2 (s)  \pi_2 (\gooda|s) \bigg). 
    \end{align*}
    Based on the first statement that $\exists s \in \goodS, \piail_1 (\gooda|s) > 0$, we have that $\forall s \in \goodS, d^{\piail}_2 (s) > 0$. According to Lemma \ref{lem:single_variable_opt_condition}, we have that $\piail_2 (\gooda|s) > 0$, which completes the proof of the second statement. 
\end{proof}
With the above auxiliary lemma, we are ready to prove Proposition \ref{prop:rabs_mdps_extension_one}.

According to the optimality condition, we have that 
    \begin{align*}
        \piail_1 \in \argmin_{\pi_1} f_1 (\pi_1; \piail_2),
    \end{align*}
    where $f_1 (\pi_1; \piail_2) = \sum_{h=1}^H \sum_{(s,a) \in \gS \times \gA} |d^{\piail}_h (s, a) - \widehat{d^{\piE}_h} (s, a)|$, which is calculated by the policy $(\pi_1, \piail_2)$. For the policy $(\pi_1, \piail_2)$, we can calculate that 
    \begin{align*}
        d^{\piail}_1 (s^1) = \rho (s^1), \; d^{\piail}_1 (s^2) = \rho (s^2), \; d^{\piail}_1 (b) = 0.  
    \end{align*}
    \begin{align*}
        \forall s \in \goodS,
        d^{\piail}_2 (s) &= \sum_{\widetilde{s} \in \goodS} \rho (\widetilde{s}) \lp \pi_1 (\gooda|\widetilde{s}) P_1 (s|\widetilde{s}, \gooda) + \pi_1 (\bada|\widetilde{s}) \varepsilon  \rp
        \\
        &= \sum_{\widetilde{s} \in \goodS} \rho (\widetilde{s}) \lp \pi_1 (\gooda|\widetilde{s}) P_1 (s|\widetilde{s}, \gooda) + (1-\pi_1 (\gooda|\widetilde{s})) \varepsilon  \rp
        \\ 
        &= \varepsilon + \sum_{\widetilde{s} \in \goodS}  \rho (\widetilde{s}) \pi_1 (\gooda|\widetilde{s}) \lp P_1 (s|\widetilde{s}, \gooda) - \varepsilon \rp.
    \end{align*}
    \begin{align*}
        d^{\piail}_2 (b) &= \sum_{\widetilde{s} \in \goodS} \rho (\widetilde{s}) \pi_1 (\bada|\widetilde{s}) \lp 1-2\varepsilon \rp
        \\
        &= \lp 1- 2 \varepsilon \rp \sum_{\widetilde{s} \in \goodS} \rho (\widetilde{s}) \lp 1- \pi_1 (\gooda|\widetilde{s})  \rp
        \\
        &=1-2\varepsilon - \lp 1- 2 \varepsilon \rp \sum_{\widetilde{s} \in \goodS} \rho (\widetilde{s}) \pi_1 (\gooda|\widetilde{s}).
    \end{align*}
    Then we can calculate that
    \begin{align*}
        &\quad \Loss_1 (\pi_1)
        \\
        &= \sum_{s \in \goodS} \lp  \labs \widehat{d^{\piE}_1} (s) - \rho (s) \pi_1 (\gooda|s) \rabs + \rho (s) \lp 1- \pi_1 (\gooda|s) \rp \rp
        \\
        &= 1 + \sum_{s \in \goodS} \lp  \labs \widehat{d^{\piE}_1} (s) - \rho (s) \pi_1 (\gooda|s) \rabs - \rho (s) \pi_1 (\gooda|s)   \rp. 
    \end{align*}
    By Lemma \ref{lem:single_variable_opt}, we have that $\pi_1 (\gooda|s) = 1, \forall s \in \goodS$ is an optimal solution of $\argmin_{\pi_1} \Loss_1 (\pi_1)$.
    \begin{align*}
        &\quad \Loss_2 (\pi_1; \piail_2)
        \\
        &=  \sum_{s \in \goodS} \lp \labs \widehat{d^{\piE}_2} (s) - d^{\piail}_2 (s) \piail_2 (\gooda|s) \rabs + d^{\piail}_2 (s) \piail_2 (\bada|s) \rp + d^{\piail}_2 (b)
        \\
        &=1-2\varepsilon - \lp 1-2\varepsilon \rp \sum_{\widetilde{s} \in \goodS} \lp \rho (\widetilde{s}) \pi_1 (\gooda|\widetilde{s}) \rp 
        \\
        &\; + \sum_{s \in \goodS} \lp \labs \widehat{d^{\piE}_2} (s) - d^{\piail}_2 (s) \piail_2 (\gooda|s) \rabs + d^{\piail}_2 (s) \piail_2 (\bada|s) \rp
        \\
        &=  \underbrace{1 - 2\varepsilon - \lp 1-2\varepsilon \rp \sum_{\widetilde{s} \in \goodS} \lp \rho (\widetilde{s}) \pi_1 (\gooda|\widetilde{s}) \rp}_{\text{Term I}} 
        \\
        &\; + \underbrace{\sum_{s \in \goodS} \labs \widehat{d^{\piE}_2} (s) - d^{\piail}_2 (s) \piail_2 (\gooda|s) \rabs}_{\text{Term II}} + \underbrace{\sum_{s \in \goodS} d^{\piail}_2 (s) \piail_2 (\bada|s) }_{\text{Term III}}. 
    \end{align*}
    For Term III, we have that
    \begin{align*}
        &\quad \sum_{s \in \goodS} d^{\piail}_2 (s) \piail_2 (\bada|s)
        \\
        &= \sum_{s \in \goodS} \lp \varepsilon + \sum_{\widetilde{s} \in \goodS}  \rho (\widetilde{s}) \pi_1 (\gooda|\widetilde{s}) \lp P_1 (s|\widetilde{s}, \gooda) - \varepsilon \rp \rp \piail_2 (\bada|s)
        \\
        &= \sum_{s \in \goodS} \lp \varepsilon + \sum_{\widetilde{s} \in \goodS}  \rho (\widetilde{s}) \pi_1 (\gooda|\widetilde{s}) \lp P_1 (s|\widetilde{s}, \gooda) - \varepsilon \rp \rp \lp 1- \piail_2 (\gooda|s) \rp
        \\
        &= \sum_{s \in \goodS} \varepsilon \lp 1- \piail_2 (\gooda|s) \rp + \sum_{s \in \goodS} \sum_{\widetilde{s} \in \goodS}  \rho (\widetilde{s}) \pi_1 (\gooda|\widetilde{s}) \lp P_1 (s|\widetilde{s}, \gooda) - \varepsilon \rp - \sum_{s \in \goodS} \sum_{\widetilde{s} \in \goodS}  \rho (\widetilde{s}) \pi_1 (\gooda|\widetilde{s}) \lp P_1 (s|\widetilde{s}, \gooda) - \varepsilon \rp \piail_2 (\gooda|s)
        \\
        &= \sum_{s \in \goodS} \varepsilon \lp 1- \piail_2 (\gooda|s) \rp + \sum_{\widetilde{s} \in \goodS}  \rho (\widetilde{s}) \pi_1 (\gooda|\widetilde{s}) \sum_{s \in \goodS} \lp P_1 (s|\widetilde{s}, \gooda) - \varepsilon \rp - \sum_{s \in \goodS} \sum_{\widetilde{s} \in \goodS}  \rho (\widetilde{s}) \pi_1 (\gooda|\widetilde{s}) \lp P_1 (s|\widetilde{s}, \gooda) - \varepsilon \rp \piail_2 (\gooda|s)
        \\
        &= \sum_{s \in \goodS} \varepsilon \lp 1- \piail_2 (\gooda|s) \rp + \sum_{\widetilde{s} \in \goodS}  \rho (\widetilde{s}) \pi_1 (\gooda|\widetilde{s})  \lp 1- |\goodS| \varepsilon \rp - \sum_{s \in \goodS} \sum_{\widetilde{s} \in \goodS}  \rho (\widetilde{s}) \pi_1 (\gooda|\widetilde{s}) \lp P_1 (s|\widetilde{s}, \gooda) - \varepsilon \rp \piail_2 (\gooda|s).
    \end{align*}
    Then we have that
    \begin{align*}
        &\quad \Loss_2 (\pi_1; \piail_2)
        \\
        &=1-2\varepsilon - \lp 1-2\varepsilon \rp \sum_{\widetilde{s} \in \goodS} \lp \rho (\widetilde{s}) \pi_1 (\gooda|\widetilde{s}) \rp + \sum_{s \in \goodS} \labs \widehat{d^{\piE}_2} (s) - d^{\piail}_2 (s) \piail_2 (\gooda|s) \rabs + \sum_{s \in \goodS} d^{\piail}_2 (s) \piail_2 (\bada|s)
        \\
	        &=1-2\varepsilon - \lp 1-2\varepsilon \rp \sum_{\widetilde{s} \in \goodS} \lp \rho (\widetilde{s}) \pi_1 (\gooda|\widetilde{s}) \rp + \sum_{s \in \goodS} \labs \widehat{d^{\piE}_2} (s) - d^{\piail}_2 (s) \piail_2 (\gooda|s) \rabs
	        \\
	        &\quad + \sum_{s \in \goodS} \varepsilon \lp 1- \piail_2 (\gooda|s) \rp + \sum_{\widetilde{s} \in \goodS}  \rho (\widetilde{s}) \pi_1 (\gooda|\widetilde{s})  \lp 1- 2 \varepsilon \rp
	        \\
	        &\quad - \sum_{s \in \goodS} \sum_{\widetilde{s} \in \goodS}  \rho (\widetilde{s}) \pi_1 (\gooda|\widetilde{s}) \lp P_1 (s|\widetilde{s}, \gooda) - \varepsilon \rp \piail_2 (\gooda|s)
        \\
	        &= 1-2\varepsilon +  \sum_{s \in \goodS} \labs \widehat{d^{\piE}_2} (s) - d^{\piail}_2 (s) \piail_2 (\gooda|s) \rabs + \sum_{s \in \goodS} \varepsilon \lp 1- \piail_2 (\gooda|s) \rp
	        \\
	        &\quad - \sum_{s \in \goodS} \sum_{\widetilde{s} \in \goodS}  \rho (\widetilde{s}) \pi_1 (\gooda|\widetilde{s}) \lp P_1 (s|\widetilde{s}, \gooda) - \varepsilon \rp \piail_2 (\gooda|s)
        \\
	        &= \sum_{s \in \goodS} \labs \widehat{d^{\piE}_2} (s) - d^{\piail}_2 (s) \piail_2 (\gooda|s) \rabs
	        \\
	        &\quad - \sum_{s \in \goodS} \sum_{\widetilde{s} \in \goodS}  \rho (\widetilde{s}) \pi_1 (\gooda|\widetilde{s}) \lp P_1 (s|\widetilde{s}, \gooda) - \varepsilon \rp \piail_2 (\gooda|s) + \operatorname{const}_1.
    \end{align*}
    Here $\operatorname{const}_1 = 1-2\varepsilon + \sum_{s \in \goodS} \varepsilon \lp 1- \piail_2 (\gooda|s) \rp$ which is independent of $\pi_1$. Therefore, we can obtain that
    \begin{align*}
        & \quad \argmin_{\pi_1} \Loss_2 (\pi_1; \piail_2)
        \\
        & = \argmin_{\pi_1} \sum_{s \in \goodS} \labs \widehat{d^{\piE}_2} (s) - d^{\piail}_2 (s) \piail_2 (\gooda|s) \rabs - \sum_{s \in \goodS} \sum_{\widetilde{s} \in \goodS}  \rho (\widetilde{s}) \pi_1 (\gooda|\widetilde{s}) \lp P_1 (s|\widetilde{s}, \gooda) - \varepsilon \rp \piail_2 (\gooda|s). 
    \end{align*}
    Then we analyze the term $ \sum_{s \in \goodS} \labs \widehat{d^{\piE}_2} (s) - d^{\piail}_2 (s) \piail_2 (\gooda|s) \rabs$. In particular, we define the set of states $\gV_{2} := \{s \in  \goodS: \widehat{d^{\piE}_2} (s) > 0 \}$. Then, we have that
    \begin{align*}
        &\quad \sum_{s \in \goodS} \labs \widehat{d^{\piE}_2} (s) - d^{\piail}_2 (s) \piail_2 (\gooda|s) \rabs
        \\
        &= \sum_{s \in \gV_{H}} \labs \widehat{d^{\piE}_2} (s) - d^{\piail}_2 (s) \piail_2 (\gooda|s) \rabs + \sum_{s \in \goodS, s \notin \gV_{H}} d^{\piail}_2 (s) \piail_2 (\gooda|s).  
    \end{align*}
    Plugging the equation of $d^{\piail}_2 (s) = \varepsilon + \sum_{\widetilde{s} \in \goodS}  \rho (\widetilde{s}) \pi_1 (\gooda|\widetilde{s}) \lp P_1 (s|\widetilde{s}, \gooda) - \varepsilon \rp, \; \forall s \in \goodS$ into the above equation yields that
    \begin{align*}
        &\quad \sum_{s \in \goodS} \labs \widehat{d^{\piE}_2} (s) - d^{\piail}_2 (s) \piail_2 (\gooda|s) \rabs
        \\
        &= \sum_{s \in \gV_{H}} \bigg| \widehat{d^{\piE}_2} (s) - \lp \varepsilon + \sum_{\widetilde{s} \in \goodS}  \rho (\widetilde{s}) \pi_1 (\gooda|\widetilde{s}) \lp P_1 (s|\widetilde{s}, \gooda) - \varepsilon \rp \rp \piail_2 (\gooda|s) \bigg|  
        \\
        &\; + \sum_{s \in \goodS, s \notin \gV_{H}} \lp \varepsilon + \sum_{\widetilde{s} \in \goodS}  \rho (\widetilde{s}) \pi_1 (\gooda|\widetilde{s}) \lp P_1 (s|\widetilde{s}, \gooda) - \varepsilon \rp \rp \piail_2 (\gooda|s)
        \\
        &= \sum_{s \in \gV_{H}} \bigg| \widehat{d^{\piE}_2} (s) -  \varepsilon  \piail_2 (\gooda|s) - \sum_{\widetilde{s} \in \goodS}  \rho (\widetilde{s})  \lp P_1 (s|\widetilde{s}, \gooda) - \varepsilon \rp \piail_2 (\gooda|s) \pi_1 (\gooda|\widetilde{s}) \bigg| 
        \\
        &\;+ \varepsilon \sum_{s \in \goodS, s \notin \gV_{H}}   \piail_2 (\gooda|s) + \sum_{s \in \goodS, s \notin \gV_{H}} \sum_{\widetilde{s} \in \goodS}  \rho (\widetilde{s}) \pi_1 (\gooda|\widetilde{s}) \lp P_1 (s|\widetilde{s}, \gooda) - \varepsilon \rp  \piail_2 (\gooda|s).
    \end{align*}
    Then we can obtain that
    \begin{align*}
        &\quad \sum_{s \in \goodS} \labs \widehat{d^{\piE}_2} (s) - d^{\piail}_2 (s) \piail_2 (\gooda|s) \rabs - \sum_{s \in \goodS} \sum_{\widetilde{s} \in \goodS}  \rho (\widetilde{s}) \pi_1 (\gooda|\widetilde{s}) \lp P_1 (s|\widetilde{s}, \gooda) - \varepsilon \rp \piail_2 (\gooda|s)
        \\
	        &= \sum_{s \in \gV_{H}} \bigg| \widehat{d^{\piE}_2} (s) -  \varepsilon  \piail_2 (\gooda|s) - \sum_{\widetilde{s} \in \goodS}  \rho (\widetilde{s})  \lp P_1 (s|\widetilde{s}, \gooda) - \varepsilon \rp \piail_2 (\gooda|s) \pi_1 (\gooda|\widetilde{s}) \bigg|
	        \\
	        &\quad + \varepsilon \sum_{s \in \goodS, s \notin \gV_{H}}   \piail_2 (\gooda|s) + \sum_{s \in \goodS, s \notin \gV_{H}} \sum_{\widetilde{s} \in \goodS}  \rho (\widetilde{s}) \pi_1 (\gooda|\widetilde{s}) \lp P_1 (s|\widetilde{s}, \gooda) - \varepsilon \rp  \piail_2 (\gooda|s)
    \end{align*}
    \begin{align*}
        &\quad  \sum_{s \in \goodS} \sum_{\widetilde{s} \in \goodS}  \rho (\widetilde{s}) \pi_1 (\gooda|\widetilde{s}) \lp P_1 (s|\widetilde{s}, \gooda) - \varepsilon \rp \piail_2 (\gooda|s)
        \\
	        &= \sum_{s \in \gV_{H}} \bigg| \widehat{d^{\piE}_2} (s) -  \varepsilon  \piail_2 (\gooda|s) - \sum_{\widetilde{s} \in \goodS}  \rho (\widetilde{s})  \lp P_1 (s|\widetilde{s}, \gooda) - \varepsilon \rp \piail_2 (\gooda|s) \pi_1 (\gooda|\widetilde{s}) \bigg|
	        \\
	        &\quad + \varepsilon \sum_{s \in \goodS, s \notin \gV_{H}}   \piail_2 (\gooda|s) - \sum_{s \in \gV_{H}} \sum_{\widetilde{s} \in \goodS}  \rho (\widetilde{s}) \pi_1 (\gooda|\widetilde{s}) \lp P_1 (s|\widetilde{s}, \gooda) - \varepsilon \rp  \piail_2 (\gooda|s)
        \\
	        &= \sum_{s \in \gV_{H}} \bigg| \widehat{d^{\piE}_2} (s) -  \varepsilon  \piail_2 (\gooda|s) - \sum_{\widetilde{s} \in \goodS}  \rho (\widetilde{s})  \lp P_1 (s|\widetilde{s}, \gooda) - \varepsilon \rp \piail_2 (\gooda|s) \pi_1 (\gooda|\widetilde{s}) \bigg|
	        \\
	        &\quad + \operatorname{const}_2 - \sum_{\widetilde{s} \in \goodS} \pi_1 (\gooda|\widetilde{s}) \lp  \sum_{s \in \gV_{H}}   \rho (\widetilde{s})  \lp P_1 (s|\widetilde{s}, \gooda) - \varepsilon \rp  \piail_2 (\gooda|s) \rp.
    \end{align*}
Here $\operatorname{const}_2 = \varepsilon \sum_{s \in \goodS, s \notin \gV_{H}}   \piail_2 (\gooda|s) $. Then we can obtain that
\begin{align*}
    & \quad \argmin_{\pi_1} \sum_{s \in \goodS} \labs \widehat{d^{\piE}_2} (s) - d^{\piail}_2 (s) \piail_2 (\gooda|s) \rabs - \sum_{s \in \goodS} \sum_{\widetilde{s} \in \goodS}  \rho (\widetilde{s}) \pi_1 (\gooda|\widetilde{s}) \lp P_1 (s|\widetilde{s}, \gooda) - \varepsilon \rp \piail_2 (\gooda|s)
    \\
    & = \argmin_{\pi_1} \sum_{s \in \gV_{H}} \bigg| \widehat{d^{\piE}_2} (s) -  \varepsilon  \piail_2 (\gooda|s) - \sum_{\widetilde{s} \in \goodS}  \rho (\widetilde{s})  \lp P_1 (s|\widetilde{s}, \gooda) - \varepsilon \rp \piail_2 (\gooda|s) \pi_1 (\gooda|\widetilde{s}) \bigg|  
        \\
        &\;  - \sum_{\widetilde{s} \in \goodS} \pi_1 (\gooda|\widetilde{s}) \lp  \sum_{s \in \gV_{H}}   \rho (\widetilde{s})  \lp P_1 (s|\widetilde{s}, \gooda) - \varepsilon \rp  \piail_2 (\gooda|s) \rp.  
\end{align*}
Then we apply Lemma \ref{lem:mn_variables_opt_unique} to prove that $\forall \widetilde{s} \in \goodS, \; \piail_1 (\gooda|\widetilde{s}) = 1$ is an \emph{unqiue} optimal solution of the above optimization problem. In particular, we apply Lemma \ref{lem:mn_variables_opt_unique} with
\begin{align*}
    & \forall s \in \gV_H, \widetilde{s} \in \goodS, \; A (s, \widetilde{s}) = \rho (\widetilde{s}) \lp P_1 (s| \widetilde{s}, \gooda) - \varepsilon \rp \piail_2 (\gooda|s), 
    \\
    &c (s) = \widehat{d^{\piE}_2} (s) - \varepsilon \piail_2 (\gooda|s),
    \\
    &d (\widetilde{s}) = \sum_{s \in \gV_H} \rho (\widetilde{s}) \lp P_1 (s|\widetilde{s}, \gooda) - \varepsilon \rp \piail_2 (\gooda|s).
\end{align*}
First, according to the second argument of Lemma \ref{lem:auxiliary_lemma_extension_one}, we have that $\forall s \in \gV_{H}, \piail_2 (\gooda|s) > 0$. Combined with the assumption $\forall s, s^\prime, P_1 (s^\prime|s, \gooda) > \varepsilon$, we can prove that $\forall s \in \gV_H, \widetilde{s} \in \goodS, A (s, \widetilde{s}) > 0$. Furthermore, on one hand, it holds that
\begin{align*}
    \sum_{s \in \gV_H} c (s) &= \sum_{s \in \gV_H} \widehat{d^{\piE}_2} (s) - \varepsilon \piail_2 (\gooda|s)
    \\
    &= 1 - \varepsilon \lp \sum_{s \in \gV_H} \piail_2 (\gooda|s) \rp.  
\end{align*}
On the other hand, we have that
\begin{align*}
     &\quad \sum_{s \in \gV_H} \sum_{\widetilde{s} \in \goodS} A (s, \widetilde{s})
     \\
     &= \sum_{s \in \gV_H} \sum_{\widetilde{s} \in \goodS} \rho (\widetilde{s}) \lp P_1 (s| \widetilde{s}, \gooda) - \varepsilon \rp \piail_2 (\gooda|s) 
     \\
     &=  \sum_{\widetilde{s} \in \goodS} \rho (\widetilde{s}) \sum_{s \in \gV_H} \lp P_1 (s| \widetilde{s}, \gooda) \piail_2 (\gooda|s)  - \varepsilon  \piail_2 (\gooda|s) \rp
     \\
     &\leq \sum_{\widetilde{s} \in \goodS} \rho (\widetilde{s}) \sum_{s \in \gV_H} \lp P_1 (s| \widetilde{s}, \gooda)  - \varepsilon  \piail_2 (\gooda|s) \rp
     \\
     &\leq \sum_{\widetilde{s} \in \goodS} \rho (\widetilde{s})  \lp 1  - \varepsilon \lp  \sum_{s \in \gV_H}  \piail_2 (\gooda|s) \rp \rp
     \\
     &= 1  - \varepsilon \lp \sum_{s \in \gV_H}  \piail_2 (\gooda|s) \rp. 
\end{align*}
We can derive that $\sum_{s \in \gV_H} \sum_{\widetilde{s} \in \goodS} A (s, \widetilde{s}) \leq \sum_{s \in \gV_H} c (s)$. Finally, for each $\widetilde{s } \in \goodS$, $\sum_{s \in \gV_{H}} A (s, \widetilde{s}) = d (\widetilde{s}) $. We have verified the conditions of Lemma \ref{lem:mn_variables_opt_unique}. According to Lemma \ref{lem:mn_variables_opt_unique}, we have that $\forall s \in \goodS$, $\pi_1 (\gooda|s) = \piE_1 (\gooda|s) =  1$ is the unique optimal solution of $\argmin_{\pi_1} \Loss_2 (\pi_1)$. In summary, we have proved that $\forall s \in \goodS$, $\pi_1 (\gooda|s) = \piE_1 (\gooda|s) =  1$ is an optimal solution of $\argmin_{\pi_1} \Loss_1 (\pi_1) $ and the unique optimal solution of $\argmin_{\pi_1} \Loss_2 (\pi_1)$. According to Lemma \ref{lem:unique_opt_solution_condition}, we can obtain that $\forall s \in \goodS$, $\pi_1 (\gooda|s) = \piE_1 (\gooda|s) =  1$ is the unique optimal solution of $\argmin_{\pi_1} f_2 (\pi_1; \piail_2)$, which completes the proof.

\subsection{Proof of Proposition \ref{prop:rabs_mdps_extension_two}}
\label{appendix:proof_of_extension_two}
In this part, we present the proof of Proposition \ref{prop:rabs_mdps_extension_two}. Different from the analysis for the original RABS MDPs, we need to perform a precise characterization of the visitation probability of good states by starting from bad states.

To prove Proposition \ref{prop:rabs_mdps_extension_two}, we need the following useful lemma.

\begin{lem}
\label{lem:auxiliary_lemma_extension_two}
    Consider the extended RABS MDPs shown in Figure \ref{fig:rabs_mdp_extension_two} and assume that $\varepsilon \leq d^{\piE}_3 (s), \forall s \in \goodS$. We have that
    \begin{align*}
        \forall h \in [3], s \in \goodS, d^{\piail}_h (s) > 0.
    \end{align*}
    Furthermore, we define $\gV_{3} := \{s \in  \goodS: \widehat{d^{\piE}_2} (s) > 0 \}$ as the set of states visited in the expert demonstrations. It holds that
    \begin{align*}
        \forall s \in \gV_3, \; \piail_3 (\gooda|s) > 0.
    \end{align*}
\end{lem}

\begin{proof}
    For the first statement, it is easy to observe that for time step $h=1$, $\forall s \in \goodS$, $d^{\piail}_h (s) = \rho (s) > 0$. Then we turn to consider time steps $h=2, 3$. We prove by contradiction argument. In particular, we assume that in time step $h=2$, $\exists s \in \goodS$, $d^{\piail}_2 (s) = 0$. According to the transition structure of the extended RABS MDPs shown in Figure \ref{fig:rabs_mdp_extension_two}, we can derive that $\forall s \in \goodS, \piail_1 (\gooda|s) = 0$. Furthermore, we can calculate the state-action distribution of $\piail$.
    \begin{align*}
        & d^{\piail}_1 (s^1, \gooda) = d^{\piail}_1 (s^2, \gooda) = d^{\piail}_1 (b) = 0,
        \\
        & d^{\piail}_2 (s^1, \gooda) = d^{\piail}_2 (s^2, \gooda) = 0, \; d^{\piail}_2 (b) = 1,
        \\
        & d^{\piail}_3 (s^1, \gooda) = \varepsilon \cdot \piail_3 (\gooda|s^1),  d^{\piail}_3 (s^2, \gooda) = \varepsilon \cdot \piail_3 (\gooda|s^2),
        \\
        &d^{\piail}_3 (b) = 1-\varepsilon.
    \end{align*}
    Then the distribution matching loss can be calculated as 
    \begin{align*}
    \Loss_1 (\piail) = 2, \Loss_2 (\piail) = 2,    
    \end{align*}
    \begin{align*}
        &\quad \Loss_3 (\piail)
        \\
        &= \sum_{s \in \goodS} \bigg(  \labs \widehat{d^{\piE}_3} (s) -  d^{\piail}_3 (s, \gooda) \rabs + d^{\piail}_3 (s, \bada) \bigg) + d^{\piail}_3 (b)   
        \\
        & = \sum_{s \in \goodS} \bigg(  \labs \widehat{d^{\piE}_3} (s) -  \varepsilon \piail_3 (\gooda|s^1) \rabs + \varepsilon ( 1- \piail_3 (\gooda|s) ) \bigg) + d^{\piail}_3 (b)
        \\
        &\geq \sum_{s \in \goodS} \bigg(   \widehat{d^{\piE}_3} (s) -  \varepsilon \piail_3 (\gooda|s) + \varepsilon ( 1- \piail_3 (\gooda|s) ) \bigg) + 1-\varepsilon
        \\
        &= 2  -  \varepsilon  2  \sum_{s \in \goodS} \piail_3 (\gooda|s)
        \\
        &\geq 2 - 4 \varepsilon.
    \end{align*}
    Then we calculate the distribution matching loss of $\piE$.
	    \begin{align*}
	        \Loss_1 (\piE) = \sum_{s \in \goodS} \labs  \widehat{d^{\piE}_1} (s) - d^{\piE}_1 (s)   \rabs < 2 = \Loss_1 (\piail),
	        \\
	        \Loss_2 (\piE) = \sum_{s \in \goodS} \labs  \widehat{d^{\piE}_2} (s) - d^{\piE}_2 (s)   \rabs < 2 = \Loss_2 (\piail).
	    \end{align*}
    The above inequalities follow that $d^{\piE}_h (\cdot)$ and $ \widehat{d^{\piE}_h} (\cdot)$ have common support. To analyze the distribution matching loss in time step $h=3$, we define the set of states $\gS^{\operatorname{G}, 1}_3 := \{ s \in \goodS: \widehat{d^{\piE}_3} (s) \geq d^{\piE}_3 (s)    \} $ and $\gS^{\operatorname{G}, 2}_3 = \goodS \setminus \gS^{\operatorname{G}, 1}_3$. It is direct to see that $|\gS^{\operatorname{G}, 1}_3| \geq 1$. Then we have that
    \begin{align*}
        &\quad \Loss_3 (\piE)
        \\
        &= \sum_{s \in \goodS} \labs  \widehat{d^{\piE}_3} (s) - d^{\piE}_3 (s)   \rabs
        \\
        &= \sum_{s \in \gS^{\operatorname{G}, 1}_3}   \widehat{d^{\piE}_3} (s) - d^{\piE}_3 (s)    + \sum_{s \in \gS^{\operatorname{G}, 2}_3} d^{\piE}_3 (s) - \widehat{d^{\piE}_3} (s)
        \\
        &= 2 - 2 \sum_{s \in \gS^{\operatorname{G}, 1}_3} d^{\piE}_3 (s) - 2 \sum_{s \in \gS^{\operatorname{G}, 2}_3} \widehat{d^{\piE}_3} (s)
        \\
        &\leq 2 - 2 \min_{s \in \goodS} d^{\piE}_3 (s). 
    \end{align*}
    According to the assumption that $\forall s \in \goodS, \varepsilon \leq d^{\piE}_3 (s) / 2 $, we have that $\Loss_3 (\piE) \leq \Loss_3 (\piail)$. In summary, we have obtained that 
    \begin{align*}
        \sum_{h=1}^3  \Loss_h (\piE) < \sum_{h=1}^3  \Loss_h (\piail), 
    \end{align*}
    which contradicts with the fact that $\piail$ is the optimal solution to the distribution matching loss. Therefore, we have that in time step $h=2$, $\forall s \in \goodS$, $d^{\piail}_2 (s) > 0$.

    Then we continue to prove that in time step $h=3$, $\forall s \in \goodS$, $d^{\piail}_3 (s) > 0$. We also prove this statement by contradiction. We assume that $\exists s \in \goodS$, $d^{\piail}_3 (s) = 0$. According to the transition structure of the extended RABS MDPs shown in Figure \ref{fig:rabs_mdp_extension_two}, we can derive that
    \begin{align*}
        d^{\piail}_2 (b) = 0, \; \forall s \in \goodS, \piail_2 (\gooda|s) = 0.
    \end{align*}
    $d^{\piail}_2 (b) = 0$ further implies that $\forall s \in \goodS, \piail_1 (\gooda|s) = \piE_1 (\gooda|s) =  1$. Then we calculate the distribution matching loss of $\piail$.
    \begin{align*}
        \Loss_1 (\piail) = \Loss_1 (\piE).
    \end{align*}
    \begin{align*}
        \Loss_2 (\piail) = \sum_{s \in \goodS} \labs  \widehat{d^{\piE}_2} (s) - d^{\piail}_2 (s, \gooda)  \rabs + d^{\piail}_2 (s, \bada) = 2. 
    \end{align*}
    \begin{align*}
        \Loss_3 (\piail) = \sum_{s \in \goodS}  \lp \widehat{d^{\piE}_3} (s) \rp + d^{\piail}_3 (b) = 2.
    \end{align*}
    Furthermore, we calculate the distribution matching loss of $\piE$. For time steps $h=2, 3$,
    \begin{align*}
        \Loss_h (\piE) = \sum_{s \in \goodS} \labs  \widehat{d^{\piE}_h} (s) - d^{\piE}_h (s)  \rabs < 2 = \Loss_h (\piail).  
    \end{align*}
    In summary, we have obtained that 
    \begin{align*}
        \sum_{h=1}^3  \Loss_h (\piE) < \sum_{h=1}^3  \Loss_h (\piail), 
    \end{align*}
    which contradicts with the fact that $\piail$ is the optimal solution to the distribution matching loss. Therefore, we have that in time step $h=3$, $\forall s \in \goodS$, $d^{\piail}_3 (s) > 0$, which completes the proof of the first statement.

    Now, we proceed to prove the second statement. According to the optimality condition, we have that
    \begin{align*}
        \piail_3 \in \argmin_{\pi_3} f_2 (\pi_3; \piail_1, \piail_2).
    \end{align*}
    Notice that $\Loss_1, \Loss_2$ are independent of $\pi_3$. Then it holds that
    \begin{align*}
        \piail_3 \in \argmin_{\pi_3} \Loss_3.
    \end{align*}
    For $\Loss_3$, we have that
    \begin{align*}
        \Loss_3 &= \sum_{s \in \goodS} \bigg( \labs \widehat{d^{\piE}_3 } (s) - d^{\piail}_3 (s) \pi_3 (\gooda|s) \rabs + d^{\piail}_3 (s) \lp 1-  \pi_3 (\gooda|s) \rp \bigg)
        \\
        &= \sum_{s \in \gV_3} \bigg( \labs \widehat{d^{\piE}_3 } (s) - d^{\piail}_3 (s) \pi_3 (\gooda|s) \rabs + d^{\piail}_3 (s) \lp 1-  \pi_3 (\gooda|s) \rp \bigg) + \sum_{s \notin \gV_3, s \in \goodS} d^{\piail}_3 (s) + d^{\piail}_3 (b)
        \\
        &= \sum_{s \in \gV_3} \bigg( \labs \widehat{d^{\piE}_3 } (s) - d^{\piail}_3 (s) \pi_3 (\gooda|s) \rabs - d^{\piail}_3 (s)  \pi_3 (\gooda|s) \bigg) + \sum_{s \in \gV_3} d^{\piail}_3  (s)   + \sum_{s \notin \gV_3, s \in \goodS} d^{\piail}_3  (s) + d^{\piail}_3 (b)
        \\
        &= \sum_{s \in \gV_3} \bigg( \labs \widehat{d^{\piE}_3 } (s) - d^{\piail}_3 (s) \pi_3 (\gooda|s) \rabs - d^{\piail}_3 (s)  \pi_3 (\gooda|s) \bigg) + 1.
    \end{align*}
    Then we have that
    \begin{align*}
        \piail_3 \in &\argmin_{\pi_3} \sum_{s \in \gV_3} \bigg( \labs \widehat{d^{\piE}_3 } (s) - d^{\piail}_3 (s) \pi_3 (\gooda|s) \rabs - d^{\piail}_3 (s)  \pi_3 (\gooda|s) \bigg) 
    \end{align*}
    For each $s \in \gV_3$, it holds that
    \begin{align*}
        \piail_3 (\gooda|s) \in &\argmin_{\pi_3 (\gooda|s) \in [0, 1]} \bigg( \labs \widehat{d^{\piE}_3 } (s) - d^{\piail}_3 (s) \pi_3 (\gooda|s) \rabs - d^{\piail}_3 (s)  \pi_3 (\gooda|s) \bigg). 
    \end{align*}
    Based on the first statement that $\forall s \in \goodS, d^{\piail}_3 (s) > 0$. According to Lemma \ref{lem:single_variable_opt_condition}, we have that $\piail_3 (\gooda|s) > 0$, which completes the proof of the second statement.
\end{proof}
Now we are ready to prove Proposition \ref{prop:rabs_mdps_extension_two}. The proof is based on a backward analysis. In particular, we first prove that in time step $h=2$, $\forall s \in \goodS$, $\piail_2 (\gooda|s) = 1$. According to the optimality condition, we have that
\begin{align*}
    \piail_2 \in \argmin_{\pi_2} f_2 (\pi_2; \piail_1, \piail_3).
\end{align*}
As $\Loss_1$ is independent of $\pi_2$, we have 
\begin{align*}
    \piail_2 \in \argmin_{\pi_2}  \Loss_2 (\pi_2) + \Loss_3 (\pi_2).
\end{align*}
For $\Loss_2 (\pi_2)$, we have that
\begin{align*}
    \Loss_2 (\pi_2) &= \sum_{s \in \goodS} \bigg( \labs \widehat{d^{\piE}_2 } (s) - d^{\piail}_2 (s) \pi_2 (\gooda|s) \rabs + d^{\piail}_2 (s) \lp 1-  \pi_2 (\gooda|s) \rp \bigg) + d^{\piail}_2 (b)
        \\
        &= \sum_{s \in \goodS} \bigg( \labs \widehat{d^{\piE}_2 } (s) - d^{\piail}_2 (s) \pi_2 (\gooda|s) \rabs - d^{\piail}_2 (s)  \pi_2 (\gooda|s) \bigg) + 1. 
\end{align*}
Then we have that
\begin{align*}
    \argmin_{\pi_2}  \Loss_2 (\pi_2) &= \argmin_{\pi_2} \sum_{s \in \goodS} \bigg( \labs \widehat{d^{\piE}_2 } (s) - d^{\piail}_2 (s) \pi_2 (\gooda|s) \rabs - d^{\piail}_2 (s)  \pi_2 (\gooda|s) \bigg).   
\end{align*}
For each state $s \in \goodS$, we have that
\begin{align*}
    \piail_2 (\gooda|s) &\in \argmin_{\pi_2 (\gooda|s) \in [0, 1]} \bigg( \labs \widehat{d^{\piE}_2 } (s) - d^{\piail}_2 (s) \pi_2 (\gooda|s) \rabs - d^{\piail}_2 (s)  \pi_2 (\gooda|s) \bigg).
\end{align*}
From Lemma \ref{lem:single_variable_opt}, we have that $ \pi_2 (\gooda|s) = 1$ is an optimal solution to the above optimization problem. Then $\forall s \in \goodS, \pi_2 (\gooda|s) = 1$ is an optimal solution to $\argmin_{\pi_2} \Loss_2 (\pi_2)$. For $\Loss_3 (\pi_2)$, we have that
\begin{align*}
    &\quad \Loss_3 (\pi_2)
    \\
    & = \sum_{s \in \goodS} \bigg( \labs \widehat{d^{\piE}_3} (s) - d^{\piail}_3 (s) \piail_3 (\gooda|s) \rabs + d^{\piail}_3 (s) \lp  1- \piail_3 (\gooda|s) \rp \bigg) + d^{\piail}_3 (b)
    \\
    & = \sum_{s \in \goodS} \bigg( \labs \widehat{d^{\piE}_3} (s) - d^{\piail}_3 (s) \piail_3 (\gooda|s) \rabs - d^{\piail}_3 (s) \piail_3 (\gooda|s) \bigg) + \sum_{s \in \goodS} d^{\piail}_3 (s)  + d^{\piail}_3 (b)
    \\
    & = \sum_{s \in \goodS} \bigg( \labs \widehat{d^{\piE}_3} (s) - d^{\piail}_3 (s) \piail_3 (\gooda|s) \rabs - d^{\piail}_3 (s) \piail_3 (\gooda|s) \bigg) + 1
    \\
	    & = \sum_{s \in \gV_3} \bigg( \labs \widehat{d^{\piE}_3} (s) - d^{\piail}_3 (s) \piail_3 (\gooda|s) \rabs - d^{\piail}_3 (s) \piail_3 (\gooda|s) \bigg)
	    \\
	    &\quad + \sum_{s \notin \gV_3, s \in \goodS} \bigg( \labs \widehat{d^{\piE}_3} (s) - d^{\piail}_3 (s) \piail_3 (\gooda|s) \rabs - d^{\piail}_3 (s) \piail_3 (\gooda|s) \bigg) + 1
    \\
    &\overset{\text{(a)}}{=} \sum_{s \in \gV_3} \bigg( \labs \widehat{d^{\piE}_3} (s) - d^{\piail}_3 (s) \piail_3 (\gooda|s) \rabs - d^{\piail}_3 (s) \piail_3 (\gooda|s) \bigg) + 1
    \\
	    &=  \underbrace{\sum_{s \in \gV_3}  \labs \widehat{d^{\piE}_3} (s) - d^{\piail}_3 (s) \piail_3 (\gooda|s) \rabs}_{\text{Term I}} - \underbrace{\sum_{s \in \gV_3} d^{\piail}_3 (s) \piail_3 (\gooda|s)}_{\text{Term II}} + 1
\end{align*}
Equation (a) follows the definition of $\gV_3$: $\gV_3 := \{ s \in \goodS: \widehat{d^{\piE}_3} (s) > 0  \}$. For $d^{\piail}_3 (s)$, we have that
\begin{align*}
    d^{\piail}_3 (s) &= \sum_{\widetilde{s} \in \goodS} d^{\piail}_2 (\widetilde{s}) \pi_2 (\gooda|\widetilde{s}) P_2 (s|\widetilde{s}, \gooda) + d^{\piail}_2 (b) \sP \lp s_3=s|s_2 = b \rp
    \\
    &= \sum_{\widetilde{s} \in \goodS} d^{\piail}_2 (\widetilde{s}) \pi_2 (\gooda|\widetilde{s}) P_2 (s|\widetilde{s}, \gooda) + d^{\piail}_2 (b) \varepsilon.
\end{align*}
For Term I, we then have that
\begin{align*}
    &\quad \sum_{s \in \gV_3}  \labs \widehat{d^{\piE}_3} (s) - d^{\piail}_3 (s) \piail_3 (\gooda|s) \rabs
    \\
     &= \sum_{s \in \gV_3}  \bigg| \widehat{d^{\piE}_3} (s) - \bigg( \sum_{\widetilde{s} \in \goodS} d^{\piail}_2 (\widetilde{s}) \pi_2 (\gooda|\widetilde{s}) P_2 (s|\widetilde{s}, \gooda) + d^{\piail}_2 (b) \varepsilon \bigg) \piail_3 (\gooda|s) \bigg|
     \\
     &= \sum_{s \in \gV_3}  \bigg| \widehat{d^{\piE}_3} (s) - d^{\piail}_2 (b) \varepsilon \piail_3 (\gooda|s) - \sum_{\widetilde{s} \in \goodS} d^{\piail}_2 (\widetilde{s}) \pi_2 (\gooda|\widetilde{s}) P_2 (s|\widetilde{s}, \gooda) \piail_3 (\gooda|s) \bigg|
     \\
     &= \sum_{s \in \gV_3}  \bigg| \widehat{d^{\piE}_3} (s) - d^{\piail}_2 (b) \varepsilon \piail_3 (\gooda|s) - \sum_{\widetilde{s} \in \goodS} d^{\piail}_2 (\widetilde{s})  P_2 (s|\widetilde{s}, \gooda) \piail_3 (\gooda|s) \pi_2 (\gooda|\widetilde{s}) \bigg|.
\end{align*}
For Term II, we have that
\begin{align*}
    &\quad \sum_{s \in \gV_3} d^{\piail}_3 (s) \piail_3 (\gooda|s)
    \\
    &= \sum_{s \in \gV_3} \bigg( \sum_{\widetilde{s} \in \goodS} d^{\piail}_2 (\widetilde{s}) \pi_2 (\gooda|\widetilde{s}) P_2 (s|\widetilde{s}, \gooda) + d^{\piail}_2 (b) \varepsilon \bigg) \piail_3 (\gooda|s)
    \\
    &= \sum_{s \in \gV_3} \sum_{\widetilde{s} \in \goodS} d^{\piail}_2 (\widetilde{s}) \pi_2 (\gooda|\widetilde{s}) P_2 (s|\widetilde{s}, \gooda) \piail_3 (\gooda|s) + \sum_{s \in \gV_3} d^{\piail}_2 (b) \varepsilon \piail_3 (\gooda|s)
    \\
    &=  \sum_{\widetilde{s} \in \goodS} \lp \sum_{s \in \gV_3} d^{\piail}_2 (\widetilde{s})  P_2 (s|\widetilde{s}, \gooda) \piail_3 (\gooda|s) \rp \pi_2 (\gooda|\widetilde{s}) + \sum_{s \in \gV_3} d^{\piail}_2 (b) \varepsilon \piail_3 (\gooda|s). 
\end{align*}
Combining the above two equations yields that
\begin{align*}
    &\quad \Loss_3 (\pi_2)
    \\
	    &= \sum_{s \in \gV_3}  \bigg| \widehat{d^{\piE}_3} (s) - d^{\piail}_2 (b) \varepsilon \piail_3 (\gooda|s) - \sum_{\widetilde{s} \in \goodS} d^{\piail}_2 (\widetilde{s})  P_2 (s|\widetilde{s}, \gooda) \piail_3 (\gooda|s) \pi_2 (\gooda|\widetilde{s}) \bigg|
	    \\
	    &\quad - \sum_{\widetilde{s} \in \goodS} \lp \sum_{s \in \gV_3} d^{\piail}_2 (\widetilde{s})  P_2 (s|\widetilde{s}, \gooda) \piail_3 (\gooda|s) \rp \pi_2 (\gooda|\widetilde{s}) - \sum_{s \in \gV_3} d^{\piail}_2 (b) \varepsilon \piail_3 (\gooda|s) + 1
    \\
    \\
	    &= \sum_{s \in \gV_3}  \bigg| \widehat{d^{\piE}_3} (s) - d^{\piail}_2 (b) \varepsilon \piail_3 (\gooda|s) - \sum_{\widetilde{s} \in \goodS} d^{\piail}_2 (\widetilde{s})  P_2 (s|\widetilde{s}, \gooda) \piail_3 (\gooda|s) \pi_2 (\gooda|\widetilde{s}) \bigg|
	    \\
	    &\quad - \sum_{\widetilde{s} \in \goodS} \lp \sum_{s \in \gV_3} d^{\piail}_2 (\widetilde{s})  P_2 (s|\widetilde{s}, \gooda) \piail_3 (\gooda|s) \rp \pi_2 (\gooda|\widetilde{s}) + \operatorname{const}.
\end{align*}
Here $\operatorname{const} = 1 - \sum_{s \in \gV_3} d^{\piail}_2 (b) \varepsilon \piail_3 (\gooda|s) $ which is independent of $\pi_2$. Then we obtain that
\begin{align*}
    &\quad \argmin_{\pi_2} \Loss_3 (\pi_2) 
    \\
	    &= \argmin_{\pi_2}  \sum_{s \in \gV_3}  \bigg| \widehat{d^{\piE}_3} (s) - d^{\piail}_2 (b) \varepsilon \piail_3 (\gooda|s) - \sum_{\widetilde{s} \in \goodS} d^{\piail}_2 (\widetilde{s})  P_2 (s|\widetilde{s}, \gooda) \piail_3 (\gooda|s) \pi_2 (\gooda|\widetilde{s}) \bigg|
	    \\
	    &\quad - \sum_{\widetilde{s} \in \goodS} \lp \sum_{s \in \gV_3} d^{\piail}_2 (\widetilde{s})  P_2 (s|\widetilde{s}, \gooda) \piail_3 (\gooda|s) \rp \pi_2 (\gooda|\widetilde{s}).
\end{align*}
We apply Lemma \ref{lem:mn_variables_opt_unique} to characterize the optimal solution of the above optimization problem. To check the conditions of Lemma \ref{lem:mn_variables_opt_unique}, we define that
	\begin{align*}
	    & m = |\gV_3|, n = | \goodS|, \forall s \in \gV_3, c(s) = \widehat{d^{\piE}_3} (s) - d^{\piail}_2 (b) \varepsilon \piail_3 (\gooda|s),
	    \\
	    & \forall s \in \gV_3, \widetilde{s} \in \goodS, A (s, \widetilde{s}) = d^{\piail}_2 (\widetilde{s})  P_2 (s|\widetilde{s}, \gooda) \piail_3 (\gooda|s),
	    \\
	    & \forall \widetilde{s} \in \goodS, d(\widetilde{s}) =  \sum_{s \in \gV_3} d^{\piail}_2 (\widetilde{s})  P_2 (s|\widetilde{s}, \gooda) \piail_3 (\gooda|s). 
	\end{align*}
According to Lemma \ref{lem:auxiliary_lemma_extension_two}, we have that $\forall \widetilde{s} \in \goodS, d^{\piail}_2 (\widetilde{s}) >0$ and $\forall s \in \gV_3, \piail_3 (\gooda|s) > 0$. Then we can get that $\forall \widetilde{s} \in \goodS, s \in \gV_3, A (s, \widetilde{s}) > 0$. Furthermore, we have that
\begin{align*}
    \sum_{s \in \gV_3} c (s) &= \sum_{s \in \gV_3} \widehat{d^{\piE}_3} (s) - d^{\piail}_2 (b) \varepsilon \piail_3 (\gooda|s)
    \\
    &= 1 - d^{\piail}_2 (b) \varepsilon \sum_{s \in \gV_3} \piail_3 (\gooda|s)
    \\
    &\geq 1 - d^{\piail}_2 (b) \varepsilon |\goodS|.
\end{align*}
\begin{align*}
    \sum_{s \in \gV_3} \sum_{\widetilde{s} \in \goodS} A (s, \widetilde{s}) &=  \sum_{s \in \gV_3} \sum_{\widetilde{s} \in \goodS} d^{\piail}_2 (\widetilde{s})  P_2 (s|\widetilde{s}, \gooda) \piail_3 (\gooda|s)
    \\
    &\leq \sum_{s \in \gV_3} \sum_{\widetilde{s} \in \goodS} d^{\piail}_2 (\widetilde{s})  P_2 (s|\widetilde{s}, \gooda)
    \\
    &\leq \sum_{\widetilde{s} \in \goodS} d^{\piail}_2 (\widetilde{s})
    \\
    &= 1 -  d^{\piail}_2 (b). 
\end{align*}
As the transition probability distribution at the bad absorbing state $b$ is a valid distribution, $\varepsilon |\goodS| \leq 1$. Then we get that $\sum_{s \in \gV_3} c (s) \geq \sum_{s \in \gV_3} \sum_{\widetilde{s} \in \goodS} A (s, \widetilde{s})$. Finally, we observe that
\begin{align*}
    \sum_{s \in \gV_3} A (s, \widetilde{s}) = \sum_{s \in \gV_3} d^{\piail}_2 (\widetilde{s})  P_2 (s|\widetilde{s}, \gooda) \piail_3 (\gooda|s) = d (\widetilde{s}).  
\end{align*}
We have checked the conditions of Lemma \ref{lem:mn_variables_opt_unique}. According to Lemma \ref{lem:mn_variables_opt_unique}, we have that $\forall \widetilde{s} \in \goodS, \pi_2 (\gooda|\widetilde{s}) = 1$ is the unique optimal solution to $\argmin_{\pi_2} \Loss_3 (\pi_2)$. Then based on Lemma \ref{lem:unique_opt_solution_condition}, we have that $\forall s \in \goodS, \pi_2 (\gooda|s) = 1 $ is the unique optimal solution of $\argmin_{\pi_2} \Loss_2 (\pi_2) + \Loss_3 (\pi_2) $, which finishes the proof for time step $h=2$.

Then we continue to prove that in time step $h=1$, $\forall s \in \goodS$, $\piail_1 (\gooda|s) = 1$. In particular, we have that
\begin{align*}
    \piail_1 \in \argmin_{\pi_1} f_1 (\pi_1; \piail_2, \piail_3) = \argmin_{\pi_1} \Loss_1 (\pi_1) + \Loss_2 (\pi_1) + \Loss_3 (\pi_1).  
\end{align*}
For $\Loss_1 (\pi_1)$, we have that
\begin{align*}
    &\quad \Loss_1 (\pi_1) 
    \\
    & = \sum_{s \in \goodS} \labs \widehat{d^{\piE}_1} (s) - \rho (s) \pi_1 (\gooda|s) \rabs + \rho (s) \lp 1 - \pi_1 (\gooda|s) \rp
    \\
    &= \sum_{s \in \goodS} \labs \widehat{d^{\piE}_1} (s) - \rho (s) \pi_1 (\gooda|s) \rabs - \rho (s) \pi_1 (\gooda|s) + 1.
\end{align*}
Then we have that for each $s \in \goodS$, we have 
\begin{align*}
    \piail_1 (\gooda|s) \in \argmin_{\pi_1 (\gooda|s) \in [0, 1]} \labs \widehat{d^{\piE}_1} (s) - \rho (s) \pi_1 (\gooda|s) \rabs - \rho (s) \pi_1 (\gooda|s). 
\end{align*}
According to Lemma \ref{lem:single_variable_opt}, we have that $\pi_1 (\gooda|s) = \piE_1 (\gooda|s) = 1$ is an optimal solution to the above optimization problem. Then $\forall s \in \goodS, \pi_1 (\gooda|s) = \piE_1 (\gooda|s) = 1$ is an optimal solution to $\argmin_{\pi_1} \Loss_1 (\pi_1)$. For $\Loss_2 (\pi_1)$, we have that
\begin{align*}
    &\quad \Loss_2 (\pi_1)
    \\
    &= \sum_{s \in \goodS} \bigg( \labs \widehat{d^{\piE}_2} (s) - d^{\piail}_2 (s) \piail_2 (\gooda|s) \rabs + d^{\piail}_2 (s) \lp 1 - \piail_2 (\gooda|s) \rp \bigg) + d^{\piail}_2 (b).  
\end{align*}
Notice that we have proved that in time step $h=2$, $\forall s \in \goodS, \piail_2 (\gooda|s) = 1$. Then we have that
\begin{align*}
    \Loss_2 (\pi_1) = \sum_{s \in \goodS} \labs \widehat{d^{\piE}_2} (s) - d^{\piail}_2 (s) \rabs + d^{\piail}_2 (b).  
\end{align*}
In particular, we can calculate that
\begin{align*}
    & \forall s \in \goodS, d^{\piail}_2 (s) = \sum_{\widetilde{s} \in \goodS} \rho (\widetilde{s}) \pi_1 (\gooda|\widetilde{s}) P_1 (s|\widetilde{s}, \gooda),
    \\
    &  d^{\piail}_2 (b) = \sum_{\widetilde{s} \in \goodS} \rho (\widetilde{s}) \lp 1- \pi_1 (\gooda|\widetilde{s}) \rp. 
\end{align*}
Then we have that
\begin{align*}
     \Loss_2 (\pi_1)= \sum_{s \in \goodS} \labs \widehat{d^{\piE}_2} (s) - \sum_{\widetilde{s} \in \goodS} \rho (\widetilde{s})  P_1 (s|\widetilde{s}, \gooda) \pi_1 (\gooda|\widetilde{s})   \rabs - \sum_{\widetilde{s} \in \goodS} \rho (\widetilde{s}) \pi_1 (\gooda|\widetilde{s}) + 1.  
\end{align*}
\begin{align*}
     \argmin_{\pi_1} \Loss_2 (\pi_1) = \argmin_{\pi_1} \sum_{s \in \goodS} \labs \widehat{d^{\piE}_2} (s) - \sum_{\widetilde{s} \in \goodS} \rho (\widetilde{s})  P_1 (s|\widetilde{s}, \gooda) \pi_1 (\gooda|\widetilde{s})   \rabs - \sum_{\widetilde{s} \in \goodS} \rho (\widetilde{s}) \pi_1 (\gooda|\widetilde{s}). 
\end{align*}
We apply Lemma \ref{lem:mn_variables_opt_unique} to analyze the optimal solution to the above optimization problem. To check the conditions in Lemma \ref{lem:mn_variables_opt_unique}, we define that
\begin{align*}
    & m = n = |\goodS|, \forall s \in \goodS, c (s) = \widehat{d^{\piE}_2} (s),
    \\
    & \forall  s, \widetilde{s} \in \goodS, A (s, \widetilde{s}) = \rho (\widetilde{s})  P_1 (s|\widetilde{s}, \gooda),
    \\
    & \forall \widetilde{s} \in \goodS, d (\widetilde{s}) = \rho (\widetilde{s}).  
\end{align*}
It is direct to verify that $\forall  s, \widetilde{s} \in \goodS, A (s, \widetilde{s}) > 0$.
\begin{align*}
    \sum_{s \in \goodS} c (s) = \sum_{s \in \goodS} \widehat{d^{\piE}_2} (s) = 1.  
\end{align*}
\begin{align*}
    \sum_{s \in \goodS} \sum_{\widetilde{s} \in \goodS} A (s, \widetilde{s}) = \sum_{s \in \goodS} \sum_{\widetilde{s} \in \goodS} \rho (\widetilde{s})  P_1 (s|\widetilde{s}, \gooda) = 1.  
\end{align*}
\begin{align*}
    \sum_{s \in \goodS} A (s, \widetilde{s}) = \sum_{s \in \goodS} \rho (\widetilde{s})  P_1 (s|\widetilde{s}, \gooda) = \rho (\widetilde{s}) = d(\widetilde{s}). 
\end{align*}
We have verified the conditions in Lemma \ref{lem:mn_variables_opt_unique}. According to Lemma \ref{lem:mn_variables_opt_unique}, we have that $\forall \widetilde{s} \in \goodS, \pi_1 (\gooda|\widetilde{s}) = \piE_1 (\gooda|\widetilde{s}) =  1$ is an \emph{unique} optimal solution of $\argmin_{\pi_1} \Loss_2 (\pi_1)$.

For $\Loss_3 (\pi_1)$, we have that
\begin{align*}
    &\quad \Loss_3 (\pi_1)
    \\
    & = \sum_{s \in \goodS} \bigg( \labs \widehat{d^{\piE}_3} (s) - d^{\piail}_3 (s) \piail_3 (\gooda|s) \rabs + d^{\piail}_3 (s) \lp  1- \piail_3 (\gooda|s) \rp \bigg) + d^{\piail}_3 (b)
    \\
    & = \sum_{s \in \goodS} \bigg( \labs \widehat{d^{\piE}_3} (s) - d^{\piail}_3 (s) \piail_3 (\gooda|s) \rabs - d^{\piail}_3 (s) \piail_3 (\gooda|s) \bigg) + \sum_{s \in \goodS} d^{\piail}_3 (s)  + d^{\piail}_3 (b)
    \\
    & = \sum_{s \in \goodS} \bigg( \labs \widehat{d^{\piE}_3} (s) - d^{\piail}_3 (s) \piail_3 (\gooda|s) \rabs - d^{\piail}_3 (s) \piail_3 (\gooda|s) \bigg) + 1
    \\
	    & = \sum_{s \in \gV_3} \bigg( \labs \widehat{d^{\piE}_3} (s) - d^{\piail}_3 (s) \piail_3 (\gooda|s) \rabs - d^{\piail}_3 (s) \piail_3 (\gooda|s) \bigg)
	    \\
	    &\quad + \sum_{s \notin \gV_3, s \in \goodS} \bigg( \labs \widehat{d^{\piE}_3} (s) - d^{\piail}_3 (s) \piail_3 (\gooda|s) \rabs - d^{\piail}_3 (s) \piail_3 (\gooda|s) \bigg) + 1
    \\
    &= \sum_{s \in \gV_3} \bigg( \labs \widehat{d^{\piE}_3} (s) - d^{\piail}_3 (s) \piail_3 (\gooda|s) \rabs - d^{\piail}_3 (s) \piail_3 (\gooda|s) \bigg) + 1
    \\
    &=  \sum_{s \in \gV_3}  \labs \widehat{d^{\piE}_3} (s) - d^{\piail}_3 (s) \piail_3 (\gooda|s) \rabs - \sum_{s \in \gV_3} d^{\piail}_3 (s) \piail_3 (\gooda|s) + 1.
\end{align*}
We can calculate that $\forall s \in \goodS$,
\begin{align*}
    &\quad d^{\piail}_3 (s)
    \\
    &= \sum_{\widetilde{s} \in \goodS} \rho (\widetilde{s}) \pi_1 (\gooda|\widetilde{s}) \sP^{\piail} (s_3=s | s_1 = \widetilde{s}, a_1 = \gooda) + \rho (\widetilde{s}) \lp 1-\pi_1 (\gooda|\widetilde{s}) \rp \sP^{\piail} (s_3=s | s_1 = \widetilde{s}, a_1 = \bada)
    \\
    &\overset{\text{(a)}}{=} \sum_{\widetilde{s} \in \goodS} \rho (\widetilde{s}) \pi_1 (\gooda|\widetilde{s}) \sP^{\piE} (s_3=s | s_1 = \widetilde{s}, a_1 = \gooda) + \rho (\widetilde{s}) \lp 1-\pi_1 (\gooda|\widetilde{s}) \rp \sP^{\piail} (s_3=s | s_1 = \widetilde{s}, a_1 = \bada)
    \\
    &\overset{\text{(b)}}{=} \sum_{\widetilde{s} \in \goodS} \rho (\widetilde{s}) \pi_1 (\gooda|\widetilde{s}) \sP^{\piE} (s_3=s | s_1 = \widetilde{s}, a_1 = \gooda) + \rho (\widetilde{s}) \lp 1-\pi_1 (\gooda|\widetilde{s}) \rp \sP^{\piail} (s_3=s | s_2 =b)
    \\
    &= \sum_{\widetilde{s} \in \goodS} \rho (\widetilde{s}) \pi_1 (\gooda|\widetilde{s}) \sP^{\piE} (s_3=s | s_1 = \widetilde{s}, a_1 = \gooda) + \rho (\widetilde{s}) \lp 1-\pi_1 (\gooda|\widetilde{s}) \rp \varepsilon
    \\
    &= \sum_{\widetilde{s} \in \goodS} \rho (\widetilde{s}) \pi_1 (\gooda|\widetilde{s}) \lp \sP^{\piE} (s_3=s | s_1 = \widetilde{s}, a_1 = \gooda) - \varepsilon \rp + \sum_{\widetilde{s} \in \goodS} \rho (\widetilde{s}) \varepsilon.
\end{align*}
Equation (a) follows that $\sP^{\piail} (s_3=s | s_1 = \widetilde{s}, a_1 = \gooda)$ depends on $\forall \widehat{s} \in \goodS, \piail_2 (\cdot|\widehat{s}) $ and $\forall \widehat{s} \in \goodS, \piail_2 (\gooda|\widehat{s}) = \piE_2 (\gooda|\widehat{s}) = 1 $ which has been proved previously. Equation (b) follows that taking the non-expert action $\bada$ on any good state deterministically transits into the bad state. Then we have that
\begin{align*}
    &\quad \Loss_3 (\pi_1)
    \\
    &=  \sum_{s \in \gV_3}  \labs \widehat{d^{\piE}_3} (s) - d^{\piail}_3 (s) \piail_3 (\gooda|s) \rabs - \sum_{s \in \gV_3} d^{\piail}_3 (s) \piail_3 (\gooda|s) + 1
    \\
    &= \sum_{s \in \gV_3}  \bigg| \widehat{d^{\piE}_3} (s) -  \sum_{\widetilde{s} \in \goodS} \rho (\widetilde{s}) \varepsilon \piail_3 (\gooda|s) - \sum_{\widetilde{s} \in \goodS} \rho (\widetilde{s}) \pi_1 (\gooda|\widetilde{s}) \cdot \lp \sP^{\piE} (s_3=s | s_1 = \widetilde{s}, a_1 = \gooda) - \varepsilon \rp  \piail_3 (\gooda|s) \bigg| 
    \\
    &\; - \sum_{s \in \gV_3} \bigg( \sum_{\widetilde{s} \in \goodS} \rho (\widetilde{s}) \pi_1 (\gooda|\widetilde{s}) \lp \sP^{\piE} (s_3=s | s_1 = \widetilde{s}, a_1 = \gooda) - \varepsilon \rp + \sum_{\widetilde{s} \in \goodS} \rho (\widetilde{s}) \varepsilon  \bigg) \piail_3 (\gooda|s) + 1
    \\
    &= \sum_{s \in \gV_3}  \bigg| \widehat{d^{\piE}_3} (s) -  \sum_{\widetilde{s} \in \goodS} \rho (\widetilde{s}) \varepsilon \piail_3 (\gooda|s) - \sum_{\widetilde{s} \in \goodS} \rho (\widetilde{s}) \pi_1 (\gooda|\widetilde{s}) \cdot \lp \sP^{\piE} (s_3=s | s_1 = \widetilde{s}, a_1 = \gooda) - \varepsilon \rp  \piail_3 (\gooda|s) \bigg| 
    \\
    &\; - \sum_{s \in \gV_3} \bigg( \sum_{\widetilde{s} \in \goodS} \rho (\widetilde{s}) \pi_1 (\gooda|\widetilde{s}) \cdot \lp \sP^{\piE} (s_3=s | s_1 = \widetilde{s}, a_1 = \gooda) - \varepsilon \rp \bigg)   \piail_3 (\gooda|s) - \sum_{s \in \gV_3} \sum_{\widetilde{s} \in \goodS} \rho (\widetilde{s}) \varepsilon  \piail_3 (\gooda|s) + 1
    \\
	    &= \sum_{s \in \gV_3}  \bigg| \widehat{d^{\piE}_3} (s) -  \sum_{\widetilde{s} \in \goodS} \rho (\widetilde{s}) \varepsilon \piail_3 (\gooda|s) - \sum_{\widetilde{s} \in \goodS} \rho (\widetilde{s}) \piail_3 (\gooda|s) \cdot \lp \sP^{\piE} (s_3=s | s_1 = \widetilde{s}, a_1 = \gooda) - \varepsilon \rp   \pi_1 (\gooda|\widetilde{s}) \bigg| 
        \\
        &\; - \sum_{\widetilde{s} \in \goodS}  \bigg( \sum_{s \in \gV_3}  \rho (\widetilde{s}) \cdot \lp \sP^{\piE} (s_3=s | s_1 = \widetilde{s}, a_1 = \gooda) - \varepsilon \rp \cdot \piail_3 (\gooda|s) \bigg) \cdot \pi_1 (\gooda|\widetilde{s}) + \operatorname{const}.
\end{align*}
Here $\operatorname{const} = 1 - \sum_{s \in \gV_3} \sum_{\widetilde{s} \in \goodS} \rho (\widetilde{s}) \varepsilon  \piail_3 (\gooda|s)$, which is independent of $\pi_1$. Therefore, we have that
\begin{align*}
    & \argmin_{\pi_1} \Loss_3 (\pi_1)
    \\
    &= \argmin_{\pi_1} \sum_{s \in \gV_3}  \bigg| \widehat{d^{\piE}_3} (s) -  \sum_{\widetilde{s} \in \goodS} \rho (\widetilde{s}) \varepsilon \piail_3 (\gooda|s) - \sum_{\widetilde{s} \in \goodS} \rho (\widetilde{s}) \piail_3 (\gooda|s) \cdot \lp \sP^{\piE} (s_3=s | s_1 = \widetilde{s}, a_1 = \gooda) - \varepsilon \rp   \pi_1 (\gooda|\widetilde{s}) \bigg| 
    \\
    &\; - \sum_{\widetilde{s} \in \goodS}  \bigg( \sum_{s \in \gV_3}  \rho (\widetilde{s}) \cdot \lp \sP^{\piE} (s_3=s | s_1 = \widetilde{s}, a_1 = \gooda) - \varepsilon \rp  \piail_3 (\gooda|s) \bigg) \pi_1 (\gooda|\widetilde{s}). 
\end{align*}
Similarly, we apply Lemma \ref{lem:mn_variables_opt_unique} to analyze the optimal solution of the above optimization problem.  In particular, we define that
	\begin{align*}
	    & m = |\gV_3|, n = |\goodS|, \forall s \in \gV_3, c (s) = \widehat{d^{\piE}_3} (s) -  \sum_{\widetilde{s} \in \goodS} \rho (\widetilde{s}) \varepsilon \piail_3 (\gooda|s).
	\end{align*}
\begin{align*}
    &\forall s \in \gV_3, \forall \widetilde{s} \in \goodS, A (s, \widetilde{s}) = \rho (\widetilde{s}) \piail_3 (\gooda|s) \lp \sP^{\piE} (s_3=s | s_1 = \widetilde{s}, a_1 = \gooda) - \varepsilon \rp,
    \\
    &\forall \widetilde{s} \in \goodS, d ( \widetilde{s}) = \bigg( \sum_{s \in \gV_3}  \rho (\widetilde{s}) \cdot \lp \sP^{\piE} (s_3=s | s_1 = \widetilde{s}, a_1 = \gooda) - \varepsilon \rp  \piail_3 (\gooda|s) \bigg)  
\end{align*}
Recall the assumption that $\varepsilon < \sP^{\piE} (s_3=s | s_1 = \widetilde{s}, a_1 = \gooda) , \forall s, \widetilde{s} \in \goodS$. According to Lemma \ref{lem:auxiliary_lemma_extension_two}, we have that $\forall s \in \gV_3, \piail_3 (\gooda|s) > 0$. Then we can derive that $\forall s \in \gV_3, \forall \widetilde{s} \in \goodS, A (s, \widetilde{s}) > 0$. Furthermore, we have that
\begin{align*}
    \sum_{s \in \gV_3} c (s) & =  \sum_{s \in \gV_3} \lp  \widehat{d^{\piE}_3} (s) -  \sum_{\widetilde{s} \in \goodS} \rho (\widetilde{s}) \varepsilon \piail_3 (\gooda|s) \rp
    \\
    &= 1 - \sum_{s \in \gV_3} \sum_{\widetilde{s} \in \goodS} \rho (\widetilde{s}) \varepsilon \piail_3 (\gooda|s)
    \\
    &\geq 1 -  \sum_{s \in \gV_3} \sum_{\widetilde{s} \in \goodS} \rho (\widetilde{s}) \varepsilon
    \\
    &\geq 1- |\goodS| \varepsilon.
\end{align*}
	\begin{align*}
	    \sum_{s \in \gV_3} \sum_{\widetilde{s} \in \goodS} A (s, \widetilde{s}) &\leq \sum_{s \in \gV_3} \sum_{\widetilde{s} \in \goodS} \rho (\widetilde{s}) \lp \sP^{\piE} (s_3=s | s_1 = \widetilde{s}, a_1 = \gooda) - \varepsilon \rp \leq 1 - |\goodS| \varepsilon.
	\end{align*}
	\begin{align*}
	    \sum_{s \in \gV_3} A (s, \widetilde{s}) &= \sum_{s \in \gV_3}  \rho (\widetilde{s}) \piail_3 (\gooda|s) \lp \sP^{\piE} (s_3=s | s_1 = \widetilde{s}, a_1 = \gooda) - \varepsilon \rp = d (\widetilde{s}).
	\end{align*}
We have checked the conditions of Lemma \ref{lem:mn_variables_opt_unique}. According to Lemma \ref{lem:mn_variables_opt_unique}, $\forall s \in \goodS, \pi_1 (\gooda|s) = \piE_1 (\gooda|s) =  1$ is the \emph{unique} optimal solution to $\argmin_{\pi_1} \Loss_3 (\pi_1)$. In summary, we have proved that $\forall s \in \goodS, \pi_1 (\gooda|s) = \piE_1 (\gooda|s)= 1$ is an optimal solution to $\argmin_{\pi_1} \Loss_1 (\pi_1)$, the unique optimal solution to $\argmin_{\pi_1} \Loss_2 (\pi_1)$ and the unique optimal solution to $\argmin_{\pi_1} \Loss_3 (\pi_1)$. According to Lemma \ref{lem:unique_opt_solution_condition}, we can derive that $\forall s \in \goodS, \pi_1 (\gooda|s) = \piE_1 (\gooda|s)= 1$ is the unique optimal solution to $\argmin_{\pi_1} f_1 (\pi_1; \piail_2, \piail_3)$, which implies that $\forall s \in \goodS, \piail_1 (\gooda|s) = \piE_1 (\gooda|s)= 1$. We complete the proof in time step $h=1$.  

\section{Proof of Results in Section \ref{sec:a_matching_lower_bound}}

\subsection{An Example for TV-AIL in MDPs Satisfying Assumption \ref{asmp:standard_imitation}}
\label{appendix:example_standard_imitation}
\begin{example}
Consider a simple MDP where $\gS = \{ s^{1}, s^{2} \}$ and $\gA = \{ \GREEN{a}, \BLUE{a} \}$; see \cref{fig:toy_standard_imitation}. Without loss of generality, we let $H=1$ and omit the subscript. 
Suppose the initial state distribution $\rho = (0.5, 0.5)$. The agent is provided with 10 trajectories: 4 trajectories start from $s^{1}$ and the others start from $s^{2}$. 

For TV-AIL, it is easy to calculate the empirical distribution:
\begin{align*}
    \widehat{d^{\piE}} (s^{1}, \GREEN{a}) = 0.4, \widehat{d^{\piE}} (s^{1}, \BLUE{a}) = 0, \quad \widehat{d^{\piE}} (s^{2}, \GREEN{a}) = 0.6, \widehat{d^{\piE}} (s^{2}, \BLUE{a}) = 0.
\end{align*}
Note that there are multiple optimal solutions for the piece-wise linear optimization in TV-AIL; refer to \eqref{eq:ail_piece_wise_linear}. For instance, $\pi(a^1 | s^1) = 0.8, \pi(a^2 | s^1) = 0.2, \pi(a^1 | s^2) = 1.0$, and
\begin{align*}
    d^{\pi}(s^1, a^1) = 0.4, d^{\pi}(s^1, a^2) = 0.1, \quad d^{\pi}(s^2, a^1) = 0.5, d^{\pi}(s^2, a^2) = 0.0.
\end{align*}
For such an optimal policy, the state-action distribution matching loss is $0.2$ and the imitation gap is $0.1$. 
\end{example}

\subsection{Proof of Proposition \ref{proposition:ail_policy_value_gap_standard_imitation}}
\label{appendix:proof_proposition:ail_policy_value_gap_standard_imitation}

\begin{proof}[Proof of \cref{proposition:ail_policy_value_gap_standard_imitation}]

As we have analyzed, for any tabular and episodic MDP satisfying \cref{asmp:standard_imitation}, we have that $d^{\pi}_h (s) = \rho (s), \forall s \in \gS, h \in [H]$. Then we obtain
\begin{align*}
  &\quad \argmin_{\pi \in \Pi} \sum_{h=1}^{H} \sum_{(s, a) \in \gS \times \gA} \labs d^{\pi}_h(s, a) - \widehat{d^{\piE}_h}(s, a) \rabs 
  \\
  &= \argmin_{\pi \in \Pi}  \sum_{h=1}^{H} \sum_{(s, a) \in \gS \times \gA} \labs \rho (s) \pi_h (a|s) - \widehat{d^{\piE}_h}(s, a) \rabs 
  \\
  &= \argmin_{\pi \in \Pi} \bigg\{ \sum_{h=1}^{H} \sum_{s \in \gS } \bigg( \labs \rho (s) \pi_h (a^{1}|s) - \widehat{d^{\piE}_h}(s, a^{1}) \rabs + \sum_{a \in \gA \setminus \{ a^{1} \} } \labs \rho(s) \pi(a|s) - \widehat{d^{\piE}_h}(s, a) \rabs  \bigg) \bigg\}
  \\
  &= \argmin_{\pi \in \Pi}  \bigg\{ \sum_{h=1}^{H} \sum_{s \in \gS } \bigg( \labs \rho (s) \pi_h (a^{1}|s) - \widehat{d^{\piE}_h}(s) \rabs + \rho (s) \lp 1- \pi_h (a^{1}|s) \rp \bigg) \bigg\}
  \\
  &= \argmin_{\pi \in \Pi} \bigg\{  \sum_{h=1}^{H} \sum_{s \in \gS } \bigg( \labs  \rho (s) \pi_h (a^{1}|s) - \widehat{d^{\piE}_h}(s) \rabs - \rho (s) \pi_h (a^{1}|s) \bigg) \bigg\}.
\end{align*}
We see that the above multi-stage policy optimization reduces to $H$ independent state-action distribution matching problems: for each $h \in [H]$ and $s \in \gS$, we solve 
\begin{align}   \label{eq:prop_standard_imitation_step_1}
    \piail_{h} \in \min_{\pi_{h}}  \labs \rho(s) \pi_{h}(a^{1} |s ) - \widehat{d^{\piE}_h} \rabs - \rho(s) \pi_h(a^{1} | s).
\end{align}
For this optimization problem, we introduce the notation $\gW_{h} = \{s \in \gS: \widehat{d^{\piE}_h}(s) < \rho(s) \}$. Then, with \cref{lem:single_variable_opt_condition}, we have that the optimal solution set is $\{\piail_h(\cdot|s):  \piail_h(a^{1}|s) \in  [\widehat{d^{\piE}_h} (s) / \rho (s), 1]\}$. On the other hand, for any state $s \in \gW_h^{c}$ (i.e., the complement set of $\gW_h$), the problem in \eqref{eq:prop_standard_imitation_step_1} reduces to 
\begin{align*}
 \piail_h(a^{1}|s) = \argmin_{\pi_h (a^{1}|s) \in [0, 1]}  \widehat{d^{\piE}_h}(s) - 2\rho (s) \pi_h (a^{1}|s).
\end{align*}
In this case, it is easy to see that the unique optimal solution is $\pi_h (a^{1}|s) = 1$. Hence, the proof of the first point in \cref{proposition:ail_policy_value_gap_standard_imitation} is finished.

Next, we continue to prove the second and third points in \cref{proposition:ail_policy_value_gap_standard_imitation}. To start with, we note that 
\begin{align}
    &\quad  V({\piE}) - V({\piail}) \nonumber  \\
    &= \sum_{h=1}^{H} \sum_{(s, a) \in \gS \times \gA}  \lp d^{\piE}_h(s, a) - d^{\piail}_h(s, a) \rp r_h (s, a)
  \nonumber   \\
    &= \sum_{h=1}^{H} \sum_{s \in \gS}  d^{\piE}_h(s, a^{1}) - d^{\piail}_h(s, a^{1})  \nonumber 
    \\
    &= \sum_{h=1}^{H} \sum_{s \in \gS } d^{\piE}_h(s)  - \rho(s) \piail_h( a^{1}|s). \label{eq:prop_standard_imitation_step_3}
\end{align}
According to the first point in \cref{proposition:ail_policy_value_gap_standard_imitation}, we further find that among all optimal solutions, the largest imitation gap is obtained at   $\piail_h$ with  $\piail_h (a^1|s) = \widehat{d^{\piE}_h}(s) / \rho (s), \forall s \in \gW_h^1$ and $\piail_h (a^1|s) = 1, \forall s \in \gW_{h}^c$. Accordingly, the largest imitation gap is 
\begin{align} 
    &\quad V({\piE}) - V({\piail}) \nonumber  \\
    &= \sum_{h=1}^{H} \sum_{s \in \gW_h} d^{\piE}_h(s)  - \rho(s) \piail_h( a^{1}|s) \nonumber \\
    &= \sum_{h=1}^{H} \sum_{s \in \gW_h} d^{\piE}_h(s)  - \widehat{d^{\piE}_h}(s).   \label{eq:prop_standard_imitation_step_2}
\end{align}
In the sequel, we connect the term $\sum_{s \in \gW_h} d^{\piE}_h(s)  - \widehat{d^{\piE}_h}(s)$ with the estimation error. Notice that for each time step $h \in [H]$, $\sum_{s \in \gS} d^{\piE}_h(s) = \sum_{s \in \gS} \widehat{d^{\piE}_h}(s)$ = 1. Then we have that
\begin{align*}
    \sum_{s \in \gW_h} d^{\piE}_h(s)  - \widehat{d^{\piE}_h}(s) = \sum_{s \in \gW_h^c} \widehat{d^{\piE}_h}(s) - d^{\piE}_h(s).  
\end{align*}
Furthermore, we obtain
\begin{align}
     &\quad \sum_{h=1}^H \lnorm \widehat{d^{\piE}_h} - d^{\piE}_h   \rnorm_1 \nonumber  \\
     &= \sum_{h=1}^H \sum_{s \in \gS} \labs d^{\piE}_h(s) - \widehat{d^{\piE}_h} (s)  \rabs  \nonumber 
     \\
     &= \sum_{h=1}^H \bigg[ \sum_{s \in \gW_h} d^{\piE}_h(s)  - \widehat{d^{\piE}_h}(s) + \sum_{s \in \gW_h^c} \widehat{d^{\piE}_h}(s) - d^{\piE}_h(s) \bigg]  \nonumber 
     \\
     &= 2 \sum_{h=1}^H \sum_{s \in \gW_h} d^{\piE}_h(s)  - \widehat{d^{\piE}_h}(s),  \label{eq:prop_standard_imitation_step_4}
\end{align}
where the penultimate equality follows the definition $\gW_h = \{s \in \gS: \widehat{d^{\piE}_h} (s) < \rho (s) = d^{\piE}_h(s) \}$. Finally, back to \eqref{eq:prop_standard_imitation_step_2}, we get that
\begin{align*}
    V({\piE}) - V({\piail}) &= \sum_{h=1}^{H} \sum_{s \in \gW_h} d^{\piE}_h(s)  - \widehat{d^{\piE}_h}(s) = \frac{1}{2}  \sum_{h=1}^H \lnorm \widehat{d^{\piE}_h} - d^{\piE}_h   \rnorm_1.
\end{align*}
Taking the expectation over the randomness in collecting the dataset finishes the proof.
\end{proof}

\subsection{Proof of Theorem \ref{thm:tv_error_lower_bound_small_data}}

\label{appendix:proof_thm:tv_error_lower_bound_small_data}

\begin{proof}[Proof of \cref{thm:tv_error_lower_bound_small_data}]
First, we consider the small sample regime where $N \lesssim |\gX|$. To prove this lower bound, we draw a connection between the $\ell_1$-norm based estimation error and the missing mass \citep{good1953population, mcallester03concentration, rajaraman2020fundamental}. Specifically, we construct a multinomial distribution $Q^\prime$ as follows.
\begin{align*}
    Q^\prime &= \lp Q^\prime (1), \ldots, Q^\prime(\vert \gX \vert-1), Q^\prime(\vert \gX \vert)  \rp = \lp \frac{1}{|\gX|+1}, \ldots, \frac{1}{|\gX|+1}, 1 - \frac{|\gX|-1}{|\gX|+1}  \rp.
\end{align*}
We consider the regime $N \lesssim |\gX|$. Hence, there exists a constant $c > 0$ such that $N \leq c |\gX|$. Recall the estimator $\widehat{Q}$:
\begin{align*}
    \widehat{Q}(i) = \frac{N (i)}{N},
\end{align*}
where $N (i) = \sum_{j} \mathbb{I}(X_j = i)$ denotes the number that the symbol $i$ is observed in $N$ i.i.d. samples $(X_1, \ldots, X_N)$. Then we have that
\begin{align*}
    \lnorm Q^\prime - \widehat{Q} \rnorm_1 &= \sum_{i=1}^{\vert \gX \vert} \labs Q^\prime(i) - \widehat{Q} (i) \rabs
    \\
    &\geq \sum_{i=1}^{\vert \gX \vert} \labs Q^\prime(i) - \widehat{Q} (i) \rabs \indict \{ N (i) = 0 \}
    \\
    &= \sum_{i=1}^{\vert \gX \vert} Q^\prime(i) \indict \{ N (i) = 0 \}. 
\end{align*}
We note that the term in RHS is called missing mass in the statistics literature, which is defined as the probability mass of symbols unobserved in the dataset $(X_1, \ldots, X_N)$ \citep{good1953population, mcallester03concentration}. Then we have that
\begin{align*}
    \expect \ls \lnorm Q^\prime - \widehat{Q} \rnorm_1 \rs &\geq \expect \ls \sum_{i=1}^{\vert \gX \vert} Q^\prime(i) \indict \{ N (i) = 0 \} \rs 
    \\
    &= \sum_{i=1}^{\vert \gX \vert} Q^\prime(i) \sP \lp N (i) = 0  \rp
    \\
    &= \sum_{i=1}^{\vert \gX \vert} Q^\prime(i) \lp 1 - Q^\prime(i) \rp^N
    \\
    &\geq \sum_{i=1}^{\vert \gX \vert} Q^\prime(i) \lp 1 - Q^\prime(i) \rp^{c |\gX|}. 
\end{align*}
The last inequality follows $ 0 < 1 - Q^\prime(i) \leq 1$ and $N \leq c |\gX|$. Next, we derive that
\begin{align*}
    \expect \ls \lnorm Q^\prime - \widehat{Q} \rnorm_1 \rs &\geq \sum_{i=1}^{\vert \gX \vert} Q^\prime(i) \lp 1 - Q^\prime(i) \rp^{c |\gX|}
    \\
    &\geq \sum_{i=1}^{\vert \gX \vert-1} Q^\prime(i) \lp 1 - Q^\prime(i) \rp^{c |\gX|} 
    \\
    &= \sum_{i=1}^{\vert \gX \vert-1} \frac{1}{\vert \gX \vert +1} \lp 1 - \frac{1}{\vert \gX \vert +1}  \rp^{c |\gX|}
    \\
    &= \frac{\vert \gX \vert -1}{\vert \gX \vert +1} \lp 1 - \frac{1}{\vert \gX \vert +1}  \rp^{c |\gX|}
    \\
    &= \frac{\vert \gX \vert -1}{\vert \gX \vert +1} \lp  \frac{\vert \gX \vert}{\vert \gX \vert +1}  \rp^{c |\gX|} 
    \\
    &\overset{(a)}{\geq} \frac{\vert \gX \vert -1}{\vert \gX \vert +1} \cdot e^{-c}
    \\
    &\overset{(b)}{\geq} \frac{1}{3 e^{c}}. 
\end{align*}
In the inequality $(a)$, we use the fact that $(1+1/\vert \gX \vert)^{c \vert \gX \vert} \leq e^{c}$ and the inequality $(b)$ follows that $\vert \gX \vert \geq 2$. Then we get that
\begin{align*}
    \max_{Q \in \gQ} \expect \ls \lnorm Q - \widehat{Q} \rnorm_1 \rs \geq \expect \ls \lnorm Q^\prime - \widehat{Q} \rnorm_1 \rs \geq  \frac{1}{3 e^{c}} \gtrsim 1, 
\end{align*}
which completes the proof in the small sample regime. The lower bound in the large sample regime can be obtained directly from \citep[Lemma 8]{kamath2015learning}.
\end{proof}

\subsection{Proof of Proposition \ref{prop:lower_bound_vail}}

\begin{proof}[Proof of \cref{prop:lower_bound_vail}]
According to \eqref{eq:prop_standard_imitation_step_3} in the proof of \cref{proposition:ail_policy_value_gap_standard_imitation}, we have that
\begin{align*}
    V({\piE}) - V({\piail})   &= \sum_{h=1}^{H} \sum_{s \in \gW_h } d^{\piE}_h(s)  - \rho(s) \piail_h( a^{1}|s) .
\end{align*}
Note that the optimal solution is not unique on the lower bound instances. Taking expectation with respect to the uniform selection of $\piail$  yields that
\begin{align*}
    &\quad V({\piE}) - \expect_{\piail \sim \text{Unif} (\Pi^{\ail})} \ls V({\piail}) \rs \\
    &= \expect_{\piail \sim \text{Unif} (\Pi^{\ail})} \ls \sum_{h=1}^{H} \sum_{s \in \gW_h } d^{\piE}_h(s)  - \rho(s) \piail_h( a^{1}|s) \rs 
    \\
    &=  \sum_{h=1}^{H} \sum_{s \in \gW_h } \bigg[ d^{\piE}_h(s) \\
    &\quad - \rho (s) \expect_{\piail_h( a^{1}|s) \sim \text{Unif} (  [\widehat{d^{\piE}_h}(s) / \rho (s), 1] ) } \ls \piail_h( a^{1}|s) \rs \bigg] 
    \\
    &= \sum_{h=1}^{H} \sum_{s \in \gW_h } \ls d^{\piE}_h(s) - \rho (s) \frac{\widehat{d^{\piE}_h}(s) / \rho(s) + 1}{2} \rs \\
    &= \frac{1}{2} \sum_{h=1}^{H} \sum_{s \in \gW_h}  d^{\piE}_h(s)  - \widehat{d^{\piE}_h}(s) ,
\end{align*}
where in the last equation we use the fact that $d^{\piE}_h(s) = \rho(s)$. 
Combing with \eqref{eq:prop_standard_imitation_step_4}, we have that
\begin{align*}
    V({\piE}) - \expect_{\piail \sim \text{Unif} (\Pi^{\ail})} \ls V({\piail}) \rs &= \frac{1}{2} \sum_{h=1}^{H} \sum_{s \in \gW_h} d^{\piE}_h(s)  - \widehat{d^{\piE}_h}(s) = \frac{1}{4} \sum_{h=1}^H \lnorm \widehat{d^{\piE}_h} - d^{\piE}_h   \rnorm_1.
\end{align*}
We further take the expectation over the randomness in collecting expert trajectories on both sides.
\begin{align*}
    V({\piE}) - \expect \ls \expect_{\piail \sim \text{Unif} (\Pi^{\ail})} \ls V({\piail}) \rs \rs &= \frac{1}{4} \sum_{h=1}^H \expect \ls \lnorm \widehat{d^{\piE}_h} - d^{\piE}_h   \rnorm_1 \rs.
\end{align*}
Finally, we apply the lower bounds of $\ell_1$-risk of the empirical distribution $\widehat{d^{\piE}}$ (refer to \cref{thm:tv_error_lower_bound_small_data}) to obtain the desired result. 
\end{proof}

\section{Technical Lemmas}
\label{appendix:technical_lemmas}

\subsection{Basic Technical Lemmas}

\begin{lem}   \label{lemma:state_dist_discrepancy}
For any tabular and episodic MDP, considering two policies $\pi$ and $\pi^\prime$, let $d^{\pi}_h (\cdot)$ and $d^{\pi}_h (\cdot, \cdot)$ denote the state distribution and state-action distribution induced by $\pi$ in time step $h$, respectively. Then we have that
\begin{itemize}
    \item 
    $
    \Vert d^{\pi}_h (\cdot) - d^{\pi^\prime}_h (\cdot)  \Vert_{1} \leq \sum_{\ell=1}^{h-1} \expect_{s \sim d^{\pi^\prime}_{\ell} (\cdot)} [\lnorm \pi_{\ell} (\cdot|s) - \pi^\prime_{\ell} (\cdot|s)   \rnorm_1]
    $ when $h \geq 2$.
    \item $\Vert d^{\pi}_h (\cdot, \cdot) - d^{\pi^\prime}_h (\cdot, \cdot)  \Vert_{1} \leq \Vert d^{\pi}_h (\cdot) - d^{\pi^\prime}_h (\cdot)  \Vert_{1} + \expect_{s \sim d^{\pi^\prime}_h (\cdot)} \ls \lnorm \pi_h (\cdot|s)  - \pi^\prime_h (\cdot|s) \rnorm_1 \rs $.
\end{itemize}
\end{lem}

\begin{proof}
For the first statement, it is direct to obtain that $\Vert d^{\pi}_1 (\cdot) - d^{\pi^\prime}_1 (\cdot)  \Vert_{1} = \lnorm \rho (\cdot) - \rho (\cdot)  \rnorm_{1}= 0$. When $h \geq 2$, for any $\ell $ where $1 < \ell \leq h$, we will prove the following recursion:
\begin{align*}
 \lnorm d^{\pi}_\ell (\cdot) - d^{\pi^\prime}_\ell (\cdot)  \rnorm_{1} &\leq \lnorm d^{\pi}_{\ell-1} (\cdot) - d^{\pi^\prime}_{\ell-1} (\cdot)  \rnorm_{1} + \expect_{s \sim d^{\pi^\prime}_{\ell-1} (\cdot)} \ls \lnorm \pi_{\ell-1} (\cdot|s) - \pi^\prime_{\ell-1} (\cdot|s)   \rnorm_1 \rs.
\end{align*}
With Bellman-flow equation in \eqref{eq:flow_link}, we have
\begin{align*}
    &\quad \lnorm d^{\pi}_\ell (\cdot) - d^{\pi^\prime}_\ell (\cdot)  \rnorm_{1}
    \\
    &= \sum_{s \in \gS} \labs d^{\pi}_\ell (s) - d^{\pi^\prime}_\ell (s)  \rabs
    \\
    &= \sum_{s \in \gS} \bigg\vert \sum_{(s^\prime, a^\prime)} d^{\pi}_{\ell-1} (s^\prime) \pi_{\ell-1} (a^\prime|s^\prime) P_{\ell-1}(s|s^\prime, a^\prime) - \sum_{(s^\prime, a^\prime)} d^{\pi^\prime}_{\ell-1} (s^\prime) \pi^\prime_{\ell-1} (a^\prime|s^\prime) P_{\ell-1}(s|s^\prime, a^\prime)  \bigg\vert
    \\
    &= \sum_{s \in \gS} \Bigg| \sum_{(s^\prime, a^\prime)} \big( d^{\pi}_{\ell-1} (s^\prime) - d^{\pi^\prime}_{\ell-1} (s^\prime) \big) \pi_{\ell-1} (a^\prime|s^\prime) P_{\ell-1}(s|s^\prime, a^\prime) 
    \\
    &\, + \sum_{(s^\prime, a^\prime)} d^{\pi^\prime}_{\ell-1} (s^\prime) \lp \pi_{\ell-1} (a^\prime|s^\prime) -\pi^\prime_{\ell-1} (a^\prime|s^\prime) \rp  P_{\ell-1}(s|s^\prime, a^\prime)  \Bigg|
    \\
    &\leq \sum_{s \in \gS} \sum_{(s^\prime, a^\prime)} \labs d^{\pi}_{\ell-1} (s^\prime) - d^{\pi^\prime}_{\ell-1} (s^\prime) \rabs \pi_{\ell-1} (a^\prime|s^\prime) P_{\ell-1}(s|s^\prime, a^\prime) 
    \\
    &\; + \sum_{s \in \gS} \sum_{(s^\prime, a^\prime)} \bigg[ d^{\pi^\prime}_{\ell-1} (s^\prime)  \cdot \labs \pi_{\ell-1} (a^\prime|s^\prime) -\pi^\prime_{\ell-1} (a^\prime|s^\prime) \rabs  P_{\ell-1}(s|s^\prime, a^\prime) \bigg]
    \\
    &= \lnorm d^{\pi}_{\ell-1} (\cdot) - d^{\pi^\prime}_{\ell-1} (\cdot)  \rnorm_{1} + \expect_{s \sim d^{\pi^\prime}_{\ell-1} (\cdot)} \ls \lnorm \pi_{\ell-1} (\cdot|s) - \pi^\prime_{\ell-1} (\cdot|s)   \rnorm_1 \rs,  
\end{align*}
Applying the recursion with $\Vert d^{\pi}_1 (\cdot) - d^{\pi^\prime}_1 (\cdot)  \Vert_{1}=0$ finishes the proof of the first statement.

Next, we continue to prove the second statement.
\begin{align*}
    &\quad \lnorm d^{\pi}_h (\cdot, \cdot) - d^{\pi^\prime}_h (\cdot, \cdot)  \rnorm_{1} 
    \\
    &= \sum_{(s, a)} \labs d^{\pi}_h (s, a) - d^{\pi^\prime}_h (s, a) \rabs
    \\
    &= \sum_{(s, a)} \labs d^{\pi}_h (s) \pi_h (a|s) - d^{\pi^\prime}_h (s) \pi^\prime_h (a|s) \rabs
    \\
    &= \sum_{(s, a)} \bigg\vert \lp d^{\pi}_h (s) - d^{\pi^\prime}_h (s) \rp \pi_h (a|s) + d^{\pi^\prime}_h (s) \lp \pi_h (a|s) -   \pi^\prime_h (a|s) \rp \bigg\vert
    \\
    &\leq \sum_{(s, a)} \labs  d^{\pi}_h (s) - d^{\pi^\prime}_h (s) \rabs \pi_h (a|s) + \sum_{(s, a)} d^{\pi^\prime}_h (s) \labs \pi_h (a|s) -   \pi^\prime_h (a|s) \rabs
    \\
    &= \lnorm d^{\pi}_h (\cdot) - d^{\pi^\prime}_h (\cdot)  \rnorm_{1} + \expect_{s \sim d^{\pi^\prime}_h (\cdot)} \ls \lnorm \pi_h (\cdot|s)  - \pi^\prime_h (\cdot|s) \rnorm_1 \rs, 
\end{align*}
which proves the second statement.
\end{proof}

\begin{lem}
\label{lem:unique_opt_solution_condition}
Consider the optimization problem: $$\min_{x \in [0, 1]^n} f (x) := \sum_{i=1}^m f_i (x),$$ where $f_i : [0, 1]^n \rightarrow \reals, \forall i \in [m]$. Suppose that 1) there exists $k \in [m]$ such that $x^{\star}$ is the unique optimal solution to $\min_{x \in [0, 1]^n} f_k (x)$; 2) for each $j \in [m], j \not= k$, $x^{\star}$ is the optimal solution to $\min_{x \in [0, 1]^n} f_j (x)$. Then, $x^{\star}$ is the unique optimal solution to $\min_{x \in [0, 1]^n} f (x)$.
\end{lem}

\begin{proof}
Since $x^{\star}$ is the unique optimal solution to $$\min_{x \in [0, 1]^n} f_k (x),$$ we have that $\forall x \in [0, 1]^n, x \not= x^{\star}$, $f_k (x^{\star}) < f_k (x)$. Furthermore, for each $j \in [m], j \not= k$, recall that $x^{\star}$ is the optimal solution to $\min_{x \in [0, 1]^n} f_j (x)$. We have that
\begin{align*}
    \forall j \in [m], j \not= k, \forall x \in [0, 1]^n, x \not= x^{\star}, f_j (x^{\star}) \leq f_j (x).  
\end{align*}
Then we derive that $\forall x \in [0, 1]^n, x \not= x^{\star}$, $f(x^{\star}) < f(x)$ and $x$ is the unique optimal solution to $\min_{x \in [0, 1]^n} f (x)$.
\end{proof}

\begin{lem}
\label{lem:n_vars_opt_greedy_structure}
Consider the optimization problem $$\min_{x_1, \ldots, x_n} f (x_1, \ldots, x_n).$$ Suppose that $x^{\star} = (x^{\star}_1, \ldots, x^{\star}_n)$ is the optimal solution, then $\forall i \in [n]$, $x^{\star}_i$ is the optimal solution to $\min_{x_i} F (x_i) := f (x^{\star}_1, \ldots, x_i, \ldots, x^{\star}_n)$. 
\end{lem}

\begin{proof}
The proof is based on contradiction. Suppose that the original statement is not true. There exists $\widetilde{x}_i \not= x^{\star}_i$ such that
\begin{align*}
    F (\widetilde{x}_i) < F (x^{\star}_i). 
\end{align*}
Consider $\widetilde{x} = (x^{\star}_1, \cdots,\widetilde{x}_i, \cdots, x^{\star}_n)$ which differs from $x^{\star}$ in the $i$-th component. Then we have that
\begin{align*}
    f (x^{\star}_1, \cdots,\widetilde{x}_i, \cdots, x^{\star}_n) &= F (\widetilde{x}_i) \\
    &< F (x^{\star}_i) \\
    &= f (x^{\star}_1, \cdots, x^{\star}_i, \cdots, x^{\star}_n), 
\end{align*}
which contradicts the fact that $x^{\star} = (x^{\star}_1, \cdots, x^{\star}_n)$ is the optimal solution to $\min_{x_1, \cdots, x_n} f (x_1, \cdots, x_n)$. Hence, the original statement is true. 
\end{proof}

\begin{lem}
\label{lem:single_variable_opt}
For any constants $a, c \geq 0$, we define the function $f(x) = \vert c - ax \vert - ax$. Consider the optimization problem $\min_{x \in [0, 1]} f(x) $, then $x^{\star} = 1$ is the optimal solution. 
\end{lem}

\begin{proof}
We assume that $x^{\star} = 1$ is not the optimal solution. There exists $\widetilde{x^{\star}} \in [0, 1)$ such that $f(\widetilde{x^{\star}}) < f (x^{\star})$. That is
\begin{align*}
    \vert c - a\widetilde{x^{\star}} \vert - a\widetilde{x^{\star}} - \vert c - a \vert + a < 0, 
\end{align*}
which implies that $\vert c - a \vert - \vert c - a\widetilde{x^{\star}} \vert > a - a \widetilde{x^{\star}}$. On the other hand, according to the inequality that $\labs p \rabs - \labs q \rabs \leq \labs p-q \rabs$ for $p, q \in \reals$, we have
\begin{align*}
    \vert c - a \vert - \vert c - a\widetilde{x^{\star}} \vert \leq \vert a\widetilde{x^{\star}} - a \vert = a- a\widetilde{x^{\star}},
\end{align*}
where the last equality follows that $\widetilde{x^{\star}} < 1$. We construct a contradiction. Therefore, the original statement is true.

\end{proof}

\begin{lem}
\label{lem:single_variable_opt_condition}
For any constants $a, c > 0$, we define the function $f(x) = \vert c - ax \vert - ax$. Consider the optimization problem $\min_{x \in [0, 1]} f(x) $, If $x^{\star}$ is an optimal solution, then $x^{\star} > 0$. Furthermore, if $c < a$, then the optimal solution set is $ [c/a, 1]$.  
\end{lem}

\begin{proof}
To begin with, we prove the first statement. The proof is based on contradiction. We assume that $x = 0$ is the optimal solution. We compare the function value on $x=1$ and $x=0$.
\begin{align*}
    f (0) - f(1) &= c +a - \labs c-a \rabs > 0,
\end{align*}
where the strict inequality follows that $a, c>0$. We obtain that $f(1) < f(0)$, which contradicts with the assumption that $x=0$ is the optimal solution. Therefore, the original statement is true and we finish the proof.

Then we prove the second statement. It is easy to see that
\begin{align*}
    f (x) = \begin{cases}
      c-2ax & x \in [0, \frac{c}{a}), \\
      - c & x \in [\frac{c}{a}, 1].
    \end{cases}
\end{align*}
$f(x)$ is continuous piece-wise linear function. $f(x)$ is strictly decreasing when $x \in [0, c / a)$ and is constant when $x \in [c / a, 1]$. Therefore, we can get that the optimal solutions are $x^{\star} \in [c / a, 1]$. 
\end{proof}

\begin{lem}
\label{lem:single_variable_regularity}
For any constants $a > 0$ and $c \geq 0$, we define the function $f(x) = \vert c - ax \vert - ax$. For any $x \leq \min\{ c/a, 1 \}$, we have $f(x) - f(1) = 2 a ( \min\{ c/a, 1 \} - x)$. 
\end{lem}

\begin{proof}
We consider two cases: $c \geq a$ and $c < a$. When $c \geq a$, the function $f (x)$ at $[0, 1]$ is formulated as $f(x) = c - 2ax$. For any $x \leq \min\{ c/a, 1 \} = 1 $, $f(x) - f(1) = 2a (1-x) = 2 a ( \min\{ c/a, 1 \} - x)$. On the other hand, when $c < a$, the function $f (x)$ at $[0, 1]$ is formulated as
\begin{align*}
    f(x) = \begin{cases}
      c-2ax & x \in [0, \frac{c}{a}), \\
      - c & x \in [\frac{c}{a}, 1].
    \end{cases}
\end{align*}
For any $x \leq \min\{ c/a, 1 \} = c/a$, $f(x) - f(1) = 2 a (c/a-x) = 2 a ( \min\{ c/a, 1 \} - x) $. Therefore, we finish the proof.
\end{proof}

\begin{lem}
\label{lem:mn_variables_opt_unique}
Consider that $A = (a_{ij}) \in \reals^{m \times n}, c \in \reals^{m}, d \in \reals^{n}$ where $a_{ij} > 0$, $\sum_{i=1}^m c_i \geq \sum_{i=1}^m \sum_{j=1}^n a_{ij}$ and for each $j \in [n]$, $\sum_{i=1}^m a_{ij} = d_j$. Consider the following optimization problem:
\begin{align*}
    \min_{x \in [0, 1]^n}f (x) &:= \lnorm c - A x \rnorm_{1} - d^{\top} x = \sum_{i=1}^m \labs c_i - \sum_{j=1}^n a_{ij} x_j \rabs - \sum_{j=1}^n d_j x_j.
\end{align*}
Then $x^{\star} = \mathbf{1}$ is the unique optimal solution, where $\mathbf{1}$ is the vector that each element is 1. 
\end{lem}

\begin{proof}
For $x = (x_1, \ldots, x_n)$, the function $f(x)$ is formulated as
\begin{align*}
    f (x) = \sum_{i=1}^m \labs c_i - \sum_{j=1}^n a_{ij} x_j \rabs - \sum_{j=1}^n d_j x_j.
\end{align*}
The proof is based on contradiction. We assume that the original statement is not true and there exists $x = (x_1, \ldots, x_n) \not= \mathbf{1}$ such that $x$ is the optimal solution. Let $k \in [n]$ denote some index where $x_k \not= 1$. We construct $\widetilde{x} = \lp \widetilde{x}_1, \ldots, \widetilde{x}_n  \rp \in [0, 1]^n$ in the following way.
\begin{align*}
    \widetilde{x}_j = x_j, \forall j \in [n] \setminus \{k\}, \quad \widetilde{x}_k = 1. 
\end{align*}
We compare the function value of $x$ and $\widetilde{x}$.
\begin{align*}
    &\quad f(\widetilde{x}) - f(x) \\
    &= \sum_{i=1}^m \lp \labs c_i - \sum_{j=1}^n a_{ij} \widetilde{x}_j \rabs - \labs c_i - \sum_{j=1}^n a_{ij} x_j  \rabs  \rp - d_k (1-x_k)
    \\
    &< \sum_{i=1}^m \lp a_{ik} (1-x_k) \rp - d_k (1-x_k) = 0.
\end{align*}
Here the strict inequality follows the statement that there exists $i^{\star} \in [m]$ such that
\begin{align*}
    &\quad \labs c_{i^{\star}} - \sum_{j=1}^n a_{i^{\star} j} \widetilde{x}_j \rabs - \labs c_{i^{\star}} - \sum_{j=1}^n a_{i^{\star} j} x_j  \rabs  \\
    &<  \labs \lp c_{i^{\star}} - \sum_{j=1}^n a_{i^{\star} j} \widetilde{x}_j \rp - \lp c_{i^{\star}} - \sum_{j=1}^n a_{i^{\star} j} x_j \rp  \rabs \\
    &=  a_{i^{\star} k} (1-x_k).
\end{align*}
We will prove this statement later. As for $i \in [m], i \not= i^{\star}$, with the inequality that $\labs a \rabs - \labs b \rabs \leq \labs a-b \rabs$ for $a, b \in \reals$, we obtain that
\begin{align*}
    &\quad \labs c_i - \sum_{j=1}^n a_{ij} \widetilde{x}_j \rabs - \labs c_i - \sum_{j=1}^n a_{ij} x_j  \rabs  \\
    &\leq \labs \lp c_i - \sum_{j=1}^n a_{ij} \widetilde{x}_j \rp - \lp c_i - \sum_{j=1}^n a_{ij} x_j \rp  \rabs=  a_{ik} (1-x_k) .
\end{align*}
Hence the strict inequality holds and we construct $\widetilde{x}$ such that $f(\widetilde{x}) < f(x) $, which contradicts with the assumption that $x$ is the optimal solution. Therefore, we prove that the original statement is true and finish the proof.

Now we proceed to prove the statement that there exists $i^{\star} \in [m]$ such that
\begin{align*}
    &\quad \labs c_{i^{\star}} - \sum_{j=1}^n a_{i^{\star} j} \widetilde{x}_j \rabs - \labs c_{i^{\star}} - \sum_{j=1}^n a_{i^{\star} j} x_j  \rabs <  \labs \lp c_{i^{\star}} - \sum_{j=1}^n a_{i^{\star} j} \widetilde{x}_j \rp - \lp c_{i^{\star}} - \sum_{j=1}^n a_{i^{\star} j} x_j \rp  \rabs
\end{align*}
We also prove this statement by contradiction. We assume that for all $i \in [m]$,
\begin{align*}
    &\quad \labs c_i - \sum_{j=1}^n a_{ij} \widetilde{x}_j \rabs - \labs c_i - \sum_{j=1}^n a_{ij} x_j  \rabs \geq \labs \lp c_i - \sum_{j=1}^n a_{ij} \widetilde{x}_j \rp - \lp c_i - \sum_{j=1}^n a_{ij} x_j \rp  \rabs
\end{align*}
According to the inequality that $\labs a \rabs - \labs b \rabs \leq \labs a-b \rabs$ for $a, b \in \reals$, we have $ \forall i \in [m]$, 
\begin{align*}
   &\quad \labs c_i - \sum_{j=1}^n a_{ij} \widetilde{x}_j \rabs - \labs c_i - \sum_{j=1}^n a_{ij} x_j  \rabs = \labs \lp c_i - \sum_{j=1}^n a_{ij} \widetilde{x}_j \rp - \lp c_i - \sum_{j=1}^n a_{ij} x_j \rp  \rabs
\end{align*}
Furthermore, consider the inequality $\labs a \rabs - \labs b \rabs \leq \labs a-b \rabs$ for $a, b \in \reals$. Notice that the equality holds iff $(b-a) b \leq 0$. Hence we have that
\begin{align*}
     \forall i \in [m], \lp a_{ik} (1-x_k)  \rp  \lp  c_i - \sum_{j=1}^n a_{ij} x_j \rp   \leq 0. 
\end{align*}
Since $\lp a_{ik} (1-x_k)  \rp > 0$, we obtain that
\begin{align*}
    \forall i \in [m], c_i - \sum_{j=1}^n a_{ij} x_j \leq 0. 
\end{align*}
This implies that
\begin{align*}
    \sum_{i=1}^m c_i \leq \sum_{i=1}^m \sum_{j=1}^n a_{ij} x_j < \sum_{i=1}^m \sum_{j=1}^n a_{ij} \leq \sum_{i=1}^m c_i,  
\end{align*}
where the strict inequality follows that $x_k < 1$ and $a_{ij} > 0$. The last inequality follows the assumption of \cref{lem:mn_variables_opt_unique}. Here we find a contradiction that $\sum_{i=1}^m c_i < \sum_{i=1}^m c_i$ and hence the original statement is true.
\end{proof}

\begin{lem}
\label{lem:mn_variables_opt_regularity}
Under the same conditions in \cref{lem:mn_variables_opt_unique}, for any $x \in [0,1]^{n}$, we have that
\begin{align*}
    f (x) - f(x^{\star}) \geq \sum_{j=1}^n \min_{i \in [m]} \{a_{ij} \} (1-x_{j}),
\end{align*}
where $x^{\star} = \mathbf{1}$, which is the vector that each element is 1. 
\end{lem}

\begin{proof}
Recall that $$f (x) = \sum_{i=1}^m \vert c_i - \sum_{j=1}^n a_{ij} x_j \vert - \sum_{j=1}^n d_j x_j.$$ We first claim that when $x \in [0, 1]^n$, $c_i - \sum_{j=1}^n a_{ij} x_j < 0$ does not hold simultaneously for all $i \in [m]$. We prove this claim via contradiction. Assume that there exists $x \in [0, 1]^n$ such that $c_i - \sum_{j=1}^n a_{ij} x_j < 0, \forall i \in [m]$. Then we have that
\begin{align*}
    \sum_{i=1}^m c_i < \sum_{i=1}^m \sum_{j=1}^n a_{ij} x_j \overset{(1)}{\leq} \sum_{i=1}^m \sum_{j=1}^n a_{ij} \overset{(2)}{\leq} \sum_{i=1}^m c_i.   
\end{align*}
The inequality $(1)$ follows that $A > 0$ and $x \in [0, 1]^n$ and the inequality $(2)$ follows that original assumption of \cref{lem:mn_variables_opt_regularity}. Thus we constructs a contradiction, which implies that the original claim is true.

Let $x_{p:q}$ be the shorthand of $(x_p, x_{p+1}, \ldots, x_{q})$ for any $1 \leq p \leq q \leq n$. With telescoping, we have that
\begin{align*}
    f(x) - f(x^{\star}) = \sum_{j=1}^n f(x^{\star}_{1:j-1}, x_{j:n}) - f(x^{\star}_{1:j}, x_{j+1:n}).
\end{align*}
Note that $f(x^{\star}_{1:j-1}, x_{j:n}) $ and $f(x^{\star}_{1:j}, x_{j+1:n})$ only differ in the $j$-th variable. For each $j \in [n]$, with fixed $x^{\star}_1, \ldots, x^{\star}_{j-1}, \\ x_{j+1}, \ldots, x_n \in [0, 1]$, we define one-variable function $F_j (t) = f (x^{\star}_{1:j-1}, t, x_{j+1:n}), \forall t \in [0, 1]$. Notice that $F_j (t)$ is also a continuous piece-wise linear function.

On the one hand, $F_j (t)$ is differentiable at any interior point $t_0$ and it holds that 
\begin{align*}
    &\quad F_j^\prime (t_0) \\
    &= \sum_{i=1}^m \indict \lb \bigg( c_i - \sum_{k=1}^{j-1} a_{ik} x^{\star}_k - a_{ij} t_0 - \sum_{k=j+1}^n a_{ik} x_k \bigg) < 0  \rb a_{ij} - d_j  \\
    &\leq - \min_{i \in [m]} \{ a_{ij} \}.
\end{align*}
The last inequality follows that $\forall x \in [0, 1]^n, c_i - \sum_{j=1}^n a_{ij} x_j \leq 0$ does not hold simultaneously for all $i \in [m]$ and $d_j = \sum_{i=1}^m a_{ij}$. On the other hand, the number of boundary points of $F_j (t)$ is $m$ at most. Let $b_j^1, b_j^2, \ldots, b_j^{n_j}$ denote the boundary point of $F_j (t)$ when $t \in [x_j, x_j^{\star}]$. With fundamental theorem of calculus, we have that
\begin{align*}
   &\quad f(x) - f(x^{\star}) \\
   &= \sum_{j=1}^n f(x^{\star}_{1:j-1}, x_{j:n}) - f(x^{\star}_{1:j}, x_{j+1:n})
   \\
   &= \sum_{j=1}^n F_j (x_j) - F_j (x_j^{\star})
   \\
   &= \sum_{j=1}^n \bigg[ F_j (x_j) - F_j (b_j^1) + \sum_{k=1}^{n_j-1} F_j (b_j^k) - F_j (b_j^{k+1}) + F_j (b_j^{n_j}) - F(x_j^{\star})   \bigg]  
   \\
   &= - \sum_{j=1}^n \bigg[ \int_{x_j}^{b_j^1} F_j^\prime (t) dt + \sum_{k=1}^{n_j-1} \int_{b_j^k}^{b_j^{k+1}} F_j^\prime (t) dt + \int_{b_j^{n_j}}^{x_j^{\star}} F_j^\prime (t) dt    \bigg]
   \\
   &\geq  \sum_{j=1}^n \min_{i \in [m]} \{ a_{ij} \} \lp x_j^{\star} - x_j  \rp \\
   &=  \sum_{j=1}^n \min_{i \in [m]} \{ a_{ij}\} \lp 1 - x_j  \rp.
\end{align*}
\end{proof}

\subsection{Proof of Technical Lemmas in Appendix}

\subsubsection{Proof of Lemma \ref{lemma:ail_policy_ail_objective_equals_expert_policy_ail_objective}}
\label{appendix:proof_lemma:ail_policy_ail_objective_equals_expert_policy_ail_objective}

\begin{proof}[Proof of \cref{lemma:ail_policy_ail_objective_equals_expert_policy_ail_objective}]
For $h, h^\prime \in [H], h \leq h^\prime$, we use $\pi_{h:h^\prime}$ denote the shorthand of $\lp \pi_h, \pi_{h+1}, \cdots, \pi_{h^\prime} \rp$. From \cref{prop:ail_general_reset_cliff}, we have that $\forall h \in [H-1], s \in \goodS, \piail_{h} (a^{1}|s) = \piE_{h} (a^{1}|s) = 1$. Hence, $\piail$ and $\piE$ never visit bad states. Furthermore, notice that for any time step $h \in [H]$, $\text{Loss}_{h}$ only depends on $\pi_{1:h}$. Therefore, we have
\begin{align*}
    \sum_{h=1}^{H-1} \text{Loss}_{h} (\piail) = \sum_{h=1}^{H-1} \text{Loss}_{h} (\piE).
\end{align*}
It remains to prove that $\text{Loss}_{H} (\piail) = \text{Loss}_{H} (\piE)$. From \cref{lem:n_vars_opt_greedy_structure}, we have that 
\begin{align*}
    \piail_{H} &\in \argmin_{\pi_H}  \text{Loss}_{H} := \sum_{(s, a)} \labs \widehat{d^{\piE}_H}(s, a) - d^{\pi}_H(s, a) \rabs,
\end{align*}
where $d^{\pi}_H$ is computed by $\piail_{1:H-1}$. Then, we have that 
\begin{align*}
    \piail_{H}  &\in \argmin_{\pi_H} \sum_{s \in \gS} \sum_{a \in \gA} \labs \widehat{d^{\piE}_H}(s, a) - d^{\piE}_H (s) \pi_H (a|s)  \rabs
    \\
    &= \argmin_{\pi_H} \bigg\{ \sum_{s \in \goodS} \labs \widehat{d^{\piE}_H}(s) - d^{\piE}_H (s) \pi_H (a^{1}|s)  \rabs + d^{\piE}_H (s) \lp 1 - \pi_{H} (a^{1}|s) \rp \bigg\}
    \\
    &= \argmin_{\pi_H} \bigg\{ \sum_{s \in \goodS} \labs \widehat{d^{\piE}_H}(s) - d^{\piE}_H (s) \pi_H (a^{1}|s)  \rabs - d^{\piE}_H (s) \pi_{H} (a^{1}|s) \bigg\}. 
\end{align*}
In the penultimate equality, we use the facts that 1) for each $s \in \goodS$, we have $\widehat{d^{\piE}_H}(s, a^{1}) = \widehat{d^{\piE}_H}(s)$, and $\widehat{d^{\piE}_H} (s, a) = 0, \forall a \in \gA \setminus \{a^{1}\}$; 2) for each $s \in \badS$, $\widehat{d^{\piE}_H}(s) = d^{\piE}_H (s) = 0$. The last equality follows that $d^{\piE}_H (s)$ is independent of $\pi_H$. Since the optimization variables $\pi_H (a^{1}|s)$ for different $s \in \goodS$ are independent, we can view the above optimization problem for each $\pi_H (a^{1}|s)$ individually.
\begin{align*}
    \piail_{H}(a^{1}|s) &\in \argmin_{\pi_H (a^{1}|s) \in [0, 1]} \bigg\{ \labs \widehat{d^{\piE}_H}(s)  - d^{\piE}_H (s) \pi_H (a^{1}|s)  \rabs - d^{\piE}_H (s) \pi_{H} (a^{1}|s) \bigg\}.
\end{align*}
By \cref{lem:single_variable_opt}, we have that $\piE_H(a^{1}|s) = 1$ is the optimal solution.  Therefore, we have that $\Vert d^{\piail}_H - \widehat{d^{\piE}_H} \Vert_1 = \Vert d^{\piE}_H - \widehat{d^{\piE}_H} \Vert_1$. Combing the above steps, we finish the proof.
\end{proof}

\subsubsection{Proof of Proposition \ref{prop:ail_general_reset_cliff_approximate_solution}}
\label{appendix:proof_prop:ail_general_reset_cliff_approximate_solution}

\begin{proof}
Suppose that $\piail$ is an optimal solution to \eqref{eq:ail}. Since $\widebar{\pi}$ is $\varepsilon$-optimal, we have that
\begin{align*}
    f (\widebar{\pi}) - f(\piail) \leq \varepsilon,
\end{align*}
where $f(\pi)$ denotes the state-action distribution matching loss, i.e.,
\begin{align*}
    f(\pi) = \sum_{h=1}^{H} \sum_{(s, a)} \labs d^{\pi}_h(s, a) - \widehat{d^{\piE}_h}(s, a) \rabs.
\end{align*}
By \cref{lemma:ail_policy_ail_objective_equals_expert_policy_ail_objective}, it holds that $f(\piail) = f(\piE)$. Furthermore, with the decomposition of $f(\pi)$, we have 
\begin{align}
    f (\widebar{\pi}) - f(\piail) &= f (\widebar{\pi}) - f(\piE) \nonumber \\
    &=  \sum_{h=1}^{H} \Loss_h (\widebar{\pi}) - \Loss_h (\piE) \nonumber  \\
    &\leq \varepsilon, \label{eq:ail_objective_pi_bar_minus_piE}
\end{align}
where $\Loss_h$ refers to the one-stage state-action distribution matching loss. For any $h, h^\prime \in [H]$ with $h \leq h^\prime$, we use $\pi_{h:h^\prime}$ denote the shorthand of $\lp \pi_h, \pi_{h+1}, \cdots, \pi_{h^\prime} \rp$. Note that $\Loss_h$ only depends on $\pi_{1:h}$ and thus we have
\begin{align*}
    \sum_{h=1}^{H} \Loss_h (\widebar{\pi}_{1:h}) - \Loss_h (\piE_{1:h}) \leq \varepsilon.
\end{align*}
We define the policy candidate set $\Pi^{\opt} = \{ \pi \in \Pi: \forall h \in [H], \exists s \in \goodS, \pi_h (a^{1}|s) > 0 \}$ and assume that $\widebar{\pi} \in \Pi^{\opt}$. We will analyze $\sum_{h=1}^{H} \Loss_h (\pi_{1:h}) - \Loss_h (\piE_{1:h})$ for any $\pi \in \Pi^{\opt}$. For each $h \in [H]$, we have the following key composition by telescoping: 
\begin{align}
\label{eq:sum_telescoping}
\Loss_h (\pi_{1:h}) - \Loss_h (\piE_{1:h}) &= \sum_{\ell=1}^{h} \Loss_h (\pi_{1:\ell}, \piE_{\ell+1:h}) - \Loss_h (\pi_{1:\ell-1}, \piE_{\ell:h}). 
\end{align}
In the following part, we consider two cases: Case I: $ h<H$ and Case II: $ h = H$.

First, we consider Case I and focus on the term $\Loss_h (\pi_{1:\ell}, \piE_{\ell+1:h}) - \Loss_h (\pi_{1:\ell-1}, \piE_{\ell:h})$. In Case I, we consider two situations: $\ell = h$ and $\ell < h$.
\begin{itemize}
    \item When $\ell = h$, we consider the term $\Loss_h (\pi_{1:h}) - \Loss_h (\pi_{1:h-1}, \piE_{h})$. Note that $\pi_{1:h}$ and $(\pi_{1:h-1}, \piE_{h})$ differ in the policy in time step $h$. Hence, we take the policy in time step $h$ as variable and focus on
    \begin{align*}
        g (\pi_{h}) - g(\piE_{h}),
    \end{align*}
    where we define that $g (\pi_{h}) = \Loss_h (\pi_{1:h})$ and $g(\piE_{h}) = \Loss_h (\pi_{1:h-1}, \piE_{h})$. In the following part, we formulate $g (\pi_{h}) = \Loss_h (\pi_{1:h})$ as
    \begin{align*}
        &\quad g (\pi_{h}) \\
        &=  \sum_{(s, a)} | \widehat{d^{\piE}_h}(s, a) - d^{\pi}_h(s, a)  |
        \\
        &= \sum_{s \in \goodS} \sum_{a \in \gA} \labs \widehat{d^{\piE}_{h}} (s, a) - d^{\pi}_{h} (s) \pi_{h} (a|s)  \rabs + \sum_{s \in \badS} \sum_{a \in \gA} d^{\pi}_{h} (s, a)
        \\
        &= \sum_{s \in \goodS} \bigg[ \labs \widehat{d^{\piE}_{h}} (s, a^1) - d^{\pi}_{h} (s) \pi_h(a^1|s) \rabs + d^{\pi}_{h} (s) \lp 1 - \pi_h(a^1|s)  \rp  \bigg] + \sum_{s \in \badS} d^{\pi}_{h} (s)  
        \\
        &= \sum_{s \in \goodS}\bigg[ \labs \widehat{d^{\piE}_{h}} (s) - d^{\pi}_{h} (s) \pi_h(a^1|s) \rabs + d^{\pi}_{h} (s) \lp 1 - \pi_h(a^1|s)  \rp   \bigg] + \sum_{s \in \badS} d^{\pi}_{h} (s). 
    \end{align*}
    Note that $d^{\pi}_{h} (s)$ is independent of the policy in time step $h$. Then we have that
    \begin{align*}
       &\quad g (\pi_{h}) - g(\piE_{h})  \\
       &= \sum_{s \in \goodS}\lp \labs \widehat{d^{\piE}_{h}} (s) - d^{\pi}_{h} (s) \pi_h(a^1|s) \rabs - d^{\pi}_{h} (s)   \pi_h(a^1|s)     \rp - \lp \labs \widehat{d^{\piE}_{h}} (s)  - d^{\pi}_{h} (s) \piE_h(a^1|s) \rabs - d^{\pi}_{h} (s)   \piE_h(a^1|s)     \rp. 
    \end{align*}
    For each $s \in \goodS$, we may apply \cref{lem:single_variable_opt} and obtain that 
    \begin{align}
    \label{eq:case_one_situation_one_result}
        g (\pi_{h}) - g(\piE_{h}) = \Loss_h (\pi_{1:h}) - \Loss_h (\pi_{1:h-1}, \piE_{h}) \geq 0.
    \end{align}
    \item When $\ell < h$, we consider the term $\Loss_h (\pi_{1:\ell}, \piE_{\ell+1:h}) - \\ \Loss_h (\pi_{1:\ell-1}, \piE_{\ell:h})$. We notice that $(\pi_{1:\ell}, \piE_{\ell+1:h})$ and $(\pi_{1:\ell-1}, \piE_{\ell:h})$ only differ in the policy in time step $\ell$. Therefore, we take the policy in time step $\ell$ as variable and focus on
    \begin{align*}
        g (\pi_{\ell}) - g(\piE_{\ell}),
    \end{align*}
    where we recall that $g (\pi_{\ell}) = \Loss_h (\pi_{1:\ell}, \piE_{\ell+1:h})$ and $g (\piE_{\ell}) = \Loss_h (\pi_{1:\ell-1}, \piE_{\ell:h})$. We can calculate $g (\pi_{\ell})$ as
    \begin{align*}
    &\quad g (\pi_{\ell}) \\
    &=  \sum_{(s, a)} | \widehat{d^{\piE}_h}(s, a) - d^{\pi}_h(s, a)  |
    \\
    &= \sum_{s \in \goodS} \sum_{a \in \gA} \labs \widehat{d^{\piE}_{h}} (s, a) - d^{\pi}_{h} (s) \pi_{h} (a|s)  \rabs + \sum_{s \in \badS} \sum_{a \in \gA} d^{\pi}_{h} (s, a)
    \\
    &\overset{(a)}{=} \sum_{s \in \goodS} \labs \widehat{d^{\piE}_{h}} (s, a^1) - d^{\pi}_{h} (s, a^1) \rabs + \sum_{s \in \badS} d^{\pi}_{h} (s)  
    \\
    &\overset{(b)}{=} \sum_{s \in \goodS} \labs \widehat{d^{\piE}_{h}} (s) - d^{\pi}_{h} (s) \rabs + \sum_{s \in \badS} d^{\pi}_{h} (s).
    \end{align*}
    Here $d^{\pi}_h(s, a)$ and $d^{\pi}_{h} (s)$ are decided by $(\pi_{1:\ell}, \piE_{\ell+1:h})$, so equality $(a)$ and $(b)$ follow that $\pi_h(a^{1}|s) = 1$ for all $s \in \goodS$. This is the difference with the result in the previous case. Similar to the proof of \cref{prop:ail_general_reset_cliff}, with Bellman-flow equation in \eqref{eq:flow_link}, we have $\forall s \in \goodS$, 
    \begin{align*}
    &\quad d^{\pi}_{h} (s) \\
    &= \sum_{s^\prime \in \gS} \sum_{a \in \gA} d^{\pi}_{\ell} (s^\prime) \pi_\ell (a|s^\prime) \sP^{\pi} \lp s_{h} = s |s_\ell = s^\prime, a_\ell = a \rp  
        \\
        &= \sum_{s^\prime \in \goodS} d^{\pi}_{\ell} (s^\prime) \pi_{\ell} (a^{1}|s^\prime) \sP^{\pi} \lp s_{h} = s |s_{\ell} = s^\prime, a_{h} = a^{1} \rp.
    \end{align*}
    Keep in mind that the conditional probability $\sP^{\pi} ( s_{h} = s |s_{\ell} = s^\prime, a_{h} = a^{1} )$ is independent of $\pi_\ell$. Besides, for the visitation probability on bad states in time step $h$, we have
    \begin{align*}
        &\quad \sum_{s \in \badS} d^{\pi}_{h} (s) \\
        &= \sum_{s^\prime \in \badS} d^{\pi}_{\ell} (s^\prime) + \sum_{s^\prime \in \goodS} \sum_{a \in \gA \setminus \{a^1 \}} d^{\pi}_{\ell} (s^\prime)  \pi_{\ell} (a|s^\prime)
        \\
        &= \sum_{s^\prime \in \badS} d^{\pi}_{\ell} (s^\prime) + \sum_{s^\prime \in \goodS} d^{\pi}_{\ell} (s^\prime) \lp 1 - \pi_{\ell} (a^{1}|s^\prime) \rp .
    \end{align*}
    Plugging the above two equations into $g (\pi_{\ell})$ yields that
    \begin{align*}
        &\quad g (\pi_{\ell})  \\
        &= \sum_{s \in \goodS} \bigg\vert \widehat{d^{\piE}_{h}} (s) - \sum_{s^\prime \in \goodS} d^{\pi}_{\ell} (s^\prime) \pi_{\ell} (a^{1}|s^\prime) \sP^{\pi} \lp s_{h} = s |s_{\ell} = s^\prime, a_{\ell} = a^{1} \rp \bigg\vert + \sum_{s^\prime \in \badS} d^{\pi}_{\ell} (s^\prime) 
        \\
        &\; + \sum_{s^\prime \in \goodS} d^{\pi}_{\ell} (s^\prime) \lp 1 - \pi_{\ell} (a^{1}|s^\prime) \rp. 
    \end{align*}
    Note that $d^{\pi}_{\ell} (s)$ is independent of the policy in time step $\ell$ and we have
    \begin{align*}
    &\quad g (\pi_{\ell}) - g (\piE_{\ell}) \\
    &= \bigg[ \sum_{s \in \goodS} \bigg\vert \widehat{d^{\piE}_{h}} (s) - \sum_{s^\prime \in \goodS} d^{\pi}_{\ell} (s^\prime)  \sP^{\pi} \lp s_{h} = s |s_{\ell} = s^\prime, a_{\ell} = a^{1} \rp \pi_{\ell} (a^{1}|s^\prime) \bigg\vert - \sum_{s^\prime \in \goodS} d^{\pi}_{\ell} (s^\prime) \pi_{\ell} (a^{1}|s^\prime)   \bigg] 
    \\
    &\; - \bigg[ \sum_{s \in \goodS} \bigg\vert \widehat{d^{\piE}_{h}} (s) - \sum_{s^\prime \in \goodS} d^{\pi}_{\ell} (s^\prime)  \sP^{\pi} \lp s_{h} = s |s_{\ell} = s^\prime, a_{\ell} = a^{1} \rp \piE_{\ell} (a^{1}|s^\prime) \bigg\vert - \sum_{s^\prime \in \goodS} d^{\pi}_{\ell} (s^\prime) \piE_{\ell} (a^{1}|s^\prime)   \bigg] .
    \end{align*}
    For this function, we can use \cref{lem:mn_variables_opt_regularity} to prove that
    \begin{align}   
    &\quad g (\pi_{\ell}) - g (\piE_{\ell}) \nonumber \\
    &\geq \sum_{s^\prime \in \goodS} \bigg[ \min_{s \in \goodS} \lb d^{\pi}_{\ell} (s^\prime) \sP^{\pi} \lp s_{h} = s |s_{\ell} = s^\prime, a_{\ell} = a^{1} \rp \rb \nonumber  \lp 1 - \pi_{\ell} (a^{1}|s^\prime)  \rp \bigg]
  \nonumber  \\
    &= \sum_{s^\prime \in \goodS} \bigg[ \min_{s \in \goodS} \lb  \sP^{\pi} \lp s_{h} = s |s_{\ell} = s^\prime, a_{\ell} = a^{1} \rp \rb d^{\pi}_{\ell} (s^\prime) \lp 1 - \pi_{\ell} (a^{1}|s^\prime)  \rp \bigg].  \label{eq:approximate_optimality_condition_prop_step_1}
    \end{align}
    To check the conditions required by \cref{lem:mn_variables_opt_regularity}, we define
    \begin{align*}
        & m = n = \labs \goodS \rabs, \forall s \in \goodS, c(s) = \widehat{d^{\piE}_{h}} (s), \\
        & \forall s, s^\prime \in \goodS, A (s, s^\prime) = d^{\pi}_{\ell} (s^\prime)  \sP^{\pi} \lp s_{h} = s |s_{\ell} = s^\prime, a_{\ell} = a^{1} \rp,
        \\
        & \forall s^\prime \in \goodS, d(s^\prime) = d^{\pi}_{\ell} (s^\prime). 
    \end{align*}
    Remember that $\pi \in \Pi^{\opt}$. With the reachable assumption (refer to \cref{asmp:reset_cliff}) that $\forall h \in [H], s, s^\prime \in \goodS, P_h (s^\prime |s, a^1) > 0$, we have that $\forall s, s^\prime \in \goodS$, 
    \begin{align*}
     d^{\pi}_{\ell} (s^\prime)  > 0, \sP^{\pi} \lp s_{h} = s |s_\ell = s^\prime, a_\ell = a^{1} \rp > 0.
    \end{align*}
    Then we can obtain that $A > 0$ where $>$ means element-wise comparison. Besides, we have that
    \begin{align*}
        &\quad  \sum_{s \in \goodS} c (s) \\
        &= 1 \\
        &\geq \sum_{s \in \goodS} \sum_{s^\prime \in \goodS} d^{\pi}_\ell (s^\prime)  \sP^{\pi} \lp s_{h} = s |s_\ell = s^\prime, a_\ell = a^{1} \rp \\
        &= \sum_{s \in \goodS} \sum_{s^\prime \in \goodS} A (s, s^\prime).
    \end{align*}
    For each $s^\prime \in \goodS$, we further have that 
    \begin{align*}
        &\quad \sum_{s \in \goodS} A (s, s^\prime) \\
        &= \sum_{s \in \goodS} d^{\pi}_\ell (s^\prime)  \sP^{\pi} \lp s_{h} = s |s_\ell = s^\prime, a_\ell = a^{1} \rp \\
        &= d^{\pi}_\ell (s^\prime) = d(s^\prime).  
    \end{align*}
    Thus, we have verified the conditions in \cref{lem:mn_variables_opt_regularity} and \eqref{eq:approximate_optimality_condition_prop_step_1} is true. From \eqref{eq:approximate_optimality_condition_prop_step_1}, we have that
    \begin{align*}
    &\quad g (\pi_{\ell}) - g (\piE_{\ell}) \\
    &\geq \sum_{s^\prime \in \goodS} \bigg[ \min_{s \in \goodS} \lb  \sP^{\pi} \lp s_{h} = s |s_{\ell} = s^\prime, a_{\ell} = a^{1} \rp \rb d^{\pi}_{\ell} (s^\prime) \lp 1 - \pi_{\ell} (a^{1}|s^\prime)  \rp \bigg]
    \\
    &\geq \min_{s, s^\prime \in \goodS} \lb  \sP^{\pi} \lp s_{h} = s |s_{\ell} = s^\prime, a_{\ell} = a^{1} \rp \rb \sum_{s^\prime \in \goodS} d^{\pi}_{\ell} (s^\prime) \lp 1 - \pi_{\ell} (a^{1}|s^\prime)  \rp
    \\
    &= c_{\ell, h} \sum_{s^\prime \in \goodS} d^{\pi}_{\ell} (s^\prime) \lp 1 - \pi_{\ell} (a^{1}|s^\prime)  \rp.   
    \end{align*}
    Here $c_{\ell, h} = \min_{s, s^\prime \in \goodS} \{ \sP^{\pi} \lp s_{h} = s |s_{\ell} = s^\prime, a_{\ell} = a^{1} \rp \}$ and we have $c_{\ell, h} > 0$. In conclusion, we have proved that for each $\ell < h$,
    \begin{align}
        \Loss_h (\pi_{1:\ell}, \piE_{\ell+1:h}) - \Loss_h (\pi_{1:\ell-1}, \piE_{\ell:h}) \geq c_{\ell, h} \sum_{s^\prime \in \goodS} d^{\pi}_{\ell} (s^\prime) \lp 1 - \pi_{\ell} (a^{1}|s^\prime)  \rp, \label{eq:case_one_situation_two_result} 
    \end{align}
\end{itemize}

Then for Case I where $h < H$, we combine the results in \eqref{eq:case_one_situation_one_result} and \eqref{eq:case_one_situation_two_result} to obtain 
\begin{align}
    &\quad \Loss_h (\pi_{1:h}) - \Loss_h (\piE_{1:h}) \nonumber \\
    &= \sum_{\ell=1}^{h} \Loss_h (\pi_{1:\ell}, \piE_{\ell+1:h}) - \Loss_h (\pi_{1:\ell-1}, \piE_{\ell:h}) \nonumber
    \\
    &= \Loss_h (\pi_{1:h}) - \Loss_h (\pi_{1:h-1}, \piE_{h}) \nonumber + \sum_{\ell=1}^{h-1} \Loss_h (\pi_{1:\ell}, \piE_{\ell+1:h}) - \Loss_h (\pi_{1:\ell-1}, \piE_{\ell:h}) \nonumber
    \\
    &\geq \sum_{\ell=1}^{h-1} \Loss_h (\pi_{1:\ell}, \piE_{\ell+1:h}) - \Loss_h (\pi_{1:\ell-1}, \piE_{\ell:h}) \nonumber
    \\
    &\geq \sum_{\ell=1}^{h-1} c_{\ell, h} \sum_{s^\prime \in \goodS} d^{\pi}_{\ell} (s^\prime) \lp 1 - \pi_{\ell} (a^{1}|s^\prime)  \rp. \label{eq:case_one_result} 
\end{align}
The penultimate inequality follows \eqref{eq:case_one_situation_one_result} and the last inequality follows \eqref{eq:case_one_situation_two_result}. 

Second, we consider Case II where $h = H$. By telescoping, we have that
\begin{align}
    &\quad \Loss_h (\pi_{1:H}) - \Loss_h (\piE_{1:H}) \nonumber
    \\
    &= \sum_{\ell=1}^{H} \Loss_h (\pi_{1:\ell}, \piE_{\ell+1:H}) - \Loss_h (\pi_{1:\ell-1}, \piE_{\ell:H}) \nonumber
    \\
    &= \Loss_h (\pi_{1:H}) - \Loss_h (\pi_{1:H-1}, \piE_{H}) + \sum_{\ell=1}^{H-1} \Loss_h (\pi_{1:\ell}, \piE_{\ell+1:H}) - \Loss_h (\pi_{1:\ell-1}, \piE_{\ell:H}). \label{eq:case_two_telescoping}
\end{align}
Similar to Case I, we also consider two situations: $\ell=H$ and $\ell<H$. We first consider the situation where $\ell<H$, which is similar to the corresponding part in Case I.

\begin{itemize}
    \item When $\ell<H$, we consider $\Loss_h (\pi_{1:\ell}, \piE_{\ell+1:H}) - \\ \Loss_h (\pi_{1:\ell-1}, \piE_{\ell:H})$. The following analysis is similar to that in Case I. Note that $(\pi_{1:\ell}, \piE_{\ell+1:H})$ and $(\pi_{1:\ell-1}, \piE_{\ell:H})$ only differ in the policy in time step $\ell$. We take the policy in time step $\ell$ as variable and focus on
    \begin{align*}
        g (\pi_{\ell}) - g(\piE_{\ell}),
    \end{align*}
    where recall that $g (\pi_{\ell}) = \Loss_h (\pi_{1:\ell}, \piE_{\ell+1:h})$ and $g (\piE_{\ell}) = \Loss_h (\pi_{1:\ell-1}, \piE_{\ell:h})$. Similarly, we have 
    \begin{align*}
        g (\pi_{\ell}) = \sum_{s \in \goodS} \labs \widehat{d^{\piE}_{h}} (s) - d^{\pi}_{h} (s) \rabs + \sum_{s \in \badS} d^{\pi}_{h} (s).
    \end{align*}
    Here $d^{\pi}_h(s)$ is computed by $(\pi_{1:\ell}, \piE_{\ell+1:h})$. With Bellman-flow equation in \eqref{eq:flow_link}, it holds that $ \forall s \in \goodS$,
    \begin{align*}
        d^{\pi}_{h} (s) &= \sum_{s^\prime \in \goodS} \bigg[ d^{\pi}_{\ell} (s^\prime) \pi_{\ell} (a^{1}|s^\prime) \sP^{\pi} \lp s_{H} = s |s_{\ell} = s^\prime, a_{h} = a^{1} \rp \bigg],
    \end{align*}
    and 
    \begin{align*}
          \sum_{s \in \badS} d^{\pi}_{h} (s) &= \sum_{s \in \badS} d^{\pi}_{\ell} (s) + \sum_{s^\prime \in \goodS} d^{\pi}_{\ell} (s^\prime) \lp 1 - \pi_{\ell} (a^{1}|s^\prime) \rp .
    \end{align*}
    Plugging the above two equations into $g (\pi_{\ell})$ yields that
    \begin{align*}
        &\quad g (\pi_{\ell})  \\
        &= \sum_{s \in \goodS} \bigg[ \bigg\vert \widehat{d^{\piE}_{h}} (s) \\
        &\,\, - \sum_{s^\prime \in \goodS} d^{\pi}_{\ell} (s^\prime) \pi_{\ell} (a^{1}|s^\prime) \sP^{\pi} \lp s_{H} = s |s_{\ell} = s^\prime, a_{\ell} = a^{1} \rp \bigg\vert \bigg] + \sum_{s \in \badS} d^{\pi}_{\ell} (s) + \sum_{s^\prime \in \goodS} d^{\pi}_{\ell} (s^\prime) \lp 1 - \pi_{\ell} (a^{1}|s^\prime) \rp. 
    \end{align*}
    Notice that $d^{\pi}_{\ell} (s)$ is independent of the policy in time step $\ell$ and we have
    \begin{align*}
    &\quad g (\pi_{\ell}) - g (\piE_{\ell})
    \\
    &= \bigg\{ \sum_{s \in \goodS} \bigg[\bigg\vert \widehat{d^{\piE}_{h}} (s) - \sum_{s^\prime \in \goodS} d^{\pi}_{\ell} (s^\prime)  \sP^{\pi} \lp s_{H} = s |s_{\ell} = s^\prime, a_{\ell} = a^{1} \rp \pi_{\ell} (a^{1}|s^\prime) \bigg\vert \bigg] - \sum_{s^\prime \in \goodS} d^{\pi}_{\ell} (s^\prime) \pi_{\ell} (a^{1}|s^\prime)   \bigg\} 
    \\
    &\; - \bigg\{ \sum_{s \in \goodS} \bigg[ \bigg\vert \widehat{d^{\piE}_{h}} (s) - \sum_{s^\prime \in \goodS} d^{\pi}_{\ell} (s^\prime)  \sP^{\pi} \lp s_{H} = s |s_{\ell} = s^\prime, a_{\ell} = a^{1} \rp \piE_{\ell} (a^{1}|s^\prime) \bigg\vert \bigg] - \sum_{s^\prime \in \goodS} d^{\pi}_{\ell} (s^\prime) \piE_{\ell} (a^{1}|s^\prime)   \bigg\} .
    \end{align*}
    For this type function in RHS, we can use \cref{lem:mn_variables_opt_regularity} to prove that
    \begin{align}
    &\quad g (\pi_{\ell}) - g (\piE_{\ell}) \nonumber \\
    &\geq \sum_{s^\prime \in \goodS} \bigg[ \min_{s \in \goodS} \lb d^{\pi}_{\ell} (s^\prime) \sP^{\pi} \lp s_{H} = s |s_{\ell} = s^\prime, a_{\ell} = a^{1} \rp \rb  \nonumber \lp 1 - \pi_{\ell} (a^{1}|s^\prime)  \rp \bigg]  \nonumber
    \\
    &= \sum_{s^\prime \in \goodS} \bigg[ \min_{s \in \goodS} \lb  \sP^{\pi} \lp s_{H} = s |s_{\ell} = s^\prime, a_{\ell} = a^{1} \rp \rb  d^{\pi}_{\ell} (s^\prime) \lp 1 - \pi_{\ell} (a^{1}|s^\prime)  \rp \bigg].  \label{eq:approximate_optimality_condition_prop_step_2}
    \end{align}
    To check the conditions in \cref{lem:mn_variables_opt_regularity}, we define
    \begin{align*}
        & m = n = \labs \goodS \rabs, \forall s \in \goodS, c(s) = \widehat{d^{\piE}_{h}} (s), \\
        & \forall s, s^\prime \in \goodS, A (s, s^\prime) = d^{\pi}_{\ell} (s^\prime)  \sP^{\pi} \lp s_{H} = s |s_{\ell} = s^\prime, a_{\ell} = a^{1} \rp,
        \\
        & \forall s^\prime \in \goodS, d(s^\prime) = d^{\pi}_{\ell} (s^\prime). 
    \end{align*}
    Similar to the analysis in Case I, we obtain that $A > 0$ and
    \begin{align*}
         & \sum_{s \in \goodS} c (s) \\
         &= 1 \\
         &\geq \sum_{s \in \goodS}  \sum_{s^\prime \in \goodS} d^{\pi}_\ell (s^\prime)  \sP^{\pi} \lp s_{H} = s |s_\ell = s^\prime, a_\ell = a^{1} \rp \\
         &= \sum_{s \in \goodS} \sum_{s^\prime \in \goodS} A (s, s^\prime),
    \end{align*}
    Furthermore, for $\forall s^\prime \in \goodS$, we have 
    \begin{align*}
      &\quad \sum_{s \in \goodS} A (s, s^\prime) \\
      &= \sum_{s \in \goodS} d^{\pi}_\ell (s^\prime)  \sP^{\pi} \lp s_{h} = s |s_\ell = s^\prime, a_\ell = a^{1} \rp \\
      &= d^{\pi}_\ell (s^\prime) = d(s^\prime).
    \end{align*}
    Thus, we have verified the conditions in \cref{lem:mn_variables_opt_regularity} and \eqref{eq:approximate_optimality_condition_prop_step_2} is true. From \eqref{eq:approximate_optimality_condition_prop_step_2}, we get
    \begin{align*}
    &\quad g (\pi_{\ell}) - g (\piE_{\ell}) \\
    &\geq \sum_{s^\prime \in \goodS} \bigg[ \min_{s \in \goodS} \lb  \sP^{\pi} \lp s_{H} = s |s_{\ell} = s^\prime, a_{\ell} = a^{1} \rp \rb d^{\pi}_{\ell} (s^\prime) \lp 1 - \pi_{\ell} (a^{1}|s^\prime)  \rp \bigg]
    \\
    &\geq \min_{s, s^\prime \in \goodS} \lb  \sP^{\pi} \lp s_{H} = s |s_{\ell} = s^\prime, a_{\ell} = a^{1} \rp \rb \sum_{s^\prime \in \goodS} d^{\pi}_{\ell} (s^\prime)  \lp 1 - \pi_{\ell} (a^{1}|s^\prime)  \rp 
    \\
    &= c_{\ell, H} \sum_{s^\prime \in \goodS} d^{\pi}_{\ell} (s^\prime)  \lp 1 - \pi_{\ell} (a^{1}|s^\prime) \rp. 
    \end{align*}
    Here $c_{\ell, H} = \min_{s, s^\prime \in \goodS} \{ \sP^{\pi} \lp s_{H} = s |s_{\ell} = s^\prime, a_{\ell} = a^{1} \rp \} $. In summary, for $\ell < H$, we have proved that
    \begin{align}
    \label{eq:case_two_situation_one_result}
        \Loss_h (\pi_{1:\ell}, \piE_{\ell+1:H}) - \Loss_h (\pi_{1:\ell-1}, \piE_{\ell:H}) \geq c_{\ell, H} \sum_{s^\prime \in \goodS} d^{\pi}_{\ell} (s^\prime)  \lp 1 - \pi_{\ell} (a^{1}|s^\prime) \rp. 
    \end{align}
    \item When $\ell = H$, we consider the term $\Loss_h (\pi_{1:H}) - \Loss_h (\pi_{1:H-1}, \piE_{H})$. The analysis in this situation is more complicated. Note that $\pi_{1:H}$ and $(\pi_{1:H-1}, \piE_{H})$ only differs in the policy in the last time step $H$. Take the policy in time step $H$ as variable and we focus on
    \begin{align*}
        g (\pi_{H}) - g(\piE_{H}),
    \end{align*}
    where keep in mind that $g (\pi_{H}) = \Loss_h (\pi_{1:H})$ and $g(\piE_{H}) = \Loss_h (\pi_{1:H-1}, \piE_{H})$. Similarly, we can formulate $g (\pi_{H}) = \Loss_h (\pi_{1:H})$ as
    \begin{align*}
         g (\pi_{H}) = \sum_{s \in \goodS}\bigg[ \labs \widehat{d^{\piE}_{h}} (s) - d^{\pi}_{h} (s) \pi_H(a^1|s) \rabs + d^{\pi}_{h} (s) \lp 1 - \pi_H(a^1|s)  \rp   \bigg] + \sum_{s \in \badS} d^{\pi}_{h} (s).
    \end{align*}
    Note that $d^{\pi}_{h} (s)$ is independent of the policy in time step $H$ and we have that
    \begin{align*}
        &\quad g (\pi_{H}) - g(\piE_{H}) \\
        &= \sum_{s \in \goodS} \lp \labs \widehat{d^{\piE}_{h}} (s) - d^{\pi}_{h} (s) \pi_H (a^1|s) \rabs - d^{\pi}_{h} (s)   \pi_H (a^1|s)     \rp - \lp \labs \widehat{d^{\piE}_{h}} (s) - d^{\pi}_{h} (s) \piE_H (a^1|s) \rabs - d^{\pi}_{h} (s)   \piE_H (a^1|s)     \rp.
    \end{align*}
    Given estimation $\widehat{d^{\piE}_{h}} (s)$, we divide the set of good states into two parts. That is $\goodS = \gS_{\pi} \cup  \gS_{\pi}^c$ and $\gS_{\pi} \cap \gS_{\pi}^c = \emptyset$. Here $\gS_{\pi}  = \{s \in \goodS, \pi_H (a^1|s) \leq \min\{1, \widehat{d^{\piE}_{h}} (s) / d^{\piE}_H (s) \}  \}$. Therefore, we have that
    \begin{align*}
     &\quad g (\pi_{H}) - g(\piE_{H}) \\
     &= \sum_{s \in  \gS_{\pi}} \lp \labs \widehat{d^{\piE}_{h}} (s) - d^{\pi}_{h} (s) \pi_H (a^1|s) \rabs - d^{\pi}_{h} (s)   \pi_H (a^1|s)     \rp - \lp \labs \widehat{d^{\piE}_{h}} (s) - d^{\pi}_{h} (s) \piE_H (a^1|s) \rabs - d^{\pi}_{h} (s)   \piE_H (a^1|s)     \rp
        \\
        &+  \sum_{s \in \gS_{\pi}^{c}} \underbrace{ \lp \labs \widehat{d^{\piE}_{h}} (s) - d^{\pi}_{h} (s) \pi_H (a^1|s) \rabs - d^{\pi}_{h} (s)   \pi_H (a^1|s)     \rp}_{T_1 (s)} \underbrace{- \big( \labs \widehat{d^{\piE}_{h}} (s) - d^{\pi}_{h} (s) \piE_H (a^1|s) \rabs - d^{\pi}_{h} (s)   \piE_H (a^1|s) \big)}_{T_2 (s)}. 
    \end{align*}
    By \cref{lem:single_variable_opt}, we that $\sum_{s \in \gS_{\pi}^c} T_1 (s) + T_2 (s) \geq 0$. Then we have that
    \begin{align*}
       &\quad g (\pi_{H}) - g(\piE_{H})  \\
       &\geq \sum_{s \in  \gS_{\pi}} \lp \labs \widehat{d^{\piE}_{h}} (s) - d^{\pi}_{h} (s) \pi_H (a^1|s) \rabs - d^{\pi}_{h} (s)   \pi_H (a^1|s)     \rp - \lp \labs \widehat{d^{\piE}_{h}} (s) - d^{\pi}_{h} (s) \piE_H (a^1|s) \rabs - d^{\pi}_{h} (s)   \piE_H (a^1|s)     \rp.
    \end{align*}
    For each $s \in  \gS_{\pi}$, we aim to apply \cref{lem:single_variable_regularity} to prove that
    \begin{align*}
        &\lp \labs \widehat{d^{\piE}_{h}} (s) - d^{\pi}_{h} (s) \pi_H (a^1|s) \rabs - d^{\pi}_{h} (s)   \pi_H (a^1|s)     \rp - \lp \labs \widehat{d^{\piE}_{h}} (s) - d^{\pi}_{h} (s) \piE_H (a^1|s) \rabs - d^{\pi}_{h} (s)   \piE_H (a^1|s)     \rp
        \\
        &\geq 2 d^{\pi}_{h} (s) \lp \min\{1,  \widehat{d^{\piE}_{h}} (s) / d^{\piE}_H (s)\} -  \pi_H (a^1|s)  \rp.
    \end{align*}
    To check the conditions required by \cref{lem:single_variable_regularity}, we define
    \begin{align*}
        c = \widehat{d^{\piE}_{h}} (s), a = d^{\pi}_{h} (s), x = \pi_H (a^1|s).   
    \end{align*}
    It is easy to see that $c \geq 0$. Since $\pi \in \Pi^{\opt}$, combined with the reachable assumption (refer to \cref{asmp:reset_cliff}) that $\forall h \in [H-1], \forall s, s^\prime \in \goodS, P_{h} (s^\prime|s, a^{1}) > 0$, we have that $a = d^{\pi}_{h} (s) > 0$. According to the definition of $\gS_{\pi}$, we have that $x =  \pi_H (a^1|s) \leq \min\{ 1,  \widehat{d^{\piE}_{h}} (s) / d^{\piE}_H (s) \} \leq  \min\{ 1,  \widehat{d^{\piE}_{h}} (s) / d^{\pi}_h (s) \} = \min\{ 1,  c / a \}$. We have verified the conditions in \cref{lem:single_variable_regularity} and obtain that
    \begin{align*}
        &\lp \labs \widehat{d^{\piE}_{h}} (s) - d^{\pi}_{h} (s) \pi_H (a^1|s) \rabs - d^{\pi}_{h} (s)   \pi_H (a^1|s)     \rp - \lp \labs \widehat{d^{\piE}_{h}} (s) - d^{\pi}_{h} (s) \piE_H (a^1|s) \rabs - d^{\pi}_{h} (s)   \piE_H (a^1|s)     \rp
        \\
        &= 2 d^{\pi}_{h} (s) \lp \min\{1,  \widehat{d^{\piE}_{h}} (s) / d^{\pi}_h (s)\} -  \pi_H (a^1|s)  \rp
        \\
        &\geq 2 d^{\pi}_{h} (s) \lp \min\{1,  \widehat{d^{\piE}_{h}} (s) / d^{\piE}_H (s)\} -  \pi_H (a^1|s)  \rp, 
    \end{align*}
    where the last inequality follows that $\forall s \in \goodS, d^{\piE}_H (s) \geq d^{\pi}_h (s)$. Plugging the above inequality into $g (\pi_{H}) - g(\piE_{H})$ yields that
    \begin{align*}
        g (\pi_{H}) - g(\piE_{H}) &\geq 2 \sum_{s \in  \gS_{\pi}}  d^{\pi}_{h} (s) \lp \min\{1,  \widehat{d^{\piE}_{h}} (s) / d^{\piE}_H (s)\} -  \pi_H (a^1|s)  \rp .
    \end{align*}
    In summary, we have proved that 
    \begin{align}
    \label{eq:case_two_situation_two_result}
        \Loss_h (\pi_{1:H}) - \Loss_h (\pi_{1:H-1}, \piE_{H}) \geq  2 \sum_{s \in  \gS_{\pi}}  d^{\pi}_{h} (s) \lp \min\{1,  \widehat{d^{\piE}_{h}} (s) / d^{\piE}_H (s)\} -  \pi_H (a^1|s)  \rp . 
    \end{align}
\end{itemize}

In Case II where $h=H$, with \eqref{eq:case_two_telescoping}, \eqref{eq:case_two_situation_one_result}, and \eqref{eq:case_two_situation_two_result},   we have that
\begin{align}
    &\quad \Loss_h (\pi_{1:H}) - \Loss_h (\piE_{1:H}) \nonumber
    \\
    &= \Loss_h (\pi_{1:H}) - \Loss_h (\pi_{1:H-1}, \piE_{H}) \nonumber + \sum_{\ell=1}^{H-1} \Loss_h (\pi_{1:\ell}, \piE_{\ell+1:H}) - \Loss_h (\pi_{1:\ell-1}, \piE_{\ell:H}) \nonumber
    \\
    &\geq 2 \sum_{s \in  \gS_{\pi}}  d^{\pi}_{h} (s) \lp \min\{1,  \widehat{d^{\piE}_{h}} (s) / d^{\piE}_H (s)\} -  \pi_H (a^1|s)  \rp + \sum_{\ell=1}^{H-1} c_{\ell, H} \sum_{s \in \goodS} d^{\pi}_{\ell} (s)  \lp 1 - \pi_{\ell} (a^{1}|s) \rp, \label{eq:case_two_result} 
\end{align}
where $c_{\ell, H} = \min_{s, s^\prime \in \goodS} \{ \sP^{\pi} \lp s_{H} = s |s_{\ell} = s^\prime, a_{\ell} = a^{1} \rp \} $. The last inequality follows $\eqref{eq:case_two_situation_one_result}$ and $\eqref{eq:case_two_situation_two_result}$.

Finally, we combine the results in \eqref{eq:case_one_result} and $\eqref{eq:case_two_result}$ to obtain that 
\begin{align*}
    &\quad f (\pi) - f(\piE)
    \\ &= \sum_{h=1}^{H} \Loss_h (\pi_{1:h}) - \Loss_h (\piE_{1:h})
    \\
    &=  \sum_{h=1}^{H-1} \bigg[ \Loss_h (\pi_{1:h}) - \Loss_h (\piE_{1:H}) + \Loss_h (\pi_{1:h}) - \Loss_h (\piE_{1:H}) \bigg]
    \\
    &\geq \sum_{h=1}^{H-1} \sum_{\ell=1}^{h-1} c_{\ell, h} \sum_{s \in \goodS} d^{\pi}_{\ell} (s) \lp 1 - \pi_{\ell} (a^{1}|s)  \rp + 2 \sum_{s \in  \gS_{\pi}}  d^{\pi}_{h} (s) \lp \min\{1,  \widehat{d^{\piE}_{h}} (s) / d^{\piE}_H (s)\} -  \pi_H (a^1|s)  \rp 
    \\
    &\; + \sum_{\ell=1}^{H-1} c_{\ell, H} \sum_{s \in \goodS} d^{\pi}_{\ell} (s)  \lp 1 - \pi_{\ell} (a^{1}|s) \rp
    \\
    &= \sum_{h=1}^{H} \sum_{\ell=1}^{h-1} c_{\ell, h} \sum_{s \in \goodS} d^{\pi}_{\ell} (s) \lp 1 - \pi_{\ell} (a^{1}|s)  \rp + 2 \sum_{s \in  \gS_{\pi}}  d^{\pi}_{h} (s) \lp \min\{1,  \widehat{d^{\piE}_{h}} (s) / d^{\piE}_H (s)\} -  \pi_H (a^1|s)  \rp, 
\end{align*}
where $c_{\ell, h}  = \min_{s, s^\prime \in \goodS} \{ \sP^{\pi} \lp s_{h} = s |s_{\ell} = s^\prime, a_{\ell} = a^{1} \rp \} $. The penultimate inequality follows \eqref{eq:case_one_result} and \eqref{eq:case_two_result}. In summary, we prove that for any $\pi \in \Pi^{\mathrm{opt}}$, we have
\begin{align*}
    &\quad f (\pi) - f(\piE) \\
    &\geq \sum_{h=1}^{H} \sum_{\ell=1}^{h-1} c_{\ell, h} \sum_{s \in \goodS} d^{\pi}_{\ell} (s) \lp 1 - \pi_{\ell} (a^{1}|s)  \rp + 2 \sum_{s \in  \gS_{\pi}}  d^{\pi}_{h} (s) \big( \min\{1,  \widehat{d^{\piE}_{h}} (s) / d^{\piE}_H (s)\} -  \pi_H (a^1|s)  \big)
    \\
    &\geq c (\pi) \bigg[ \sum_{h=1}^{H} \sum_{\ell=1}^{h-1} \sum_{s \in \goodS} d^{\pi}_{\ell} (s) \lp 1 - \pi_{\ell} (a^{1}|s)  \rp \\
    &\,\, + \sum_{s \in  \gS_{\pi}}  d^{\pi}_{h} (s) \big( \min\{1,  \widehat{d^{\piE}_{h}} (s) / d^{\piE}_H (s)\} -  \pi_H (a^1|s)  \big)  \bigg].
\end{align*}
Here $c(\pi)$ is defined as 
\begin{align*}
    c (\pi) &= \min_{1 \leq \ell < h \leq H} c_{\ell, h} =  \min_{1 \leq \ell < h \leq H, s, s^\prime \in \goodS} \{ \sP^{\pi} \lp s_{h} = s |s_{\ell} = s^\prime, a_{\ell} = a^{1} \rp \}
\end{align*}
Since $\widebar{\pi} \in \Pi^{\mathrm{opt}}$, it holds that
\begin{align*}
    &\quad f (\widebar{\pi}) - f(\piE) \\
    &\geq c (\widebar{\pi}) \bigg[ \sum_{h=1}^{H} \sum_{\ell=1}^{h-1} \sum_{s \in \goodS} d^{\widebar{\pi}}_{\ell} (s) \lp 1 - \widebar{\pi}_{\ell} (a^{1}|s)  \rp + \sum_{s \in  \gS^{\widebar{\pi}}_H}  d^{\widebar{\pi}}_{H} (s) \lp \min\{1,  \widehat{d^{\piE}_{h}} (s) / d^{\piE}_H (s)\} - \widebar{\pi}_H (a^1|s)  \rp  \bigg]. 
\end{align*}
Finally, with \eqref{eq:ail_objective_pi_bar_minus_piE}, we obtain the desired result.
\end{proof}

\section{Discussion}

\subsection{Optimization Procedures for TV-AIL}
\label{appendix:optimization_procedures}
Here we discuss optimization procedures for TV-AIL. Recall the objective of TV-AIL:
\begin{align}
  \min_{\pi \in \Pi} \sum_{h=1}^{H} \sum_{(s, a) \in \gS \times \gA} | d^{\pi}_h(s, a) - \widehat{d^{\piE}_h}(s, a) |,
\end{align}
There are two main optimization approaches to the above state-action distribution matching problem. First, we can utilize linear programming to solve the above optimization problem \emph{exactly} \citep{syed08lp, nived2021provably}. The main idea is that relax the optimization variable from $\pi$ to $d^{\pi}$, and solve the matching problem in the space of state-action distributions. In particular, we observe that the optimization objective and Bellman-flow constraints are linear with respect to $d^{\pi}$. Thus, linear programming is applicable. Finally, we recover the optimal policy from the solved state-action distribution. Please see \citep{syed08lp, nived2021provably} for details.

Second, we can utilize gradient-based methods to solve this optimization problem \emph{approximately} \citep{syed07game, xu2021error}. This type of optimization approach is widely used in practice \citep{ho2016gail, ghasemipour2019divergence, ke2019imitation, Kostrikov19dac}.  We briefly introduce the approach in \citep{xu2021error} and defer interested readers to \citep{xu2021error} for more information. The main idea is to utilize the mini-max formulation for TV-AIL:
\begin{align}
    \min_{\pi \in \Pi} \max_{c \in \gC_{\TV}}  \,\,  & \sum_{h=1}^H \expect_{(s, a) \sim d^{\pi}_h} \left[c_h(s, a) \right] - \expect_{(s, a)\sim \widehat{d^{\piE}_h} } \left[c_h(s, a) \right], 
\end{align}
where $\gC_{\TV}$ is the set of functions $c_h: \gS \times \gA \rar [-1, 1]$. Then our target is to solve the saddle point of the above mini-max problem. By the dual representation of policy value in \eqref{eq:value_dual_representation}, we see that the outer problem is to maximize the policy value of $\pi$ given the reward function $- c_h(s, a)$. For the inner optimization problem, we can use online gradient descent methods \citep{shalev12online-learning} so that we can finally reach an approximate saddle point. Formally, let us define the objective $f^{(t)}(c)$:
\begin{align}
\label{eq:objective_c}
    & \sum_{h=1}^{H} \sum_{(s, a)} c_h(s, a) \lp \widehat{d^{\piE}_h} (s, a) - d^{\pi^{(t)}}_h (s, a) \rp,
\end{align}
where $\pi^{(t)}$ is the optimized policy at iteration $t$. Then the update rule for $c$ is:
\begin{align*}   
    c^{(t+1)} := \gP_{\gC_{\TV}}\lp c^{(t)} - \eta^{(t)}  \nabla f^{(t)}(c^{(t)}) \rp, 
\end{align*}
where $\eta^{(t)} > 0$ is the stepsize to be chosen later, and $\gP_{\gC_{\TV}}$ is the Euclidean projection on the set $\gC_{\TV}$, i.e., $\gP_{\gC_{\TV}} (c) := \argmin_{z \in \gC_{\TV}} \lnorm z- c \rnorm_2$. The above procedure is outlined in \cref{algo:main_aglorithm}.

\begin{algorithm}[htbp]
\caption{\textsf{TV-AIL} via Gradient-based Method \citep{xu2021nearly}}
\label{algo:main_aglorithm}
{
\begin{algorithmic}[1]
\REQUIRE{expert demonstrations $\gD$, transition model $\gP$, number of iterations $T$, step size $\eta^{(t)}$, and initialization $c^{(1)}$.}
\STATE{Obtain the estimation $\widehat{d_h^{\piE}}$ in \eqref{eq:estimate_by_count}.}
\FOR{$t = 1, 2, \cdots, T$}
\STATE{$\pi^{(t)} \lar $ solve the optimal policy with the reward function $- c^{(t)}$ in transition model $\gP$ up to an error of $\varepsilon_{\rl}$.}
\STATE{Compute the state-action distribution $d^{\pi^{(t)}}_h$ for $\pi^{(t)}$ for all $h \in [H]$ in transition model $\gP$.}
\STATE{Update $ c^{(t+1)} := \gP_{\gC_{\TV}}\lp c^{(t)} - \eta^{(t)}  \nabla f^{(t)}(c^{(t)}) \rp$ with $f^{(t)}(c)$ defined in \eqref{eq:objective_c}.}
\ENDFOR
\STATE{Compute the mean state-action distribution $\widebar{d_h} (s, a) = \sum_{t=1}^{T} d^{\pi^{(t)}}_h(s, a) / T$ for all $h \in [H], (s, a) \in \gS \times \gA$.}
\STATE{Compute $\widebar{\pi}_h (a|s) \lar \widebar{d_h}(s, a) / \sum_{a} \widebar{d_h} (s, a)$ for all $h \in [H], (s, a) \in \gS \times \gA$.}
\ENSURE{policy $\widebar{\pi}$.}
\end{algorithmic}
}
\end{algorithm}

For the optimization problem in Line 3 of \cref{algo:main_aglorithm}, value iteration and policy gradient methods are applicable. Specifically, if we use value iteration, $\varepsilon_{\opt} = 0$ and this procedure can be done in $H$ iterations for episodic MDPs. The following theoretical guarantee is provided in \citep{xu2021nearly}.

\begin{prop}[Lemma E.5 in \citep{xu2021nearly}] \label{prop:approximate_vail_opt_error} 
Fix $\varepsilon \in \lp 0, H \rp$. Consider Algorithm \ref{algo:main_aglorithm} with $\varepsilon_{\rl} \leq \varepsilon / 2$ and  $\widebar{\pi}$ being the output policy. If we take the number of iterations $T \geq 32  H^2  |\gS|  |\gA| / \varepsilon^2$, and the step size $\eta^{(t)} :=  \sqrt{|\gS||\gA| / (8T)}$, then we have
\begin{align*}
    \sum_{h=1}^H \lnorm d^{\widebar{\pi}}_h - \widehat{d^{\piE}_h} \rnorm_{1} \leq \min_{\pi \in \Pi} \sum_{h=1}^H \lnorm d^{\pi}_h - \widehat{d^{\piE}_h} \rnorm_{1} + \varepsilon. 
\end{align*}
\end{prop}

\subsection{Model-based TV-AIL}
\label{appendix:subsec:mb-tvail}

\begin{algorithm}[htbp]
\caption{Model-based TV-AIL}
\label{algo:mbtvail_detailed}
\begin{algorithmic}[1]
\REQUIRE{Expert demonstrations $\gD$.}
\STATE{$\widehat{\gP} \leftarrow$ Invoke RF-Express \citep{menard20fast-active-learning} to interact with the environment for $M$ trajectories and learn a transition model.}
\STATE{$\widebar{\pi} \lar$ Apply \cref{algo:main_aglorithm} with the learned transition model $\widehat{\gP}$ to solve the optimization problem in \eqref{eq:AIL_with_empirical_transition_model}.}
\ENSURE{Policy $\widebar{\pi}$.}
\end{algorithmic}
\end{algorithm}

In this part, we elaborate on the model-based TV-AIL method in the unknown transition setting. This extension has been briefly discussed in \cref{subsec:unknown_transition} and we provide more details here.

Consider the meta-algorithm displayed in \cref{algo:mbtvail}. First, reward-free exploration methods \citep{chi20reward-free, menard20fast-active-learning} satisfy the uniform policy evaluation condition in \cref{def:uniform_policy_evaluation}, as pointed out by \citep{xu2021nearly}. Therefore, following \citep{xu2021nearly}, we invoke the reward-free exploration method RF-Express \citep{menard20fast-active-learning} to interact with the environment for $M$ trajectories and learn a transition model $\widehat{\gP}$. With the learned transition model, we can apply \cref{algo:main_aglorithm} to solve the state-action distribution matching problem in \eqref{eq:AIL_with_empirical_transition_model}. Finally, we arrive at the model-based TV-AIL method outlined in \cref{algo:mbtvail_detailed}. We provide the following theorem, which indicates the number of online interactions required by model-based TV-AIL to achieve an approximately optimal solution.

\begin{thm}\label{theorem:sample-complexity-unknown-transition}
Fix $\varepsilon \in \lp 0, 1 \rp$ and $\delta \in (0, 1)$. Under the unknown transition setting, consider Model-based TV-AIL displayed in Algorithm \ref{algo:mbtvail_detailed} and $\widebar{\pi}$ is output policy, take $\varepsilon_{\rl} \leq \varepsilon/6$, $T \geq 288 H^2 |\gS| |\gA| / \varepsilon^2$ and $\eta^{(t)} :=  \sqrt{|\gS||\gA| / (8T)}$. If the number of online interactions satisfies,
\begin{align*}
M \gtrsim  \frac{H^3 |\gS| |\gA|}{\varepsilon^2} \lp |\gS| + \log \lp \frac{H |\gS| |\gA|}{\delta \varepsilon} \rp \rp,
\end{align*}
then with probability at least $1-\delta$, $\widebar{\pi}$ is an $\varepsilon$-optimal solution in \cref{defn:approximate_solution}.
\begin{align*}
    \sum_{h=1}^{H} \lnorm d^{\widebar{\pi}}_h - \widehat{d^{\piE}_h} \rnorm_1 \leq \min_{\pi \in \Pi} \sum_{h=1}^{H} \lnorm d^{\pi}_h - \widehat{d^{\piE}_h} \rnorm_1 + \varepsilon.
\end{align*}
\end{thm}
The theorem indicates that with the number of online interactions of $\widetilde{\gO} (H^3|\gS|^2 |\gA| / \varepsilon^2)$, model-based TV-AIL can achieve an $\varepsilon$-optimal solution in the unknown transition setting. Furthermore, by \cref{theorem:ail_approximate_reset_cliff}, we can demonstrate that model-based TV-AIL also enjoys the horizon-free imitation gap in terms of the number of expert trajectories. In summary, our theoretical results also hold even when the transition function is unknown.

\begin{proof}[Proof of Theorem \ref{theorem:sample-complexity-unknown-transition}]
The proof is based on \cref{prop:transfer_error} and the analysis in \citep{xu2021nearly}. First, we argue that when the number of online interactions satisfies
\begin{align}
\label{eq:rfexpress-interaction-complexity}
    M \gtrsim \frac{H^3 |\gS| |\gA|}{\varepsilon^2} \lp |\gS| + \log \lp \frac{H |\gS| |\gA|}{\delta \varepsilon} \rp \rp, 
\end{align}
the reward-free exploration method RF-Express is $(\varepsilon/3, \delta)$-PAC for uniform policy evaluation in \cref{def:uniform_policy_evaluation}. According to Theorem 1, Lemma 1 and the stopping rule in RF-express in \citep{menard20fast-active-learning}, with the number of online interactions in \eqref{eq:rfexpress-interaction-complexity}, for any reward function $r = \{ r_1, \ldots, r_H \}$ with $r_h : \gS \times \gA \rightarrow [0, 1]$ and any policy $\pi \in \Pi$,
\begin{align*}
    \labs V^{\pi, \gP, r} -  V^{\pi, \widehat{\gP}, r} \rabs \leq \frac{\varepsilon}{6}, 
\end{align*}
where $\widehat{\gP}$ is the transition model learned by RF-Express. Here $V^{\pi, \gP, r}$ denotes the policy value of $\pi$ under the transition model $\gP$ and reward function $r$, and $V^{\pi, \widehat{\gP}, r}$ denotes the counterpart under the learned transition model $\widehat{\gP}$. 
By the dual representation of $\ell_1$-norm, we have that
\begin{align*}
    \sum_{h=1}^H \lnorm d^{\pi, \gP}_h - d^{\pi, \widehat{\gP}}_h \rnorm_1 &= 2 \max_{w \in \gR} \sum_{h=1}^H \sum_{(s, a)} \lp d^{\pi, \gP}_h (s, a) - d^{\pi, \widehat{\gP}}_h (s, a) \rp w_h (s, a).
\end{align*}
Here $\gR = \{ r = \{r_1, \ldots, r_H \}: \forall h \in [H], \forall (s, a) \in \gS \times \gA, r_h (s, a) \in [0, 1] \}$. Furthermore, by the dual form of policy value in \eqref{eq:value_dual_representation}, we have
\begin{align*}
    \sum_{h=1}^H \lnorm d^{\pi, \gP}_h - d^{\pi, \widehat{\gP}}_h \rnorm_1 = 2 \max_{w \in \gR} V^{\pi, \gP, w} - V^{\pi, \widehat{\gP}, w} \leq \frac{\varepsilon}{3}.
\end{align*}
Therefore, we obtain that the reward-free exploration method RF-Express is $(\varepsilon/3, \delta)$-PAC for uniform policy evaluation in \cref{def:uniform_policy_evaluation}.

Second, we consider \cref{algo:main_aglorithm} and take $\varepsilon_{\rl} \leq \varepsilon/6$, $T \geq 288 H^2 |\gS| |\gA| / \varepsilon^2$ and $\eta^{(t)} :=  \sqrt{|\gS||\gA| / (8T)}$. With \cref{prop:approximate_vail_opt_error}, we have that
\begin{align*}
    \sum_{h=1}^H \lnorm d^{\widebar{\pi}, \widehat{\gP}}_h - \widehat{d^{\piE}_h} \rnorm_{1} \leq \min_{\pi \in \Pi} \sum_{h=1}^H \lnorm d^{\pi, \widehat{\gP}}_h - \widehat{d^{\piE}_h} \rnorm_{1} + \frac{\varepsilon}{3}.
\end{align*}
Therefore, we derive that \cref{algo:main_aglorithm} solves the optimization problem in \eqref{eq:AIL_with_empirical_transition_model} up to an error of $\varepsilon/3$.

In summary, we obtain that RF-Express is $(\varepsilon/3, \delta)$-PAC for uniform policy evaluation in \cref{def:uniform_policy_evaluation} and \cref{algo:main_aglorithm} solves the optimization problem in \eqref{eq:AIL_with_empirical_transition_model} up to an error of $\varepsilon/3$. Applying \cref{prop:transfer_error} with $\varepsilon_{\eval} = \varepsilon/3$ and $\varepsilon_{\opt} = \varepsilon/3$ completes the proof.  
\end{proof}

\subsection{Addressing Sample Barrier Issue of TV-AIL}
\label{appendix:vail_sample_barrier}

This section discusses how to address the sample barrier issue of TV-AIL on RBAS MDPs. That is, in the last time step, TV-AIL may make a wrong decision because there is no future guidance. We propose two approaches. The first approach is to add a one-stage terminal loss $\Vert d^{\pi}_{H+1} - \widehat{d^{\piE}_{H+1}} \Vert_1$. As a consequence, TV-AIL can exactly recover the expert policy in the first $H$ time steps. We note that this approach is widely applied in dynamic-programming-based algorithms \citep{bertsekas2012dynamic}. Another approach is to directly override the policy of TV-AIL in the final time step by a BC's policy. In this way, we can prove that the imitation gap bound becomes $\gO(\min\{1, |\gS|/N\})$, which is the one-step imitation gap of a BC policy. Notably, the improved bound for TV-AIL is better than that of BC in the whole sample regime.

\subsection{Performance of Other AIL Methods}
\label{appendix:algorithmic_behaviors_of_other_ail_methods}

In this part, we present experiment results for three representative AIL methods FEM, GTAL and GAIL, on RBAS MDPs and the lower bound instances (refer to \cref{asmp:standard_imitation}). In particular, FEM \citep{pieter04apprentice} and GTAL \citep{syed07game} perform state-action distribution matching with $\ell_2$-norm-based and $\ell_{\infty}$-norm-based divergences, respectively. Besides, GAIL \citep{ho2016gail} minimizes the state-action distribution discrepancy with the JS divergence. First, \cref{tab:reset_cliff_other_ail} summarizes the imitation gaps with different horizons on a RBAS MDP. We clearly see that the imitation gaps of FEM, GTAL and GAIL do not increase when the horizon grows, which is similar to TV-AIL. Second, we evaluate TV-AIL, FEM, GTAL and GAIL on the lower bound instance with different horizons; see the results in \cref{tab:standard_imitation_other_ail_small_data} and \cref{tab:standard_imitation_other_ail_large_data}. In both large sample regime (\cref{tab:standard_imitation_other_ail_small_data}) and small sample regime (\cref{tab:standard_imitation_other_ail_large_data}), we observe that the imitation gaps of all methods increase when the horizon grows. In summary, we observe that FEM, GTAL and GAIL exhibit similar algorithmic behaviors to TV-AIL.

\begin{table}[htbp]
\centering
\caption{Imitation gap on the RBAS MDP with $N=1$.}
\label{tab:reset_cliff_other_ail}
\begin{tabular}{@{}ccccc@{}}
\toprule
         &$H=100$  & $H=500$  & $H=1000$ &$H=2000$ \\ \midrule
         TV-AIL & \meanstd{0.69}{0.00} & \meanstd{0.70}{0.00}  & \meanstd{0.71}{0.00}   & \meanstd{0.71}{0.00} \\
         FEM &  \meanstd{0.58}{0.00} & \meanstd{0.57}{0.00} & \meanstd{0.58}{0.00} & \meanstd{0.58}{0.00}       \\
GTAL     & \meanstd{0.80}{0.00} & \meanstd{0.81}{0.00}  & \meanstd{0.81}{0.19}   & \meanstd{0.74}{0.32} \\
GAIL &    \meanstd{0.94}{0.00} & \meanstd{0.95}{0.00} & \meanstd{0.95}{0.00} & \meanstd{0.95}{0.00}
\\
\bottomrule
\end{tabular}
\end{table}

\section{Experiment Details}
\label{appendix:experiment_details}

\subsection{Experiment Details on MuJoCo}

For MuJoCo tasks, all experiments are run with $5$ random seeds. We use the expert dataset collected by the trained online SAC \citep{haarnoja2018sac} with 1 million steps. We use the deterministic policy as the expert policy, which is common in the literature \citep{ho2016gail}. The expert policy values are listed in \cref{table:expert_value_mujoco}.

For MuJoCo tasks, we implement BC according to \citep{li2022rethinking}. The implementation of TV-AIL is based on an existing AIL algorithm DAC \citep{Kostrikov19dac} (\url{https://github.com/google-research/google-research/tree/master/value_dice}). Instead of the KL-divergence in DAC, TV distance is considered in TV-AIL. To this end,  the tanh activation function is applied in the last layer of the discriminator.

\begin{table}[htbp]
\caption{Policy value of the expert policy on Hopper, HalfCheetah and Walker2d with different horizons.}
\label{table:expert_value_mujoco}
\centering
\begin{tabular}{@{}lllll@{}}
\toprule
Horizon           & H=100 &H=500 & H=1000 & H=2000 \\ \midrule
Hopper                        & 221.76 & 1533.83      &3202.19           & 6496.92           \\
HalfCheetah                        & 497.12 & 3719.69      &7770.66           & 15866.96           \\
Walker2d                        & 225.24 & 2176.84      &5046.76           & 10797.29           \\
\bottomrule
\end{tabular}
\end{table}

\begin{table*}[t]
\centering
\caption{{Imitation gap on Hopper, HalfCheetah and Walker2d with $H = 1000$. We report the mean of the imitation gap with the standard deviation over 5 independent experiments (same with \cref{tab:mujoco_sample_size_return}, \cref{tab:mujoco_horizon_unscaled_imitation_gap} and \cref{tab:mujoco_horizon_return}).}}
\label{tab:mujoco_sample_size_unscaled_imitation_gap}
\begin{tabular}{@{}cccccc@{}}
\toprule
     &    & $N=1$ & $N=4$ & $N=7$ &$N=10$ \\ \midrule
\multirow{2}{*}{Hopper } &BC      & \meanstd{2511.89}{89.90} &  \meanstd{2838.54}{100.04} &  \meanstd{2132.61}{341.06} & \meanstd{1474.30}{239.85}       \\
 &TV-AIL      & \meanstd{33.23}{36.01}  & \meanstd{5.78}{9.21}       & \meanstd{-15.87}{34.80}  & \meanstd{7.44}{33.9} \\
\midrule
\multirow{2}{*}{HalfCheetah } &BC      & \meanstd{8150.26}{63.70}& \meanstd{8209.84}{175.26} & \meanstd{7608.13}{273.49} & \meanstd{4462.39}{1318.84}       \\
 &TV-AIL      & \meanstd{-172.86}{782.73}  & \meanstd{-654.17}{129.64}       & \meanstd{-605.93}{53.75} & \meanstd{-610.52}{61.22} \\
 \midrule
\multirow{2}{*}{Walker2d } &BC      & \meanstd{5010.65}{48.42} & \meanstd{4696.33}{95.52} & \meanstd{2644.53}{907.30} & \meanstd{1114.92}{264.86}       \\
 &TV-AIL      & \meanstd{64.45}{98.05}  & \meanstd{45.20}{89.96}       & \meanstd{-24.53}{40.69}  & \meanstd{71.79}{62.42} \\ \bottomrule
\end{tabular}
\end{table*}

\begin{table*}[t]
\centering
\caption{Return on Hopper, HalfCheetah and Walker2d with $H = 1000$.}
\label{tab:mujoco_sample_size_return}
\begin{tabular}{@{}cccccc@{}}
\toprule
     &    & $N=1$ & $N=4$ & $N=7$ &$N=10$ \\ \midrule
\multirow{2}{*}{Hopper } &BC      & \meanstd{690.30}{89.90} &  \meanstd{363.65}{100.04} &  \meanstd{1069.58}{341.06} & \meanstd{1727.89}{239.85}       \\
 &TV-AIL      & \meanstd{3168.96}{36.01}  & \meanstd{3196.41}{9.21}       & \meanstd{3218.06}{34.80}  & \meanstd{3194.75}{33.9} \\
\midrule
\multirow{2}{*}{HalfCheetah } &BC      & \meanstd{-379.60}{63.70}& \meanstd{-439.18}{175.26} & \meanstd{162.53}{273.49} & \meanstd{3308.27}{1318.84}       \\
 &TV-AIL      & \meanstd{7943.52}{782.73}  & \meanstd{8423.83}{129.64}       & \meanstd{8376.59}{53.75} & \meanstd{8381.18}{61.22} \\
 \midrule
\multirow{2}{*}{Walker2d } &BC      & \meanstd{36.11}{48.42} & \meanstd{350.43}{95.52} & \meanstd{2402.23}{907.30} & \meanstd{3931.84}{264.86}       \\
 &TV-AIL      & \meanstd{4982.31}{98.05}  & \meanstd{5001.56}{89.96}       & \meanstd{5071.29}{40.69}  & \meanstd{4974.97}{62.42} \\ \bottomrule
\end{tabular}
\end{table*}

\begin{table*}[t]
\centering
\caption{Imitation gap on Hopper, HalfCheetah and Walker2d with $N=1$.}
\label{tab:mujoco_horizon_unscaled_imitation_gap}
\begin{tabular}{@{}cccccc@{}}
\toprule
     &    & $H=100$ & $H=500$ & $H=1000$ &$H=2000$ \\ \midrule
\multirow{2}{*}{Hopper} &BC      &\meanstd{2.56}{5.50} & \meanstd{571.38}{253.67} & \meanstd{2511.89}{89.90} & \meanstd{6241.36}{115.84}       \\
 &TV-AIL       & \meanstd{15.86}{14.60}   & \meanstd{-5.54}{21.65}       & \meanstd{33.23}{36.01}  & \meanstd{-31.75}{118.87} \\
\midrule
\multirow{2}{*}{HalfCheetah} &BC      &  \meanstd{434.56}{71.08} & \meanstd{3787.72}{161.79} & \meanstd{8150.26}{63.70} & \meanstd{16929.29}{99.57}       \\
 &TV-AIL      & \meanstd{65.39}{63.58}  & \meanstd{-190.44}{88.22}       & \meanstd{-172.86}{782.73}  & \meanstd{-1307.71}{123.65} \\
 \midrule
\multirow{2}{*}{Walker2d} &BC      & \meanstd{51.23}{39.47} & \meanstd{2065.12}{68.27} & \meanstd{5010.65}{48.42} & \meanstd{10794.19}{10.24}       \\
 &TV-AIL      & \meanstd{-0.91}{5.18}  & \meanstd{81.32}{113.18}       & \meanstd{64.65}{98.05}  & \meanstd{358.43}{331.51} \\ \bottomrule
\end{tabular}
\end{table*}

\begin{table*}[t]
\centering
\caption{Return on Hopper, HalfCheetah and Walker2d with $N=1$.}
\label{tab:mujoco_horizon_return}
\begin{tabular}{@{}cccccc@{}}
\toprule
     &    & $H=100$ & $H=500$ & $H=1000$ &$H=2000$ \\ \midrule
\multirow{2}{*}{Hopper} &BC      &\meanstd{219.20}{5.50} & \meanstd{962.45}{253.67} & \meanstd{690.30}{89.90} & \meanstd{255.56}{115.84}       \\
 &TV-AIL       & \meanstd{205.90}{14.60}   & \meanstd{1539.37}{21.65}       & \meanstd{3168.96}{36.01}  & \meanstd{6528.67}{118.87} \\
\midrule
\multirow{2}{*}{HalfCheetah} &BC      &  \meanstd{62.56}{71.08} & \meanstd{-68.03}{161.79} & \meanstd{-379.6}{63.70} & \meanstd{-1062.33}{99.57}       \\
 &TV-AIL      & \meanstd{431.73}{63.58}  & \meanstd{3910.13}{88.22}       & \meanstd{7943.52}{782.73}  & \meanstd{17174.67}{123.65} \\
 \midrule
\multirow{2}{*}{Walker2d} &BC      & \meanstd{174.01}{39.47} & \meanstd{111.72}{68.27} & \meanstd{36.11}{48.42} & \meanstd{3.10}{10.24}       \\
 &TV-AIL      & \meanstd{226.15}{5.18}  & \meanstd{2095.52}{113.18}       & \meanstd{4982.11}{98.05}  & \meanstd{10438.86}{331.51} \\ \bottomrule
\end{tabular}
\end{table*}

\begin{table*}[tbp]
\caption{Information about tabular MDPs.}
\label{table:task_information}
\centering

\begin{tabular*}{\textwidth}{@{\extracolsep{\fill}}p{0.34\textwidth}cccc@{}}
\toprule
Tasks & \makecell{Number of\\ states} & \makecell{Number of\\ actions} & Horizon & \makecell{Number of\\ expert trajectories} \\ \midrule
RBAS MDPs (\cref{tab:reset_cliff_sample_size}) & 20 & 2 & 1000 & $[1,4,7,10]$ \\
\midrule
RBAS MDPs (\cref{tab:reset_cliff_horizon}) & 20 & 2 & $[10,100,1000,2000]$ & 1 \\
\midrule
Lower bound instances (\cref{tab:standard_imitation_other_ail_small_data}) & 100 & 2 & $[10,100,1000,2000]$ & 1 
\\
\midrule
Lower bound instances (\cref{tab:standard_imitation_other_ail_large_data}) & 100 & 2 & $[10,100,1000,2000]$ & 100 
\\
\bottomrule
\end{tabular*}
\end{table*}

\subsection{Experiment Details on Tabular MDPs}

For tabular MDPs, all experiments are run with $20$ random seeds. The detailed task information is listed in Table \ref{table:task_information}. For RBAS MDPs, we consider that the initial state distribution is uniform over good states. For the lower bound instances (refer to \cref{asmp:standard_imitation}), we consider that the initial state distribution is uniform over all states. By construction, the policy value of the expert policy is $H$.

BC directly estimates the expert policy from expert demonstrations via \eqref{eq:pi_bc}. For TV-AIL, we run \cref{algo:main_aglorithm} to obtain an approximate solution. In our experiments, an adaptive step size~\citep{Orabona19a_modern_introduction_to_ol} is implemented for TV-AIL: 
\begin{align*}
    \eta^{(t)} = \frac{D}{ \sqrt{\sum_{i=1}^t \lnorm \nabla_{c} f^{(i)} \lp c^{(i)} \rp \rnorm_2^2}},
\end{align*}
where $D = \sqrt{2H |\gS| |\gA|}$ is the diameter of the set $\gC_{\TV}$. Notice that \cref{prop:approximate_vail_opt_error} still holds with this adaptive step size \citep{Orabona19a_modern_introduction_to_ol}.

\end{document}